\documentclass{article}

\usepackage[margin = 1in]{geometry}

\usepackage{mathpazo}
\usepackage{xcolor}
\usepackage{cancel}

\usepackage[utf8]{inputenc} 
\usepackage[T1]{fontenc}    
\usepackage{hyperref}       
\usepackage{url}            
\usepackage{booktabs}       
\usepackage{amsfonts}       
\usepackage{nicefrac}       
\usepackage{microtype}      
\usepackage{amsmath, amssymb, amsfonts, mathtools}
\usepackage{ifthen}
\usepackage[ruled, linesnumbered]{algorithm2e}
\usepackage{comment}
\usepackage{bbm}
\usepackage{pgf}
\usepackage{tikz}
\usepackage{array}
\usepackage{float}
\usepackage[nottoc,notlot,notlof]{tocbibind}
\usetikzlibrary{arrows,automata}
\usepackage{yhmath}

\usepackage{amsthm}
\usepackage{hyperref}
\usepackage{multirow, multicol}
\usepackage{subfigure, graphicx} 

\newtheorem{theorem}{Theorem}
\numberwithin{theorem}{subsection}

\newtheorem{lemma}[theorem]{Lemma}

\newtheorem{corollary}[theorem]{Corollary}
\newtheorem{remark}[theorem]{Remark}

\newtheorem{definition}[theorem]{Definition}

\newcommand{\abs}[1]{\ensuremath{\left|#1\right|}}
\newcommand{\paren}[1]{\ensuremath{\left(#1\right)}}
\newcommand{\brac}[1]{\ensuremath{\left\{#1\right\}}}
\newcommand{\sqbrac}[1]{\ensuremath{\left[#1\right]}}

\newcommand{\eps}{\epsilon}
\newcommand{\sign}[1]{\ensuremath{\mathrm{sgn}(#1)}}
\newcommand{\minmrg}[1]{\ensuremath{\mathrm{minmrg}(#1)}}
\newcommand{\mindiag}[1]{\ensuremath{\mathrm{mindiag}(#1)}}
\newcommand{\mindisc}[1]{\ensuremath{\mathrm{mindisc}(#1)}}
\newcommand{\conv}[2]{\ensuremath{\mathrm{conv}_{#1}(#2)}}

\newcommand{\dtv}[2]{d_{\rm TV}(#1,#2)}

\title{Tree-structured Ising models can be learned efficiently, from an optimal number of samples}
\title{Computationally Efficient and Sample-Optimal Algorithms fo Learning Tree-structured Ising models}
\title{Efficient and Sample-Optimal Learning of Tree-structured Ising models}
\title{Tree-structured Ising models can be learned efficiently}

\title{Sample-Optimal and Efficient Learning of Tree Ising models}

\author{
  Constantinos Daskalakis\\
  EECS and CSAIL, MIT
   \and
 Qinxuan Pan \\
 EECS and CSAIL, MIT
}

\date{October 29, 2020}
\begin{document}

\let\endtitlepage\relax
\begin{titlepage}
\maketitle
\end{titlepage}

\begin{abstract}
We show that $n$-variable tree-structured Ising models can be learned computationally-efficiently to within total variation distance~$\eps$ from an optimal $O(n \ln n/\eps^2)$ samples, where  $O(\cdot)$ hides an absolute constant which, importantly, does not depend on the model being learned---neither its tree nor the magnitude of its edge strengths, on which we place no assumptions. Our guarantees hold, in fact, for the celebrated Chow-Liu algorithm~\cite{ChowL68}, using the plug-in estimator for estimating mutual information. While this (or any other) algorithm may fail to identify the structure of the underlying model correctly from a finite sample, we show that it will still learn a tree-structured  model that is $\eps$-close to the true one in total variation distance, a guarantee called~``proper~learning.''  

Our guarantees do not follow from known results for the Chow-Liu algorithm~\cite{ChowW73} and the ensuing literature on learning graphical models, including the very recent renaissance of algorithms on this learning challenge (see e.g.~\cite{Bresler15,VuffrayMLC16,KlivansM17,hamilton2017information,WuSD19,vuffray2019efficient}), which only yield asymptotic consistency results, or sample-inefficient and/or time-inefficient algorithms, unless further assumptions are placed on the graphical model, such as bounds on the ``strengths'' of the model's edges/hyperedges. While we establish guarantees for a widely known and simple algorithm, the analysis that this algorithm succeeds and is sample-optimal is quite complex, requiring a hierarchical classification of the edges into layers with different reconstruction guarantees, depending on their strength, combined with delicate uses of the subadditivity of the squared Hellinger distance over graphical models to control the error accumulation.
\end{abstract} 

\thispagestyle{empty}
\newpage
\setcounter{page}{1}

\thispagestyle{empty}
\tableofcontents
\thispagestyle{empty}
\newpage
\setcounter{page}{1}

\section{Introduction}
\label{sec:intro}

Markov Random Fields (MRFs) and Bayesian Networks (Bayesnets) are popular frameworks for graphically representing high-dimensional distributions, allowing the explicit specification of conditional independence properties that a high-dimensional distribution may satisfy. These properties can be exploited to decrease the sample and computational requirements to learn the distribution, as well as to perform inference once the distribution has been learned. A particularly simple special case of both frameworks occurs when the underlying graph is a tree. In this case, many basic inference tasks that are computationally intractable even to approximate in general, e.g.~computing marginals or the mode of the distribution, can be carried out exactly by simple, linear-time algorithms, such as the popular sum-product and max-product algorithms. For an introduction to MRFs and Bayesnets, their uses, and associated algorithms see, e.g.,~\cite{Lauritzen96,WainwrightJ08,Pearl2014}.

The attractiveness of these graphical models and their widespread use in application domains have motivated a vast literature on learning them from samples; see e.g.~the references above and~\cite{ChowL68,ChowW73,NarasimhanB04,ravikumar2010high,TanATW11,jalali2011learning,SanthanamW12,Bresler15,VuffrayMLC16,KlivansM17,hamilton2017information,WuSD19,vuffray2019efficient}. Even when the model to be learned is tree-structured, however, these methods only yield asymptotic consistency results, or sample-suboptimal and/or time-inefficient algorithms, unless further assumptions are placed on the distribution. Indeed, recent works on efficiently learning Ising models (and more general MRFs) only provide guarantees that the learned distribution is close (under some probabilistic distance) to the one being learned,  assuming a bound on the total ``strength'' of the edges (resp.~hyper-edges) adjacent to each node. Under even stronger assumptions, some of these works also identify the graphical structure of the Ising model (resp.~MRF).

Our goal in this paper is to study whether a tree-structured Ising model (equivalently a tree-structured binary-alphabet Bayesnet or a tree-structured binary-alphabet MRF) can be learned {\em without making any assumptions about the distribution}. More precisely, we study the following problem.

\medskip \noindent \begin{minipage}{\textwidth}{\em {\bf Learning Problem:} Given samples from an arbitrary $n$-variable tree-structued Ising model $\mathrm{P}$, learn some $n$-variable distribution $\rm Q$ that is within some desired total variation distance $\eps>0$ from $\rm P$, using a number of samples and total computation time that scale polynomially in $n$ and $1/\eps$, and are independent~of~any~property~of~$\rm P$.}
\end{minipage} 

\medskip \noindent 
We will additionally insist on what is called {\em proper learning}, i.e.~we will aim to compute a {\em tree-structured Ising model}~$\mathrm{Q}$ such that $\dtv{\mathrm{P}}{\mathrm{Q}} \leq \eps$. We want to identify a tree-structured Ising model $\rm Q$ as we want to maintain the computational benefits of tree-structured models when performing inference tasks on the learned $\mathrm{Q}$. And, of course, we want that $\mathrm{Q}$ is close to the true model $\mathrm{P}$ in total variation distance so that the inferences  on $\mathrm{Q}$ are approximately correct. We 
show that the celebrated Chow-Liu algorithm~\cite{ChowL68} attains these guarantees from an optimal number of samples.

%
%

\begin{theorem}[Main Theorem---Restated as Theorem~\ref{theoremgen} in Section~\ref{gen}]
\label{thm:main}
For $n \in \mathbb{N}, \eps, \gamma \in (0,1]$, given $O\left({n (\ln n + \ln {1/\gamma}) \over \eps^2}\right)$ samples from an arbitrary $n$-variable binary-alphabet Bayesnet $\mathrm{P}$  defined on an {\em unknown} tree, the Chow-Liu algorithm with the plug-in estimator for estimating mutual information (namely Algorithm~\ref{algo:genalgo2}), learns an $n$-variable tree-structured Bayesian Network~$\mathrm{Q}$ on a {\em (possibly different)} tree such that $\dtv{\mathrm{P}}{\mathrm{Q}} \leq \eps$ with probability at least $1-\gamma$. The running time of the algorithm is polynomial, and is dominated by the time to compute the mutual information between all pairs of variables on the empirical distribution defined by the samples, and running a Maximum Spanning tree algorithm. Finally, the number of samples is optimal up to a constant factor. Importantly, the $O(\cdot)$ in our sample complexity is hiding an absolute constant which does not depend on the model being learned---neither its tree nor the magnitude of the edge strengths.
\end{theorem}

\begin{algorithm}
\label{algo:genalgo2}
\DontPrintSemicolon
\caption{{\sc Finite Sample Chow-Liu Algorithm for Proper Learning of a Binary Tree-Structured Bayesnet} (see Section~\ref{gen} for a more detailed description of the same algorithm)}
\KwIn{$\epsilon, \gamma \in (0, 1]$; sample access to tree-structured Bayesnet $\mathrm{P}$ over binary variables $X_1,\ldots,X_n$}
\KwOut{specification of a binary tree-structured Bayesnet $\mathrm{Q}$, approximating $\mathrm{P}$}
Draw $\biggl\lceil  \frac{B \cdot n}{\epsilon^2} \cdot { \ln \frac{n}{\gamma} } \biggr\rceil$ samples from $\mathrm{P}$, for some absolute constant $B$~given~in~the~proof~of~Theorem~\ref{thm:main}. \;
Let $\widehat{\mathrm{P}}$ be the empirical joint distribution over $(X_1, \ldots, X_n)$ induced by those samples. \;
Compute the  mutual information, $\mathrm{I}^{\widehat{\mathrm{P}}}(X_i;X_j)$, between every pair of variables $X_i, X_j$ under $\widehat{\mathrm{P}}$.\label{line:costis}\;  
Run Kruskal's algorithm to find a maximum weight spanning tree $\widehat{T}$ of the complete weighted graph over $[n]$ with weight $\mathrm{I}^{\widehat{\mathrm{P}}}(X_i;X_j)$ on edge $(i,j)$.\;
\Return{(1) $\widehat{T}$ as $\mathrm{Q}$'s underlying tree; (2) the empirical distribution $\widehat{\mathrm{P}}_{ij}$ over $X_i, X_j$ computed from the sample, as $\mathrm{Q}$'s marginal on the pair of variables $X_i,X_j$, for each edge $(i,j)$ of $\widehat{T}$. } {\em /* Note that (1) and (2) suffice to uniquely determine a tree-structured Bayesnet over variables $X_1,\ldots,X_n$---see Section~\ref{sec:asymmetric}*/}
\end{algorithm}

We note that it is well-understood that the celebrated Chow-Liu algorithm~\cite{ChowL68} can be used to obtain an {\em asymptotically consistent} estimator of tree-structured models~\cite{ChowW73}. Recall that the Chow-Liu algorithm computes a tree on random variables $X_1,\ldots,X_n$ by running a Maximum Spanning Tree algorithm on a weighted clique with nodeset $V=\{X_1,\ldots,X_n\}$ and edge weight $w_{ij}$ between every pair of nodes $X_i$ and $X_j$ that is set equal to an estimate of the mutual information, $I(X_i;X_j)$, between $X_i$ and $X_j$. (These are precisely Steps 2--4 of Algorithm~\ref{algo:genalgo2}.) When the distribution  $\rm P$ of the random variables is a tree-structured model defined on some tree $T$ wherein the distributions of the endpoints $(X_i,X_j)$ of each edge satisfy some non-degeneracy condition, and when the estimation error of each $I(X_i;X_j)$ is small enough, it is known that the computed maximum spanning tree will equal $T$. Moreover, once $T$ known, it is not hard to see---using e.g.~the results of~\cite{DP17}---that $O(n\ln n/\eps^2)$ samples from ${\mathrm P}$ suffice to define some $\rm Q$ on this tree such that $\dtv{{\mathrm P}} {{\mathrm Q}} \le \eps$.

The above discussion might suggest that a natural learning algorithm for tree-structured models could aim to first identify the tree of some unknown model ${\rm P}$, and then identify a model ${\mathrm Q}$ on that tree which is close to ${\mathrm P}$. Unfortunately, it is easy to see that no finite number of samples suffices to identify the true tree even for three variable models. For an obvious example, imagine that the true model $\rm P$ is a Markov Chain with structure $A \rightarrow B \rightarrow C$. The smaller the mutual informations $I(A,B)$ and $I(B,C)$ are (i.e.~the closer to independent these variables become), the more samples are required to distinguish between the true Markov chain structure, and an alternative Markov chain structure, e.g.~$B \rightarrow A \rightarrow C$.

Indeed, despite the wide use of the Chow-Liu algorithm and the breadth of follow up work that it has spun for over five decades, we are aware of no work that provides sample- and time-efficient algorithms for learning graphical models, even tree-structured ones, without making assumptions about the model. On the non-efficient learning front, recent work of~\cite{devroye2019minimax} provides an exponential-time algorithm which uses an optimal $O({n\ln n \over \eps^2})$ samples to properly learn an arbitrary tree-structured Ising model, to within total variation distance $\eps$. Our main theorem is that Chow-Liu is an efficient algorithm that matches this~optimal~sample~bound.

The main challenge in establishing Theorem~\ref{thm:main} lies in dealing with ``structural mistakes,'' that is, the discrepancy between the true tree underlying the distribution that is being learned and the tree output by the Chow-Liu algorithm. Our proof exploits a recent squared Hellinger subadditivity theorem, established in~\cite{DP17} for the purposes of testing the identity of Bayesian networks from samples.  It involves an intricate, nesting argument that uses a hierarchical classification of the nodes into groups, induced by a classification of the edges based on some measure of strength, which captures how far from independent their endpoints are. We argue that the edges of the true model may be hard to identify because of them being and/or having in their proximity very strong or very weak edges. Despite the unavoidable mistakes that the algorithm will make in adding edges to the reconstructed tree, we argue that the contribution of these mistakes to the total variation distance between the reconstructed model and the true one can be controlled. Roughly speaking, strong edges partition the nodes into clusters of highly correlated/anti-correlated nodes, and we argue that the algorithm will correctly identify medium strength edges between the correct pairs of clusters, despite possibly assigning them incorrect endpoints. This type of argument, combined with the squared Hellinger subadditivity for bounding the error, is repeated a constant number of times at different scales of edge-strengths. It is worth noting that the partition of the nodes into the hierarchical clustering used in our argument is constructive and based on measures of strength of the edges of the tree $\widehat T$ that the Chow-Liu algorithm reconstructs. As a result, a description of the reconstruction guarantees, characterizing the symmetric difference between the true tree and $\widehat{T}$, can be computed explicitly and efficiently. Figures~\ref{fig:layer structure} and~\ref{genhierarchyfigure}, for respectively symmetric and asymmetric models, illustrate the types of reconstruction guarantees that can be derived from the learned tree.

\paragraph{A Deterministic Condition for Learning.} An interesting feature of our proof is that it identifies a deterministic condition on the sample, which suffices to guarantee that Algorithm~\ref{algo:genalgo2} will succeed in identifying a tree-structured Ising model $\mathrm{Q}$ that is within $O \paren{\epsilon}$ from $\mathrm{P}$ in total variation distance. This condition, called ``Strong $4$-Consistency'' and presented formally in Definition~\ref{gen4consistency}, roughly states  that the empirical distribution $\widehat{\mathrm{P}}$ induced by the sample well-approximates the true distribution $\mathrm{P}$ on all events involving up to $4$ variables in the model, and is especially sensitive to small probability events. It can be shown that Strong $4$-Consistency holds, with probability at least $1-\gamma$, if we draw $\Omega ( \frac{n}{\epsilon^2} \cdot { \ln \frac{n}{\gamma} } )$ i.i.d.~samples from $\mathrm{P}$. Regardless of the provenance of the samples, however, our proof establishes that, if we are handed an empirical distribution $\widehat{\mathrm{P}}$ which satisfies Strong $4$-Consistency and use that to run Algorithm~\ref{algo:genalgo2} from Line~\ref{line:costis} onwards, the algorithm will succeed in identifying a tree-structured Ising model  $\mathrm{Q}$ such that  $\dtv{\mathrm{P}}{\mathrm{Q}} \leq O \paren{\epsilon}$. In particular, we show the following.

\begin{theorem}[Restated as Theorem~\ref{gensufficiency} in Section~\ref{gen}] \label{thm:main2} For $n \in \mathbb{N}, \eps, \gamma \in (0,1]$, suppose $\widehat{\mathrm{P}}$ is a distribution over binary variables $(X_1,\ldots,X_n)$ satisfying the ``Strong $4$-Consistency''  condition of Definition~\ref{gen4consistency} with respect to distribution $\mathrm{P}$.  If we use $\widehat{\mathrm{P}}$ to run Algorithm~\ref{algo:genalgo2} from Line~\ref{line:costis} onwards, the algorithm will succeed in identifying a tree-structured Ising model  $\mathrm{Q}$ such that  $\dtv{\mathrm{P}}{\mathrm{Q}} \leq O \paren{\epsilon}$.
\end{theorem}

\paragraph{Sample Optimality.} We conclude this section by discussing the optimality of the sample requirements in our main theorem. It was shown in~\cite{DaskalakisDK19} (see Theorem 14 and Remark 4 of their paper) that, given sample access to an Ising model $\mathrm{P}$, which is known to have a tree structure, one needs $\Omega(n/\eps^2)$ samples to distinguish, with success probability at least $2/3$, whether $\mathrm{P}$ is equal to the uniform distribution over $\{0, 1\}^n$ versus $\mathrm{P}$ is $\eps$-far in total variation distance from the uniform distribution. This immediately yields an $\Omega(n/\eps^2)$ lower bound for the problem of properly learning tree-structured Ising models, as the ability to learn implies the ability to test. Using a similar but tighter construction, \cite{Koehler20} shows an $\Omega(n \ln n/\eps^2)$ sample lower bound for proper learning of tree-structured Ising models with a constant success probability, exactly matching our upper bound.

\paragraph{Paper Roadmap.} Section~\ref{sec:preliminariesss} provides the necessary preliminaries, including our notation, basic facts about Bayesian networks and Markov random fields on trees, and the probability distances that are used in this paper. Importantly, this section presents subadditivity properties satisfied by squared Hellinger, which are crucial to obtain our results. 

Section~\ref{sym} shows a bonus result, that a simple variant of the Chow-Liu algorithm, using absolute correlations rather than mutual informations as weights to build the maximum spanning tree, succeeds in the symmetric case, i.e.~when all variables of the tree-structured model that is being learned have marginal distributions that are unbiased. We provide this result because it is simpler to prove and captures most of the main ideas of the general case. We also give it to illustrate the flexibility of our argument, i.e.~that it is not specific to the use of mutual information. Section~\ref{gen} presents the proof of our main theorem, i.e.~that the Chow-Liu algorithm succeeds in the general case of arbitrary tree-structured models. Note that the constants appearing in our proofs have not been optimized; instead, our main focus has been to present our proofs in a manner that is as easy to digest as possible. Section~\ref{sec:experiments} provides an experiment whose goal is to shed light on the magnitude of the absolute constant that hides in the $O(\cdot)$ expression for the sample complexity of Theorem~\ref{thm:main}. Section~\ref{sec:conclusion} discusses potential future work motivated by the results in this paper. 

Lastly, we provide a few potential shortcuts for going through our paper. For readers whose only goal is to become familiar with the central ideas in our argument along with a basic outline of the proof, we recommend going through Sections~\ref{probnotation} to \ref{sec:bernoulli} (as needed, for basic notation and definitions) and the part of Section~\ref{sym} up to and including Section~\ref{sec:outline of symmetric case} (which contains an informal outline of our analysis of the symmetric case, and discusses the major ideas as well as the guiding intuition). The analysis of the general case employs the same basic ideas and overall paradigm as that of the symmetric case, but is much more complicated due to a few new challenges. For readers who are interested in learning about those new challenges, as well as the required modifications to the proof (of the symmetric case), we recommend additionally Section~\ref{sec:genonly} (which contains a few more notations and definitions that are relevant only to the general case) and the part of Section~\ref{gen} up to and including Section~\ref{sec:outline of general case} (which contains an informal outline of our analysis of the general case). It may also be worthwhile to go through Sections~\ref{sec:symproof} and \ref{sec:genproof} (which contain the proof details for the symmetric and the general case, respectively) but skipping the proofs of all the lemmas. We arranged our lemmas in a way so that their statements, when read in order, would illustrate the organizing logic underneath our proof. 

\paragraph{Comparison to Bhattacharyya et al~\cite{bhattacharyya2020near}.}

After the initial appearance of our paper on the arXiv, we received a manuscript by Arnab Bhattacharyya, Sutanu Gayen, Eric Price, and N. V. Vinodchandran, which also provided results for proper learning of tree-structured models, using different techniques. Their work, which appeared just ten days after ours on the arXiv~\cite{bhattacharyya2020near}, compares to ours as follows. Their work is stronger in that their analysis of Chow-Liu also accommodates
non-binary discrete alphabets. However, our results are stronger in that our sample complexity of $O(n \ln n/\epsilon^2)$ for learning Ising models is information-theoretically optimal. Their sample complexity is off by a factor of $\Omega(\ln {n \over \epsilon})$ for learning under total variation distance, and it can provably not be made optimal using their approach, which learns in KL-divergence and applies Pinsker's inequality to learn in total variation. Moreover, our analysis is constructive, leading to a characterization of the structural mistakes that Chow-Liu will make, which can also be (partially) computed from the output model. This characterization allows our method to be more flexible and able to accommodate variants of the Chow-Liu algorithm where, instead of using mutual information, other measures of node similarity are used to construct the Maximum Spanning Tree, as shown in Section~\ref{sym}.

\newpage

\section{Preliminaries} \label{sec:preliminariesss}

\subsection{Probability Notation}
\label{probnotation}

In this paper we use the term \textbf{random variable} (or just \textbf{variable}) to refer to some symbol, say $X$, that takes value in some alphabet set, say $\mathcal{A}$, whose exact distribution is left to be assigned. This deviates slightly from the general convention in probability literature where the term ``random variable'' also encompasses the assignment of a distribution. Our choice makes it so that it is possible to speak of different distributions for the same random variable. For example, it now makes sense to say ``suppose $\mathrm{P}$ and $\mathrm{Q}$ are distributions for $X$''. This will greatly simplify our notation.

Throughout the text, we use subscripts at probability symbols to help zoom in on subsets of random variables. If $\mathrm{P}$ is a joint distribution for a set of random variables $X_1, \ldots, X_n$, and $S$ is a subset of those variables, or just a subset of the indexes $[n] = \{1, \ldots, n\}$, then we use $\mathrm{P}_S$ to denote the marginal distribution of $\mathrm{P}$ for those variables in $S$. For example, we use $\mathrm{P}_{X_i X_j X_k}$, or just $\mathrm{P}_{ijk}$, to denote the marginal of $\mathrm{P}$ for the triplet of variables $X_i,X_j,X_k$. When the variables (or indexes) are spelled out explicitly at the subscript, the values inside the parentheses that may follow are to be assigned to the variables in the order they (or their indexes) are spelled out. For example, $\mathrm{P}_{ijk}(a,b,c)$ is the probability of the event ``$X_i = a, X_j = b, X_k = c$''. We list the indexes in alphabetical order at the subscripts whenever possible to make such variable assignments easier to read.

If the event of interest is made clear inside the parentheses, we may drop the subscript at the probability symbol, as in $\mathrm{P}(X_i = 1, X_j = -1, X_k = 1)$ or $\mathrm{P}(X_i = X_j)$. The generalization of everything so far to conditional distributions is straightforward, with examples including $\mathrm{P}_{S|S'}$, $\mathrm{P}_{ik|j}(a,b|c)$, and $\mathrm{P}(X_i = a, X_k = b \mid X_j = c)$.

When $S$ is a subset of the indexes, we use $X_S$ to denote the set of variables with indexes in $S$, and $x_S$ to denote a specific assignment of values to those variables, perhaps as a restriction of the general assignment $x$.

\subsection{Markov Random Fields and Bayesian Networks}
\label{MRF+BN}

We provide the basic definitions of Markov Random Fields and Bayesian Networks.

\begin{definition}
\label{def:MRF}
A \textbf{Markov Random Field (MRF)} is a joint distribution $\mathrm{P}$ for $X_1, \ldots, X_n$ with an underlying undirected graph $G$ on $[n]$, where each $X_i$ takes value in $\mathcal{A}$. Associated with each maximal clique $C \in \mathfrak{cl}(G)$ is a potential function $\psi_C: \mathcal{A}^C \rightarrow [0,1]$, where $\mathfrak{cl}(G)$ is the set of all maximal cliques of $G$. In terms of these potentials, $\mathrm{P}$ assigns to each vector $x \in \mathcal{A}^n$ a probability $\mathrm{P}(x)$ satisfying
\begin{align}
\mathrm{P}(x) = {1 \over Z} \prod_{C \in \mathfrak{cl}(G)} \psi_C(x_C), \label{eq:graphical model}
\end{align}
where $Z$ is a normalization constant making sure that $\mathrm{P}$, as defined above, is a distribution. In the degenerate case where the products on the RHS of~\eqref{eq:graphical model} always evaluate to $0$, we assume that $\mathrm{P}$ is the uniform distribution over $\mathcal{A}^n$. In that case, we get the same distribution by assuming that all potential functions are identically $1$. Hence, we can assume that the products on the RHS of~\eqref{eq:graphical model} cannot always evaluate to $0$.
\end{definition}

\begin{definition}
\label{def:Bayesnet}
A \textbf{Bayesian network}, or \textbf{Bayesnet}, is a joint distribution $\mathrm{P}$ for $X_1, \ldots, X_n$ with an underlying directed acyclic graph $G$, where each $X_i$ takes value in $\mathcal{A}$. To describe $\mathrm{P}$, one specifies conditional probabilities $\mathrm{P}_{i |\Pi_i}(x_i|x_{\Pi_i})$, for all $i \in [n]$, and configurations $x_i \in \mathcal{A}$ and $x_{\Pi_i} \in \mathcal{A}^{\Pi_i}$, where $\Pi_i$ represents the set of parents of $i$ in $G$, taken to be $\emptyset$ if $i$ has no parents. In terms of these conditional probabilities, $\mathrm{P}$ assigns to each vector $x \in \mathcal{A}^n$ a probability $\mathrm{P}(x)$ satisfying
$$\mathrm{P}(x)  = \prod_{i} \mathrm{P}_{i | \Pi_i} (x_i | x_{\Pi_i}).$$
\end{definition}

It is important to note that both MRFs and Bayesnets allow the study of distributions in their full generality, as long as the graphs on which they are defined are sufficiently dense. In particular, the graphs underlying these models captures conditional independence relations, and is sufficiently flexible to capture the structure of intricate dependencies in the data. 

In this paper, we focus on the special case of \textbf{tree-structured} MRFs and Bayesnets, namely:
\begin{itemize}
\item MRFs whose underlying undirected graphs are trees;
\item Bayesnets whose underlying DAGs are directed rooted trees (with a designated root and all edges oriented away from the root).
\end{itemize}
It is easy to see that tree-structured MRFs and Bayesnets have the exact same expressive power. More specifically, for any directed rooted tree $T$, the set of all Bayesnets on $T$ coincides with the set of all MRFs on the undirected version of $T$. In view of this equivalence, we will only use the term ``tree-structured Bayesnet'', but think of the underlying tree $T$ as undirected, unless otherwise noted.

Tree-structured Bayesnets on $T$ satisfy the following graph separation criterion for conditional independence relations: if $A,B,C$ are disjoint subsets of $[n]$ such that the path in $T$ between any $i \in A$ and any $j \in B$ goes through at least one node in $C$ (that is, $C$ \textbf{separates} $A$ from $B$ in $T$), then $X_A$ and $X_B$ are independent conditional on $X_C$. By ``path'', we mean the unique shortest path (in the tree being specified), always.

Markov chains are special cases of tree-structured Bayesnets whose underlying trees are simply lines. A \textbf{Markov chain} for $X_1, \ldots, X_n$, in that order, is a tree-structured Bayesnet whose underlying tree has edges $(1,2), (2,3), \ldots, (n-1,n)$. In this paper, we mostly utilize the Markov chain property of certain subsets of the variables. We will come back to this point in Section~\ref{sec:asymmetric}.

In this paper, we also focus on the case of \textbf{binary} alphabets, i.e. Ising models. We use $\mathcal{A} = \{1, -1\}$ throughout.

Finally, we call a binary Bayesnet \textbf{symmetric} iff each variable takes each of $1,-1$ with probability $\frac{1}{2}$.

\subsection{Tree-Structured Bayesnets}
\label{sec:asymmetric}

A tree-structured Bayesnet $\mathrm{P}$ for $X_1, \ldots, X_n$ can be uniquely determined by specifying its underlying tree $T$, together with its marginal distributions for pairs of variables corresponding to edges of $T$ (in the binary case, for example, each such marginal $\mathrm{P}_{ij}$ is described by four nonnegative numbers $\mathrm{P}_{ij}(1,1)$, $\mathrm{P}_{ij}(1,-1)$, $\mathrm{P}_{ij}(-1,1)$, and $\mathrm{P}_{ij}(-1,-1)$ that sum to $1$). The only requirement is that the set of all those pairwise marginals is \textbf{consistent}, in the sense that if two edges share some node $i$, then further marginalizations of the two corresponding pairwise marginals onto $X_i$ provide the same result.

To see that, root $T$ at $1$, and assume without loss of generality (by possibly reindexing) that $i < j$ implies the depth of $i$ is at most the depth of $j$. For $2 \leq i \leq n$, let $\pi(i)$ denote the parent of $i$ (our assumption implies $\pi(i) < i$). Note that the edges of $T$ are then $\{ (i, \pi(i)) \mid 2 \leq i \leq n \}$. Since $\pi(i)$ separates $i$ from everything else in $1, \ldots, i-1$, we have
\begin{align}
\label{treefactorization}
\mathrm{P}(x) = \mathrm{P}_1(x_1) \prod_{i=2}^n \mathrm{P}_{i|1 \ldots i-1}(x_i|x_1, \ldots, x_{i-1}) = \mathrm{P}_1(x_1) \prod_{i=2}^n \mathrm{P}_{i|\pi(i)}(x_i|x_{\pi(i)}).
\end{align}
Note that the values of the terms in the rightmost expression can be determined from the marginal distributions for the edges of $T$, which we assume to have been specified.

Given a tree-structured Bayesnet $\mathrm{P}$ specified as above, we might be interested in the properties of $\mathrm{P}_S$ for some subset $S$ of $[n]$. We will characterize a class of subsets for which this description is simple. We need some definitions first.

\begin{definition}[\textbf{Induced Subgraph,} $\pmb{T}$\textbf{-connectedness}]
\label{treeconnectedness}
For a tree $T$, and a subset of nodes $S$, we say that $S$ is $\pmb{T}$\textbf{-connected} if and only if $T[S]$ is connected, where $\pmb{T[S]}$ is the subgraph in $T$ induced by $S$. Equivalently, $S$ is $T$-connected iff for any $i,j \in S$, all the nodes on the path in $T$ between $i$ and $j$ are in $S$.
\end{definition}

If a tree-structured Bayesnet $\mathrm{P}$ has underlying tree $T$, and $S$ is a $T$-connected subset of $[n]$, then $\mathrm{P}_S$ is also a tree-structured Bayesnet. Furthermore, its underlying tree is $T[S]$ (a tree according to Definition~\ref{treeconnectedness}), and for an edge $(i,j)$ of $T[S]$, its marginal distribution for $(i,j)$ is $\mathrm{P}_{ij}$. To see that, notice that if $A,B,C$ are disjoint subsets of $S$, then $C$ separates $A$ from $B$ in $T[S]$ if and only if $C$ separates $A$ from $B$ in $T$. Thus, $\mathrm{P}_S$ satisfies all conditional independence relations implied by $T[S]$, and so admits a factorization with respect to $T[S]$ (follow steps similar to those leading to (\ref{treefactorization}) above). The assertion on the marginals of $\mathrm{P}_S$ for edges of $T[S]$ is obvious.

There are also subsets that are not necessarily $T$-connected for which the marginals of $\mathrm{P}$ are simple to describe. Instead of striving for the most general results in this vein, we content ourselves with the following special case that will appear many times throughout the paper. Suppose indexes $i_1, i_2 \ldots, i_r$ \textbf{lie on a path in} $\pmb{T}$, that is, $i_2$ sits on the path in $T$ between $i_1$ and $i_r$, $i_3$ sits on the path in $T$ between $i_2$ and $i_r$, and so on. Then it is easy to see that $\mathrm{P}_{i_1 i_2 \cdots i_r}$ is a Markov chain for $X_{i_1}, X_{i_2}, \ldots, X_{i_r}$, in that order, with marginal $\mathrm{P}_{i_s i_{s+1}}$ for the pair of variables $X_{i_s},X_{i_{s+1}}$ for each $s = 1, \ldots, r-1$.

\subsection{Binary Symmetric Tree-Structured Bayesnets}
\label{sec:symmetric}

For a binary symmetric tree-structured Bayesnet $\mathrm{P}$, its marginal distribution for each pair $(X_i,X_j)$ (not necessarily an edge of the underlying tree) is determined simply by the probability of the two nodes being equal, namely $\mathrm{P}(X_i=X_j)$. For convenience, we reparametrize $\mathrm{P}(X_i = X_j)$ using $\alpha_{ij}$, such that $\mathrm{P}(X_i=X_j) = \frac{1+\alpha_{ij}}{2}$, and $\mathrm{P}(X_i = -X_j) = \frac{1-\alpha_{ij}}{2}$. So overall we have $\mathrm{P}(X_i=1,X_j=1) = \frac{1}{2} \frac{1 + \alpha_{ij}}{2}$, $\mathrm{P}(X_i=1,X_j=-1) = \frac{1}{2} \frac{1 - \alpha_{ij}}{2}$, $\mathrm{P}(X_i=-1,X_j=1) = \frac{1}{2} \frac{1 - \alpha_{ij}}{2}$, and $\mathrm{P}(X_i=-1,X_j=-1) = \frac{1}{2} \frac{1 + \alpha_{ij}}{2}$. We call $\alpha_{ij}$ the $\pmb{\alpha}$\textbf{-value} of the pair $(X_i,X_j)$, or just of $(i,j)$, for $\mathrm{P}$. A binary symmetric tree-structured Bayesnet $\mathrm{P}$ can therefore be uniquely determined by specifying its underlying tree $T$, together with the $\alpha$-value of each edge of $T$. Edges with $\alpha$-values close to $1$ or $-1$ are strong; the two variables are almost always equal or almost always unequal under $\mathrm{P}$. Edges with $\alpha$-values close to $0$ are weak; the two variables are nearly independent under $\mathrm{P}$.

The convenience of using the $\alpha$-value lies in their \textbf{multiplicativity}, which we prove after first proving an independence result.

\begin{lemma}
\label{independence}
Suppose a binary symmetric tree-structured Bayesnet $\mathrm{P}$ has underlying tree $T$, and $i,j,k$ lie on a path in $T$, then
$$\mathrm{P}(X_i=X_j, X_j=X_k) = \mathrm{P}(X_i=X_j) \mathrm{P}(X_j=X_k).$$
The identity still holds if we replace ``$X_i=X_j$'' with ``$X_i = -X_j$'' (on both sides), or ``$X_j=X_k$'' with ``$X_j = -X_k$'', or both.
\end{lemma}
\begin{proof}
It is easy to see that, since $\mathrm{P}$ is symmetric, we have
$$ \mathrm{P}_{j|i}(1|1) = \mathrm{P}_{j|i}(-1|-1) = \mathrm{P}(X_i = X_j), $$
and similarly for $X_j$ and $X_k$. By the Markov property, we have
\begin{align*}
\mathrm{P}(X_i=X_j, X_j=X_k) &= \mathrm{P}_{ijk}(1,1,1) + \mathrm{P}_{ijk}(-1,-1,-1) \\
	&= \mathrm{P}_i(1) \mathrm{P}_{j|i}(1|1) \mathrm{P}_{k|j}(1|1) + \mathrm{P}_i(-1) \mathrm{P}_{j|i}(-1|-1) \mathrm{P}_{k|j}(-1|-1) \\
	&= \frac{1}{2} \mathrm{P}(X_i=X_j) \mathrm{P}(X_j=X_k) + \frac{1}{2} \mathrm{P}(X_i=X_j) \mathrm{P}(X_j=X_k) \\
	&= \mathrm{P}(X_i=X_j) \mathrm{P}(X_j=X_k)
\end{align*} 
The other three cases follow in similar fashion.
\end{proof}

\begin{lemma}[$\pmb{\alpha}$\textbf{-value Multiplicativity}]
\label{multiplicativity}
Suppose a binary symmetric tree-structured Bayesnet $\mathrm{P}$ has underlying tree $T$, and $i_1,i_2, \ldots ,i_r$ lie on a path in $T$, then $\alpha_{i_1 i_r} = \alpha_{i_1 i_2} \ldots \alpha_{i_{r-1} i_r}$.
\end{lemma}
\begin{proof}
\begin{align*} \alpha_{i_1 i_r} & = 2 \mathrm{P}(X_{i_1}=X_{i_r}) - 1 \\
                                      & = 2 \sqbrac{ \mathrm{P}(X_{i_1}=X_{i_{r-1}}, X_{i_{r-1}}=X_{i_r}) + \mathrm{P}(X_{i_1}=-X_{i_{r-1}}, X_{i_{r-1}}=-X_{i_r}) } - 1\\
                                      & \stackrel{(\ast)}{=} 2 \sqbrac{ \mathrm{P}(X_{i_1}=X_{i_{r-1}}) \mathrm{P}(X_{i_{r-1}}=X_{i_r}) + \mathrm{P}(X_{i_1}=-X_{i_{r-1}}) \mathrm{P}(X_{i_{r-1}}=-X_{i_r}) } - 1\\
                                      & = 2 \paren{ \frac{1+\alpha_{i_1 i_{r-1}}}{2} \frac{1+\alpha_{i_{r-1} i_r}}{2} + \frac{1-\alpha_{i_1 i_{r-1}}}{2} \frac{1-\alpha_{i_{r-1} i_r}}{2} } - 1\\
                                      & = \alpha_{i_1 i_{r-1}}\alpha_{i_{r-1} i_r},
\end{align*}
where $(\ast)$ follows from Lemma~\ref{independence}. The desired statement thus follows by induction on $r$.
\end{proof}

\subsection{Distances and the Subadditivity of Squared Hellinger}
\label{sec:hellinger}

Even though our final goal is to learn in the total variation distance, $d_{\rm TV}$, for the most part of our argument we will work instead with the Hellinger distance $H$, and its square $H^2$. Let's define those quantities.

\begin{definition}[\textbf{Total Variation, Hellinger}]
For two discrete distributions $\mathrm{p}' = (p'_1, \ldots, p'_L)$ and $\mathrm{p}'' = (p''_1, \ldots, p''_L)$ over a domain of size $L$, their \textbf{total variation distance} is defined as
$$\dtv{\mathrm{p}'}{\mathrm{p}''} = \frac{1}{2} \sum_{l=1}^L \abs{p'_l - p''_l},$$
whereas their \textbf{Hellinger distance} is defined as
$$H(\mathrm{p}',\mathrm{p}'') = \frac{1}{\sqrt{2}} \sqrt{\sum_{l=1}^L \paren{ \sqrt{p'_l} - \sqrt{p''_l} }^2}.$$
The \textbf{squared Hellinger} is therefore
$$H^2(\mathrm{p}',\mathrm{p}'') = \frac{1}{2} \sum_{l=1}^L \paren{ \sqrt{p'_l} - \sqrt{p''_l} }^2 = 1 - \sum_{l=1}^L \sqrt{p'_l p''_l}.$$
\end{definition}

Note that while $d_{\rm TV}$ and $H$ satisfy the triangle inequality, $H^2$ does not, so it is not a distance metric on probability distributions. The Hellinger distance always takes value in $[ 0,1 ]$. Compared with the total variation distance, Hellinger distance satisfies the following inequalities.

\begin{lemma}
\label{tvvshel}
$$ H^2(\mathrm{p}',\mathrm{p}'') \leq \dtv{\mathrm{p}'}{\mathrm{p}''} \leq \sqrt{2} H(\mathrm{p}',\mathrm{p}''). $$
\end{lemma}
\begin{proof}
\begin{align*}
H^2(\mathrm{p}',\mathrm{p}'') = \frac{1}{2} \sum_{l=1}^L \paren{ \sqrt{p'_l} - \sqrt{p''_l} }^2 & \leq \frac{1}{2} \sum_{l=1}^L \abs{\sqrt{p'_l} - \sqrt{p''_l}} \cdot \paren{ \sqrt{p'_l} + \sqrt{p''_l} } \\
& = \frac{1}{2} \sum_{l=1}^L \abs{p'_l - p''_l} \\
&= \dtv{\mathrm{p}'}{\mathrm{p}''}
\end{align*}
\begin{align*}
\dtv{\mathrm{p}'}{\mathrm{p}''} = \frac{1}{2} \sum_{l=1}^L \abs{p'_l - p''_l} &= \frac{1}{2} \sum_{l=1}^L \abs{\sqrt{p'_l} - \sqrt{p''_l}} \cdot \paren{ \sqrt{p'_l} + \sqrt{p''_l} } \\
                       &\leq \frac{1}{2} \sqrt{\sum_{l=1}^L \paren{ \sqrt{p'_l} - \sqrt{p''_l} }^2} \sqrt{\sum_{l=1}^L \paren{ \sqrt{p'_l} + \sqrt{p''_l} }^2}\\
                       &\leq \frac{1}{2} \sqrt{\sum_{l=1}^L \paren{ \sqrt{p'_l} - \sqrt{p''_l} }^2} \sqrt{\sum_{l=1}^L 2(p'_l + p''_l)}\\
                       &= \sqrt{\sum_{l=1}^L \paren{ \sqrt{p'_l} - \sqrt{p''_l} }^2} \\
                       &= \sqrt{2} H(\mathrm{p}',\mathrm{p}'')
\end{align*}
\end{proof}

The squared Hellinger satisfies the following subadditivity inequality, whose proof is contained in~\cite{DP17}. The fact is interesting considering that $H^2$ is not even a distance metric.

\begin{theorem}[\textbf{Squared Hellinger Subadditivity}]
\label{subadditivity}
Let $S_1, \ldots, S_{\Lambda}$ be pairwise disjoint sets whose union is $[n]$. Suppose that $\mathrm{P}'$ and $\mathrm{P}''$ are joint distributions for $X_1, \ldots, X_n$ with common factorization structure
\begin{align*}
\mathrm{P}'(x) &= \mathrm{P}'_{S_1} (x_{S_1}) \prod_{\lambda=2}^{\Lambda} \mathrm{P}'_{S_{\lambda} | \Pi_{\lambda}} (x_{S_{\lambda}} | x_{\Pi_{\lambda}}) \\
\mathrm{P}''(x) &= \mathrm{P}''_{S_1} (x_{S_1}) \prod_{{\lambda}=2}^{\Lambda} \mathrm{P}''_{S_{\lambda} | \Pi_{\lambda}} (x_{S_{\lambda}} | x_{\Pi_{\lambda}})
\end{align*}
where $\Pi_{\lambda} \subset S_1 \cup \cdots \cup S_{\lambda-1}$ corresponds to the set of variables conditioned on which $X_{S_{\lambda}}$ is independent from everything else in $S_1, \ldots, S_{\lambda-1}$. Then $$H^2(\mathrm{P}',\mathrm{P}'') \leq H^2(\mathrm{P}'_{S_1}, \mathrm{P}''_{S_1}) + H^2(\mathrm{P}'_{S_2 \cup \Pi_2}, \mathrm{P}''_{S_2 \cup \Pi_2}) + \cdots + H^2(\mathrm{P}'_{S_{\Lambda} \cup \Pi_{\Lambda}}, \mathrm{P}''_{S_{\Lambda} \cup \Pi_{\Lambda}}).$$
\end{theorem}

The following important corollary of Theorem~\ref{subadditivity} will be used heavily in bounding the Hellinger distance between the true distribution and our learned distribution.

\begin{corollary}
\label{edgeswitch}
Suppose $\mathrm{P}'$ and $\mathrm{P}''$ are tree-structured Bayesnets for $X_1, \ldots, X_n$, with underlying trees $T'$ and $T''$, respectively. Suppose $A_1, \ldots, A_{\Lambda}$ are nonempty, pairwise disjoint sets whose union is $[n]$, such that
\begin{enumerate}
\item $A_{\lambda}$ is both $T'$-connected and $T''$-connected, for each $\lambda$;
\item there is an edge in $T'$ between $A_{\lambda}$ and $A_{\mu}$ iff there is an edge in $T''$ between $A_{\lambda}$ and $A_{\mu}$, for each pair $\lambda, \mu$.
\end{enumerate}
Let $\mathcal{E}$ be the set of all unordered pairs $(\lambda, \mu)$ such that there is an edge in $T'$ and an edge in $T''$ straddling $A_{\lambda}$ and $A_{\mu}$ (note that each straddling edge must be unique due to the $T'/T''$-connectedness of $A_{\lambda}$ and $A_{\mu}$). For each $(\lambda, \mu) \in \mathcal{E}$, let $W_{\lambda \mu}$ be the set of endpoints of those two straddling edges, so it has cardinality 2, 3, or 4, depending on how many endpoints are shared. Then,
\begin{equation}
\label{edgeswitchineq}
H^2(\mathrm{P}',\mathrm{P}'') \leq \sum_{\lambda=1}^{\Lambda} H^2(\mathrm{P}'_{A_\lambda}, \mathrm{P}''_{A_\lambda}) + \sum_{(\lambda, \mu) \in \mathcal{E}} H^2(\mathrm{P}'_{W_{\lambda \mu}}, \mathrm{P}''_{W_{\lambda \mu}}).
\end{equation}
\end{corollary}
\begin{proof}
It is easy to see that the edges in $\mathcal{E}$ form a tree on $[\Lambda]$. Call it $\mathcal{T}$. Consider $\mathcal{T}$ as rooted at $1$, we can assume without loss of generality (by possibly reindexing the $A_{\lambda}$'s) that $\lambda < \mu$ implies the depth of $\lambda$ is at most the depth of $\mu$. For $2 \leq \lambda \leq \Lambda$, let $\pi(\lambda)$ denote the parent of $\lambda$ (our assumption implies $\pi(\lambda) < \lambda$). Note that the edge set $\mathcal{E}$ then has the form $\{ (\lambda, \pi(\lambda)) \mid 2 \leq \lambda \leq \Lambda \}$. To simplify notation, for $2 \leq \lambda \leq \Lambda$, let $U_{\lambda} = W_{\lambda,\pi(\lambda)} \cap A_{\pi(\lambda)}$, so it consists of the endpoints of the edges in $T'$ and $T''$ straddling $A_{\lambda}$ and $A_{\pi(\lambda)}$ that are on the $A_{\pi(\lambda)}$ side. Let $V_{\lambda} = W_{\lambda,\pi(\lambda)} \cap A_{\lambda}$, so it consists of the endpoints of the edges in $T'$ and $T''$ straddling $A_{\lambda}$ and $A_{\pi(\lambda)}$ that are on the $A_{\lambda}$ side. Then, $\mathrm{P}'$ and $\mathrm{P}''$ have the common factorization
\begin{align*}
\mathrm{P}'(x) & = \mathrm{P}'_{A_1} (x_{A_1}) \prod_{\lambda=2}^{\Lambda} \sqbrac{ \mathrm{P}'_{V_{\lambda} | U_{\lambda}} (x_{V_{\lambda}} | x_{U_{\lambda}}) \cdot \mathrm{P}'_{A_{\lambda} \setminus V_{\lambda} \, | \, V_{\lambda}} (x_{A_{\lambda} \setminus V_{\lambda}} | x_{V_{\lambda}}) }; \\
\mathrm{P}''(x) & = \mathrm{P}''_{A_1} (x_{A_1}) \prod_{\lambda=2}^{\Lambda} \sqbrac{ \mathrm{P}''_{V_{\lambda} | U_{\lambda}} (x_{V_{\lambda}} | x_{U_{\lambda}}) \cdot \mathrm{P}''_{A_{\lambda} \setminus V_{\lambda} \, | \, V_{\lambda}} (x_{A_{\lambda} \setminus V_{\lambda}} | x_{V_{\lambda}}) }.
\end{align*}
Since $U_{\lambda} \cup V_{\lambda} = W_{\lambda, \pi(\lambda)}$, we can apply Theorem~\ref{subadditivity} to obtain the desired inequality.
\end{proof}

\subsection{Notation for Creating Markov Chains on 3 or 4 Nodes with Various Orderings}
\label{sec:smallmarkov}

The notation to be introduced next is useful later, for example, when we reason about the relative positioning of a small subset (3 or 4) of nodes in the (unknown) underlying tree of the distribution we wish to learn. The definition relies on the method for specifying a tree-structured Bayesnet mentioned in Section~\ref{sec:asymmetric}. The subscript at each symbol is a drawing of the underlying line graph of the Markov chain being represented. For example, $i\wideparen{-j \,~~~ }k$ is the line graph on $i,j,k$ with edges $(i,j)$ and $(i,k)$. It is the underlying tree of $\mathrm{P}_{i\wideparen{-j \,~~ }k}$, a Markov chain for $X_j,X_i,X_k$, in that order (we want to keep the indexes at the subscripts in alphabetical order, which is why we adopted $i\wideparen{-j \,~~~ }k$ instead of $j-i-k$, for example).
\begin{definition}
\label{arcnotation}
Given an arbitrary joint distribution $\mathrm{P}$ for random variables $X_1, \ldots, X_n$, and $i,j,k \in [n]$, we adopt the following notations:
\begin{itemize}
\item[] $\mathrm{P}_{i\wideparen{-j \,~~ }k}$ - the Markov chain for $X_j,X_i,X_k$, in that order, with the marginal for $(i,j)$ equal to $\mathrm{P}_{ij}$, and the marginal for $(i,k)$ equal to $\mathrm{P}_{ik}$;
\item[] $\mathrm{P}_{i-j-k}$ - the Markov chain for $X_i,X_j,X_k$, in that order, with the marginal for $(i,j)$ equal to $\mathrm{P}_{ij}$, and the marginal for $(j,k)$ equal to $\mathrm{P}_{jk}$;
\item[] $\mathrm{P}_{i\wideparen{ \,~~ j-}k}$ - the Markov chain for $X_i,X_k,X_j$, in that order, with the marginal for $(i,k)$ equal to $\mathrm{P}_{ik}$, and the marginal for $(j,k)$ equal to $\mathrm{P}_{jk}$. 
\end{itemize}
At the subscript of each of the three symbols listed above, the indexes are regarded as being spelled out in the order of $i,j,k$. The convention regarding assigning values to variables introduced in Section~\ref{probnotation} still applies, and we use, for example, $\mathrm{P}_{i\wideparen{-j \,~~ }k} (a,b,c)$ to denote the probability of the event $``X_i = a, X_j = b, X_k = c"$ under $\mathrm{P}_{i\wideparen{-j \,~~ }k}$.

We can generalize the paradigm above to subsets of four nodes. For $h,i,j,k \in [n]$, we have for example
\begin{itemize}
\item[] $\mathrm{P}_{h-i\wideparen{ \,~~ j-}k}$ - the Markov chain for $X_h,X_i,X_k,X_j$, in that order, with the marginal for $(h,i)$ equal to $\mathrm{P}_{hi}$, the marginal for $(i,k)$ equal to $\mathrm{P}_{ik}$, and the marginal for $(j,k)$ equal to $\mathrm{P}_{jk}$;
\item[] $\mathrm{P}_{h\wideparen{-i \,~~ j-}k}$ - the Markov chain for $X_i,X_h,X_k,X_j$, in that order, with the marginal for $(h,i)$ equal to $\mathrm{P}_{hi}$, the marginal for $(h,k)$ equal to $\mathrm{P}_{hk}$, and the marginal for $(j,k)$ equal to $\mathrm{P}_{jk}$,
\end{itemize}
and so on. At the subscript of each of the two symbols listed above, the indexes are regarded as being spelled out in the order of $h,i,j,k$. So we use, for example, $\mathrm{P}_{h-i\wideparen{ \,~~ j-}k} (a,b,c,d)$ to denote the probability of the event $``X_h = a, X_i = b, X_j = c, X_k = d"$ under $\mathrm{P}_{h-i\wideparen{ \,~~ j-}k}$.
\end{definition}
\begin{remark}
Note that while $\mathrm{P}_{ijk}$ (or $\mathrm{P}_{hijk}$) may not satisfy any conditional independence relation, the five distributions listed in Definition~\ref{arcnotation} are constrained to be Markov chains with the specified orderings of nodes, by design.
\end{remark}

The following lemma is a good exercise to get familiar with the notation just introduced. It also involves some simple applications of squared Hellinger subadditivity (Lemma~\ref{subadditivity}) introduced in Section~\ref{sec:hellinger}.

\begin{lemma}
\label{hel4nodes}
Suppose a tree-structured Bayesnet $\mathrm{P}$ has underlying tree $T$, and $h,i,j,k$ lie on a path in $T$. Then, $H^2(\mathrm{P}_{hijk}, \mathrm{P}_{h\wideparen{-i \,~~ j-}k}) \leq \paren{H(\mathrm{P}_{ijk}, \mathrm{P}_{i\wideparen{ \,~~ j-}k}) + H(\mathrm{P}_{hik}, \mathrm{P}_{h\wideparen{-i \,~~ }k})}^2$. 
\end{lemma}
\begin{proof}
We will use $\mathrm{P}_{h-i\wideparen{ \,~~ j-}k}$ as an intermediate step between $\mathrm{P}_{hijk}$ and $\mathrm{P}_{h\wideparen{-i \,~~ j-}k}$. See Figure~\ref{hel4nodesfigure} for depictions of the three distributions involved).

\begin{figure}[h]
\centering
\includegraphics[width=16.5cm]{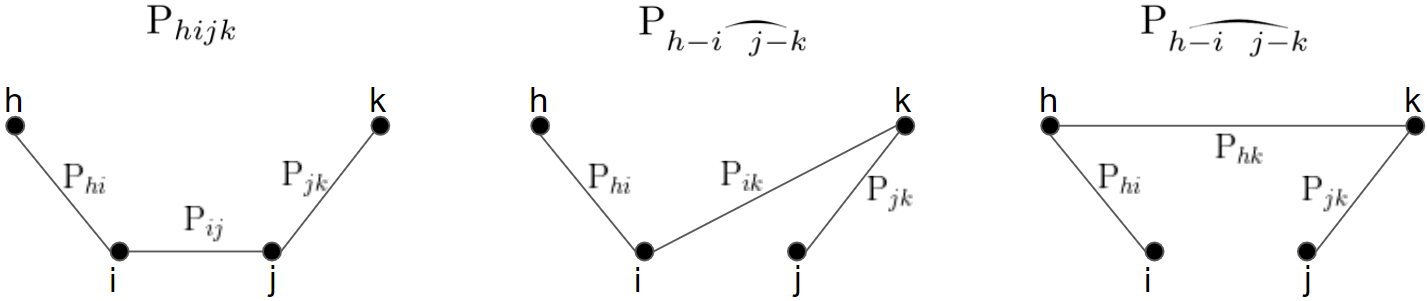}
\caption{depictions of $\mathrm{P}_{hijk}$, $\mathrm{P}_{h-i\wideparen{ \,~~ j-}k}$, and $\mathrm{P}_{h\wideparen{-i \,~~ j-}k}$ in terms of their underlying trees and their pairwise marginals on the edges of the respective underlying trees}
\label{hel4nodesfigure}
\end{figure}

To bound the distance between $\mathrm{P}_{hijk}$ and $\mathrm{P}_{h-i\wideparen{ \,~~ j-}k}$, notice that they share the common factorization structure:
\begin{align*}
\mathrm{P}_{hijk}(x_{hijk}) &= \mathrm{P}_{ijk} (x_{ijk}) \mathrm{P}_{h | i} (x_h | x_i) \\
\mathrm{P}_{h-i\wideparen{ \,~~ j-}k}(x_{hijk}) &= \mathrm{P}_{i\wideparen{ \,~~ j-}k} (x_{ijk}) \mathrm{P}_{h | i} (x_h | x_i)
\end{align*}
By Theorem~\ref{subadditivity}, we have
\begin{align*}
H^2(\mathrm{P}_{hijk}, \mathrm{P}_{h-i\wideparen{ \,~~ j-}k}) &\leq H^2(\mathrm{P}_{ijk}, \mathrm{P}_{i\wideparen{ \,~~ j-}k}) + H^2(\mathrm{P}_{hi}, \mathrm{P}_{hi}) = H^2(\mathrm{P}_{ijk}, \mathrm{P}_{i\wideparen{ \,~~ j-}k}) \\
&\Longrightarrow H(\mathrm{P}_{hijk}, \mathrm{P}_{h-i\wideparen{ \,~~ j-}k}) \leq H(\mathrm{P}_{ijk}, \mathrm{P}_{i\wideparen{ \,~~ j-}k})
\end{align*}
(In fact we have equality, which can be seen by directly applying the definition of $H^2$ to $H^2(\mathrm{P}_{hijk}, \mathrm{P}_{h-i\wideparen{ \,~~ j-}k})$ and $H^2(\mathrm{P}_{ijk}, \mathrm{P}_{i\wideparen{ \,~~ j-}k})$, and using the common factorization of $\mathrm{P}_{hijk}$ and $\mathrm{P}_{h-i\wideparen{ \,~~ j-}k}$ above. However, the inequality version suffices, and demonstrates the use of Lemma~\ref{subadditivity}.)

Simiarly, to bound the distance between $\mathrm{P}_{h-i\wideparen{ \,~~ j-}k}$ and $\mathrm{P}_{h\wideparen{-i \,~~ j-}k}$, notice that they share the common factorization structure:
\begin{align*}
\mathrm{P}_{h-i\wideparen{ \,~~ j-}k}(x_{hijk}) &= \mathrm{P}_{hik} (x_{hik}) \mathrm{P}_{j | k} (x_j | x_k) \\
\mathrm{P}_{h\wideparen{-i \,~~ j-}k}(x_{hijk}) &= \mathrm{P}_{h\wideparen{-i \,~~ }k} (x_{hik}) \mathrm{P}_{j | k} (x_j | x_k)
\end{align*}
By Theorem~\ref{subadditivity}, we have
\begin{align*}
H^2(\mathrm{P}_{h-i\wideparen{ \,~~ j-}k}, \mathrm{P}_{h\wideparen{-i \,~~ j-}k}) &\leq H^2(\mathrm{P}_{hik}, \mathrm{P}_{h\wideparen{-i \,~~ }k}) + H^2(\mathrm{P}_{jk}, \mathrm{P}_{jk}) = H^2(\mathrm{P}_{hik}, \mathrm{P}_{h\wideparen{-i \,~~ }k}) \\
&\Longrightarrow H(\mathrm{P}_{h-i\wideparen{ \,~~ j-}k}, \mathrm{P}_{h\wideparen{-i \,~~ j-}k}) \leq H(\mathrm{P}_{hik}, \mathrm{P}_{h\wideparen{-i \,~~ }k})
\end{align*}
(Again we have in fact equality.)

Thus, overall we have
\begin{align*}
H^2(\mathrm{P}_{hijk}, \mathrm{P}_{h\wideparen{-i \,~~ j-}k}) &\leq \paren{H(\mathrm{P}_{hijk}, \mathrm{P}_{h-i\wideparen{ \,~~ j-}k}) + H(\mathrm{P}_{h-i\wideparen{ \,~~ j-}k}, \mathrm{P}_{h\wideparen{-i \,~~ j-}k})}^2 \\
		&\leq \paren{H(\mathrm{P}_{ijk}, \mathrm{P}_{i\wideparen{ \,~~ j-}k}) + H(\mathrm{P}_{hik}, \mathrm{P}_{h\wideparen{-i \,~~ }k})}^2.
\end{align*}
\end{proof}

\subsection{Estimating Bernoulli Parameters}
\label{sec:bernoulli}

The following lemma deals with the precision in which we can approximate a Bernoulli parameter $p$, in the absolute value sense, using the empirical average. Note that our bound here depends on $p$.

\begin{lemma}
\label{concentration}
For any $0 < c \leq 1$, there exists $B_c > 0$, depending only on $c$, such that given any $0 \leq p \leq 1$, $\epsilon > 0$, and $\gamma < \frac{1}{2}$, if we let $m \geq B_c \cdot \frac{1}{\epsilon} \ln{\frac{1}{\gamma}}$, and $\widehat{p} = \frac{X_1 + \cdots + X_m}{m}$, where $X_1, \ldots, X_m$ are i.i.d Bernoulli random variables with parameter $p$, then we have $\abs{\widehat{p}-p} \leq c \cdot \max \brac{ \sqrt{p \epsilon}, \epsilon }$ with probability at least $1 - \gamma$.
\end{lemma}
\begin{remark}
Obviously by symmetry the precision bound in Lemma~\ref{concentration} can be improved to $\abs{\widehat{p}-p} \leq c \cdot \max \brac{ \sqrt{\min \brac{p,1-p} \cdot  \epsilon}, \epsilon }$. We stated the bound as it appears there because it is cleaner and sufficient for our purpose.
\end{remark}
\begin{proof}
Fix $0 < c \leq 1$. The lemma is obviously true when $p=0$. Suppose $p>0$. We argue by cases.

\underline{\textit{Case 1}} \textit{:} $p \geq c \epsilon$

We use the following multiplicative form of Chernoff Bound, which says that for $0 \leq \delta \leq 1$ we have
\begin{align*}
\mathrm{Pr}(\widehat{p} \leq (1-\delta)p) \leq e^{-\frac{\delta^2 mp}{3}}; \\
\mathrm{Pr} (\widehat{p} \geq (1+\delta)p) \leq e^{-\frac{\delta^2 mp}{3}}.
\end{align*}
Setting $\delta = c \cdot \sqrt{\frac{\epsilon}{p}} \leq 1$, we have
\begin{align*}
\mathrm{Pr}(\abs{\widehat{p}-p} \geq c \cdot \max \brac{ \sqrt{p \epsilon}, \epsilon }) & \leq \mathrm{Pr}(\abs{\widehat{p}-p} \geq c \cdot \sqrt{p \epsilon})\\
		& = \mathrm{Pr}(\abs{\widehat{p}-p} \geq \delta p)\\
                & \leq 2 e^{-\frac{\delta^2 mp}{3}} = 2 e^{-\frac{c^2}{3} \epsilon m} \leq 2 e^{-\frac{c^2}{3} \epsilon \cdot B_c \cdot \frac{1}{\epsilon} \ln{\frac{1}{\gamma}}} \leq 2 \gamma^2 < \gamma,
\end{align*}
for all big enough $B_c$ (depending only on $c$).

\underline{\textit{Case 2}} \textit{:} $0 < p < c \epsilon$

We use the following alternative multiplicative form of Chernoff Bound, which says that for $\delta \geq 1$ we have
$$ \mathrm{Pr}(\widehat{p} \geq (1+\delta)p) \leq e^{-\frac{\delta mp}{3}}. $$
Setting $\delta = c \cdot \frac{\epsilon}{p} \geq 1$, we have
\begin{align*}
\mathrm{Pr}(\abs{\widehat{p}-p} \geq c \cdot \max \brac{ \sqrt{p \epsilon}, \epsilon }) & = \mathrm{Pr}(\abs{\widehat{p}-p} \geq c \epsilon) \\
		& = \mathrm{Pr}(\abs{\widehat{p}-p} \geq \delta p) \\
                & \leq e^{-\frac{\delta mp}{3}} = e^{-\frac{c}{3} \epsilon m} \leq e^{-\frac{c}{3} \epsilon \cdot B_c \cdot \frac{1}{\epsilon} \ln{\frac{1}{\gamma}}}  \leq \gamma,
\end{align*}
for all big enough $B_c$ (depending only on $c$).
\end{proof}

The next lemma turns the absolute value bound in Lemma~\ref{concentration} into a bound on the corresponding term in the expression for squared Hellinger. The bound here is uniform in $p$. This is one of the reasons why for the majority of our analysis we prefer working with the Hellinger distance to the total variation distance.

\begin{lemma}
\label{bernoullihellinger}
For $p,\widehat{p},\epsilon \geq 0$, $0 < c \leq 1$, if $\abs{\widehat{p}-p} \leq c \cdot \max \brac{ \sqrt{p \epsilon}, \epsilon }$, then $\paren{ \sqrt{\widehat{p}} - \sqrt{p} }^2 \leq c \epsilon.$
\end{lemma}
\begin{proof}
The condition implies that $\abs{\widehat{p}-p} \leq c \cdot \max \brac{ \frac{\sqrt{p \epsilon}}{\sqrt{c}}, \epsilon }$. We argue by cases.

\underline{\textit{Case 1}} \textit{:} $p \geq c \epsilon$

In this case $\max \brac{ \frac{\sqrt{p \epsilon}}{\sqrt{c}}, \epsilon } = \frac{\sqrt{p \epsilon}}{\sqrt{c}}$, so we have $\widehat{p} \in \sqbrac{ p - \sqrt{c} \sqrt{p \epsilon}, p + \sqrt{c} \sqrt{p \epsilon} }$.

We have $\widehat{p} \leq p + \sqrt{c}\sqrt{p \epsilon} \leq p + 2 \sqrt{p}\sqrt{c \epsilon} + c \epsilon = \paren{ \sqrt{p} + \sqrt{c \epsilon} }^2$.

We also have $\widehat{p} \geq \widehat{p} - \sqrt{c}\sqrt{p \epsilon} \geq p - 2 \sqrt{c}\sqrt{p \epsilon} + c \epsilon = \paren{ \sqrt{p} - \sqrt{c \epsilon} }^2$.

Combining the two we see that $\paren{ \sqrt{\widehat{p}} - \sqrt{p} }^2 \leq c \epsilon.$

\underline{\textit{Case 2}} \textit{:} $p < c \epsilon$

In this case $\max \brac{ \frac{\sqrt{p \epsilon}}{\sqrt{c}}, \epsilon } = \epsilon$, so we have $\widehat{p} \in \sqbrac{ 0,p+c \epsilon }$.

Since $\paren{ 0 - \sqrt{p} }^2 = p \leq c \epsilon$, and $\paren{ \sqrt{p + c \epsilon} - \sqrt{p} }^2 \leq \sqrt{c \epsilon}^2 = c \epsilon$, we have $\paren{ \sqrt{\widehat{p}} - \sqrt{p} }^2 \leq c \epsilon.$
\end{proof}

\subsection{Additional Preparation for the General Case}
\label{sec:genonly}

This subsection provides some additional definitions, as well as a lemma, that are used only in the general case in Section~\ref{gen}.

We start with an operation that allows us to obtain an independent pairwise distribution from an arbitrary pairwise distribution in a way that respects the marginal distribution for each individual variable. We can use this operation to ``cut'' certain edges of the underlying tree of a tree-structured Bayesnet.

\begin{definition}[\textbf{ind}]
\label{ind}
Given an arbitrary joint distribution $\mathrm{P}$ on a set of random variables that includes $X$ and $Y$, we define $\mathrm{P}_{XY}^{\rm{(ind)}}$ to be the independent pairwise distribution for $X$ and $Y$ with the same individual marginals at $X$ and $Y$, respectively, as $\mathrm{P}_{XY}$ does.
\end{definition}

The following definitions for the KL-divergence, the mutual information, and the conditional mutual information are standard.

\begin{definition}[\textbf{KL-Divergence, Mutual Information, Conditional Mutual Information}]
\label{klidef}
For two discrete distributions $\mathrm{p}' = (p'_1, \ldots, p'_L)$ and $\mathrm{p}'' = (p''_1, \ldots, p''_L)$ over a domain of size $L$, the \textbf{KL-divergence} from $\mathrm{p}''$ to $\mathrm{p}'$ is defined as
$$KL(\mathrm{p}' \mid \mid \mathrm{p}'') = \sum_{l=1}^L \mathrm{p}'_l \ln{\frac{\mathrm{p}'_l}{\mathrm{p}''_l}}.$$
Given an arbitrary joint distribution $\mathrm{P}$ for a set of random variables that includes $X$, $Y$, and $Z$ (with alphabet sets $\mathcal{A}_X$, $\mathcal{A}_Y$, and $\mathcal{A}_Z$, respectively), the \textbf{mutual information} of $X$ and $Y$ with respect to $\mathrm{P}$ is defined as the KL-divergence from $\mathrm{P}_{XY}^{\rm{(ind)}}$ to $\mathrm{P}_{XY}$:
$$\mathrm{I}^{\mathrm{P}}(X;Y) = \sum_{x \in \mathcal{A}_X, y \in \mathcal{A}_Y} \mathrm{P}_{XY}(x,y) \ln{\frac{\mathrm{P}_{XY}(x,y)}{\mathrm{P}_{X}(x) \mathrm{P}_{Y}(y)}},$$
and the \textbf{conditional mutual information} of $X$ and $Y$ given $Z$ with respect to $\mathrm{P}$ is defined as the expected value of the mutual information of $X$ and $Y$ with respect to their conditional pairwise distribution conditioning on $Z = z$, with $z$ drawn according to $\mathrm{P}_Z$:
$$\mathrm{I}^{\mathrm{P}}(X;Y|Z) = \sum_{z \in \mathcal{A}_Z} \brac{\mathrm{P}_Z(z) \sum_{x \in \mathcal{A}_X, y \in \mathcal{A}_Y} \mathrm{P}_{XY|Z}(x,y|z) \ln{\frac{\mathrm{P}_{XY|Z}(x,y|z)}{\mathrm{P}_{X|Z}(x|z) \mathrm{P}_{Y|Z}(y|z)}} }.$$
\end{definition}

We will also need the following extension of the notation introduced in Definition~\ref{arcnotation}.

\begin{definition}
\label{genarcnotation}
Given an arbitrary joint distribution $\mathrm{P}$ for random variables $X_1, \ldots, X_n$, and $h,i,j,k \in [n]$, we adopt the following notations:
\begin{itemize}
\item[] $\mathrm{P}_{h-i-(jk)}$ - the Markov chain for $X_h,X_i,X_{jk}$, in that order, with the marginal for $(h,i)$ equal to $\mathrm{P}_{hi}$, and the marginal for $i,j,k$ equal to $\mathrm{P}_{ijk}$;
\item[] $\mathrm{P}_{(hi)\wideparen{ \,~~ j-}k}$ - the Markov chain for $X_{hi},X_k,X_j$, in that order, with the marginal for $h,i,k$ equal to $\mathrm{P}_{hik}$, and the marginal for $(j,k)$ equal to $\mathrm{P}_{jk}$,
\end{itemize}
and so on. At the subscript of each of the two symbols listed above, the indexes are regarded as being spelled out in the order of $h,i,j,k$. So we use, for example, $\mathrm{P}_{(hi)\wideparen{ \,~~ j-}k} (a,b,c,d)$ to denote the probability of the event $``X_h = a, X_i = b, X_j = c, X_k = d"$ under $\mathrm{P}_{(hi)\wideparen{ \,~~ j-}k}$.
\end{definition}
\begin{remark}
In Definition~\ref{genarcnotation} when we say for example $\mathrm{P}_{(hi)\wideparen{ \,~~ j-}k}$ is a Markov chain for $X_{hi},X_k,X_j$, in that order, we are treating the pair $(X_h,X_i)$ as one single random variable $X_{hi}$ so that the Markov structure is really on three (super) random variables. We don't order between $X_h$ and $X_i$. In other words, the only conditional independence property implied by the Markov structure is that $X_{hi}$ and $X_j$ are independent conditioning on $X_k$.
\end{remark}

Next we define four measures that are important in classifying the edges and reasoning about the layers for the general case.

\begin{definition}[\textbf{minmrg}]
\label{minmrg}
Given an arbitrary joint distribution $\mathrm{P}$ for binary random variables $X_1, \ldots, X_n$, and $U \subset [n]$, we define $\minmrg{\mathrm{P}_U} = \min \brac{\mathrm{P}_i (x_i) \mid i \in U, x_i = \pm 1}.$
\end{definition}
That is, $\minmrg{\mathrm{P}_U}$ denotes the minimum probability among those assigned to either $1$ or $-1$ by any of the marginals of $\mathrm{P}_U$ for individual variables. A small $\mathrm{P}_U$ implies that at least one of those variables is very biased towards either $1$ or $-1$.

\begin{definition}[\textbf{mindiag}]
\label{mindiag}
Given an arbitrary joint distribution $\mathrm{P}$ for binary random variables $X_1, \ldots, X_n$, and $i,j \subset [n]$, we define $\mindiag{\mathrm{P}_{ij}} = \min \brac{\mathrm{P}_{ij} (X_i = X_j), \mathrm{P}_{ij} (X_i = -X_j)}.$
\end{definition}
That is, $\mindiag{\mathrm{P}_{ij}}$ is the minimum amount of probability mass we need to ``move'' in order to make $X_i$ and $X_j$ always equal or always unequal. The smaller it is, the ``stronger'' the relation between $X_i$ and $X_j$.

\begin{definition}[\textbf{mindisc}]
\label{mindisc}
Given an arbitrary joint distribution $\mathrm{P}$ for binary random variables $X_1, \ldots, X_n$, and $i,j \subset [n]$, we define $\mindisc{\mathrm{P}_{ij}} = \min \brac{\abs{\mathrm{P}_{i|j} (1|1) - \mathrm{P}_{i|j}(1|-1)}, \abs{\mathrm{P}_{j|i} (1|1) - \mathrm{P}_{j|i}(1|-1)}}.$
\end{definition}
Note that $\mindisc{\mathrm{P}_{ij}} = 0$ is equivalent to $X_i$ and $X_j$ being independent, and $\mindisc{\mathrm{P}_{ij}} = 1$ is equivalent to $X_i$ and $X_j$ being always equal or always unequal. Therefore, $\mindisc{\mathrm{P}_{ij}}$ provides another way to measure the strength of interaction between $X_i$ and $X_j$.

\begin{definition}[$\boldsymbol{I_H, I_{H^2}}$]
\label{ih}
Given an arbitrary joint distribution $\mathrm{P}$ for random variables $X_1, \ldots, X_n$, and $i,j \subset [n]$, we define $I_H (\mathrm{P}_{ij}) = H (\mathrm{P}_{ij}, \mathrm{P}_{ij}^{\rm{(ind)}})$ and $I_{H^2} (\mathrm{P}_{ij}) = H^2 (\mathrm{P}_{ij}, \mathrm{P}_{ij}^{\rm{(ind)}})$.
\end{definition}

Note that $I_{H^2} (\mathrm{P}_{ij})$ measures how far from independent $\mathrm{P}_{ij}$ is in $H^2$, and that $I_{H^2} (\mathrm{P}_{ij}) = 0$ is equivalent to $X_i$ and $X_j$ being independent. Therefore, $\mindisc{\mathrm{P}_{ij}}$ provides yet another way to measure the strength of interaction between $X_i$ and $X_j$.

The final lemma of this section is a strengthening of Lemma~\ref{concentration} in the case $p$ is small. It will be used to bound certain mutual information-related terms for the general case in Section~\ref{gen}.

 \begin{lemma}
\label{smallprob}
There exists $B > 0$ such that given any $\epsilon > 0$, $\gamma < \frac{1}{2}$, and $0 < p < \epsilon$, if we let $m \geq B \cdot \frac{1}{\epsilon} \ln{\frac{1}{\gamma}}$, and $\widehat{p} = \frac{X_1 + \cdots + X_m}{m}$, where $X_1, \ldots, X_m$ are i.i.d Bernoulli random variables with parameter $p$, then we have $\widehat{p} \ln{\frac{\epsilon}{p}} \leq \epsilon$ with probability at least $1 - \gamma$.
\end{lemma}
\begin{proof}
By Lemma~\ref{concentration}, for all big enough $B$ we have $\abs{\widehat{p}-p} \leq \frac{1}{100} \cdot \max \brac{ \sqrt{p \epsilon}, \epsilon } = \frac{1}{100} \epsilon$. Assuming that, we argue by cases.

\underline{\textit{Case 1}} \textit{:} $\frac{1}{100} \epsilon < p < \epsilon$

Let $p = a \epsilon$, so that $\frac{1}{100} < a < 1$. We have
$$\widehat{p} \ln{\frac{\epsilon}{p}} \leq (a + \frac{1}{100}) \cdot \epsilon \cdot \ln{\frac{1}{a}} \leq 2a \epsilon \cdot \ln{\frac{1}{a}} = 2 \epsilon \cdot \frac{\ln{\frac{1}{a}}}{\frac{1}{a}} < \epsilon,$$
where the last inequality follows from the fact that $\frac{\ln{t}}{t} < \frac{1}{2}$ for $t > 1$.

\underline{\textit{Case 2}} \textit{:} $p \leq \frac{1}{100} \epsilon$

Again let $p = a \epsilon$, so that $a \leq \frac{1}{100}$. Consider the following multiplicative form of Chernoff Bound, which says that for $\delta > 0$ we have
$$\mathrm{Pr}(\widehat{p} > (1+\delta) p) < \paren{\frac{e^{\delta}}{\paren{1+\delta}^{1+\delta}}}^{mp}.$$
Let $\beta = \frac{\frac{1}{a}}{\ln{\frac{1}{a}}}$. Setting $\delta = \beta - 1$ in the Chernoff Bound, we have
\begingroup
\allowdisplaybreaks
\begin{align*}
\mathrm{Pr}(\widehat{p} > \beta p) &< \paren{\frac{e^{\beta - 1}}{\beta^{\beta}}}^{mp}\\
		&< \paren{\frac{e}{\beta}}^{\beta mp} \\
		&\stackrel{(\ast)}{<} \beta^{-\frac{1}{2} \beta mp} \\
		&= \exp{\brac{-\frac{1}{2} \beta \ln{\beta} \cdot B \cdot \frac{1}{\epsilon} \ln{\frac{1}{\gamma}} \cdot a \epsilon}} \\
		&= \exp{\brac{-\frac{1}{2} \cdot \frac{\frac{1}{a}}{\ln{\frac{1}{a}}} \cdot \ln{\frac{\frac{1}{a}}{\ln{\frac{1}{a}}}} \cdot B \cdot \frac{1}{\epsilon} \ln{\frac{1}{\gamma}} \cdot a \epsilon}} \\
		&\stackrel{(\ast \ast)}{\leq} \exp{\brac{-\frac{1}{2} \cdot \frac{\frac{1}{a}}{\ln{\frac{1}{a}}} \cdot \frac{1}{2} \ln{\frac{1}{a}} \cdot B \cdot \frac{1}{\epsilon} \ln{\frac{1}{\gamma}} \cdot a \epsilon}} \\
		&\leq \gamma,
\end{align*}
\endgroup
as long as $B \geq 4$. Note that $(\ast)$ holds because $a \leq \frac{1}{100}$ implies $\beta \geq e^2$, and $(\ast \ast)$ holds because $\frac{t^{\frac{1}{2}}}{\ln{t}} > 1$ for $t > 1$.

Thus, we have $\widehat{p} \leq \beta p$ with probability at least $1 - \gamma$. When that happens, we have
$$\widehat{p} \ln{\frac{\epsilon}{p}} \leq \beta p \cdot \ln{\frac{\epsilon}{p}} = \frac{\frac{1}{a}}{\ln{\frac{1}{a}}} \cdot a \epsilon \cdot \ln{\frac{1}{a}} = \epsilon.$$
\end{proof}

\newpage

\section{The Symmetric Case}
\label{sym}

As the analysis for general tree-structured Ising models is much more complex, here we first focus on learning symmetric models, wherein every node has a uniform marginal. Even in this case, proving learnability guarantees for the Chow-Liu algorithm is non-trivial, so we focus on a closely related, but slightly easier to analyze algorithm. The analysis of this variation of the Chow-Liu algorithm for proper learnings of binary symmetric tree-structured Bayesnets will nevertheless illustrate all the major ideas present in the analysis of the general case. In Section~\ref{gen}, we provide the necessary modifications to our argument to analyze Chow-Liu for general tree-structured Ising models. 

Throughout Section~\ref{sym}, we reserve the symbol $\mathrm{P}$ to represent the unknown binary symmetric tree-structured Bayesnet for $X_1, \ldots, X_n$ that we wish to properly learn. The dimension of $\mathrm{P}$ is $n$. We reserve $T$ to represent the true (but unknown) underlying tree of $\mathrm{P}$, and $\alpha_{ij}$ to represent the true (but unknown) $\alpha$-value of the pair $(i,j)$ under $\mathrm{P}$ (see Section~\ref{sec:symmetric} for a discussion of $\alpha$-values). Our proper learning algorithm for the symmetric case is presented as Algorithm~\ref{algo:symalgo} below. See Section~\ref{sec:symmetric} for a discussion on specifying a binary symmetric tree-structured Bayesnet.  

\begin{algorithm}
\label{algo:symalgo}
\DontPrintSemicolon
\caption{{\sc Proper Learning of a Binary Symmetric Tree-Structured Bayesnet}}
\KwIn{$\epsilon,\gamma \in (0, 1]$}
\KwOut{specification of a binary symmetric tree-structured Bayesnet $\mathrm{Q}$, approximating $\mathrm{P}$}
Draw $m = \biggl\lceil B \cdot \frac{n}{\epsilon^2} \cdot \paren{ \ln n + \ln \frac{1}{\gamma} } \biggr\rceil$ i.i.d samples $x^{(1)}, \ldots, x^{(m)}$ from $\mathrm{P}$. \;
Let $\widehat{\mathrm{P}}$ be the empirical joint distribution of $(X_1, \ldots, X_n)$ induced by the $m$ samples, so that $\widehat{\mathrm{P}}(x) = \frac{1}{m} \sum_{t=1}^m \mathbbm{1}_{x^{(t)} = x}$, $\forall x \in {\{1, -1\}}^n$.\;
Compute $\widehat{\alpha}_{ij} = 2\widehat{\mathrm{P}}(X_i = X_j) - 1$, $\forall (i,j)$. \;
Run Kruskal's algorithm to find a maximum weight spanning tree $\widehat{T}$ of the complete graph on $[n]$ with weight $\abs{\widehat{\alpha}_{ij}}$ for $(i,j)$.\;\label{line:mst}
\Return{(1) $\widehat{T}$ as $\mathrm{Q}$'s underlying tree, (2) $\widehat{\alpha}_{ij}$ as $\mathrm{Q}$'s $\alpha$-value for $(i,j)$, for each edge $(i,j)$ of $\widehat{T}$. }\;
\end{algorithm}
\noindent Throughout the section, we will also reserve the symbols $\widehat{\mathrm{P}}$, $\widehat{\alpha}_{ij}$, $\widehat{T}$, and $\mathrm{Q}$ for the respective meanings they carry in Algorithm~\ref{algo:symalgo}. We call $\widehat{\alpha}_{ij}$ the $\pmb{\widehat{\alpha}}$\textbf{-estimate} of the pair $(i,j)$, and $\abs{\widehat{\alpha}_{ij}}$ the $\abs{\pmb{\widehat{\alpha}}}$\textbf{-estimate} of the pair $(i,j)$. Our main result for Algorithm~\ref{algo:symalgo} is the following.

\begin{theorem}
\label{theoremsym}
There exists $B > 0$ such that for any $ \epsilon, \gamma \in (0,1]$, Algorithm~\ref{algo:symalgo} outputs a binary symmetric tree-structured Bayesnet $\mathrm{Q}$ such that $\dtv{\mathrm{P}}{\mathrm{Q}} \leq \epsilon$, with error probability at most $\gamma$.
\end{theorem}

In fact, we establish something stronger and perhaps surprising. We show that Algorithm~\ref{algo:symalgo} is successful as long as the empirical distribution $\widehat{\mathrm{P}}$ satisfies a {\em deterministic criterion}, presented below. 

\begin{definition}[\textbf{3-Consistency}]
\label{3consistency} 
For a given run of Algorithm~\ref{algo:symalgo}, $\widehat{\mathrm{P}}$ is said to satisfy \textbf{3-consistency} if for every $S \subseteq [n]$ of cardinality $3$ and every $W \subseteq {\{1, -1\}}^3$ we have
\begin{equation}
\label{eq:consistency}
\abs{\widehat{\mathrm{P}}(X_S \in W) - \mathrm{P}(X_S \in W)} \leq \frac{1}{10} \cdot \max \brac{ \sqrt{\mathrm{P}(X_S \in W) \cdot \frac{\epsilon^2}{n}}, \frac{\epsilon^2}{n} }.
\end{equation}
\end{definition}
\noindent The criterion of Definition~\ref{3consistency} specifies that $\widehat{\mathrm{P}}$ well-approximates the true distribution ${\mathrm{P}}$ for all events involving up to $3$ variables. Some concrete examples of the event ``$X_S \in W$'' in Definition~\ref{3consistency} are ``$X_i = -X_j, X_j = X_k$'', ``$X_i =  X_j$'', and ``$X_i = 1$''. For example, ``$X_i =  X_j$'' is expressed by taking $S = \{i,j,k\}$ (where $k$ is any index other than $i,j$) and $W = \{(1,1,1), (1,1,-1), (-1,-1,1), (-1,-1,-1)\}$. 

\smallskip We show that whenever the samples provided to Algorithm~\ref{algo:symalgo} induce $\widehat{\mathrm{P}}$ which satisfies 3-consistency, the algorithm succeeds (up to a constant), as per the following theorem.

\begin{theorem}
\label{sufficiency}
For a run of Algorithm~\ref{algo:symalgo}, if $\widehat{\mathrm{P}}$ satisfies 3-consistency, then $\dtv{\mathrm{P}}{\mathrm{Q}} \leq O \paren{\epsilon}$.
\end{theorem}

Notice that there are less than $n^3$ choices for $S$, and $2^8$ choices for $W$ for each $S$. By Lemma~\ref{concentration} and the union bound, 3-consistency fails with probability at most $\gamma$ if $B$ is chosen large enough. Consequently, showing Theorem~\ref{sufficiency} implies Theorem~\ref{theoremsym} (up to a constant). We will focus on proving Theorem~\ref{sufficiency}.

\subsection{Outline of the Proof}
\label{sec:outline of symmetric case}

In this section, we outline our analysis, presenting the intuition guiding our proof as well as the critical techniques that we develop to complete it. All lemmas presented in this section are restated copies of lemmas with the same label stated and proved in Section~\ref{sec:symproof}.

\paragraph{Narrowing into a deterministic condition for success.} A few words first on the 3-consistency condition (see Definition~\ref{3consistency}). We'd ideally like to have a simple deterministic condition describing the samples that is satisfied with probability at least $1-\gamma$, and that yields $\dtv{\mathrm{P}}{\mathrm{Q}} \leq O \paren{\epsilon}$ if satisfied. Such a deterministic condition serves as an intermediate step between the i.i.d samples and the desired bound in total variation distance between $\mathrm{P}$ and $\mathrm{Q}$. What can we hope for with $m = \biggl\lceil B \cdot \frac{n}{\epsilon^2} \cdot \paren{ \ln n + \ln \frac{1}{\gamma} } \biggr\rceil$ samples? A simple thing we can hope for is that for each $X_i$ the probabilities that it equals $1$ or $-1$ are similar under $\mathrm{P}$ and  $\widehat{\mathrm{P}}$. Specifically, by standard Chernoff bounds (see Lemma~\ref{concentration}) and the union bound, for any $c>0$ we can choose $B$ big enough so that $\abs{\widehat{\mathrm{P}}(X_i = x_i) - \mathrm{P}(X_i = x_i)} \leq c \cdot \max \brac{ \sqrt{\mathrm{P}(X_i = x_i) \cdot \frac{\epsilon^2}{n}}, \frac{\epsilon^2}{n} }$ for all $i$ and $x_i \in \{1,-1\}$, with probability at least $1-\gamma$. In fact, for any fixed constant $K$, by choosing a big enough $B$ we can get similar guarantees simultaneously for all events of the form ``$X_S \in W$'' for all subsets $S$ of $K$ indexes and all $W \subseteq \brac{1,-1}^K$, with probability at least $1-\gamma$. Our surprising finding is that taking $K=3$ (and $c = \frac{1}{10}$) is sufficient to guarantee the success of the algorithm.

From now on we assume that $\widehat{\mathrm{P}}$ satisfies 3-consistency. It suffices to show that this alone implies $\dtv{\mathrm{P}}{\mathrm{Q}} \leq O \paren{\epsilon}$, for the output $Q$ of the algorithm (that is, it suffices to prove Theorem~\ref{sufficiency}). It is worth emphasizing that our argument from this point on is completely deterministic. In fact, if $\widehat{\mathrm{P}}$ were obtained in some way other than being the empirical joint distribution of i.i.d samples (for example, from dependent samples or even an oracle), as long as it satisfies 3-consistency Algorithm~\ref{algo:symalgo} would still output $\mathrm{Q}$ such that $\dtv{\mathrm{P}}{\mathrm{Q}} \leq O \paren{\epsilon}$. 

\paragraph{Long hybrid arguments vs squared Hellinger subadditivity.} We first illustrate how  $H$ (Hellinger distance) and the subadditivity of $H^2$ (squared Hellinger) play an important role in the argument even though our final goal is to provide a bound in $d_{\mathrm{TV}}$ (total variation distance). Suppose, for the sake of this illustration, that we happened to learn the structure of the tree correctly, that is $\widehat{T} = T$. In this case, $\mathrm{P}$ and $\mathrm{Q}$ would only differ in their $\alpha$-values on the edges of the true $T$ ($\alpha_{ij}$ for $\mathrm{P}$ and $\widehat{\alpha}_{ij}$ for $\mathrm{Q}$). How could we bound $\dtv{\mathrm{P}}{\mathrm{Q}}$? One way is to use a hybrid argument modifying $\mathrm{P}$ to $\mathrm{Q}$ by changing the $\alpha$-value of one edge of $T$ at a time from $\alpha_{ij}$ to $\widehat{\alpha}_{ij}$, resulting in a sequence of intermediate distributions all with the same underlying tree $T$ and with increasingly more edges of $T$ having $\alpha$-value $\widehat{\alpha}_{ij}$ instead of $\alpha_{ij}$. It easily follows from 3-consistency that the difference between $\alpha_{ij}$ and $\widehat{\alpha}_{ij}$ is $O {\paren{ \max{ \brac{ \sqrt{\paren{1-\abs{\alpha_{ij}}} \cdot \frac{\epsilon^2}{n}}, \frac{\epsilon^2}{n}} } }}$, and this bounds the cost in $d_{\mathrm{TV}}$ when we change the $\alpha$-value of an edge from $\alpha_{ij}$ to $\widehat{\alpha}_{ij}$. In general, this cost is upper bounded by $O{\paren{\frac{\epsilon}{\sqrt{n}}}}$, so for $n-1$ changes of $\alpha$-value the triangle inequality gives an overall $d_{\mathrm{TV}}$ bound between $P$ and $Q$ of $O{\paren{\sqrt{n} \epsilon}}$, which is not good enough.

$H$ suffers from the same problem. It can be shown that the cost in $H$ when we change the $\alpha$-value of an edge from $\alpha_{ij}$ to $\widehat{\alpha}_{ij}$ is $O{\paren{\frac{\epsilon}{\sqrt{n}}}}$ (see Lemma~\ref{helfromestimate} in Section~\ref{sec:symlemmas}), which is also not enough to make the hybrid argument work. But, here is an important insight. The $H^2$ subadditivity property (see Corollary~\ref{edgeswitch}) applied to our case implies that $H^2(\mathrm{P},\mathrm{Q})$ is bounded by the sum of $H^2(\mathrm{P}_{ij},\mathrm{Q}_{ij})$ for all edges $(i,j)$ of $T$. Note that $H^2$ is not a distance metric, so the subadditivity is not the result of some hybrid argument changing the $\alpha$-values one edge at a time. Rather, we are changing the $\alpha$-values for all $n-1$ edges of $T$ in one round. Since we pay $O{\paren{\frac{\epsilon^2}{n}}}$ in $H^2$ per edge, in total we pay $O{\paren{\epsilon^2}}$ in squared Hellinger, yielding $\dtv{\mathrm{P}}{\mathrm{Q}} \leq \sqrt{2} H(\mathrm{P},\mathrm{Q}) \leq O \paren{\epsilon}$, as desired.

\paragraph{Causes of structural uncertainty.} Of course, we can't expect $\widehat{T} = T$. It is possible that $\widehat{T}$ may be very different from $T$. If a subset of nodes are linked together in $T$ by edges with $\alpha$-values $1$, then those nodes will always equal each other in the sample, and there is no way Algorithm~\ref{algo:symalgo} can tell what the true edges  are. Algorithm~\ref{algo:symalgo} will be similarly clueless if the linking edges all have $\alpha$-values $1$ or $-1$, or if they all have $\alpha$-values $0$ (in which case the variables will be mutually independent). Intuitively, we should expect Algorithm~\ref{algo:symalgo} to also have difficulty picking out the true edges if the $\alpha$-values are close enough to $1$ or $-1$, or close enough to $0$. When different strength scales interact the algorithm may get a ``partial clue'' of the structure. For example, suppose $n=3$ and $X_1,X_2,X_3$ form a Markov chain, in that order, with $\alpha_{12} = \frac{1}{2}$ and $\alpha_{23}$ close to $1$. Then Algorithm~\ref{algo:symalgo} should be able to pick out edge $(2,3)$, but might have trouble distinguishing between $(1,2)$ and $(1,3)$ (in other words, $\abs{\widehat{\alpha}_{23}}$ should be the clear winner but the order between $\abs{\widehat{\alpha}_{12}}$ and $\abs{\widehat{\alpha}_{13}}$ might be reversed). The case $\alpha_{12}$ close to $0$ and $\alpha_{23} = \frac{1}{2}$ is similar. 

As $n$ gets larger the picture gets complicated very quickly, and it seems very hard to reason about the effect of interactions among different edges and the order relationships among different $\abs{\widehat{\alpha}_{ij}}$'s. The central difficulty in our argument is controlling the error when $\widehat{T}$ and $T$ differ in many edges. The main difficulty is to bound the error when switching the edges from those of $T$ to those of $\widehat{T}$. Once that is done, further switching the $\alpha$-values of the edges of $\widehat{T}$ from $\alpha_{ij}$ to $\widehat{\alpha}_{ij}$ (with the underlying tree fixed to be $\widehat{T}$) can be dealt with relatively easily as illustrated above. More specifically, if we define the {\em hybrid distribution} $\mathrm{P}^{(4)}$ (the reason for calling it $\mathrm{P}^{(4)}$ will be clear later), with underlying tree $\widehat{T}$ and $\alpha$-value equal to $\alpha_{ij}$ for every edge $(i,j)$ in $\widehat{T}$, it follows from our earlier $H^2$-subadditivity argument that $\dtv{\mathrm{P}^{(4)}}{\mathrm{Q}} \le O \paren{\epsilon}$. So our focus from now on is bounding~$\dtv{\mathrm{P}}{\mathrm{P}^{(4)}}$.

\paragraph{Some modest understanding.} Let's start simple and consider only three nodes $i,j,k$ that lie on a path in $T$ ($n$ can be arbitrary, and $i,j,k$ don't have to be adjacent to each other). The multiplicativity of $\alpha$-values along a path (Lemma~\ref{multiplicativity}) implies that $\alpha_{ik} = \alpha_{ij} \alpha_{jk}$ is no greater in absolute value than either $\alpha_{ij}$ or $\alpha_{jk}$. Algorithm~\ref{algo:symalgo} may ``reverse the order'' between $(i,j)$ and $(i,k)$ if $\abs{\widehat{\alpha}_{ij}} \leq \abs{\widehat{\alpha}_{ik}}$. The following lemma characterizes when this can possibly happen (assuming 3-consistency):

\newtheorem*{lem:wrongedgecondition}{Lemma \ref{wrongedgecondition}}
\begin{lem:wrongedgecondition}
If $i,j,k$ lie on a path in $T$, and $\abs{\widehat{\alpha}_{ij}} \leq \abs{\widehat{\alpha}_{ik}}$, then $\alpha_{ij}^2 (1-\abs{\alpha_{jk}}) \leq \frac{\epsilon^2}{n}$.
\end{lem:wrongedgecondition}

The next lemma bounds the $H^2$ distance between $\mathrm{P}_{ijk}$ and $\mathrm{P}_{i\wideparen{ \,~~ j-}k}$ (see Definition~\ref{arcnotation} for the definitions of symbols such as $\mathrm{P}_{i\wideparen{ \,~~ j-}k}$). It is the error in $H^2$ when we switch from $(i,j)$ to $(i,k)$. See Figure~\ref{3nodesfigure} for depictions of those two distributions.

\begin{figure}[h]
\centering
\includegraphics[width=15cm]{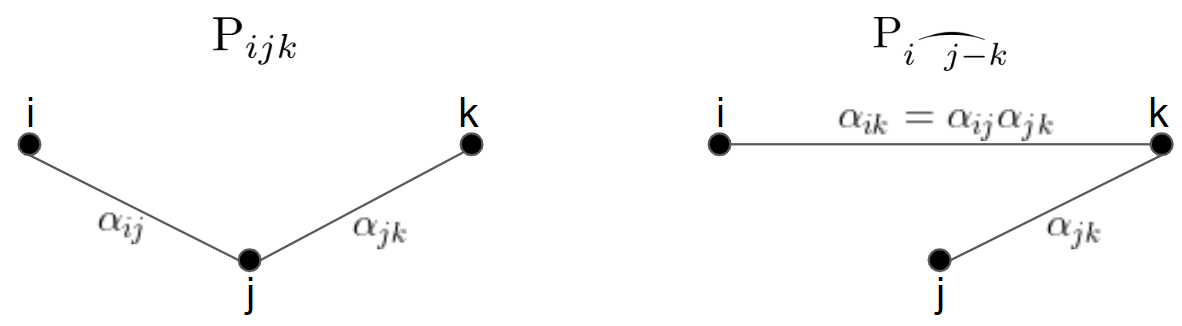}
\caption{depictions of $\mathrm{P}_{ijk}$ and $\mathrm{P}_{i\wideparen{ \,~~ j-}k}$ in terms of their underlying trees and their $\alpha$-values for the edges of the respective underlying trees}
\label{3nodesfigure}
\end{figure}

\newtheorem*{lem:heltwochains}{Lemma \ref{heltwochains}}
\begin{lem:heltwochains}
If $i,j,k$ lie on a path in $T$, then $H^2(\mathrm{P}_{ijk}, \mathrm{P}_{i\wideparen{ \,~~ j-}k}) \leq 2 \alpha_{ij}^2 (1-\abs{\alpha_{jk}} )$.
\end{lem:heltwochains}

The next lemma follows directly from Lemmas~\ref{wrongedgecondition} and~\ref{heltwochains}. It says that if Algorithm~\ref{algo:symalgo} makes a structural mistake and ``reverses the order'' between $(i,j)$ and $(i,k)$, then the resulting error in $H^2$ between the true structure and the wrong structure must be small (assuming 3-consistency).

\newtheorem*{lem:helwrongedge3}{Lemma \ref{helwrongedge3}}
\begin{lem:helwrongedge3}
If $i,j,k$ lie on a path in $T$, and $\abs{\widehat{\alpha}_{ij}} \leq \abs{\widehat{\alpha}_{ik}}$, then $H^2(\mathrm{P}_{ijk}, \mathrm{P}_{i\wideparen{ \,~~ j-}k}) \leq 2 \frac{\epsilon^2}{n}$.
\end{lem:helwrongedge3}

We also have the following four-node version of Lemma~\ref{helwrongedge3} (see Figure~\ref{4nodesfigure} for depictions of the two distributions involved):

\begin{figure}[h]
\centering
\includegraphics[width=15cm]{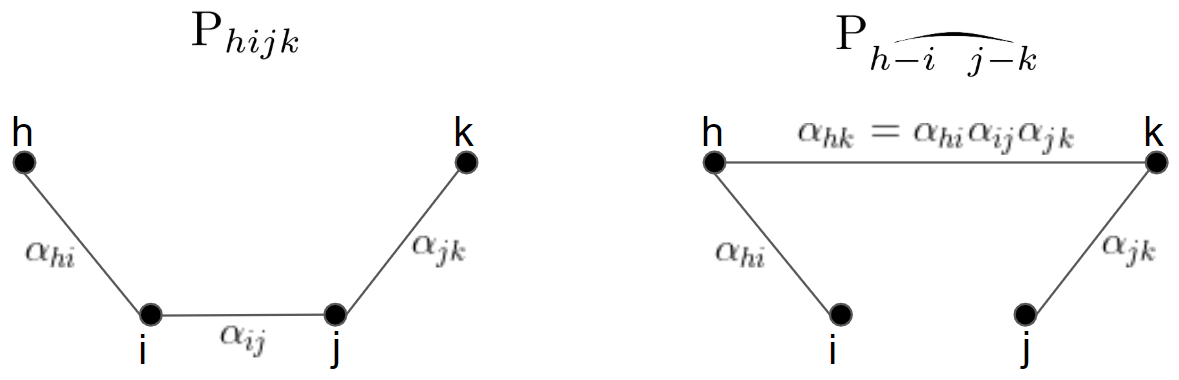}
\caption{depictions of $\mathrm{P}_{hijk}$ and $\mathrm{P}_{h\wideparen{-i \,~~ j-}k}$ in terms of their underlying trees and their $\alpha$-values for the edges of the respective underlying trees}
\label{4nodesfigure}
\end{figure}

\newtheorem*{lem:helwrongedge4}{Lemma \ref{helwrongedge4}}
\begin{lem:helwrongedge4}
If $h,i,j,k$ lie on a path in $T$, and $\abs{\widehat{\alpha}_{ij}} \leq \abs{\widehat{\alpha}_{hk}}$, then $H^2(\mathrm{P}_{hijk}, \mathrm{P}_{h\wideparen{-i \,~~ j-}k}) \leq 8 \frac{\epsilon^2}{n}$.
\end{lem:helwrongedge4}

\paragraph{The challenges of going bigger.} It is not hard to see using Lemmas~\ref{helwrongedge3} and~\ref{helwrongedge4} that if Algorithm~\ref{algo:symalgo} picks a single wrong edge, then $H(\mathrm{P},\mathrm{P}^{(4)}) \le O{\paren{\frac{\epsilon}{\sqrt{n}}}}$. So what if the algorithm picks multiple wrong edges? An obvious idea is to use the triangle inequality (satisfied by $H$) and create many intermediate distributions between $\mathrm{P}$ and $\mathrm{P}^{(4)}$, going from $T$ to $\widehat{T}$ replacing one edge at a time (possibly through the use of some intermediate edges that belong to neither $T$ nor $\widehat{T}$), assigning to each edge $(i,j)$ of the underlying tree of each intermediate distribution $\alpha$-value $\alpha_{ij}$, and using Lemmas~\ref{helwrongedge3} and~\ref{helwrongedge4} to bound the error in $H$ after each edge replacement. Unfortunately, this does not work, not only because the triangle inequality would lead to a bound that is potentially $\sqrt{n}$ times too big, but also because there are examples of $T$ and $\widehat{T}$ for which there does not even exist a sequence of replacements that would allow us to properly apply Lemmas~\ref{helwrongedge3} or~\ref{helwrongedge4} to bound the error in $H$ after every replacement.\footnote{To get a sense of the potential issue, say for one of the steps we replace $(i,j)$ with $(h,k)$. Say $h$ is closer to $i$ than to $j$ with respect to the intermediate tree right before the replacement. Then, in order for us to be able to apply Lemma~\ref{helwrongedge4}, we need for example that the $\alpha$-value of $(h,i)$ for the intermediate distribution right before the replacement to be equal to $\alpha_{hi}$, which is the $\alpha$-value of $(h,i)$ for $\mathrm{P}$. This may not hold due to all the edge replacements that occurred up to this step, which may have created extra nodes on the path in the intermediate tree between $h$ and $i$ that are not on the path in $T$ between $h$ and $i$.} See Figure~\ref{counterexample} for such a pair of $T$ and $\widehat{T}$.

\begin{figure}[h]
\centering
\includegraphics[width=15cm]{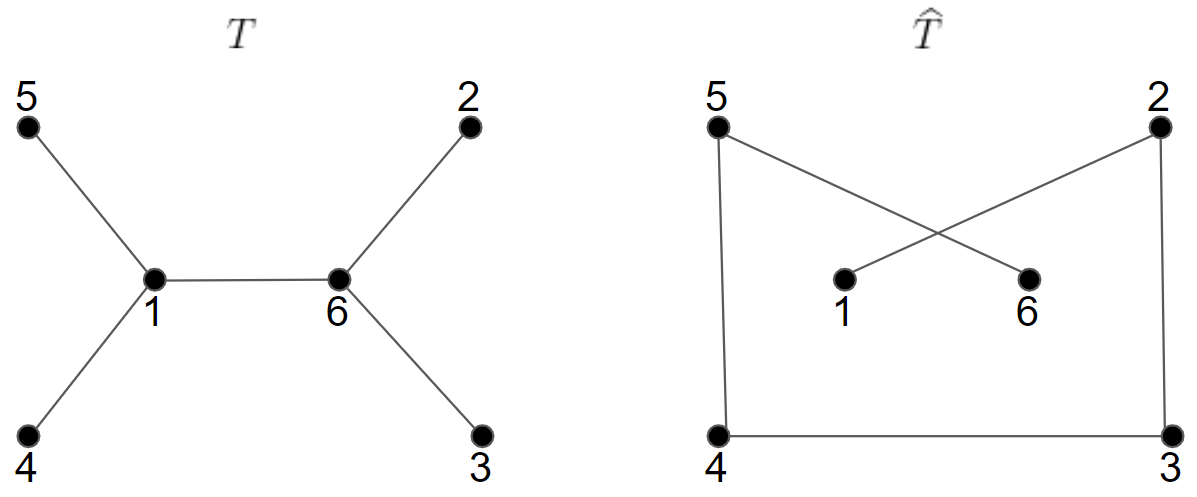}
\caption{A pair of $T$ and $\widehat{T}$ for which there is no way to go from $T$ to $\widehat{T}$ by replacing one edge at a time in a way that allows us to properly apply Lemmas~\ref{helwrongedge3} and~\ref{helwrongedge4} to bound the error in $H$ after every replacement}
\label{counterexample}
\end{figure}

To salvage the situation we choose to use Corollary~\ref{edgeswitch} (squared Hellinger subadditivity) which allows us to replace many edges in one round, as long as they are ``compatible,'' as per the statement of that corollary. If we allow $r$ general rounds of edge replacements, then we may need to replace $\Theta{\paren{\frac{n}{r}}}$ edges per round. Assume that either Lemma~\ref{helwrongedge3} or~\ref{helwrongedge4} is applicable for each individual edge replacement, then we can get an error bound of $O{\paren{\sqrt{\frac{n}{r} \cdot \frac{\epsilon^2}{n}}}} = O{\paren{\frac{\epsilon}{\sqrt{r}}}}$ in $H$ per round, and so $O{\paren{\sqrt{r} \epsilon}}$ in $H$ in total. This suggests aiming for a constant number of rounds. It is not clear at all how this might be achieved. For that, we introduce another crucial structural ingredient---{\em layering}---which allows us to be more precise about the manner in which $\widehat{T}$ may differ from $T$.
 
\paragraph{Building a hierarchy.} We classify the edges of $\widehat{T}$ based on their $\abs{\widehat{\alpha}}$-estimates. This induces a hierarchical classification of nodes into groups, which  in turn induces a classification of the edges of $T$.\footnote{\label{footnote:classification choice} We could have let $T$ dictate the layering, starting by classifying the edges of $T$ instead. The proof would be somewhat more complicated if we went this way. We could also have chosen to base the classification on $\abs{\alpha}$-values rather than $\abs{\widehat{\alpha}}$-estimates. We chose $\abs{\widehat{\alpha}}$-estimates because those are known after running Algorithm~\ref{algo:symalgo} (just like $\widehat{T}$ is). This will enable us to actually tell which edge of $\widehat{T}$ is classified as what, and visualize the partitioning of nodes into groups and therefore various reconstruction guarantees obtained by the algorithm.} We emphasize that this entire layering is merely a thought experiment to facilitate the proof. 

At the top level of our layering, we classify an edge $(i,j)$ of $\widehat{T}$ as a $\pmb{\widehat{T}}$\textbf{-road} if $\abs{\widehat{\alpha}_{ij}} \geq 1 - 10 \frac{\epsilon^2}{n}$. The $\widehat{T}$-roads cluster the $n$ nodes into connected components, each of which is called a \textbf{city}. An edge $(i,j)$ of $T$ with $i,j$ belonging to the same city is classified as a $\pmb{T}$\textbf{-road}. We can prove the following structural results.

\newtheorem*{lem:roads}{Lemma \ref{roads}}
\begin{lem:roads}
If $(i,j)$ is a $T$-road or a $\widehat{T}$-road, then $\abs{\alpha_{ij}} \geq 1 - 11 \frac{\epsilon^2}{n}$.
\end{lem:roads}

\newtheorem*{lem:difcities}{Lemma \ref{difcities}}
\begin{lem:difcities}
If $i,j$ belong to different cities, then $\abs{\alpha_{ij}} < 1 - 9 \frac{\epsilon^2}{n}$.
\end{lem:difcities}

\newtheorem*{lem:cityconnected}{Lemma \ref{cityconnected}}
\begin{lem:cityconnected}
Every city is $T$-connected.
\end{lem:cityconnected}

Lemma~\ref{cityconnected} says that each city is not only spanned by the $\widehat{T}$-roads in it (this is by definition), but also spanned by the $T$-roads in it. Intuitively, the $T$-roads and $\widehat{T}$-roads are so strong that within each city it doesn't make much difference (in terms of $H^2$) if the nodes were connected by $T$-roads or by $\widehat{T}$-roads. Later, we will replace the set of all $T$-roads with the set of all $\widehat{T}$-roads in a single round of edge replacement in the process of going from $T$ to $\widehat{T}$.

For the next level of our layering, we classify an edge $(i,j)$ of $\widehat{T}$ as a $\pmb{\widehat{T}}$\textbf{-highway} if $\frac{1}{2} \leq \abs{\widehat{\alpha}_{ij}} < 1 - 10 \frac{\epsilon^2}{n}$. The $\widehat{T}$-highways cluster the cities into bigger connected components, each of which is called a \textbf{country}. An edge $(i,j)$ of $T$ with $i,j$ belonging to  different cities in the same country is classified as a $\pmb{T}$\textbf{-highway}. We can prove the following structural results.

\newtheorem*{lem:highwayparallel}{Lemma \ref{highwayparallel}}
\begin{lem:highwayparallel}
There is a $\widehat{T}$-highway between a pair of cities if and only if there is a $T$-highway between the same pair of cities.
\end{lem:highwayparallel}
\begin{proof}
(sketch of the ``only if'' direction). Suppose for the sake of contradiction that $(i,k)$ is a $\widehat{T}$-highway between cities $C$ and $D$, and that there is no edge of $T$ between $C$ and $D$ (any such edge would in fact be classified as a $T$-highway because $C$ and $D$ are clearly different cities in the same country). Then, the path in $T$ between $i$ and $k$ must contain some $j$ that belongs to neither $C$ nor $D$. Since $\abs{\widehat{\alpha}_{ik}} \geq \frac{1}{2}$, by 3-consistency $\abs{\alpha_{ik}} \geq \frac{1}{2} - \frac{1}{5} \frac{\epsilon}{\sqrt{n}}$, and by the multiplicativity of $\alpha$-values $\abs{\alpha_{ij}} \geq \frac{1}{2} - \frac{1}{5} \frac{\epsilon}{\sqrt{n}}$. Since $j,k$ belong to different cities, by Lemma~\ref{difcities} $\abs{\alpha_{jk}} \leq 1 - 9 \frac{\epsilon^2}{n}$. Since $\paren{\frac{1}{2} - \frac{1}{5} \frac{\epsilon}{\sqrt{n}}}^2 \cdot 9 \frac{\epsilon^2}{n} > \frac{\epsilon^2}{n}$, it follows from (the contrapositive of) Lemma~\ref{helwrongedge3} that $\abs{\widehat{\alpha}_{ij}} > \abs{\widehat{\alpha}_{ik}}$. Symmetrically we also have $\abs{\widehat{\alpha}_{jk}} > \abs{\widehat{\alpha}_{ik}}$. Together they imply that $(i,k)$ could not possibly be picked to be an edge of $\widehat{T}$ by the maximum spanning tree algorithm.
\end{proof}

\newtheorem*{lem:countryconnected}{Lemma \ref{countryconnected}}
\begin{lem:countryconnected}
Every country is $T$-connected.
\end{lem:countryconnected}

\newtheorem*{lem:difcountries}{Lemma \ref{difcountries}}
\begin{lem:difcountries}
If $i,j$ belong to different countries, then $\abs{\alpha_{ij}} < \frac{1}{2} + \frac{1}{5} \frac{\epsilon}{\sqrt{n}}$.
\end{lem:difcountries}

The lower threshold $\frac{1}{2}$ in the criterion for $\widehat{T}$-highways has been chosen in view of Lemma~\ref{wrongedgecondition} and Lemma~\ref{difcities} so that Lemma~\ref{highwayparallel} will work. We want to roughly aim for
$$\paren{\textit{lower threshold for }\widehat{T}\textit{-highways}}^2 \cdot \paren{1 - \textit{lower threshold for }\widehat{T}\textit{-roads}} > \frac{\epsilon^2}{n}.$$
See the inequality $\paren{\frac{1}{2} - \frac{1}{5} \frac{\epsilon}{\sqrt{n}}}^2 \cdot 9 \frac{\epsilon^2}{n} > \frac{\epsilon^2}{n}$ in the proof sketch above for a precise version that takes estimation errors into account. Lemma~\ref{highwayparallel} says that the $T$-highways and $\widehat{T}$-highways can be matched into parallel pairs, where each parallel pair consists of a $T$-highway and a $\widehat{T}$-highway going between the same pair of cities (in the same country). Later, we will replace the set of all $T$-highways with the set of all $\widehat{T}$-highways in a single round of edge replacement in the process of going from $T$ to $\widehat{T}$.

For the next level of our layering, we classify an edge $(i,j)$ of $\widehat{T}$ as a $\pmb{\widehat{T}}$\textbf{-railway} if $2 \frac{\epsilon}{\sqrt{n}} \leq \abs{\widehat{\alpha}_{ij}} < \frac{1}{2}$. The $\widehat{T}$-railways cluster the countries into bigger connected components, each of which is called a \textbf{continent}. An edge $(i,j)$ of $T$ with $i,j$ belonging to different countries in the same continent is classified as a $\pmb{T}$\textbf{-railway}. We can prove the following structural results.

\newtheorem*{lem:railwayparallel}{Lemma \ref{railwayparallel}}
\begin{lem:railwayparallel}
There is a $\widehat{T}$-railway between a pair of countries if and only if there is a $T$-railway between the same pair of countries.
\end{lem:railwayparallel}

\newtheorem*{lem:continentconnected}{Lemma \ref{continentconnected}}
\begin{lem:continentconnected}
Every continent is $T$-connected.
\end{lem:continentconnected}

\newtheorem*{lem:difcontinents}{Lemma \ref{difcontinents}}
\begin{lem:difcontinents}
If $i,j$ belong to different continents, then $\abs{\alpha_{ij}} < \frac{11}{5} \frac{\epsilon}{\sqrt{n}}$.
\end{lem:difcontinents}

The proof of Lemma~\ref{railwayparallel} is similar to that of Lemma~\ref{highwayparallel}. The lower threshold $2 \frac{\epsilon}{\sqrt{n}}$ in the criterion for $\widehat{T}$-railways has been chosen in view of Lemma~\ref{wrongedgecondition} and Lemma~\ref{difcountries} so that Lemma~\ref{railwayparallel} will work. We want to roughly aim for
$$\paren{\textit{lower threshold for }\widehat{T}\textit{-railways}}^2 \cdot \paren{1 - \textit{lower threshold for }\widehat{T}\textit{-highways}} > \frac{\epsilon^2}{n}.$$
Lemma~\ref{railwayparallel} says that the $T$-railways and $\widehat{T}$-railways can be matched into parallel pairs, where each parallel pair consists of a $T$-railway and a $\widehat{T}$-railway going between the same pair of countries (in the same continent). Later, we will replace the set of all $T$-railways with the set of all $\widehat{T}$-railways in a single round of edge replacement in the process of going from $T$ to $\widehat{T}$.

At the bottom layer, we classify the remaining edges of $\widehat{T}$ as $\pmb{\widehat{T}}$\textbf{-airways}, and the remaining edges of $T$ as $\pmb{T}$\textbf{-airways}. Those are edges that go between different continents. Intuitively, the $T$-airways and $\widehat{T}$-airways are so weak (with $\alpha$-value $O{\paren{\frac{\epsilon}{\sqrt{n}}}}$) that it doesn't make much difference (in terms of $H^2$) if the continents were linked together with each other by $T$-airways or by $\widehat{T}$-airways. In fact, those edges can be cut (i.e. change the $\alpha$-value to $0$) without incurring too much error. Later, we will replace the set of all $T$-airways with the set of all $\widehat{T}$-airways in a single round of edge replacement in the process of going from $T$ to $\widehat{T}$.

\paragraph{Summary of layering.} To summarize, the edges of $\widehat{T}$ are classified into $\widehat{T}$-roads, $\widehat{T}$-highways, $\widehat{T}$-railways, and $\widehat{T}$-airways based on their $\abs{\widehat{\alpha}}$-estimates. The $n$ nodes are clustered into cities by the $\widehat{T}$-roads, which are further clustered into countries by the $\widehat{T}$-highways, which are further clustered into continents by the $\widehat{T}$-railways, which are finally linked into one component by the $\widehat{T}$-airways. This hierarchical classification of nodes into groups then induces a classification of the edges of $T$: an edge of $T$ is classified as a $T$-road if its two end-nodes belong to the same city, a $T$-highway if its two end-nodes belong to different cities in the same country, a $T$-railway if its two end-nodes belong to different countries in the same continent, or a $T$-airway if its two end-nodes belong to different continents.

Note that each city, country, or continent is $\widehat{T}$-connected by definition. We can prove that every city, country, or continent is also $T$-connected. Furthermore, we can prove that the $T$-highways and $\widehat{T}$-highways can be matched into parallel pairs, that is, there is a $\widehat{T}$-highway between two cities if and only if there is a $T$-highway between them. The same goes for railways: there is a $\widehat{T}$-railway between two countries if and only if there is a $T$-railway between them. One consequence of everything we stated so far is that the exact same clusters (i.e. cities, countries, or continents) at each level of the hierarchy would have been achieved if we were to cluster the $n$ nodes first by $T$-roads, then $T$-highways, then $T$-railways, and finally $T$-airways.

For a diagrammatic illustration of our structural results see Figure~\ref{fig:layer structure}. The solid connections represent edges of $T$, and the dashed connections represent edges of $\widehat{T}$. The thickness of the line increases as we go down in the hierarchy, with the roads being the thinnest and the airways being the thickest. The circles represent the cities, the rectangular boxes represent the countries, and the saw-toothed regions represent the continents. Note that the $T$-highways and $\widehat{T}$-highways can be matched into parallel pairs, each going between the same pair of cities (in the same country). An example is $(1,2)$ and $(3,4)$. Similarly, the $T$-railways and $\widehat{T}$-railways can be matched into parallel pairs, each going between the same pair of countries (in the same continent). An example is $(5,6)$ and $(7,8)$. This parallelism needs not hold for roads and airways. For example, there is no $\widehat{T}$-airway that goes between the same pair of continents as the $T$-airway $(9,10)$ does, and there is no $T$-airway that goes between the same pair of continents as the $\widehat{T}$-airway $(11,12)$ does.

\begin{figure}[h!]
\centering
\includegraphics[width=14.5cm]{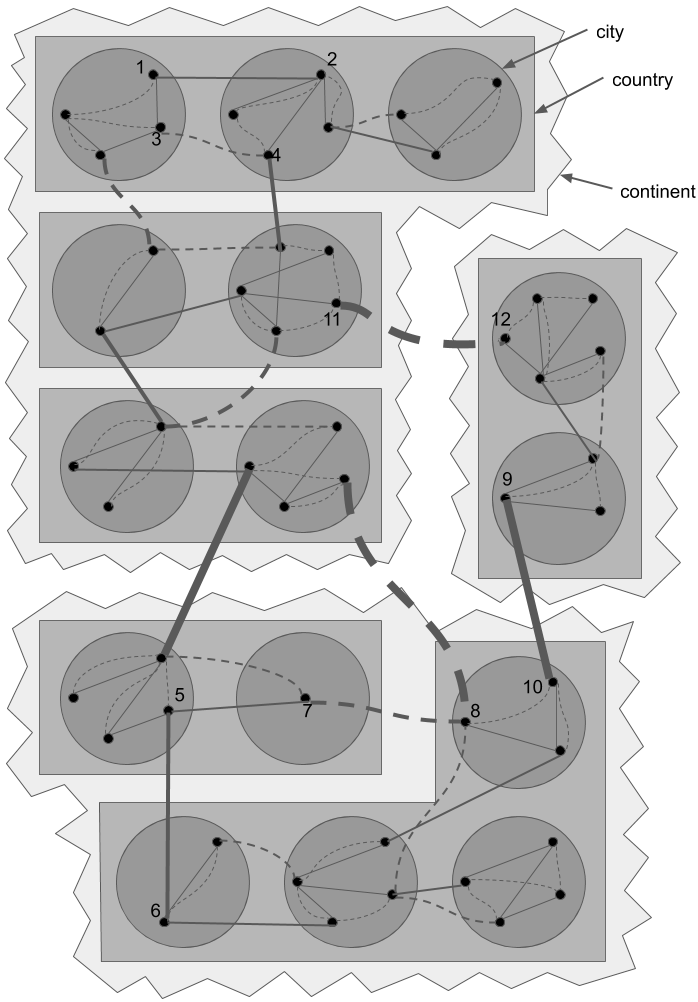}
\caption{A diagrammatic illustration of our hierarchy. For a description of the figure see the last paragraph of ``Summary of layering" in Section~\ref{sec:outline of symmetric case}.}
\label{fig:layer structure}
\end{figure}

\paragraph{Putting everything together.} We are now ready to bound $\dtv{\mathrm{P}}{\mathrm{Q}}$. Recall that $\mathrm{P}^{(4)}$ was defined as the distribution with underlying tree $\widehat{T}$ and $\alpha$-value equal to $\alpha_{ij}$ for every edge $(i,j)$ in $\widehat{T}$. We already argued that $\dtv{\mathrm{P}^{(4)}}{\mathrm{Q}} \le O(\epsilon)$, so it suffices to bound $H(\mathrm{P},\mathrm{P}^{(4)})$ (remember $d_{\mathrm{TV}}$ is upper bounded by $\sqrt{2}$ times $H$). There will be four rounds of edge replacements to go from $T$ to $\widehat{T}$. Starting with $T$, first we replace the set of all $T$-airways by the set of all $\widehat{T}$-airways, then replace the set of all $T$-railways by the set of all $\widehat{T}$-railways, then replace the set of all $T$-highways by the set of all $\widehat{T}$-highways, and finally replace the set of all $T$-roads by the set of all $\widehat{T}$-roads (as will soon be seen, it is crucial that we do those edge replacements in a bottom-up fashion, starting with the airways and ending with the roads). This creates three intermediate distributions $\mathrm{P}^{(1)},\mathrm{P}^{(2)},\mathrm{P}^{(3)}$ between $\mathrm{P}$ and $\mathrm{P}^{(4)}$. The underlying trees of $\mathrm{P}^{(1)},\mathrm{P}^{(2)},\mathrm{P}^{(3)}$ are various hybrids between $T$ and $\widehat{T}$ (for example, the underlying tree of $\mathrm{P}^{(2)}$ consists of the $T$-roads, the $T$-highways, the $\widehat{T}$-railways, and the $\widehat{T}$-airways). For each of these distributions the $\alpha$-value for each edge $(i,j)$ of its underlying tree is $\alpha_{ij}$.

To bound $H(\mathrm{P},\mathrm{P}^{(1)})$, we perform one sub-step in which we ``cut'' all the $T$-airways (i.e. change their $\alpha$-values to $0$) so that the continents become mutually independent, and another step in which we ``install'' all the $\widehat{T}$-airways. For each of those two sub-steps, we apply Corollary~\ref{edgeswitch} with $A_1, \ldots, A_\Lambda$ being the continents. The potentially nonzero terms on the right hand of (\ref{edgeswitchineq}) in Corollary~\ref{edgeswitch} are those in the second summand, with the $W_{\lambda \mu}$'s corresponding to the $T$-airways for the first sub-step, and to the $\widehat{T}$-airways for the second sub-step. By Lemma~\ref{difcontinents}, every $T$-airway or $\widehat{T}$-airway has $\abs{\alpha}$-value $O{\paren{\frac{\epsilon}{\sqrt{n}}}}$. It can be shown that (see Lemma~\ref{helfromindep} in Section~\ref{sec:symlemmas}) such an edge is $O{\paren{\frac{\epsilon^2}{n}}}$ away in $H^2$ from an edge with $\alpha$-value $0$. Since there are less than $n$ $T$-airways and $\widehat{T}$-airways, Corollary~\ref{edgeswitch} implies a loss of $O{\paren{\epsilon}}$ in $H$ in each of the two sub-steps.

To bound $H(\mathrm{P}^{(1)},\mathrm{P}^{(2)})$, we apply Corollary~\ref{edgeswitch} with $A_1, \ldots, A_\Lambda$ being the countries. The potentially nonzero terms on the right hand side of (\ref{edgeswitchineq}) in Corollary~\ref{edgeswitch} are those in the second summand, each of which corresponds to switching between a parallel pair of $T$-railway and $\widehat{T}$-railway that go between the same pair of countries ($W_{\lambda \mu}$ is the set of endpoints of those two railways). By Lemma~\ref{helwrongedge4} (or Lemma~\ref{helwrongedge3} if the parallel pair of railways share an end-node), each such switch costs us $O{\paren{\frac{\epsilon^2}{n}}}$ in $H^2$. Note that those lemmas are indeed applicable because each country is both $T$-connected and $\widehat{T}$-connected, and that the edges within each country haven't been replaced yet and are still edges of the true tree $T$; here we see that it is crucial to replace the edges from the bottom layer to the top layer). So overall we lose $O{\paren{\epsilon}}$ in $H$ in this round.

The bounding of $H(\mathrm{P}^{(2)},\mathrm{P}^{(3)})$ is similar; we apply Corollary~\ref{edgeswitch} with $A_1, \ldots, A_\Lambda$ being the cities.

Finally, to bound $H(\mathrm{P}^{(3)},\mathrm{P}^{(4)})$, we first ``tighten'' all the $T$-roads (i.e. turn their $\alpha$-values to $1$ or $-1$ depending on which is closer). At this point we can equivalently think of each city as being held together by the $\widehat{T}$-roads in it with appropriately chosen $1$ or $-1$ as $\alpha$-values. We then ``relax'' all the $\widehat{T}$-roads back to their original $\alpha$-values. For each of those two sub-steps we apply Corollary~\ref{edgeswitch} with $A_1, \ldots, A_\Lambda$ being the individual nodes. The potentially nonzero terms on the right hand of (\ref{edgeswitchineq}) in Corollary~\ref{edgeswitch} are those in the second summand, with the $W_{\lambda \mu}$'s corresponding to the $T$-roads for the first sub-step, and to the $\widehat{T}$-roads for the second sub-step. Every $T$-road or $\widehat{T}$-road has $\abs{\alpha}$-value that is  $1 - O{\paren{\frac{\epsilon^2}{n}}}$. It can be shown that (see Lema~\ref{helfromdet} in Section~\ref{sec:symlemmas}) ``tightening'' or ``relaxing'' such an edge costs $O{\paren{\frac{\epsilon^2}{n}}}$ in $H^2$. Since there are less than $n$ $T$-roads and $\widehat{T}$-roads, Corollary~\ref{edgeswitch} implies a loss of $O{\paren{\epsilon}}$ in $H$ in each of the two sub-steps.

\subsection{The Detailed Proof of the Symmetric Case}
\label{sec:symproof}

We present the detailed proof of Theorem~\ref{sufficiency}, following the steps sketched in Section~\ref{sec:outline of symmetric case}. Throughout the section, we assume that $\widehat{\mathrm{P}}$ satisfies 3-consistency, and our goal is to demonstrate that under this condition, Algorithm~\ref{algo:symalgo} will compute a distribution $\mathrm{Q}$ satisfying $\dtv{\mathrm{P}}{\mathrm{Q}} \leq O \paren{\epsilon}$. Section~\ref{sec:symlemmas} contains a collection of lemmas each involving a small number (2, 3, or 4) of variables. Section~\ref{sec:symstructure} contains structural results on the relation between $T$ and $\widehat{T}$, stated in the language of various hierarchies to be defined at the beginning of that section. The layers defined here dictate how we switch from $T$ to $\widehat{T}$ in the hybrid argument later. Section~\ref{sec:symbounding} proves the desired bound in total variation distance between $\mathrm{P}$ and $\mathrm{Q}$ with the aid of a few intermediate hybrid distributions. Throughout the section we assume $\epsilon \leq \frac{1}{10}$ (it's easy to see that it suffices to deal only with sufficiently small $\epsilon$). 

\subsubsection{Lemmas on 2, 3, or 4 Variables}
\label{sec:symlemmas}

This subsection includes a collection of lemmas each involving a small number (2, 3, or 4) of variables that form the basis of our analysis later of more complex structures (e.g. those of $T$ and $\widehat{T}$) and distributions.

Our first lemma indicates the degree to which $\widehat{\mathrm{P}}$ can be gauranteed to approximate $\mathrm{P}$ on events involving (up to) $3$ variables, as a consequence of assuming 3-consistency.

\begin{lemma}
\label{probprecision}
Let $S \subseteq [n]$ be of cardinality $3$, and $W \subseteq {\{1, -1\}}^3$, then
\begin{itemize}
\item[(i)] $\abs{\widehat{\mathrm{P}}(X_S \in W) - \mathrm{P}(X_S \in W)} \leq \frac{1}{10} \frac{\epsilon}{\sqrt{n}}$;
\item[(ii)] If $\mathrm{P}(X_S \in W) \leq 10 \frac{\epsilon^2}{n}$, then $\abs{\widehat{\mathrm{P}}(X_S \in W) - \mathrm{P}(X_S \in W)} \leq \frac{1}{2} \frac{\epsilon^2}{n}$;
\item[(iii)] If $\widehat{\mathrm{P}}(X_S \in W) \leq 5 \frac{\epsilon^2}{n}$, then $\abs{\widehat{\mathrm{P}}(X_S \in W) - \mathrm{P}(X_S \in W)} \leq \frac{1}{2} \frac{\epsilon^2}{n}$.
\end{itemize}
\end{lemma}
\begin{proof}
To prove (i), note that 3-consistency implies
\begin{align*}
\abs{\widehat{\mathrm{P}}(X_S \in W) - \mathrm{P}(X_S \in W)} & \leq \frac{1}{10} \cdot \max \brac{ \sqrt{\mathrm{P}(X_S \in W) \cdot \frac{\epsilon^2}{n}}, \frac{\epsilon^2}{n} } \\
	& \leq \frac{1}{10} \cdot \max \brac{ \sqrt{1 \cdot \frac{\epsilon^2}{n}}, \frac{\epsilon^2}{n} } \\
	& = \frac{1}{10} \frac{\epsilon}{\sqrt{n}}.
\end{align*}

To prove (ii), note that if $\mathrm{P}(X_S \in W) \leq 10 \frac{\epsilon^2}{n}$, then 3-consistency implies
\begin{align*}
\abs{\widehat{\mathrm{P}}(X_S \in W) - \mathrm{P}(X_S \in W)} & \leq \frac{1}{10} \cdot \max \brac{ \sqrt{\mathrm{P}(X_S \in W) \cdot \frac{\epsilon^2}{n}}, \frac{\epsilon^2}{n} } \\
	& \leq \frac{1}{10} \cdot \max \brac{ \sqrt{10 \frac{\epsilon^2}{n} \cdot \frac{\epsilon^2}{n}}, \frac{\epsilon^2}{n} } \\
	& < \frac{1}{2} \frac{\epsilon^2}{n}.
\end{align*}

To prove (iii), assume $\widehat{\mathrm{P}}(X_S \in W) \leq 5 \frac{\epsilon^2}{n}$. If we were to have $\mathrm{P}(X_S \in W) \geq 10 \frac{\epsilon^2}{n}$, then 3-consistency implies 
$$\abs{\widehat{\mathrm{P}}(X_S \in W) - \mathrm{P}(X_S \in W)} \leq \frac{1}{10} \cdot \max \brac{ \sqrt{\mathrm{P}(X_S \in W) \cdot \frac{\epsilon^2}{n}}, \frac{\epsilon^2}{n} } = \frac{1}{10} \cdot \sqrt{\mathrm{P}(X_S \in W) \cdot \frac{\epsilon^2}{n}},$$
and so
$$ \frac{\widehat{\mathrm{P}}(X_S \in W)}{\mathrm{P}(X_S \in W)} \geq 1 - \frac{\abs{\widehat{\mathrm{P}}(X_S \in W) - \mathrm{P}(X_S \in W)}}{\mathrm{P}(X_S \in W)} \geq 1 - \frac{1}{10} \cdot \sqrt{\frac{\frac{\epsilon^2}{n}}{\mathrm{P}(X_S \in W)}} > \frac{1}{2},$$
leading to $\widehat{\mathrm{P}}(X_S \in W) > 5 \frac{\epsilon^2}{n}$, a contradiction.

Thus, we must have $\mathrm{P}(X_S \in W) \leq 10 \frac{\epsilon^2}{n}$, and we can apply (ii).
\end{proof}

The next lemma indicates the degree to which $\widehat{\alpha}_{ij}$ can be guaranteed to approximate $\alpha_{ij}$. It follows from Lemma~\ref{probprecision} and the fact that $\alpha_{ij}$ and $\widehat{\alpha}_{ij}$ are simply linear transformations of the probabilities of the event ``$X_i = X_j$'' (or the event ``$X_i = -X_j$") under $\mathrm{P}$ and $\widehat{\mathrm{P}}$, respectively.

\begin{lemma}
\label{alphaprecision}
For any $i,j \in [n]$, we have the following:
\begin{itemize}
\item[(i)] $\abs{\widehat{\alpha}_{ij} - \alpha_{ij}} \leq \frac{1}{5} \frac{\epsilon}{\sqrt{n}}$;
\item[(ii)] If $\abs{\alpha_{ij}} \geq 1 - 20 \frac{\epsilon^2}{n}$, then $\abs{\widehat{\alpha}_{ij} - \alpha_{ij}} \leq \frac{\epsilon^2}{n}$;
\item[(iii)] If $\abs{\widehat{\alpha}_{ij}} \geq 1 - 10 \frac{\epsilon^2}{n}$, then $\abs{\widehat{\alpha}_{ij} - \alpha_{ij}} \leq \frac{\epsilon^2}{n}$.
\end{itemize}
\end{lemma}
\begin{proof}
Recall that $\alpha_{ij} = 2 \mathrm{P}(X_i = X_j) - 1$, $\widehat{\alpha}_{ij} = 2 \widehat{\mathrm{P}}(X_i = X_j) - 1$. So we have
\begin{equation}
\label{eq:alphaboundfromP}
\abs{\widehat{\alpha}_{ij} - \alpha_{ij}} = 2 \abs{\widehat{\mathrm{P}}(X_i = X_j) - \mathrm{P}(X_i = X_j)},
\end{equation}
and (i) follows from Lemma~\ref{probprecision} (i).

Furthermore, $\alpha_{ij} \leq -(1 - 20 \frac{\epsilon^2}{n})$ implies  $\mathrm{P}(X_i = X_j) \leq 10 \frac{\epsilon^2}{n}$, and $\widehat{\alpha}_{ij} \leq -(1 - 10 \frac{\epsilon^2}{n})$ implies $\widehat{\mathrm{P}}(X_i = X_j) \leq 5 \frac{\epsilon^2}{n}$. In both cases, the desired bound follows from Lemma~\ref{probprecision} and (\ref{eq:alphaboundfromP}).

The other cases $\alpha_{ij} \geq 1 - 20 \frac{\epsilon^2}{n}$ and $\widehat{\alpha}_{ij} \geq 1 - 10 \frac{\epsilon^2}{n}$ are handled by starting from the identities $\alpha_{ij} = 1 - 2 \mathrm{P}(X_i = -X_j)$ and $\widehat{\alpha}_{ij} = 1 - 2 \widehat{\mathrm{P}}(X_i = -X_j)$ instead, and proceed similarly.
\end{proof}

The next lemma bounds the $H^2$ between a pair of variables $i,j$ with $\alpha$-value $\alpha_{ij}$ and the same pair of variable with $\alpha$-value $0$.  It is the error in $H^2$ when we ``cut'' an edge with $\alpha$-value $\alpha_{ij}$.

\begin{lemma}
\label{helfromindep}
Fix any $i,j \in [n]$. Consider the following pair of distributions:
\begin{itemize}
\item[] $\mathrm{P}_{ij}$ - the binary symmetric distribution for $X_i,X_j$ with $\alpha$-value $\alpha_{ij}$;
\item[] $\mathrm{P}_{ij}^{(\rm{ind})}$ - the independent binary symmetric distribution for $X_i,X_j$, i.e. $\alpha$-value $0$. 
\end{itemize}
Then $H^2(\mathrm{P}_{ij}, \mathrm{P}_{ij}^{(\rm{ind})}) \leq \frac{1}{2} \alpha_{ij}^2$.
\end{lemma}
\begin{proof}
We have the following:
\begin{align*}
\mathrm{P}_{ij} (1,1) = \mathrm{P}_{ij} (-1,-1) &= \frac{1}{2} \cdot \frac{1+\alpha_{ij}}{2} &\mathrm{P}_{ij}^{(\rm{ind})} (1,1) = \mathrm{P}_{ij}^{(\rm{ind})} (-1,-1) &= \frac{1}{4} \\
\mathrm{P}_{ij} (1,-1) = \mathrm{P}_{ij} (-1,1) &= \frac{1}{2} \cdot \frac{1-\alpha_{ij}}{2} & \mathrm{P}_{ij}^{(\rm{ind})} (1,-1) = \mathrm{P}_{ij}^{(\rm{ind})} (-1,1) &= \frac{1}{4} 
\end{align*}
Thus, we have
\begin{align*}
H^2(\mathrm{P}_{ij}, \mathrm{P}_{ij}^{(\rm{ind})}) &= \frac{1}{2} \cdot \sqbrac{ \ 2 \cdot \paren{ \sqrt{\frac{1+\alpha_{ij}}{4}} - \sqrt{\frac{1}{4}} }^2 + 2 \cdot \paren{ \sqrt{\frac{1-\alpha_{ij}}{4}} - \sqrt{\frac{1}{4}} }^2 } \\
                     &= \frac{1}{4} \cdot \sqbrac{ \paren{ \sqrt{1+\alpha_{ij}} - 1 }^2 + \paren{ \sqrt{ 1-\alpha_{ij} } - 1 }^2 } \\
		     & \stackrel{(\ast)}{\leq} \frac{1}{4} \cdot ( \alpha_{ij}^2 + \alpha_{ij}^2 )  \\
		     &= \frac{1}{2} \alpha_{ij}^2,
\end{align*}
where $(\ast)$ follows from the fact that $\abs{\sqrt{1+a} - 1} = \frac{\abs{a}}{\sqrt{1+a}+1} \leq \abs{a}$, $\forall \abs{a} \leq 1$.
\end{proof}

The next lemma bounds the $H^2$ between a pair of variables $i,j$ with $\alpha$-value $\alpha_{ij}$ and the same pair of variable with $\alpha$-value $1$ or $-1$, whichever is closer to $\alpha_{ij}$. It is the error in $H^2$ when we ``tighten'' an edge with $\alpha$-value $\alpha_{ij}$.

\begin{lemma}
\label{helfromdet}
Fix any $i,j \in [n]$. Consider the following pair of distributions:
\begin{itemize}
\item[] $\mathrm{P}_{ij}$ - the binary symmetric distribution for $X_i,X_j$ with $\alpha$-value $\alpha_{ij}$;
\item[] $\mathrm{P}_{ij}^{(\rm{det})}$ - the binary symmetric distribution for $X_i,X_j$ with $\alpha$-value $\sign{\alpha_{ij}}$,
\end{itemize}
where $\sign{t}$ equals $1$ for $t \geq 0$, and $-1$ otherwise. Then $H^2(\mathrm{P}_{ij}, \mathrm{P}_{ij}^{(\rm{det})}) \leq \frac{1}{2} (1-\abs{\alpha_{ij}})$.
\end{lemma}
\begin{proof}
Let $\abs{\alpha_{ij}} = 1 - \delta_{ij}$.

Assume $\alpha_{ij} \geq 0$ (the case $\alpha_{ij} < 0$ follows symmetrically). We have
\begin{align*}
\mathrm{P}_{ij} (1,1) = \mathrm{P}_{ij} (-1,-1) &= \frac{1}{2} \cdot \frac{2-\delta_{ij}}{2} & \mathrm{P}_{ij}^{(\rm{det})} (1,1) = \mathrm{P}_{ij}^{(\rm{det})} (-1,-1) &= \frac{1}{2} \\
\mathrm{P}_{ij} (1,-1) = \mathrm{P}_{ij} (-1,1) &= \frac{1}{2} \frac{\delta_{ij}}{2}  & \mathrm{P}_{ij}^{(\rm{det})} (1,-1) = \mathrm{P}_{ij}^{(\rm{det})} (-1,1) &= 0
\end{align*}
Thus, by Lemma~\ref{tvvshel} we have
\begin{align*}
H^2(\mathrm{P}_{ij}, \mathrm{P}_{ij}^{(\rm{det})}) &\leq \dtv{\mathrm{P}_{ij}}{\mathrm{P}_{ij}^{(\rm{det})}} \\
                     &= \frac{1}{2} \cdot \paren{ 2 \cdot \abs{ \frac{1}{2} \cdot \frac{2-\delta_{ij}}{2} - \frac{1}{2} } + 2 \cdot \abs{ \frac{1}{2} \frac{\delta_{ij}}{2} - 0 } } \\
		     &= \frac{1}{2} \delta_{ij},
\end{align*}
which is exactly what we want.
\end{proof}

The next lemma bounds the $H^2$ between a pair of variables $i,j$ with $\alpha$-value $\alpha_{ij}$ and the same pair of variable with $\alpha$-value $\widehat{\alpha}_{ij}$.

\begin{lemma}
\label{helfromestimate}
Fix any $i,j \in [n]$. Consider the following pair of distributions:
\begin{itemize}
\item[] $\mathrm{P}_{ij}$ - the binary symmetric distribution for $X_i,X_j$ with $\alpha$-value $\alpha_{ij}$;
\item[] $\mathrm{P}_{ij}^{(\rm{est})}$ - the binary symmetric distribution for $X_i,X_j$ with $\alpha$-value $\widehat{\alpha}_{ij}$. 
\end{itemize}
Then $H^2(\mathrm{P}_{ij}, \mathrm{P}_{ij}^{(\rm{est})}) \leq \frac{1}{10} \frac{\epsilon^2}{n}$.
\end{lemma}
\begin{proof}
Since $\frac{1+\widehat{\alpha}_{ij}}{2} = \widehat{\mathrm{P}} (X_i = X_j)$ and $\frac{1-\widehat{\alpha}_{ij}}{2} = \widehat{\mathrm{P}} (X_i = -X_j)$, we have the following:
\begin{align*}
\mathrm{P}_{ij} (1,1) = \mathrm{P}_{ij} (-1,-1) &= \frac{1}{2} \mathrm{P} (X_i = X_j)   & \mathrm{P}_{ij}^{(\rm{est})} (1,1) = \mathrm{P}_{ij}^{(\rm{est})} (-1,-1) &= \frac{1}{2} \widehat{\mathrm{P}} (X_i = X_j) \\
\mathrm{P}_{ij} (1,-1) = \mathrm{P}_{ij} (-1,1) &= \frac{1}{2} \mathrm{P} (X_i = -X_j)    & \mathrm{P}_{ij}^{(\rm{est})} (1,-1) = \mathrm{P}_{ij}^{(\rm{est})} (-1,1) &= \frac{1}{2} \widehat{\mathrm{P}} (X_i = -X_j)
\end{align*}
Thus, we have
\begin{align*}
H^2(\mathrm{P}_{ij}, \mathrm{P}_{ij}^{(\rm{est})}) &= \frac{1}{2} \cdot \Bigg[ \ 2 \cdot \paren{ \sqrt{\frac{1}{2} \widehat{\mathrm{P}} (X_i = X_j)} - \sqrt{\frac{1}{2} \mathrm{P} (X_i = X_j)} }^2 \\
                     & ~~~~~~~~~~~~~~~~~~~~~~ + 2 \cdot \paren{ \sqrt{\frac{1}{2} \widehat{\mathrm{P}} (X_i = -X_j)} - \sqrt{\frac{1}{2} \mathrm{P} (X_i = -X_j)} }^2 \Bigg] \\
                     &= \frac{1}{2} \cdot \bigg[ \paren{ \sqrt{\widehat{\mathrm{P}} (X_i = X_j)} - \sqrt{ \mathrm{P} (X_i = X_j)} }^2 \\
                     & ~~~~~~~~~~~~~~~~~~~~~~ + \paren{\sqrt{ \widehat{\mathrm{P}} (X_i = -X_j)} - \sqrt{ \mathrm{P} (X_i = -X_j)} }^2 \bigg] \\
		     & \stackrel{(\ast)}{\leq} \frac{1}{2} \cdot \paren{\frac{1}{10} \frac{\epsilon^2}{n} + \frac{1}{10} \frac{\epsilon^2}{n}}  \\
		     &= \frac{1}{10} \frac{\epsilon^2}{n}, 
\end{align*}
where ($\ast$) follows from 3-consistency and Lemma~\ref{bernoullihellinger}.
\end{proof}

For any $i,j,k$ that lie on a path in $T$, the next lemma bounds the $H^2$ between the true distribution on those three nodes (i.e. $\mathrm{P}_{ijk}$) and the alternative distribution if $i$ were to be mistakenly regarded as closer to $k$ than to $j$ (i.e. $\mathrm{P}_{i\wideparen{ \,~~ j-}k}$). See Definition~\ref{arcnotation} for definitions of symbols such as $\mathrm{P}_{i\wideparen{ \,~~ j-}k})$. See Figure~\ref{3nodesfigure} in Section~\ref{sec:outline of symmetric case} for a depiction of those two distributions.

\begin{lemma}
\label{heltwochains}
If $i,j,k$ lie on a path in $T$, then $H^2(\mathrm{P}_{ijk}, \mathrm{P}_{i\wideparen{ \,~~ j-}k}) \leq 2 \alpha_{ij}^2 (1-\abs{\alpha_{jk}})$.
\end{lemma}
\begin{proof}
Let $\delta_{jk} = 1 - \abs{\alpha_{jk}}$.

Assume $\alpha_{ij} \geq 0, \alpha_{jk} \geq 0$ (the other three cases follow symmetrically). Since by Lemma~\ref{multiplicativity} $\alpha_{ik} = \alpha_{ij} \alpha_{jk} = \alpha_{ij} (1 - \delta_{jk})$, we have
\begin{align*}
\mathrm{P}_{ijk} (1,1,1) &= \frac{1}{2} \cdot \frac{1+\alpha_{ij}}{2} \cdot \frac{2-\delta_{jk}}{2} & \mathrm{P}_{i\wideparen{ \,~~ j-}k} (1,1,1) &= \frac{1}{2} \cdot \frac{1+\alpha_{ij} (1-\delta_{jk})}{2} \cdot \frac{2-\delta_{jk}}{2} \\
\mathrm{P}_{ijk} (1,1,-1) &= \frac{1}{2} \cdot \frac{1+\alpha_{ij}}{2} \cdot \frac{\delta_{jk}}{2} & \mathrm{P}_{i\wideparen{ \,~~ j-}k} (1,1,-1) &= \frac{1}{2} \cdot \frac{1-\alpha_{ij} (1-\delta_{jk})}{2} \cdot \frac{\delta_{jk}}{2} \\
\mathrm{P}_{ijk} (1,-1,1) &= \frac{1}{2} \cdot \frac{1-\alpha_{ij}}{2} \cdot \frac{\delta_{jk}}{2} & \mathrm{P}_{i\wideparen{ \,~~ j-}k} (1,-1,1) &= \frac{1}{2} \cdot \frac{1+\alpha_{ij} (1-\delta_{jk})}{2} \cdot \frac{\delta_{jk}}{2} \\
\mathrm{P}_{ijk} (1,-1,-1) &= \frac{1}{2} \cdot \frac{1-\alpha_{ij}}{2} \cdot \frac{2-\delta_{jk}}{2} & \mathrm{P}_{i\wideparen{ \,~~ j-}k} (1,-1,-1) &= \frac{1}{2} \cdot \frac{1-\alpha_{ij} (1-\delta_{jk})}{2} \cdot \frac{2-\delta_{jk}}{2}
\end{align*}
We omit listing the probabilities for the four other combinations $(-1,1,1)$, $(-1,1,-1)$, $(-1,-1,1)$, and $(-1,-1,-1)$, due to the simple observation that, because of symmetry, negating all three variables simultaneously gives the same probability.

First, we have
\begingroup
\allowdisplaybreaks
\begin{align*}
\abs{\sqrt{\mathrm{P}_{ijk} (1,1,-1)} - \sqrt{\mathrm{P}_{i\wideparen{ \,~~ j-}k} (1,1,-1)}} &= \sqrt{\frac{\delta_{jk}}{8}} \cdot \abs{\sqrt{1+\alpha_{ij}} - \sqrt{1-\alpha_{ij} (1-\delta_{jk})}} \\
		& \leq \sqrt{\frac{\delta_{jk}}{8}} \cdot \paren{\abs{\sqrt{1+\alpha_{ij}} - 1} + \abs{1 - \sqrt{1-\alpha_{ij} (1-\delta_{jk})}}} \\
		& \stackrel{(\ast)}{\leq} \sqrt{\frac{\delta_{jk}}{8}} \cdot \paren{\alpha_{ij} + \alpha_{ij} \paren{1-\delta_{jk}}} \\
		& \leq \frac{1}{\sqrt{2}} \alpha_{ij} \sqrt{\delta_{jk}},
\end{align*}
\endgroup
where $(\ast)$ follows from the fact that $\abs{\sqrt{1+a} - 1} = \frac{\abs{a}}{\sqrt{1+a}+1} \leq \abs{a}$, $\forall \abs{a} \leq 1$.

Thus, $\paren{ \sqrt{\mathrm{P}_{ijk} (1,1,-1)} - \sqrt{\mathrm{P}_{i\wideparen{ \,~~ j-}k} (1,1,-1)} }^2 \leq \frac{1}{2} \alpha_{ij}^2 \delta_{jk}$.

Similarly, $\paren{ \sqrt{\mathrm{P}_{ijk} (1,-1,1)} - \sqrt{\mathrm{P}_{i\wideparen{ \,~~ j-}k} (1,-1,1)} }^2 \leq \frac{1}{2} \alpha_{ij}^2 \delta_{jk}$.

Next, we have
\begingroup
\allowdisplaybreaks
\begin{align*}
\abs{\sqrt{\mathrm{P}_{ijk} (1,1,1)} - \sqrt{\mathrm{P}_{i\wideparen{ \,~~ j-}k} (1,1,1)}} &= \sqrt{\frac{2-\delta_{jk}}{8}} \cdot \abs{\sqrt{1+\alpha_{ij}} - \sqrt{1+\alpha_{ij} (1-\delta_{jk})}} \\
		&= \sqrt{\frac{2-\delta_{jk}}{8}} \cdot \sqrt{1+\alpha_{ij}} \cdot \abs {1 - \sqrt{1 - \frac{\alpha_{ij} \delta_{jk}}{1 + \alpha_{ij}}}} \\
		&\leq \sqrt{\frac{2-\delta_{jk}}{8}} \cdot \sqrt{1+\alpha_{ij}} \cdot \frac{\alpha_{ij} \delta_{jk}}{1 + \alpha_{ij}} \\
		&\leq \frac{1}{2} \alpha_{ij} \delta_{jk} \\
		&\leq \frac{1}{2} \alpha_{ij} \sqrt{\delta_{jk}},
\end{align*}
\endgroup
and so $\paren{ \sqrt{\mathrm{P}_{ijk} (1,1,1)} - \sqrt{\mathrm{P}_{i\wideparen{ \,~~ j-}k} (1,1,1)} }^2 \leq \frac{1}{4} \alpha_{ij}^2 \delta_{jk}$.

Finally, to bound $\abs{\sqrt{\mathrm{P}_{ijk} (1,-1,-1)} - \sqrt{\mathrm{P}_{i\wideparen{ \,~~ j-}k} (1,-1,-1)}}$, we need to consider two cases.

If $\alpha_{ij} \leq \frac{1}{2}$, we have
\begingroup
\allowdisplaybreaks
\begin{align*}
\abs{\sqrt{\mathrm{P}_{ijk} (1,-1,-1)} - \sqrt{\mathrm{P}_{i\wideparen{ \,~~ j-}k} (1,-1,-1)}} &= \sqrt{\frac{2-\delta_{jk}}{8}} \cdot \abs{\sqrt{1-\alpha_{ij}} - \sqrt{1-\alpha_{ij} (1-\delta_{jk})}} \\
		&= \sqrt{\frac{2-\delta_{jk}}{8}} \cdot \sqrt{1-\alpha_{ij}} \cdot \abs {1 - \sqrt{1 + \frac{\alpha_{ij} \delta_{jk}}{1 - \alpha_{ij}}}} \\
		&\leq \sqrt{\frac{2-\delta_{jk}}{8}} \cdot \sqrt{1-\alpha_{ij}} \cdot \frac{\alpha_{ij} \delta_{jk}}{1 - \alpha_{ij}} \\
		&= \sqrt{\frac{2-\delta_{jk}}{8}} \cdot \frac{\alpha_{ij} \delta_{jk}}{\sqrt{1 - \alpha_{ij}}} \\
		&\leq \frac{1}{\sqrt{2}} \alpha_{ij} \delta_{jk} \\
		&\leq \frac{1}{\sqrt{2}} \alpha_{ij} \sqrt{\delta_{jk}},
\end{align*}
\endgroup
while if $\alpha_{ij} > \frac{1}{2}$, we have
\begingroup
\allowdisplaybreaks
\begin{align*}
\abs{\sqrt{\mathrm{P}_{ijk} (1,-1,-1)} - \sqrt{\mathrm{P}_{i\wideparen{ \,~~ j-}k} (1,-1,-1)}} &= \sqrt{\frac{2-\delta_{jk}}{8}} \cdot \abs{\sqrt{1-\alpha_{ij} (1-\delta_{jk})} - \sqrt{1-\alpha_{ij}}} \\
		&\stackrel{(\ast)}{\leq} \sqrt{\frac{2-\delta_{jk}}{8}} \cdot \sqrt{\alpha_{ij} \delta_{jk}} \\
		&= \sqrt{\frac{2-\delta_{jk}}{8}} \cdot \frac{1}{\sqrt{\alpha_{ij}}} \cdot \alpha_{ij} \sqrt{\delta_{jk}} \\
		&\leq \frac{1}{\sqrt{2}} \alpha_{ij} \sqrt{\delta_{jk}},
\end{align*}
\endgroup
where $(\ast)$ follows from the fact that $\sqrt{a+b} \leq \sqrt{a} + \sqrt{b}$, $\forall a,b \geq 0$.

So, in either case, we have $\paren{ \sqrt{\mathrm{P}_{ijk} (1,-1,-1)} - \sqrt{\mathrm{P}_{i\wideparen{ \,~~ j-}k} (1,-1,-1)} }^2 \leq \frac{1}{2} \alpha_{ij}^2 \delta_{jk}$.

Combining everything above, we get $H^2(\mathrm{P}_{ijk}, \mathrm{P}_{i\wideparen{ \,~~ j-}k}) \leq 2 \alpha_{ij}^2 \delta_{jk}$. This is exactly what we want.
\end{proof}

For any $i,j,k$ that lie on a path in $T$, the multiplicativity of $\alpha$-values along a path (Lemma~\ref{multiplicativity}) implies that $\alpha_{ik} = \alpha_{ij} \alpha_{jk}$ is no greater in absolute value than either $\alpha_{ij}$ or $\alpha_{jk}$. Algorithm~\ref{algo:symalgo} may ``reverse the order'' between $(i,j)$ and $(i,k)$ if $\abs{\widehat{\alpha}_{ij}} \leq \abs{\widehat{\alpha}_{ik}}$. The following lemma characterizes when this can possibly happen (assuming 3-consistency).

\begin{lemma}
\label{wrongedgecondition}
If $i,j,k$ lie on a path in $T$, and $\abs{\widehat{\alpha}_{ij}} \leq \abs{\widehat{\alpha}_{ik}}$, then $\alpha_{ij}^2 (1-\abs{\alpha_{jk}}) \leq \frac{\epsilon^2}{n}$.
\end{lemma}
\begin{proof}
Let $\delta_{jk} = 1 - \abs{\alpha_{jk}}$.

Assume $\alpha_{ij} \geq 0, \alpha_{jk} \geq 0$ (the other three cases follow symmetrically). Note that this implies $\alpha_{ik} = \alpha_{ij} \alpha_{jk} \geq 0$. If either $\alpha_{ij} \leq \frac{\epsilon}{\sqrt{n}}$ or $\delta_{jk} \leq \frac{\epsilon^2}{n}$, the conclusion obviously holds. So for the rest of the proof we assume both $\alpha_{ij} \geq \frac{\epsilon}{\sqrt{n}}$ and $\delta_{jk} \geq \frac{\epsilon^2}{n}$.

By Lemma~\ref{alphaprecision} (i), we have $\abs{\widehat{\alpha}_{ij} - \alpha_{ij}} \leq \frac{1}{5} \frac{\epsilon}{\sqrt{n}}$, $\abs{\widehat{\alpha}_{ik} - \alpha_{ik}} \leq \frac{1}{5} \frac{\epsilon}{\sqrt{n}}$, so
\begin{align*}
\widehat{\alpha}_{ij} &\geq \alpha_{ij} - \abs{\widehat{\alpha}_{ij} - \alpha_{ij}} \geq \frac{\epsilon}{\sqrt{n}} - \frac{1}{5} \frac{\epsilon}{\sqrt{n}} \geq \frac{4}{5} \frac{\epsilon}{\sqrt{n}}; \\
\widehat{\alpha}_{ik} &\geq \alpha_{ik} - \abs{\widehat{\alpha}_{ik} - \alpha_{ik}} \geq 0 - \frac{1}{5} \frac{\epsilon}{\sqrt{n}} \geq -\frac{1}{5} \frac{\epsilon}{\sqrt{n}} > -\frac{4}{5} \frac{\epsilon}{\sqrt{n}}.
\end{align*}
This implies $\abs{\widehat{\alpha}_{ij}} = \widehat{\alpha}_{ij}$, and $-\widehat{\alpha}_{ij} < \widehat{\alpha}_{ik}$. Because we assumed that $\abs{\widehat{\alpha}_{ij}} \leq \abs{\widehat{\alpha}_{ik}}$, we must have $\widehat{\alpha}_{ij} \leq \widehat{\alpha}_{ik}$. So $\widehat{\alpha}_{ij} - \widehat{\alpha}_{ik} \leq 0$.

Consider the following identities:
\begin{align}
\begin{split}
\label{eq:gapbound1}
\widehat{\mathrm{P}}(X_i = X_j) &= \widehat{\mathrm{P}}(X_i = X_j, X_j = X_k) + \widehat{\mathrm{P}}(X_i = X_j, X_j = -X_k); \\
\widehat{\mathrm{P}}(X_i = X_k) &= \widehat{\mathrm{P}}(X_i = X_j, X_j = X_k) + \widehat{\mathrm{P}}(X_i = -X_j, X_j = -X_k); \\
\widehat{\alpha}_{ij} - \widehat{\alpha}_{ik} &= \paren{2 \widehat{\mathrm{P}}(X_i = X_j) - 1} - \paren{2 \widehat{\mathrm{P}}(X_i = X_k) - 1} \\
		&= 2 \paren{\widehat{\mathrm{P}}(X_i = X_j) - \widehat{\mathrm{P}}(X_i = X_k)} \\
		&= 2 \paren{\widehat{\mathrm{P}}(X_i = X_j, X_j = -X_k) - \widehat{\mathrm{P}}(X_i = -X_j, X_j = -X_k)}.
\end{split}
\end{align}
The cancelation of the term $\widehat{\mathrm{P}}(X_i = X_j, X_j = X_k)$ during the last step is crucial, as it allows us to apply 3-consistency to events with smaller probability masses under $\mathrm{P}$, thereby obtaining tighter bounds (the argument below wouldn't have worked if we were to stop at the penultimate line of (\ref{eq:gapbound1}) and apply 3-consistency to the events ``$X_i = X_j$'' and ``$X_i = X_k$'') . We have
\begin{align}
\label{eq:gapbound2}
\abs{ \widehat{\mathrm{P}}(X_i = X_j, X_j = -X_k) - \underbrace{\mathrm{P}(X_i = X_j, X_j = -X_k)}_{=\frac{1+\alpha_{ij}}{2} \cdot \frac{\delta_{jk}}{2}} } &\leq \frac{1}{10} \cdot \max \brac{ \sqrt{\frac{1+\alpha_{ij}}{2} \cdot \frac{\delta_{jk}}{2} \cdot \frac{\epsilon^2}{n}}, \frac{\epsilon^2}{n}} \nonumber \\
&\leq \frac{1}{10} \cdot \max \brac{ \sqrt{\delta_{jk} \frac{\epsilon^2}{n}}, \frac{\epsilon^2}{n}} \nonumber \\
&= \frac{1}{10} \sqrt{\delta_{jk} \frac{\epsilon^2}{n}},
\end{align}
since we assumed $\delta_{jk} \geq \frac{\epsilon^2}{n}$, and similarly
\begin{equation}
\label{eq:gapbound3}
\abs{ \widehat{\mathrm{P}}(X_i = -X_j, X_j = -X_k) - \underbrace{\mathrm{P}(X_i = -X_j, X_j = -X_k)}_{=\frac{1-\alpha_{ij}}{2} \cdot \frac{\delta_{jk}}{2}} } \leq \frac{1}{10} \sqrt{\delta_{jk} \frac{\epsilon^2}{n}}.
\end{equation}

Noting that $\mathrm{P}(X_i = X_j, X_j = -X_k) - \mathrm{P}(X_i = -X_j, X_j = -X_k) = \frac{\alpha_{ij} \delta_{jk}}{2}$, (\ref{eq:gapbound1}), (\ref{eq:gapbound2}), (\ref{eq:gapbound3}) now imply
\begin{equation}
\label{eq:gapbound}
\widehat{\alpha}_{ij} - \widehat{\alpha}_{ik} \geq 2 \paren{\frac{\alpha_{ij} \delta_{jk}}{2} - \frac{1}{10} \sqrt{\delta_{jk} \frac{\epsilon^2}{n}} - \frac{1}{10} \sqrt{\delta_{jk} \frac{\epsilon^2}{n}}},
\end{equation}
which, together with $\widehat{\alpha}_{ij} - \widehat{\alpha}_{ik} \leq 0$, implies $\alpha_{ij}^2 \delta_{jk} \leq \frac{4}{25} \frac{\epsilon^2}{n} < \frac{\epsilon^2}{n}$. This is exactly what we want.
\end{proof}

Lemma~\ref{wrongedgecondition} is important because it tells us, at least for the setting of three nodes lying on a path in $T$, when it's possible for Algorithm~\ref{algo:symalgo} to make a structural mistake and when it's not. It will serve as a guideline on how to classify the edges and define the layers in Section~\ref{sec:symstructure} (which allows us to speak more precisely about the manners in which $\widehat{T}$ might differ from $T$). Furthermore, when combined with Lemma~\ref{heltwochains}, it bounds the error in $H^2$ if Algorithm~\ref{algo:symalgo} does make a mistake in inferring the structure in the setting of three nodes lying on a path in $T$, as the following lemma shows.

\begin{lemma}
\label{helwrongedge3}
If $i,j,k$ lie on a path in $T$, and $\abs{\widehat{\alpha}_{ij}} \leq \abs{\widehat{\alpha}_{ik}}$, then $H^2(\mathrm{P}_{ijk}, \mathrm{P}_{i\wideparen{ \,~~ j-}k}) \leq 2 \frac{\epsilon^2}{n}$.
\end{lemma}
\begin{proof}
This follows directly from Lemma~\ref{heltwochains} and Lemma~\ref{wrongedgecondition}.
\end{proof}

We also have the following four-node version of Lemma~\ref{helwrongedge3} (see Figure~\ref{4nodesfigure} in Section~\ref{sec:outline of symmetric case} for a depiction of the two distributions involved):

\begin{lemma}
\label{helwrongedge4}
If $h,i,j,k$ lie on a path in $T$, and $\abs{\widehat{\alpha}_{ij}} \leq \abs{\widehat{\alpha}_{hk}}$, then $H^2(\mathrm{P}_{hijk}, \mathrm{P}_{h\wideparen{-i \,~~ j-}k}) \leq 8 \frac{\epsilon^2}{n}$.
\end{lemma}
\begin{proof}
The proof is analogous to that of Lemma~\ref{helwrongedge3} through Lemma~\ref{heltwochains} and Lemma~\ref{wrongedgecondition}.

Assume $\alpha_{hi} \geq 0, \alpha_{ij} \geq 0, \alpha_{jk} \geq 0$ (the other seven cases follow symmetrically). Note that this implies $\alpha_{hk} = \alpha_{hi} \alpha_{ij} \alpha_{jk} \geq 0$. For ease of notation, let $\delta_{hi} = 1 - \alpha_{hi}$, $\delta_{jk} = 1 - \alpha_{jk}$.

We have by Lemma~\ref{hel4nodes}
$$H^2(\mathrm{P}_{hijk}, \mathrm{P}_{h\wideparen{-i \,~~ j-}k}) \leq \paren{H(\mathrm{P}_{ijk}, \mathrm{P}_{i\wideparen{ \,~~ j-}k}) + H(\mathrm{P}_{hik}, \mathrm{P}_{h\wideparen{-i \,~~ }k})}^2.$$
Lemma~\ref{heltwochains} applied to $i,j,k$ gives $H(\mathrm{P}_{ijk}, \mathrm{P}_{i\wideparen{ \,~~ j-}k}) \leq \sqrt{2 \alpha_{ij}^2 \delta_{jk}}$, and applied to $k,i,h$ (noticing that $\alpha_{ik} = \alpha_{ij} (1 - \delta_{jk})$) gives $H(\mathrm{P}_{hik}, \mathrm{P}_{h\wideparen{-i \,~~ }k}) \leq \sqrt{2 \alpha_{ij}^2 \paren{1 - \delta_{jk}}^2 \delta_{hi}}$. So we have
$$H^2(\mathrm{P}_{hijk}, \mathrm{P}_{h\wideparen{-i \,~~ j-}k}) \leq \paren{\sqrt{2 \alpha_{ij}^2 \delta_{jk}} + \sqrt{2 \alpha_{ij}^2 \paren{1 - \delta_{jk}}^2 \delta_{hi}}}^2 \leq 8 \alpha_{ij}^2 \max \brac{\delta_{hi}, \delta_{jk}}.$$

To rest of our job is to prove $\alpha_{ij}^2 \max \brac{\delta_{hi}, \delta_{jk}} \leq \frac{\epsilon^2}{n}$. This obviously holds if either $\alpha_{ij} \leq \frac{\epsilon}{\sqrt{n}}$ or $\max \brac{\delta_{hi}, \delta_{jk}} \leq \frac{\epsilon^2}{n}$, so from now on we assume $\alpha_{ij} \geq \frac{\epsilon}{\sqrt{n}}$ and $\max \brac{\delta_{hi}, \delta_{jk}} \geq \frac{\epsilon^2}{n}$.

We argue as in Lemma~\ref{wrongedgecondition}. By Lemma~\ref{alphaprecision} (i), we have $\abs{\widehat{\alpha}_{ij} - \alpha_{ij}} \leq \frac{1}{5} \frac{\epsilon}{\sqrt{n}}$, $\abs{\widehat{\alpha}_{hk} - \alpha_{hk}} \leq \frac{1}{5} \frac{\epsilon}{\sqrt{n}}$, so
\begin{align*}
\widehat{\alpha}_{ij} &\geq \alpha_{ij} - \abs{\widehat{\alpha}_{ij} - \alpha_{ij}} \geq \frac{\epsilon}{\sqrt{n}} - \frac{1}{5} \frac{\epsilon}{\sqrt{n}} \geq \frac{4}{5} \frac{\epsilon}{\sqrt{n}}; \\
\widehat{\alpha}_{hk} &\geq \alpha_{hk} - \abs{\widehat{\alpha}_{hk} - \alpha_{hk}} \geq 0 - \frac{1}{5} \frac{\epsilon}{\sqrt{n}} \geq -\frac{1}{5} \frac{\epsilon}{\sqrt{n}} > -\frac{4}{5} \frac{\epsilon}{\sqrt{n}}.
\end{align*}
This implies $\abs{\widehat{\alpha}_{ij}} = \widehat{\alpha}_{ij}$, and $-\widehat{\alpha}_{ij} < \widehat{\alpha}_{hk}$. Because we assumed that $\abs{\widehat{\alpha}_{ij}} \leq \abs{\widehat{\alpha}_{hk}}$, we must have $\widehat{\alpha}_{ij} \leq \widehat{\alpha}_{hk}$. So $\widehat{\alpha}_{ij} - \widehat{\alpha}_{hk} \leq 0$.

We break $\widehat{\alpha}_{ij} - \widehat{\alpha}_{hk}$ into the sum of $\widehat{\alpha}_{ij} - \widehat{\alpha}_{ik}$ and $\widehat{\alpha}_{ik} - \widehat{\alpha}_{hk}$, and deal with them separately.

For $\widehat{\alpha}_{ij} - \widehat{\alpha}_{ik}$, we can obtain in a similar way as in (\ref{eq:gapbound1}),(\ref{eq:gapbound2}),(\ref{eq:gapbound3}),(\ref{eq:gapbound}) in Lemma~\ref{wrongedgecondition} (remembering our assumption that $\max \brac{\delta_{hi}, \delta_{jk}} \geq \frac{\epsilon^2}{n}$) that
\begin{gather*}
\widehat{\alpha}_{ij} - \widehat{\alpha}_{ik} = 2 \paren{\widehat{\mathrm{P}}(X_i = X_j, X_j = -X_k) - \widehat{\mathrm{P}}(X_i = -X_j, X_j = -X_k)}; \\
\abs{ \widehat{\mathrm{P}}(X_i = X_j, X_j = -X_k) - \underbrace{\mathrm{P}(X_i = X_j, X_j = -X_k)}_{=\frac{1+\alpha_{ij}}{2} \cdot \frac{\delta_{jk}}{2}} } \leq \frac{1}{10} \sqrt{\max \brac{\delta_{hi}, \delta_{jk}} \cdot \frac{\epsilon^2}{n}}; \\
\abs{ \widehat{\mathrm{P}}(X_i = -X_j, X_j = -X_k) - \underbrace{\mathrm{P}(X_i = -X_j, X_j = -X_k)}_{=\frac{1-\alpha_{ij}}{2} \cdot \frac{\delta_{jk}}{2}} } \leq \frac{1}{10} \sqrt{\max \brac{\delta_{hi}, \delta_{jk}} \cdot \frac{\epsilon^2}{n}}; \\
\mathrm{P}(X_i = X_j, X_j = -X_k) - \mathrm{P}(X_i = -X_j, X_j = -X_k) = \frac{\alpha_{ij} \delta_{jk}}{2},
\end{gather*}
leading to $\widehat{\alpha}_{ij} - \widehat{\alpha}_{ik} \geq \alpha_{ij} \delta_{jk} - \frac{2}{5} \sqrt{\max \brac{\delta_{hi}, \delta_{jk}} \cdot \frac{\epsilon^2}{n}}.$

For $\widehat{\alpha}_{ik} - \widehat{\alpha}_{hk}$, we can obtain in a similar way (but working with $k,i,h$ instead) that
\begin{gather*}
\widehat{\alpha}_{ik} - \widehat{\alpha}_{hk} = 2 \paren{\widehat{\mathrm{P}}(X_h = -X_i, X_i = X_k) - \widehat{\mathrm{P}}(X_h = -X_i, X_i = -X_k)}; \\
\abs{ \widehat{\mathrm{P}}(X_h= -X_i, X_i = X_k) - \underbrace{\mathrm{P}(X_h = -X_i, X_i = X_k)}_{=\frac{\delta_{hi}}{2} \cdot \frac{1+\alpha_{ij}(1-\delta_{jk})}{2}} } \leq \frac{1}{10} \sqrt{\max \brac{\delta_{hi}, \delta_{jk}} \cdot \frac{\epsilon^2}{n}}; \\
\abs{ \widehat{\mathrm{P}}(X_h = -X_i, X_i = -X_k) - \underbrace{\mathrm{P}(X_h = -X_i, X_i = -X_k)}_{=\frac{\delta_{hi}}{2} \cdot \frac{1-\alpha_{ij}(1-\delta_{jk})}{2}} } \leq \frac{1}{10} \sqrt{\max \brac{\delta_{hi}, \delta_{jk}} \cdot \frac{\epsilon^2}{n}}; \\
\mathrm{P}(X_h = -X_i, X_i = X_k) - \mathrm{P}(X_h = -X_i, X_i = -X_k) = \alpha_{ij}(1 - \delta_{jk}) \frac{\delta_{hi}}{2},
\end{gather*}
leading to $\widehat{\alpha}_{ik} - \widehat{\alpha}_{hk} \geq \alpha_{ij}(1 - \delta_{jk}) \delta_{hi} - \frac{2}{5} \sqrt{\max \brac{\delta_{hi}, \delta_{jk}} \cdot \frac{\epsilon^2}{n}}.$

Adding up the two bounds, we get
\begin{align*}
\widehat{\alpha}_{ij} - \widehat{\alpha}_{hk} &\geq \alpha_{ij}(\delta_{hi}+\delta_{jk}-\delta_{hi}\delta_{jk}) - \frac{4}{5} \sqrt{\max \brac{\delta_{hi}, \delta_{jk}} \cdot \frac{\epsilon^2}{n}} \\
		&\geq \alpha_{ij} \max \brac{\delta_{hi}, \delta_{jk}} - \frac{4}{5} \sqrt{\max \brac{\delta_{hi}, \delta_{jk}} \cdot \frac{\epsilon^2}{n}},
\end{align*}
which, together with $\widehat{\alpha}_{ij} - \widehat{\alpha}_{hk} \leq 0$, implies $\alpha_{ij}^2 \max \brac{\delta_{hi}, \delta_{jk}} \leq \frac{16}{25} \frac{\epsilon^2}{n} < \frac{\epsilon^2}{n}$.
\end{proof}

\subsubsection{Hierarchies and Structural Results for $T$ vs. $\widehat{T}$}
\label{sec:symstructure}

This subsection contains structural results on the relation between $T$ and $\widehat{T}$, stated in the language of various hierarchies to be defined first. The layers we define here dictate how we switch from $T$ to $\widehat{T}$ in the hybrid argument for bounding $\dtv{\mathrm{P}}{\mathrm{Q}}$ in the next subsection.

Recall that $\widehat{T}$ is a maximum weight spanning tree with respect to the $\abs{\widehat{\alpha}}$-estimates (Line~\ref{line:mst} of Algorithm~\ref{algo:symalgo}). We first classify all the edges of $\widehat{T}$ based on their $\abs{\widehat{\alpha}}$-estimates. This induces a hierarchical classification of nodes into groups, which in turn induces a classification of the edges of $T$.
\begin{definition}[\textbf{Hierarchies}]
\label{hierarchy}
We establish the following hierarchies, starting with the empty graph $G$ on $[n]$ and adding edges incrementally:
\begin{itemize}
\item An edge $(i,j)$ of $\widehat{T}$ with $\abs{\widehat{\alpha}_{ij}} \geq 1 - 10 \frac{\epsilon^2}{n}$ is classified as a $\pmb{\widehat{T}}$\textbf{-road}. Adding the $\widehat{T}$-roads to $G$ links the $n$ nodes into connected components, each of which is called a \textbf{city}. An edge $(i,j)$ of $T$ with $i,j$ belonging to the same city is classified as a $\pmb{T}$\textbf{-road}.
\item An edge $(i,j)$ of $\widehat{T}$ with $\frac{1}{2} \leq \abs{\widehat{\alpha}_{ij}} < 1 - 10 \frac{\epsilon^2}{n}$ is classified as a $\pmb{\widehat{T}}$\textbf{-highway}. Adding the $\widehat{T}$-highways to $G$ further links the cities into bigger connected components, each of which is called a \textbf{country}. An edge $(i,j)$ of $T$ with $i,j$ belonging to different cities in the same country is classified as a $\pmb{T}$\textbf{-highway}.
\item An edge $(i,j)$ of $\widehat{T}$ with $2 \frac{\epsilon}{\sqrt{n}} \leq \abs{\widehat{\alpha}_{ij}} < \frac{1}{2}$ is classified as a $\pmb{\widehat{T}}$\textbf{-railway}. Adding the $\widehat{T}$-railways to $G$ further links the countries into even bigger connected components, each of which is called a \textbf{continent}. An edge $(i,j)$ of $T$ with $i,j$ belonging to different countries in the same continent is classified as a $\pmb{T}$\textbf{-railway}.
\item An edge $(i,j)$ of $\widehat{T}$ with $\abs{\widehat{\alpha}_{ij}} < 2 \frac{\epsilon}{\sqrt{n}}$ is classified as a $\pmb{\widehat{T}}$\textbf{-airway}. Adding the $\widehat{T}$-airways to $G$ further links the continents into one single connected component of all $n$ nodes. An edge $(i,j)$ of $T$ with $i,j$ belonging to different continents is classified as a $\pmb{T}$\textbf{-airway}.
\end{itemize}
\end{definition}

Our first lemma shows that each $T$-road or $\widehat{T}$-road must be strong, in the sense that its $\abs{\alpha}$-value must be close to $1$. In other words, its two end-nodes are either almost always equal or almost always unequal under $\mathrm{P}$. This will enable us to prove that within a city it doesn't matter much in terms of $H^2$ whether the nodes are held together by the $T$-roads in the city or by the $\widehat{T}$-roads in the city.

\begin{lemma}
\label{roads}
If $(i,j)$ is a $T$-road or a $\widehat{T}$-road, then $\abs{\alpha_{ij}} \geq 1 - 11 \frac{\epsilon^2}{n}$.
\end{lemma}
\begin{proof}
Suppose $(i,j)$ is a $\widehat{T}$-road. We have $\abs{\widehat{\alpha}_{ij}} \geq 1 - 10 \frac{\epsilon^2}{n}$ according to Definition~\ref{hierarchy}, and so $\abs{\alpha_{ij}} \geq 1 - 11 \frac{\epsilon^2}{n}$ by Lemma~\ref{alphaprecision} (iii).

Suppose $(i,j)$ is a $T$-road. Imagine removing $(i,j)$ from $T$. This partitions $[n]$ into two $T$-connected subsets $U$ and $V$, with $i \in U$ and $j \in V$. By Definition~\ref{hierarchy}, $i,j$ belong to the same city, which is $\widehat{T}$-connected. So all the edges on the path in $\widehat{T}$ between $i$ and $j$ are $\widehat{T}$-roads. At least one of those $\widehat{T}$-roads goes between $U$ and $V$. Say $(h,k)$ does, with $h \in U$ and $k \in V$. Then, $h,i,j,k$ lie on a path in $T$, and we have $\alpha_{hk} = \alpha_{hi} \alpha_{ij} \alpha_{jk}$ by Lemma~\ref{multiplicativity}. We already proved that $\abs{\alpha_{hk}} \geq 1 - 11 \frac{\epsilon^2}{n}$. As a result, $\abs{\alpha_{ij}} \geq 1 - 11 \frac{\epsilon^2}{n}$.
\end{proof}

The following lemma can be seen as a partial converse to Lemma~\ref{roads}.

\begin{lemma}
\label{difcities}
If $i,j$ belong to different cities, then $\abs{\alpha_{ij}} < 1 - 9 \frac{\epsilon^2}{n}$.
\end{lemma}
\begin{proof}
Suppose for the sake of contradiction that $\abs{\alpha_{ij}} \geq 1 - 9 \frac{\epsilon^2}{n}$. By Lemma~\ref{alphaprecision} (ii), we have $\abs{\widehat{\alpha}_{ij}} \geq 1 - 10 \frac{\epsilon^2}{n}$. Consider the path in $\widehat{T}$ between $i$ and $j$. By the cycle property of maximum weight spanning trees, each edge on that path must have an $\abs{\widehat{\alpha}}$-estimate at least that of $\abs{\widehat{\alpha}_{ij}} \geq 1 - 10 \frac{\epsilon^2}{n}$, and is therefore classified as a $\widehat{T}$-road according to Definition~\ref{hierarchy}. Consequently, $i$ and $j$ must belong to the same city, a contradiction.
\end{proof}

Obviously, every city is $\widehat{T}$-connected by definition. We show that it is also $T$-connected. Intuitively, this means that we can regard every city as a ``supernode'' in the context of both $T$ and $\widehat{T}$.

\begin{lemma}
\label{cityconnected}
Every city is $T$-connected.
\end{lemma}
\begin{proof}
We want to show that for a city $C$ and $i,j \in C$, all the nodes on the path in $T$ between $i$ and $j$ are in $C$. Since any two nodes in the same city are linked by a series of $\widehat{T}$-roads, it suffices to argue only for the case where $(i,j)$ is a $\widehat{T}$-road.

Suppose the path in $T$ between $i$ and $j$ consists of $i = i_0, i_1, \ldots, i_r = j$, in that order. Then we have
\begin{equation}
\label{eq:multidentity}
\alpha_{ij} = \alpha_{i_0 i_1} \ldots \alpha_{i_{r-1} i_r}.
\end{equation}
Since $(i,j)$ is a $\widehat{T}$-road, we have $\abs{\alpha_{ij}} \geq 1 - 11 \frac{\epsilon^2}{n}$ by Lemma~\ref{roads}. Note that $1 - 11 \frac{\epsilon^2}{n} > \paren{1 - 9 \frac{\epsilon^2}{n}}^2$ holds whenever $\frac{\epsilon^2}{n} < \frac{7}{81}$, which is satisfied because we assumed $\epsilon \leq \frac{1}{10}$. As a result, (\ref{eq:multidentity}) implies that $\abs{\alpha_{i_s i_{s+1}}} \geq 1 - 9 \frac{\epsilon^2}{n}$ for all but at most one index $s$. By Lemma~\ref{difcities}, $i_s$ and $i_{s+1}$ belong to the same city, for all but at most one index $s$. Consequently, for any $s$, either $i_0, \ldots, i_s$ all belong to the same city, or $i_s, \ldots, i_r$ all belong to the same city. Either case leads to the desired conclusion that $i_s \in C$, because $i_0 = i$ and $i_r = j$ are both in $C$.
\end{proof}

Our next goal is to emulate Lemma~\ref{cityconnected} and prove that every country is $T$-connected (it is obviously $\widehat{T}$-connected by definition). In fact, we show something stronger in the next lemma, the proof of which also explains the choice of the lower threshold $\frac{1}{2}$ in the criterion for a $\widehat{T}$-highway.

\begin{lemma}
\label{highwayparallel}
There is a $\widehat{T}$-highway between a pair of cities if and only if there is a $T$-highway between the same pair of cities.
\end{lemma}
\begin{proof}

First, we prove that if there is a $\widehat{T}$-highway between a pair of cities, then there is a $T$-highway between the same pair of cities.

Suppose for the sake of contradiction that $(i,k)$ is a $\widehat{T}$-highway between cities $C$ and $D$, with $i \in C$ and $k \in D$, such that there is no $T$-highway between $C$ and $D$. Note that $C$ and $D$ are different cities in the same country, because they are linked together directly by the $\widehat{T}$-highway $(i,k)$. Consequently, there must exist no edge of $T$ at all between $C$ and $D$, because any such edge would have been classified as a $T$-highway according to Definition~\ref{hierarchy}.

As a result, the path in $T$ between $i$ and $k$ must contain some $j$ that belongs to neither $C$ nor $D$. Since $(i,k)$ is an edge of $\widehat{T}$, either $k$ sits on the path in $\widehat{T}$ between $i$ and $j$, or $i$ sits on the path in $\widehat{T}$ between $k$ and $j$. Without loss of generality we assume it is the former case. Then, by the cycle property of maximum weight spanning trees, we have $\abs{\widehat{\alpha}_{ij}} \leq \abs{\widehat{\alpha}_{ik}}$. Since $i,j,k$ lie on a path in $T$, Lemma~\ref{wrongedgecondition} implies that $\alpha_{ij}^2 (1-\abs{\alpha_{jk}}) \leq \frac{\epsilon^2}{n}$.

On the other hand, since $(i,k)$ is a $\widehat{T}$-highway, we have $\abs{\widehat{\alpha}_{ik}} \geq \frac{1}{2}$ according to Definition~\ref{hierarchy}, and so $\abs{\alpha_{ik}} \geq \frac{1}{2} - \frac{1}{5} \frac{\epsilon}{\sqrt{n}}$ by Lemma~\ref{alphaprecision} (i), which implies $\abs{\alpha_{ij}} \geq \frac{1}{2} - \frac{1}{5} \frac{\epsilon}{\sqrt{n}}$ since $\alpha_{ik} = \alpha_{ij} \alpha_{jk}$. Also, since $j,k$ belong to different cities, Lemma~\ref{difcities} implies $\abs{\alpha_{jk}} < 1 - 9 \frac{\epsilon^2}{n}$. As a result, $\alpha_{ij}^2 (1-\abs{\alpha_{jk}}) > \paren{\frac{1}{2} - \frac{1}{5} \frac{\epsilon}{\sqrt{n}}}^2 \cdot 9 \frac{\epsilon^2}{n} > \frac{\epsilon^2}{n}$ (recall that $\epsilon \leq \frac{1}{10}$), a contradiction.

Next, we prove that if there is a $T$-highway between a pair of cities, then there is a $\widehat{T}$-highway between the same pair of cities.

Suppose for the sake of contradiction that $(i,j)$ is a $T$-highway between cities $C$ and $D$, with $i \in C$ and $k \in D$, such that there is no $\widehat{T}$-highway between $C$ and $D$. According to Definition~\ref{hierarchy} the fact that $(i,j)$ is a $T$-highway implies that $C$ and $D$ are different cities in the same country. Since there is no $\widehat{T}$-highway between $C$ and $D$, according to Definition~\ref{hierarchy} there must exist a sequence of distinct cities $C = C_0, C_1, \ldots, C_r = D$, $r \geq 2$, such that there is a $\widehat{T}$-highway between $C_s$ and $C_{s+1}$ for each $s$. By the only if part of this Lemma that we already established, there is a $T$-highway between $C_s$ and $C_{s+1}$ for each $s$. Since each $C_s$ is $T$-connected by Lemma~\ref{cityconnected}, it is easy to see that there is a path in $T$ between $i$ and $j$ that visits the cities $C = C_0, C_1, \ldots, C_r = D$, in that order. On the other hand, $(i,j)$ by itself constitutes another path in $T$ between $i$ and $j$. So we have more than one paths in $T$ between $i$ and $j$, contradicting the fact that $T$ is a tree. 
\end{proof}

Now we can easily prove that every country is $T$-connected.

\begin{lemma}
\label{countryconnected}
Every country is $T$-connected.
\end{lemma}
\begin{proof}
Recall that a country consists of a collection of cities linked together by $\widehat{T}$-highways. Its $T$-connectedness then follows from Lemma~\ref{cityconnected} and Lemma~\ref{highwayparallel}.
\end{proof}

The following lemma is the country-level version of Lemma~\ref{difcities}.

\begin{lemma}
\label{difcountries}
If $i,j$ belong to different countries, then $\abs{\alpha_{ij}} < \frac{1}{2} + \frac{1}{5} \frac{\epsilon}{\sqrt{n}}$.
\end{lemma}
\begin{proof}
Suppose for the sake of contradiction that $\abs{\alpha_{ij}} \geq \frac{1}{2} + \frac{1}{5} \frac{\epsilon}{\sqrt{n}}$. By Lemma~\ref{alphaprecision} (i), we have $\abs{\widehat{\alpha}_{ij}} \geq \frac{1}{2}$. Consider the path in $\widehat{T}$ between $i$ and $j$. By the cycle property of maximum weight spanning trees, each edge on that path must have an $\abs{\widehat{\alpha}}$-estimate at least that of $\abs{\widehat{\alpha}_{ij}} \geq \frac{1}{2}$, and is therefore classified either as a $\widehat{T}$-road or as a $\widehat{T}$-highway, according to Definition~\ref{hierarchy}. In either case, the two end-nodes of that edge belong to the same country. Consequently, $i$ and $j$ must belong to the same country, a contradiction.
\end{proof}

The next three lemmas parallelize the three previous lemmas but are one layer lower in the hierarchy.

\begin{lemma}
\label{railwayparallel}
There is a $\widehat{T}$-railway between a pair of countries if and only if there is a $T$-railway between the same pair of countries.
\end{lemma}
\begin{proof}

First, we prove that if there is a $\widehat{T}$-railway between a pair of countries, then there is a $T$-railway between the same pair of countries.

Suppose for the sake of contradiction that $(i,k)$ is a $\widehat{T}$-railway between countries $\mathcal{C}$ and $\mathcal{D}$, with $i \in \mathcal{C}$ and $k \in \mathcal{D}$, such that there is no $T$-railway between $\mathcal{C}$ and $\mathcal{D}$. Note that $\mathcal{C}$ and $\mathcal{D}$ are different countries in the same continent, because they are linked together directly by the $\widehat{T}$-railway $(i,k)$. Consequently, there must exist no edge of $T$ at all between $\mathcal{C}$ and $\mathcal{D}$, because any such edge would have been classified as a $T$-railway according to Definition~\ref{hierarchy}. 

As a result, the path in $T$ between $i$ and $k$ must contain some $j$ that belongs to neither $\mathcal{C}$ nor $\mathcal{D}$. Since $(i,k)$ is an edge of $\widehat{T}$, either $k$ sits on the path in $\widehat{T}$ between $i$ and $j$, or $i$ sits on the path in $\widehat{T}$ between $k$ and $j$. Without loss of generality we assume it is the former case. Then, by the cycle property of maximum weight spanning trees, we have $\abs{\widehat{\alpha}_{ij}} \leq \abs{\widehat{\alpha}_{ik}}$. Since $i,j,k$ lie on a path in $T$, Lemma~\ref{wrongedgecondition} implies that $\alpha_{ij}^2 (1-\abs{\alpha_{jk}}) \leq \frac{\epsilon^2}{n}$.

On the other hand, since $(i,k)$ is a $\widehat{T}$-railway, we have $\abs{\widehat{\alpha}_{ik}} \geq 2 \frac{\epsilon}{\sqrt{n}}$ according to Definition~\ref{hierarchy}, and so $\abs{\alpha_{ik}} \geq \frac{9}{5} \frac{\epsilon}{\sqrt{n}}$ by Lemma~\ref{alphaprecision} (i), which implies $\abs{\alpha_{ij}} \geq \frac{9}{5} \frac{\epsilon}{\sqrt{n}}$ since $\alpha_{ik} = \alpha_{ij} \alpha_{jk}$. Also, since $j,k$ belong to different countries, Lemma~\ref{difcountries} implies $\abs{\alpha_{jk}} < \frac{1}{2} + \frac{1}{5} \frac{\epsilon}{\sqrt{n}}$. As a result, $\alpha_{ij}^2 (1-\abs{\alpha_{jk}}) > \paren{\frac{9}{5} \frac{\epsilon}{\sqrt{n}}}^2 \cdot \paren{\frac{1}{2} - \frac{1}{5} \frac{\epsilon}{\sqrt{n}}} > \frac{\epsilon^2}{n}$ (recall that $\epsilon \leq \frac{1}{10}$), a contradiction.

Next, we prove that if there is a $T$-railway between a pair of countries, then there is a $\widehat{T}$-railway between the same pair of countries.

Suppose for the sake of contradiction that $(i,j)$ is a $T$-railway between countries $\mathcal{C}$ and $\mathcal{D}$, with $i \in \mathcal{C}$ and $k \in \mathcal{D}$, such that there is no $\widehat{T}$-railway between $\mathcal{C}$ and $\mathcal{D}$. According to Definition~\ref{hierarchy} the fact that $(i,j)$ is a $T$-railway implies that $\mathcal{C}$ and $\mathcal{D}$ are different countries in the same continent. Since there is no $\widehat{T}$-railway between $\mathcal{C}$ and $\mathcal{D}$, according to Definition~\ref{hierarchy} there must exist a sequence of distinct countries $\mathcal{C} = \mathcal{C}_0, \mathcal{C}_1, \ldots, \mathcal{C}_r = \mathcal{D}$, $r \geq 2$, such that there is a $\widehat{T}$-railway between $\mathcal{C}_s$ and $\mathcal{C}_{s+1}$ for each $s$. By the only if part of this Lemma that we already established, there is a $T$-railway between $\mathcal{C}_s$ and $\mathcal{C}_{s+1}$ for each $s$. Since each $\mathcal{C}_s$ is $T$-connected by Lemma~\ref{countryconnected}, it is easy to see that there is a path in $T$ betwenn $i$ and $j$ that visits the countries $\mathcal{C} = \mathcal{C}_0, \mathcal{C}_1, \ldots, \mathcal{C}_r = \mathcal{D}$, in that order. On the other hand, $(i,j)$ by itself constitutes another path in $T$ between $i$ and $j$. So we have more than one paths in $T$ between $i$ and $j$, contradicting the fact that $T$ is a tree. 
\end{proof}

\begin{lemma}
\label{continentconnected}
Every continent is $T$-connected.
\end{lemma}
\begin{proof}
Recall that a continent consists of a collection of countries linked together by $\widehat{T}$-railways. Its $T$-connectedness then follows from Lemma~\ref{countryconnected} and Lemma~\ref{railwayparallel}.
\end{proof}

\begin{lemma}
\label{difcontinents}
If $i,j$ belong to different continents, then $\abs{\alpha_{ij}} < \frac{11}{5} \frac{\epsilon}{\sqrt{n}}$.
\end{lemma}
\begin{proof}
Suppose for the sake of contradiction that $\abs{\alpha_{ij}} \geq \frac{11}{5} \frac{\epsilon}{\sqrt{n}}$. By Lemma~\ref{alphaprecision} (i), we have $\abs{\widehat{\alpha}_{ij}} \geq 2 \frac{\epsilon}{\sqrt{n}}$. Consider the path in $\widehat{T}$ between $i$ and $j$. By the cycle property of maximum weight spanning trees, each edge on that path must have an $\abs{\widehat{\alpha}}$-estimate at least that of $\abs{\widehat{\alpha}_{ij}} \geq 2 \frac{\epsilon}{\sqrt{n}}$, and is therefore classified as a $\widehat{T}$-road, a $\widehat{T}$-highway, or a $\widehat{T}$-railway according to Definition~\ref{hierarchy}. In any case, the two end-nodes of that edge belong to the same continent. Consequently, $i$ and $j$ must belong to the same continent, a contradiction.
\end{proof}

For a summary of the structural results proved in this section, as well as a diagrammatic illustration, see the part that begins with ``Summary of layering'' in Section~\ref{sec:outline of symmetric case}.

\subsubsection{Bounding the Distance between $\mathrm{P}$ and $\mathrm{Q}$}
\label{sec:symbounding}

This subsection puts together everything we have so far and proves the desired bound in total variation distance between $\mathrm{P}$ and $\mathrm{Q}$.

As outlined in Section~\ref{sec:outline of symmetric case}, we need a few hybrid distributions to serve as intermediate steps in bounding the total variation distance between $\mathrm{P}$ and $\mathrm{Q}$. We first replace the edges of $T$ by the edges of $\widehat{T}$ one layer at a time from the bottom to the top, assigning $\alpha$-value $\alpha_{ij}$ to each edge $(i,j)$ of the resulting tree after each step to get the hybrid distributions $\mathrm{P}^{(1)}$, $\mathrm{P}^{(2)}$, $\mathrm{P}^{(3)}$, and $\mathrm{P}^{(4)}$. The step going from $\mathrm{P}^{(4)}$ to $\mathrm{Q}$ keeps the underlying tree fixed at $\widehat{T}$ but changes the $\alpha$-value of each edge $(i,j)$ of $\widehat{T}$ from $\alpha_{ij}$ to $\widehat{\alpha}_{ij}$. We now formally specify the underlying tree and the $\alpha$-values for the tree edges for all the distributions involved:
\begin{itemize}
\item $\mathrm{P}$: the underlying tree is $T$; the $\alpha$-value for each edge $(i,j)$ of $T$ is $\alpha_{ij}$;
\item $\mathrm{P}^{(1)}$: the underlying tree $T^{(1)}$ consists of all the $T$-roads, $T$-highways, $T$-railways, and $\widehat{T}$-airways; the $\alpha$-value for each edge $(i,j)$ of $T^{(1)}$ is $\alpha_{ij}$;
\item $\mathrm{P}^{(2)}$: the underlying tree $T^{(2)}$ consists of all the $T$-roads, $T$-highways, $\widehat{T}$-railways, and $\widehat{T}$-airways; the $\alpha$-value for each edge $(i,j)$ of $T^{(2)}$ is $\alpha_{ij}$;
\item $\mathrm{P}^{(3)}$: the underlying tree $T^{(3)}$ consists of all the $T$-roads, $\widehat{T}$-highways, $\widehat{T}$-railways, and $\widehat{T}$-airways; the $\alpha$-value for each edge $(i,j)$ of $T^{(3)}$ is $\alpha_{ij}$;
\item $\mathrm{P}^{(4)}$: the underlying tree is $\widehat{T}$; the $\alpha$-value for each edge $(i,j)$ of $\widehat{T}$ is $\alpha_{ij}$;
\item $\mathrm{Q}$: the underlying tree is $\widehat{T}$; the $\alpha$-value for each edge $(i,j)$ of $\widehat{T}$ is $\widehat{\alpha}_{ij}$.
\end{itemize}
We bound the Hellinger distance separately between each adjacent pair in the list above. The desired total variation distance bound between $\mathrm{P}$ and $\mathrm{Q}$ as stated in Theorem~\ref{sufficiency} then follows from the triangle inequality and the fact that $d_{\mathrm{TV}}$ is upper bounded by $\sqrt{2}$ times $H$ (Lemma~\ref{tvvshel}).

\subsubsection*{\underline{Bounding $H(\mathrm{P}, \mathrm{P}^{(1)})$}}

To bound the Hellinger distance between $\mathrm{P}$ and $\mathrm{P}^{(1)}$, we introduce two more binary symmetric tree-structured Bayesnets as intermediate steps:
\begin{itemize}
\item $\mathrm{P}^{(\frac{1}{3})}$: the underlying tree is $T$; the $\alpha$-value for each edge $(i,j)$ of $T$ is $\alpha_{ij}$, except if $(i,j)$ is a $T$-airway, in which case the $\alpha$-value is $0$;
\item $\mathrm{P}^{(\frac{2}{3})}$: the underlying tree is $T^{(1)}$; the $\alpha$-value for each edge $(i,j)$ of $T^{(1)}$ is $\alpha_{ij}$, except if $(i,j)$ is a $\widehat{T}$-airway, in which case the $\alpha$-value is $0$.
\end{itemize}

Note that $\mathrm{P}^{(\frac{1}{3})}$ and $\mathrm{P}^{(\frac{2}{3})}$ have the same marginal for each continent, and that different continents are mutually independent according to both. So $\mathrm{P}^{(\frac{1}{3})}$ and $\mathrm{P}^{(\frac{2}{3})}$ are in fact equal!

To bound the Hellinger distance between $\mathrm{P}$ and $\mathrm{P}^{(\frac{1}{3})}$, we can apply Corollary~\ref{edgeswitch} with $\mathrm{P}' = \mathrm{P}$, $\mathrm{P}'' = \mathrm{P}^{(\frac{1}{3})}$, $T' = T'' = T$, and the sets $A_{\lambda}$'s being the continents $\mathfrak{C}_1, \ldots, \mathfrak{C}_{\Lambda}$, to obtain
\begin{equation}
\label{01edgeswitch}
H^2(\mathrm{P},\mathrm{P}^{(\frac{1}{3})}) \leq \sum_{\lambda=1}^{\Lambda} H^2(\mathrm{P}_{\mathfrak{C}_\lambda}, \mathrm{P}^{(\frac{1}{3})}_{\mathfrak{C}_\lambda}) + \sum_{(\lambda, \mu) \in \mathcal{E}} H^2(\mathrm{P}_{W_{\lambda \mu}}, \mathrm{P}^{(\frac{1}{3})}_{W_{\lambda \mu}}),
\end{equation}
where $\mathcal{E}$ contains all $(\lambda, \mu)$ for which $\mathfrak{C}_{\lambda}$ and $\mathfrak{C}_{\mu}$ are straddled by a $T$-airway, and $W_{\lambda \mu}$ consists of the two end-nodes of that straddling edge.

The first summation on the right hand side of (\ref{01edgeswitch}) is zero.

For a term in the second summation corresponding to $(\lambda,\mu) \in \mathcal{E}$, let the $T$-airway straddling $\mathfrak{C}_{\lambda}$ and $\mathfrak{C}_{\mu}$ be $(i,j)$. So $W_{\lambda \mu} = \{i,j\}$. Note that $\mathrm{P}^{(\frac{1}{3})}_{ij}$ equals $\mathrm{P}^{\rm{(ind)}}_{ij}$ as defined in Lemma~\ref{helfromindep}. By Lemma~\ref{helfromindep} and Lemma~\ref{difcontinents} we thus have,
$$H^2 (\mathrm{P}_{W_{\lambda \mu}}, \mathrm{P}^{(\frac{1}{3})}_{W_{\lambda \mu}}) = H^2(\mathrm{P}_{ij}, \mathrm{P}^{(\frac{1}{3})}_{ij}) = H^2(\mathrm{P}_{ij}, \mathrm{P}^{\rm{(ind)}}_{ij}) \leq \frac{1}{2} \alpha_{ij}^2 < \frac{1}{2} \paren{\frac{11}{5} \frac{\epsilon}{\sqrt{n}}}^2 < 3 \frac{\epsilon^2}{n}.$$

Since $\abs{\mathcal{E}} < n$, (\ref{01edgeswitch}) implies $H^2(\mathrm{P}, \mathrm{P}^{(\frac{1}{3})}) \leq n \cdot 3 \frac{\epsilon^2}{n} = 3 \epsilon^2$, and so $H(\mathrm{P}, \mathrm{P}^{(\frac{1}{3})}) \leq \sqrt{3} \epsilon.$

In a similar fashion we can also get $H(\mathrm{P}^{(\frac{2}{3})}, \mathrm{P}^{(1)}) \leq \sqrt{3} \epsilon$. In this case we apply Corollary~\ref{edgeswitch} with $\mathrm{P}' = \mathrm{P}^{(\frac{2}{3})}$, $\mathrm{P}'' = \mathrm{P}^{(1)}$, $T' = T'' = T^{(1)}$, and the sets $A_{\lambda}$'s being the continents $\mathfrak{C}_1, \ldots, \mathfrak{C}_{\Lambda}$. The only possibly nonzero terms on the right hand side of the resulting inequality correspond to $(\lambda, \mu)$ for which $\mathfrak{C}_{\lambda}$ and $\mathfrak{C}_{\mu}$ are straddled by a $\widehat{T}$-airway.

Overall, we have $H(\mathrm{P}, \mathrm{P}^{(1)}) \leq H(\mathrm{P}, \mathrm{P}^{(\frac{1}{3})})+H(\mathrm{P}^{(\frac{2}{3})}, \mathrm{P}^{(1)}) \leq 2 \sqrt{3} \epsilon < 4 \epsilon.$

\subsubsection*{\underline{Bounding $H(\mathrm{P}^{(1)}, \mathrm{P}^{(2)})$}}

To bound the Hellinger distance between $\mathrm{P}^{(1)}$ and $\mathrm{P}^{(2)}$, we can apply Corollary~\ref{edgeswitch} with $\mathrm{P}' = \mathrm{P}^{(1)}$, $\mathrm{P}'' = \mathrm{P}^{(2)}$, $T' = T^{(1)}$, $T'' = T^{(2)}$, and the sets $A_{\lambda}$'s being the countries $\mathcal{C}_1, \ldots, \mathcal{C}_{\Lambda}$ to obtain
\begin{equation}
\label{12edgeswitch}
H^2(\mathrm{P}^{(1)},\mathrm{P}^{(2)}) \leq \sum_{\lambda=1}^{\Lambda} H^2(\mathrm{P}^{(1)}_{\mathcal{C}_\lambda}, \mathrm{P}^{(2)}_{\mathcal{C}_\lambda}) + \sum_{(\lambda, \mu) \in \mathcal{E}} H^2(\mathrm{P}^{(1)}_{W_{\lambda \mu}}, \mathrm{P}^{(2)}_{W_{\lambda \mu}}),
\end{equation}
where $\mathcal{E}$ contains all $(\lambda, \mu)$ for which $\mathcal{C}_{\lambda}$ and $\mathcal{C}_{\mu}$ are straddled by (1) a $\widehat{T}$-airway or (2) a $T$-railway and a $\widehat{T}$-railway, and $W_{\lambda \mu}$ consists of all end-nodes of those straddling edges.

The only possibly nonzero terms on the right hand side of (\ref{12edgeswitch}) are those in the second summation corresponding to $(\lambda,\mu)$ such that $\mathcal{C}_{\lambda}$ and $\mathcal{C}_{\mu}$ are straddled by a $T$-railway and a $\widehat{T}$-railway. For such a pair $(\lambda,\mu)$, let the $T$-railway straddling $\mathcal{C}_{\lambda}$ and $\mathcal{C}_{\mu}$ be $(i,j)$, with $i \in \mathcal{C}_{\lambda}$ and $j \in \mathcal{C}_{\mu}$, and the $\widehat{T}$-railway straddling $\mathcal{C}_{\lambda}$ and $\mathcal{C}_{\mu}$ be $(h,k)$, with $h \in \mathcal{C}_{\lambda}$ and $k \in \mathcal{C}_{\mu}$.

First suppose $h \neq i$ and $j \neq k$. In this case, $W_{\lambda \mu} = \{h,i,j,k\}$, and $h,i,j,k$ lie on a path in $T$ because every country is $T$-connected. Also, $i,h,k,j$ lie on a path in $\widehat{T}$ because every country is $\widehat{T}$-connected. So the cycle property of maximum weight spanning trees implies that $\abs{\widehat{\alpha}_{ij}} \leq \abs{\widehat{\alpha}_{hk}}$.  Note that $\mathrm{P}^{(1)}_{hijk}$ equals $\mathrm{P}_{hijk}$, and $\mathrm{P}^{(2)}_{hijk}$ equals $\mathrm{P}_{h\wideparen{-i \,~~ j-}k}$ (because the edges within each country haven't been replaced yet and are still edges of the true tree $T$). By Lemma~\ref{helwrongedge4} we thus have
$$H^2(\mathrm{P}^{(1)}_{W_{\lambda \mu}}, \mathrm{P}^{(2)}_{W_{\lambda \mu}}) = H^2 (\mathrm{P}^{(1)}_{hijk}, \mathrm{P}^{(2)}_{hijk}) = H^2(\mathrm{P}_{hijk}, \mathrm{P}_{h\wideparen{-i \,~~ j-}k}) \leq 8 \frac{\epsilon^2}{n}.$$

Next suppose $h = i$ and $j \neq k$ (the case $h \neq i$ and $j = k$ is symmetric). In this case, $W_{\lambda \mu} = \{i,j,k\}$, and $i,j,k$ lie on a path in $T$. It is easy to see that $i,k,j$ lie on a path in $\widehat{T}$. So the cycle property of maximum weight spanning trees implies that $\abs{\widehat{\alpha}_{ij}} \leq \abs{\widehat{\alpha}_{ik}}$. Note that $\mathrm{P}^{(1)}_{ijk}$ equals $\mathrm{P}_{ijk}$, and $\mathrm{P}^{(2)}_{ijk}$ equals $\mathrm{P}_{i\wideparen{ \,~~ j-}k}$. By Lemma~\ref{helwrongedge3} we thus have
$$H^2(\mathrm{P}^{(1)}_{W_{\lambda \mu}}, \mathrm{P}^{(2)}_{W_{\lambda \mu}}) = H^2 (\mathrm{P}^{(1)}_{ijk}, \mathrm{P}^{(2)}_{ijk}) = H^2(\mathrm{P}_{ijk}, \mathrm{P}_{i\wideparen{ \,~~ j-}k}) \leq 2 \frac{\epsilon^2}{n}.$$

Lastly, if $h = i$ and $j = k$, then $W_{\lambda \mu} = \{i,j\}$. It is easy to see that $H^2(\mathrm{P}^{(1)}_{W_{\lambda \mu}}, \mathrm{P}^{(2)}_{W_{\lambda \mu}}) = 0$.

Thus, in any case, we have $H^2(\mathrm{P}^{(1)}_{W_{\lambda \mu}}, \mathrm{P}^{(2)}_{W_{\lambda \mu}}) \leq 8 \frac{\epsilon^2}{n}$.

Since $\abs{\mathcal{E}} < n$, (\ref{12edgeswitch}) implies $H^2(\mathrm{P}^{(1)}, \mathrm{P}^{(2)}) \leq n \cdot 8 \frac{\epsilon^2}{n} = 8 \epsilon^2$, and so $H(\mathrm{P}^{(1)}, \mathrm{P}^{(2)}) < 3 \epsilon.$

\subsubsection*{\underline{Bounding $H(\mathrm{P}^{(2)}, \mathrm{P}^{(3)})$}}

This is similar to the bounding between $\mathrm{P}^{(1)}$ and $\mathrm{P}^{(2)}$, but centered at the city/highway level of the hierarchy instead of the country/railway level. We apply Corollary~\ref{edgeswitch} with $\mathrm{P}' = \mathrm{P}^{(2)}$, $\mathrm{P}'' = \mathrm{P}^{(3)}$, $T' = T^{(2)}$, $T'' = T^{(3)}$, and the sets $A_{\lambda}$'s being the cities $C_1, \ldots, C_{\Lambda}$. The only possibly nonzero terms on the right hand side of the resulting inequality correspond to $(\lambda, \mu)$ for which $C_\lambda$ and $C_\mu$ are straddled by a $T$-highway and a $\widehat{T}$-highway. An analysis parallel to that for the pair $\mathrm{P}^{(1)}$ and $\mathrm{P}^{(2)}$ leads to $H(\mathrm{P}^{(2)}, \mathrm{P}^{(3)}) \leq 3 \epsilon.$

\subsubsection*{\underline{Bounding $H(\mathrm{P}^{(3)}, \mathrm{P}^{(4)})$}}

To bound the Hellinger distance between $\mathrm{P}^{(3)}$ and $\mathrm{P}^{(4)}$, we can apply Corollary~\ref{edgeswitch} with $\mathrm{P}' = \mathrm{P}^{(3)}$, $\mathrm{P}'' = \mathrm{P}^{(4)}$, $T' = T^{(3)}$, $T'' = T^{(4)}$, and the sets $A_{\lambda}$'s being the cities $C_1, \ldots, C_{\Lambda}$ to obtain
\begin{equation}
\label{34edgeswitch}
H^2(\mathrm{P}^{(3)},\mathrm{P}^{(4)}) \leq \sum_{\lambda=1}^{\Lambda} H^2(\mathrm{P}^{(3)}_{C_\lambda}, \mathrm{P}^{(4)}_{C_\lambda}) + \sum_{(\lambda, \mu) \in \mathcal{E}} H^2(\mathrm{P}^{(3)}_{W_{\lambda \mu}}, \mathrm{P}^{(4)}_{W_{\lambda \mu}}),
\end{equation}
where $\mathcal{E}$ contains all $(\lambda, \mu)$ for which $C_\lambda$ and $C_\mu$ are straddled by a $\widehat{T}$-airway, a $\widehat{T}$-railway, or a $\widehat{T}$-highway, and $W_{\lambda \mu}$ consists of both end-nodes of that straddling edge.

The second summation on the right hand side of (\ref{34edgeswitch}) is zero.

We now focus on a term in the first summation corresponding to $\lambda$. Let $E_{C_\lambda}$ be the set of all $T$-roads in $C_\lambda$, and $\widehat{E}_{C_\lambda}$ be the set of all $\widehat{T}$-roads in $C_\lambda$.

To bound the Hellinger distance between $\mathrm{P}^{(3)}_{C_\lambda}$ and $\mathrm{P}^{(4)}_{C_\lambda}$, we introduce two more binary symmetric tree-structured Bayesnets on $C_\lambda$ as intermediate steps:
\begin{itemize}
\item $\mathrm{P}^{(3 \frac{1}{3})}_{C_\lambda}$: the underlying tree $T_{C_\lambda}$ consists of all the edges in $E_{C_\lambda}$; the $\alpha$-value for $(i,j) \in E_{C_\lambda}$ is $\sign{\alpha_{ij}}$;
\item $\mathrm{P}^{(3 \frac{2}{3})}_{C_\lambda}$: the underlying tree $\widehat{T}_{C_\lambda}$ consists of all the edges in $\widehat{E}_{C_\lambda}$; the $\alpha$-value for $(i,j) \in \widehat{E}_{C_\lambda}$ is $\sign{\alpha_{ij}}$,
\end{itemize}
where $\sign{t}$ equals $1$ if $t \geq 0$, and $-1$ otherwise.

We show $\mathrm{P}^{(3 \frac{1}{3})}_{C_\lambda}$ and $\mathrm{P}^{(3 \frac{2}{3})}_{C_\lambda}$ are actually equal! Fix any $u \in C_\lambda$. Define $x_{C_\lambda} = \paren{x_i}_{i \in C_\lambda} \in \brac{1,-1}^{C_\lambda}$ such that $x_u = 1$, and that $x_i = \sign{\alpha_{ij}} \cdot x_j$ for all $(i,j) \in E_{C_\lambda}$. It is easy to see that $\mathrm{P}^{(3 \frac{1}{3})}_{C_\lambda}$ places probability $\frac{1}{2}$ on each of $x_{C_\lambda}$ and $-x_{C_\lambda}$.

Also define $\widehat{x}_{C_\lambda} = \paren{\widehat{x}_i}_{i \in C_\lambda} \in \brac{1,-1}^{C_\lambda}$ such that $\widehat{x}_u = 1$, and that $\widehat{x}_i = \sign{\alpha_{ij}} \cdot \widehat{x}_j$ for all $(i,j) \in \widehat{E}_{C_\lambda}$. It is easy to see that $\mathrm{P}^{(3 \frac{2}{3})}_{C_\lambda}$ places probability $\frac{1}{2}$ on each of $\widehat{x}_{C_\lambda}$ and $-\widehat{x}_{C_\lambda}$.

For any $(i,j) \in \widehat{E}_{C_\lambda}$, let $i = i_0, \ldots, i_r = j$ be the path in $T$ between $i$ and $j$. Then we have by Lemma~\ref{multiplicativity} that
$$\frac{\widehat{x}_i}{\widehat{x}_j} = \sign{\alpha_{ij}} = \sign{\alpha_{i_0 i_1} \cdots \alpha_{i_{r-1} i_r}} =  \sign{\alpha_{i_0 i_1}} \cdots  \sign{\alpha_{i_{r-1} i_r}} = \frac{x_{i_0}}{x_{i_1}} \cdots \frac{x_{i_{r-1}}}{x_{i_r}} = \frac{x_i}{x_j},$$
where the multiplicativity of the $\sign{\cdot}$ function holds because none of the arguments is $0$ (which is the case because $\abs{\alpha_{ij}} \geq 1 - 11 \frac{\epsilon^2}{n} > 0$ by Lemma~\ref{roads} and the fact that $\epsilon \leq \frac{1}{10}$). Since the $\widehat{T}$-roads in $C_{\lambda}$ span $C_\lambda$, we have $x_{C_\lambda} = \widehat{x}_{C_\lambda}$, and so $\mathrm{P}^{(3 \frac{1}{3})}_{C_\lambda}$ and $\mathrm{P}^{(3 \frac{2}{3})}_{C_\lambda}$ are equal.

To bound the Hellinger distance between $\mathrm{P}^{(3)}_{C_\lambda}$ and $\mathrm{P}^{(3 \frac{1}{3})}_{C_\lambda}$, we can apply Corollary~\ref{edgeswitch} with $\mathrm{P}' = \mathrm{P}^{(3)}_{C_\lambda}$, $\mathrm{P}'' = \mathrm{P}^{(3 \frac{1}{3})}_{C_\lambda}$, $T' = T'' = T_{C_\lambda}$, and the sets $A_{\lambda}$'s being the single-node sets $\{i\}$, $i \in C_\lambda$, to obtain
\begin{equation}
\label{34edgeswitch1}
H^2(\mathrm{P}^{(3)}_{C_\lambda},\mathrm{P}^{(3 \frac{1}{3})}_{C_\lambda}) \leq \sum_{i \in C_\lambda} H^2(\mathrm{P}^{(3)}_i, \mathrm{P}^{(3 \frac{1}{3})}_i) + \sum_{(i,j) \in E_{C_\lambda}} H^2(\mathrm{P}^{(3)}_{ij}, \mathrm{P}^{(3 \frac{1}{3})}_{ij}).
\end{equation}

The first summation on the right hand side of (\ref{34edgeswitch1}) is zero.

For a term in the second summation corresponding to $(i,j)$, note that $\mathrm{P}^{(3)}_{ij}$ has $\alpha$-value $\alpha_{ij}$, and $\mathrm{P}^{(3 \frac{1}{3})}_{ij}$ has $\alpha$-value $\sign{\alpha_{ij}}$. By Lemma~\ref{helfromdet} and Lemma~\ref{roads} we thus have
$$H^2 (\mathrm{P}^{(3)}_{ij}, \mathrm{P}^{(3 \frac{1}{3})}_{ij}) \leq \frac{1}{2} \paren{1 - \abs{\alpha_{ij}}} \leq \frac{11}{2} \frac{\epsilon^2}{n}.$$

Since $\abs{E_{C_\lambda}} = \abs{C_\lambda}-1$, (\ref{34edgeswitch1}) implies $H^2(\mathrm{P}^{(3)}_{C_\lambda}, \mathrm{P}^{(3 \frac{1}{3})}_{C_\lambda}) \leq \abs{C_\lambda} \cdot \frac{11}{2} \frac{\epsilon^2}{n}$, and so we obtain $H(\mathrm{P}^{(3)}_{C_\lambda}, \mathrm{P}^{(3 \frac{1}{3})}_{C_\lambda}) \leq \sqrt{\abs{C_\lambda}} \cdot \sqrt{\frac{11}{2}} \frac{\epsilon}{\sqrt{n}}.$

In a similar fashion we can also get $H(\mathrm{P}^{(3 \frac{2}{3})}_{C_\lambda}, \mathrm{P}^{(4)}_{C_\lambda}) \leq \sqrt{\abs{C_\lambda}} \cdot \sqrt{\frac{11}{2}} \frac{\epsilon}{\sqrt{n}}$. In this case we apply Corollary~\ref{edgeswitch} with $\mathrm{P}' = \mathrm{P}^{(3 \frac{2}{3})}_{C_\lambda}$, $\mathrm{P}'' = \mathrm{P}^{(4)}_{C_\lambda}$, $T' = T'' = \widehat{T}_{C_\lambda}$, and the sets $A_{\lambda}$'s being the single-node sets $\{i\}$, $i \in C_\lambda$.

By the triangle inequality  $H(\mathrm{P}^{(3)}_{C_\lambda}, \mathrm{P}^{(4)}_{C_\lambda}) \leq 2 \sqrt{\abs{C_\lambda}} \cdot \sqrt{\frac{11}{2}} \frac{\epsilon}{\sqrt{n}}$, and so $H^2(\mathrm{P}^{(3)}_{C_\lambda}, \mathrm{P}^{(4)}_{C_\lambda}) \leq 4 \abs{C_\lambda} \cdot \frac{11}{2} \frac{\epsilon^2}{n}$. By (\ref{34edgeswitch}) we thus have $H^2(\mathrm{P}^{(3)}, \mathrm{P}^{(4)}) \leq \sum_{\lambda=1}^{\Lambda} 4 \abs{C_\lambda} \cdot \frac{11}{2} \frac{\epsilon^2}{n} = 22 \epsilon^2$. So $H(\mathrm{P}^{(3)}, \mathrm{P}^{(4)}) \leq 5 \epsilon$.

\subsubsection*{\underline{Bounding $H(\mathrm{P}^{(4)}, \mathrm{Q})$}}

To bound the Hellinger distance between $\mathrm{P}^{(4)}$ and $\mathrm{Q}$, we can apply Corollary~\ref{edgeswitch} with $\mathrm{P}' = \mathrm{P}^{(4)}$, $\mathrm{P}'' = \mathrm{Q}$, $T' = T'' = \widehat{T}$, and the sets $A_{\lambda}$'s being the single-node sets $\{1\}, \ldots, \{n\}$, to obtain
\begin{equation}
\label{4qedgeswitch}
H^2(\mathrm{P}^{(4)},\mathrm{Q}) \leq \sum_{i=1}^{n} H^2(\mathrm{P}^{(4)}_i, \mathrm{Q}_i) + \sum_{(i,j) \in E} H^2(\mathrm{P}^{(4)}_{ij}, \mathrm{Q}_{ij}),
\end{equation}
where $E$ is the set of edges of $\widehat{T}$.

The first summation on the right hand side of (\ref{4qedgeswitch}) is clearly zero.

For a term in the second summation corresponding to $(i,j) \in E$, note that $\mathrm{P}^{(4)}_{ij}$ equals $\mathrm{P}_{ij}$, and $\mathrm{Q}_{ij}$ equals $\mathrm{P}_{ij}^{\rm{(est)}}$ as defined in Lemma~\ref{helfromestimate}. By Lemma~\ref{helfromestimate} we thus have
$$H^2 (\mathrm{P}^{(4)}_{ij}, \mathrm{Q}_{ij}) = H^2(\mathrm{P}_{ij}, \mathrm{P}_{ij}^{\rm{(est)}}) \leq \frac{1}{10} \frac{\epsilon^2}{n}.$$

Since $\abs{E} < n$, (\ref{4qedgeswitch}) implies $H^2(\mathrm{P}^{(4)}, \mathrm{Q}) \leq n \cdot \frac{1}{10} \frac{\epsilon^2}{n} = \frac{1}{10} \epsilon^2$, and so $H(\mathrm{P}^{(4)}, \mathrm{Q}) < \epsilon.$

\begin{proof}[\textbf{\underline{Proof of Theorem~\ref{sufficiency}}}]

Combining the bounds between adjacent pairs from the list of distributions at the beginning of Section~\ref{sec:symbounding}, we have by the triangle inequality
\begin{align*}
H(\mathrm{P}, \mathrm{Q}) &\leq H(\mathrm{P}, \mathrm{P}^{(1)}) + H(\mathrm{P}^{(1)}, \mathrm{P}^{(2)}) + H(\mathrm{P}^{(2)}, \mathrm{P}^{(3)}) + H(\mathrm{P}^{(3)}, \mathrm{P}^{(4)}) + H(\mathrm{P}^{(4)}, \mathrm{Q}) \\
		& \leq 4 \epsilon + 3 \epsilon + 3 \epsilon + 5 \epsilon + \epsilon \\
		& = 16 \epsilon.
\end{align*}
Finally, we have $\dtv{\mathrm{P}}{\mathrm{Q}} \leq \sqrt{2} H(\mathrm{P}, \mathrm{Q}) < 23 \epsilon$ by Lemma~\ref{tvvshel}.
\end{proof}

\newpage

\section{The General Case}
\label{gen}

In this section, we prove our main result for proper learning general binary-alphabet tree-structured Bayesnets. In particular, we show that this is achieved by the Chow-Liu algorithm using the plug-in estimator for estimating all required mutual information quantities (or equivalently, by running the Chow-Liu algorithm on the empirical joint distribution induced by the samples). We prove our main result, which is stated as Theorem~\ref{thm:main} in the introduction and restated as Theorem~\ref{theoremgen} below. (Our claim about the sample optimality of the algorithm, up to a constant factor, which appears in the statement of Theorem~\ref{thm:main} and follows from \cite{Koehler20} as discussed in the introduction, will not be discussed in this section.)

Throughout Section~\ref{gen}, we reserve the symbol $\mathrm{P}$ to represent the unknown binary tree-structured Bayesnet on $X_1, \ldots, X_n$ that we wish to properly learn. The dimension of $\mathrm{P}$ is $n$. We reserve $T$ to represent the true (but unknown) underlying tree of $\mathrm{P}$.

The Chow-Liu Algorithm, using the plug-in estimator for estimating mutual information, is presented in the introduction as Algorithm~\ref{algo:genalgo2} and restated below in a more detailed form. See Section~\ref{sec:asymmetric} for a discussion on specifying a (binary-alphabet) tree-structured Bayesnet.

\setcounter{algocf}{0}
\begin{algorithm}
\DontPrintSemicolon
\caption{{\sc Finite Sample Chow-Liu Algorithm for Proper Learning of a Binary-Alphabet Tree-Structured Bayesnet}}
\label{algo:genalgo}
\KwIn{$\epsilon, \gamma \in (0, 1]$}
\KwOut{specification for a binary tree-structured Bayesnet $\mathrm{Q}$, approximating $\mathrm{P}$}
Draw $m = \biggl\lceil B \cdot \frac{n}{\epsilon^2} \cdot \paren{ \ln n + \ln \frac{1}{\gamma} } \biggr\rceil$ i.i.d samples $x^{(1)}, \ldots, x^{(m)}$ from $\mathrm{P}$. \;
Let $\widehat{\mathrm{P}}$ be the empirical joint distribution for $X_1, \ldots, X_n$ constructed from the $m$ samples, so that $\widehat{\mathrm{P}}(x) = \frac{1}{m} \sum_{t=1}^m \mathbbm{1}_{x^{(t)} = x}$, $\forall x \in {\{1, -1\}}^n$.\;
Compute $\mathrm{I}^{\widehat{\mathrm{P}}}(X_i;X_j) = \sum_{x_i,x_j = \pm 1} \widehat{\mathrm{P}}_{ij}(x_i,x_j) \ln{\frac{\widehat{\mathrm{P}}_{ij}(x_i,x_j)}{\widehat{\mathrm{P}}_i(x_i) \widehat{\mathrm{P}}_j(x_j)}} $, $\forall (i,j)$. \;
Run Kruskal's algorithm to find a maximum weight spanning tree $\widehat{T}$ of the complete graph on $[n]$ with weight $\mathrm{I}^{\widehat{\mathrm{P}}}(X_i;X_j)$ for $(i,j)$.\;
\Return{(1) $\widehat{T}$ as $\mathrm{Q}$'s underlying tree, (2) $\widehat{\mathrm{P}}_{ij}$ as $\mathrm{Q}$'s marginal for the pair $X_i,X_j$, for each edge $(i,j)$ of $\widehat{T}$. }\;
\end{algorithm}

Throughout this section, we will also reserve the symbols $\widehat{\mathrm{P}}$, $\widehat{T}$, and $\mathrm{Q}$ for the respective meanings they carry in Algorithm~\ref{algo:genalgo}. Our main result for Algorithm~\ref{algo:genalgo} is the following:

\begin{theorem}
\label{theoremgen}
There exists $B>0$ such that for any $\epsilon, \gamma \in (0,1]$, Algorithm~\ref{algo:genalgo} specifies a binary tree-structured Bayesnet $\mathrm{Q}$ such that $\dtv{\mathrm{P}}{\mathrm{Q}} \leq \epsilon$, with error probability at most $\gamma$.
\end{theorem}

In fact, we establish something stronger and perhaps surprising. We show that Algorithm~\ref{algo:genalgo} is successful as long as the empirical distribution $\widehat{\mathrm{P}}$ satisfies a \textit{deterministic criterion}, presented below.

\begin{definition}[\textbf{4-Consistency, Strong 4-Consistency}]
\label{gen4consistency}
For a given run of Algorithm~\ref{algo:genalgo}, $\widehat{\mathrm{P}}$ is said to satisfy \textbf{4-consistency} if for every $S \subseteq [n]$ of cardinality $4$ and every $W \subseteq {\{1, -1\}}^4$ we have
\begin{equation}
\label{eq:genconsistency}
\abs{\widehat{\mathrm{P}}(X_S \in W) - \mathrm{P}(X_S \in W)} \leq \frac{1}{10^{20}} \cdot \max \brac{ \sqrt{\mathrm{P}(X_S \in W) \cdot \frac{\epsilon^2}{n}}, \frac{\epsilon^2}{n} }.
\end{equation}
It is said to satisfy \textbf{strong 4-consistency} if furthermore whenever $\mathrm{P}(X_S \in W) < \frac{\epsilon^2}{n}$ we also have
\begin{equation}
\label{eq:gensmallprob}
\widehat{\mathrm{P}}(X_S \in W) \cdot \ln{\frac{\frac{\epsilon^2}{n}}{\mathrm{P}(X_S \in W)}} \leq \frac{\epsilon^2}{n}.
\end{equation}
\end{definition}

The criterion (\ref{eq:genconsistency}) of Definition~\ref{gen4consistency} specifies that $\widehat{\mathrm{P}}$ well-approximates the true distribution $\mathrm{P}$ for all events involving up to $4$ variables, and the criterion \eqref{eq:gensmallprob} is a strengthening of (\ref{eq:genconsistency}) for events (involving up to $4$ variables) that have small probabilities under $\mathrm{P}$. Some concrete examples of the event ``$X_S \in W$'' in Definition~\ref{gen4consistency} are ``$X_h = -1, X_i = X_j = X_k$'', ``$X_i =  -X_j, X_j  = X_k$'', and ``$X_h = X_i = 1$''. For example, ``$X_i =  -X_j, X_j = X_k$'' is expressed by taking $S = \{h,i,j,k\}$ (where $h$ is any index other than $i,j,k$) and $W = \{(1,1,-1,-1), (-1,1,-1,-1), (1,-1,1,1), (-1,-1,1,1)\}$. 

\smallskip We show that whenever the samples provided to Algorithm~\ref{algo:genalgo} induce $\widehat{\mathrm{P}}$ which satisfies strong 4-consistency, the algorithm succeeds (up to a constant), as per the following theorem.

\begin{theorem}
\label{gensufficiency}
For a run of Algorithm~\ref{algo:genalgo}, if $\widehat{\mathrm{P}}$ satisfies strong 4-consistency, then $\dtv{\mathrm{P}}{\mathrm{Q}} \leq O \paren{\epsilon}$.
\end{theorem}

Notice that there are less than $n^4$ choices for $S$, and $2^{16}$ choices for $W$ for each $S$. By Lemma~\ref{concentration}, Lemma~\ref{smallprob}, and the union bound, strong 4-consistency fails with probability at most $\gamma$ if $B$ is chosen large enough. Consequently, showing Theorem~\ref{gensufficiency} implies Theorem~\ref{theoremgen} (up to a constant). We will focus on proving Theorem~\ref{gensufficiency}.

\subsection{Outline of the Proof}
\label{sec:outline of general case}

In this section, we outline our proof for the general case. The proof that  the Chow-Liu algorithm (Algorithm~\ref{algo:genalgo}) succeeds in properly learning arbitrary binary tree-structured Bayesnets employs the same basic ideas --- namely, the squared Hellinger subadditivity and layering --- as our proof that the Chow-Liu variant (Algorithm~\ref{algo:symalgo}) succeeds in the symmetric case, given in Section~\ref{sym}. However, the general case involves a few new challenges and the argument is much more complicated. Our proof sketch in this section will be less detailed than the proof sketch in Section~\ref{sec:outline of symmetric case} for the symmetric case, and we will focus mainly on the new challenges encountered in the general case and the necessary modifications required, as compared to the symmetric case. All lemmas presented in this section are restated copies of lemmas with the same label stated and proved in Section~\ref{sec:genproof}, which includes the full proof for the general case.

\paragraph{The deterministic condition for the success of Chow-Liu.} Recall that when analyzing Algorithm~\ref{algo:symalgo} in the symmetric case we were able to identify a deterministic condition on the relationship between $\mathrm{P}$ and $\widehat{\mathrm{P}}$, called 3-consistency, that holds with probability at least $1-\gamma$ for a run of Algorithm~\ref{algo:symalgo} (for a big enough constant $B$) and that can be shown to imply the desired bound $\dtv{P}{Q} \leq O \paren{\epsilon}$. Roughly speaking, 3-consistency says that $\mathrm{P}$ and $\widehat{\mathrm{P}}$ assign similar probabilities to all events involving up to three variables. For the general case we need something a bit stronger. First, we need precision guarantees for events involving up to four variables. See (\ref{eq:genconsistency}) in Definition~\ref{gen4consistency}. We call it 4-consistency. But even that is not enough. Recall that Algorithm~\ref{algo:genalgo} uses the plug-in estimator to estimate mutual informations, based on which it runs Kruskal's Algorithm to infer the underlying tree. Naturally, we need to ensure that the estimated mutual informations are close to the actual mutual informations. However, each term in the expression for an estimated mutual information involves the appearances of estimated probabilities in the denominator (inside the logarithm), and so the value of such a term is particularly sensitive to the estimation error of a probability whose estimate appears in the denominator if this probability is small. Thus, we need an extra condition that places more stringent precision requirements on events involving up to $4$ variables that have small probabilities under $\mathrm{P}$. See (\ref{eq:gensmallprob}) in Definition~\ref{gen4consistency}. We call the combination of (\ref{eq:genconsistency}) and (\ref{eq:gensmallprob}) strong 4-consistency. It can be shown that strong 4-consistency holds with probability at least $1-\gamma$ for a run of Algorithm~\ref{algo:genalgo} (for a big enough constant $B$), and is sufficient to guarantee the success of the algorithm. Note that strong 4-consistency is still made up of conditions that each involves only a constant number (i.e. 4) of variables, so in some sense it is still ``local''. From now on we assume that $\widehat{\mathrm{P}}$ satisfies strong 4-consistency.

\paragraph{The core of the argument.} As in the symmetric case, the main challenge is to control the error as a result of structural mistakes (i.e. incorrectly learned edges), as opposed to ``parameter mistakes'' (i.e. placing the estimated pairwise distribution $\widehat{\mathrm{P}}_{ij}$ rather than the actual pairwise distribution $\mathrm{P}_{ij}$ on each edge $(i,j)$ of the tree $\widehat{T}$ of the output distribution $\mathrm{Q}$).  So we focus on bounding the distance between $\mathrm{P}$ on one hand, and the hybrid distribution with underlying tree $\widehat{T}$ but with the correct marginal $\mathrm{P}_{ij}$ assigned to each edge $(i,j)$ of $\widehat{T}$, on the other. After getting from $T$ to $\widehat{T}$ in the sense just described, the remaining error due to the ``parameter mistakes'' (that is, having $\mathrm{P}_{ij}$ versus having $\widehat{\mathrm{P}}_{ij}$ on each edge $(i,j)$ of $\widehat{T}$) can be dealt with easily via a single application of Corollary~\ref{edgeswitch}.\footnote{See the passage entitled ``Long hybrid arguments vs squared Hellinger subadditivity'' in Section~\ref{sec:outline of symmetric case} for an explanation on why it is important to use the squared Hellinger and its subadditivity in situations like this, even though our ultimate goal is to bound the error in total variation distance.}

As will be seen, we will go from $T$ to $\widehat{T}$ in a constant number of rounds\footnote{See the passage entitled ``The challenges of going deeper'' in Section~\ref{sec:outline of symmetric case} for a discussion (applicable also to the general case here) of why it is important that the number of rounds is constant.} where in each round we either ``cut'' a set of edges, ``install'' a set of edges, or replace a set of edges by another, resulting in a sequence of forests interpolating between $T$ and $\widehat{T}$. For each forest in the sequence, we define a hybrid distribution that is a Bayesnet with that forest as underlying graph and with $\mathrm{P}_{ij}$ assigned to each edge $(i,j)$ of the forest.\footnote{More specifically, it is the distribution under which the set of variables corresponding to each connected component of the forest are distributed according to a tree-structured Bayesnet with underlying tree being that connected component, and with $\mathrm{P}_{ij}$ assigned to each edge $(i,j)$ of that connected component; furthermore, the sets of variables corresponding to different connected components of the forest are mutually independent. Note that, in Section~\ref{sec:preliminariesss} we did not formally define forest-structured Bayesnets, as these can easily be seen to be special cases of tree-structured Bayesnets, and our formal proof later we will do everything in terms of tree-structured Bayesnets. We speak of forest-structured Bayesnets here only, for the sake of intuition.} This yields a sequence of hybrid distributions interpolating between $\mathrm{P}$ on one hand, and the hybrid distribution with underlying tree $\widehat{T}$ and with $\mathrm{P}_{ij}$ assigned to each edge $(i,j)$ of $\widehat{T}$, on the other. We will be able to prove that the total variation distance between any adjacent pair in this sequence of hybrid distributions is small, and the tool we will rely on to establish every such bound is, unsurprisingly, Corollary~\ref{edgeswitch}. The edges we ``cut'', ``install'', or swap at each of the rounds will be carefully choosen based on classifications of the edges of $T$ and $\widehat{T}$ into layers, in order to ensure that Corollary~\ref{edgeswitch} is indeed applicable and that the error introduced is indeed small. This is all quite similar to the symmetric case. The rest of our discussion focuses on the realization of the program  outlined in this paragraph.

As in the symmetric case, in order to obtain some guidance on how to classify the edges and define the layers (which allows us to speak more precisely about the manners in which $\widehat{T}$ might differ from $T$, and which dictates the edges we ``cut'', ``install'', or swap in going from $T$ to $\widehat{T}$), as well as to bound the error in $H^2$ incurred by swapping edges, we need to answer (1) when can Algorithm~\ref{algo:genalgo} possibly make a structural mistake and (2) what is the error in $H^2$ if Algorithm~\ref{algo:genalgo} does make a structural mistake, for the modest case of $i,j,k$ lying on a path in $T$ (assuming strong 4-consistency). In this case, a structural mistake is made when Algorthm~\ref{algo:genalgo} prioritizes $(i,k)$ over $(i,j)$ or prioritizes $(i,k)$ over $(j,k)$. Let's focus on it prioritizing $(i,k)$ over $(i,j)$ (as the other case is symmetric), which may happen when $\widehat{\mathrm{P}}$ is such that $\mathrm{I}^{\widehat{\mathrm{P}}}(X_i;X_j) \leq \mathrm{I}^{\widehat{\mathrm{P}}}(X_i;X_k)$ (note that $\mathrm{I}^{\mathrm{P}}(X_i;X_j) \geq \mathrm{I}^{\mathrm{P}}(X_i;X_k)$ by the data processing inequality).

Recall that for the symmetric case the first question above is answered by Lemma~\ref{wrongedgecondition}, which provides a necessary condition on $\alpha_{ij}$ and $\alpha_{jk}$ (i.e. $\alpha_{ij}^2 (1-\abs{\alpha_{jk}}) \leq \frac{\epsilon^2}{n}$) for Algorithm~\ref{algo:symalgo} to make the structural mistake. Combined with Lemma~\ref{heltwochains} (which says $H^2(\mathrm{P}_{ijk}, \mathrm{P}_{i\wideparen{ \,~~ j-}k}) \leq 2 \alpha_{ij}^2 (1-\abs{\alpha_{jk}} )$), it yields a bound on the error in $H^2$ when Algorithm~\ref{algo:symalgo} does make the structural mistake, answering the second question above. See Lemma~\ref{helwrongedge3}, as well as the four-node version Lemma~\ref{helwrongedge4}.

For analyzing Algorithm~\ref{algo:genalgo} for the general case the following lemma is central to answering the questions above. The proof of this lemma is the only place in the entire argument where we need those implications of strong 4-consistency that are beyond those implied by just 3-consistency.

\newtheorem*{lem:genhelwrongedge}{Lemma \ref{genhelwrongedge}}
\begin{lem:genhelwrongedge}
\begin{itemize}
\item[(i)] If $i,j,k$ lie on a path in $T$, and $\mathrm{I}^{\widehat{\mathrm{P}}}(X_i;X_j) \leq \mathrm{I}^{\widehat{\mathrm{P}}}(X_i;X_k)$, then $H^2(\mathrm{P}_{ijk}, \mathrm{P}_{i\wideparen{ \,~~ j-}k}) \leq 22 \frac{\epsilon^2}{n}$.
\item[(ii)] If $h,i,j,k$ lie on a path in $T$, and $\mathrm{I}^{\widehat{\mathrm{P}}}(X_i;X_j) \leq \mathrm{I}^{\widehat{\mathrm{P}}}(X_h;X_k)$, then $H^2(\mathrm{P}_{ijk}, \mathrm{P}_{i\wideparen{ \,~~ j-}k}) \leq 62 \frac{\epsilon^2}{n}$, $H^2(\mathrm{P}_{hik}, \mathrm{P}_{h\wideparen{-i \,~~ }k}) \leq 62 \frac{\epsilon^2}{n}$, $H^2(\mathrm{P}_{hijk}, \mathrm{P}_{h\wideparen{-i \,~~ j-}k}) \leq 248 \frac{\epsilon^2}{n}$ (symmetrically, $H^2(\mathrm{P}_{hij}, \mathrm{P}_{h\wideparen{-i \,~~ }j}) \leq 62 \frac{\epsilon^2}{n}$ and $H^2(\mathrm{P}_{hjk}, \mathrm{P}_{h\wideparen{ \,~~ j-}k}) \leq 62 \frac{\epsilon^2}{n}$).
\end{itemize}
\end{lem:genhelwrongedge}
\begin{proof}
(sketch of (i) showing only the main steps and without paying attention to the specific constants). We expand $\mathrm{I}^{\widehat{\mathrm{P}}}(X_i;X_j,X_k)$ using the chain rule for mutual information in two different ways:
$$\mathrm{I}^{\widehat{\mathrm{P}}}(X_i;X_j,X_k) = \mathrm{I}^{\widehat{\mathrm{P}}}(X_i;X_j) + \mathrm{I}^{\widehat{\mathrm{P}}}(X_i;X_k|X_j),$$
and
$$\mathrm{I}^{\widehat{\mathrm{P}}}(X_i;X_j,X_k) = \mathrm{I}^{\widehat{\mathrm{P}}}(X_i;X_k) + \mathrm{I}^{\widehat{\mathrm{P}}}(X_i;X_j|X_k).$$
Comparing the two we see that $\mathrm{I}^{\widehat{\mathrm{P}}}(X_i;X_j) \leq \mathrm{I}^{\widehat{\mathrm{P}}}(X_i;X_k)$ implies $\mathrm{I}^{\widehat{\mathrm{P}}}(X_i;X_k|X_j) \geq \mathrm{I}^{\widehat{\mathrm{P}}}(X_i;X_j|X_k)$.

Notice that $i,j,k$ lying on a path in $T$ implies that $\mathrm{I}^{\mathrm{P}}(X_i;X_k|X_j) = 0$. Using strong 4-consistency we can show that $\mathrm{I}^{\widehat{\mathrm{P}}}(X_i;X_k|X_j)$ is also close to $0$. More precisely, we can show that $\mathrm{I}^{\widehat{\mathrm{P}}}(X_i;X_k|X_j) \leq O{\paren{\frac{\epsilon^2}{n}}}$. Thus we have $\mathrm{I}^{\widehat{\mathrm{P}}}(X_i;X_j|X_k) \leq O{\paren{\frac{\epsilon^2}{n}}}$.

We can show that $H^2(\widehat{\mathrm{P}}_{ijk}, \widehat{\mathrm{P}}_{i\wideparen{ \,~~ j-}k}) \leq \frac{1}{2} \mathrm{I}^{\widehat{\mathrm{P}}}(X_i;X_j|X_k) < O{\paren{\frac{\epsilon^2}{n}}}$ (the first inequality is in fact generally true with any joint distribution taking the place of $\widehat{\mathrm{P}}$, and any $i,j,k$ not necessarily lying on a path in $T$). So $H^2(\widehat{\mathrm{P}}_{ijk}, \widehat{\mathrm{P}}_{i\wideparen{ \,~~ j-}k}) \leq O{\paren{\frac{\epsilon^2}{n}}}$. Using strong 4-consistency we can also show that $H^2(\mathrm{P}_{ijk}, \widehat{\mathrm{P}}_{ijk}) \leq O{\paren{\frac{\epsilon^2}{n}}}$ and $H^2(\widehat{\mathrm{P}}_{i\wideparen{ \,~~ j-}k}, \mathrm{P}_{i\wideparen{ \,~~ j-}k}) \leq O{\paren{\frac{\epsilon^2}{n}}}$ (in fact just 3-consistency is enough here). Taking square roots and applying triangle inequality for Hellinger distance, we get $H^2(\mathrm{P}_{ijk}, \mathrm{P}_{i\wideparen{ \,~~ j-}k}) \leq O{\paren{\frac{\epsilon^2}{n}}}$.

A full proof of this lemma can be found in Section~\ref{sec:genlemmas}.
\end{proof}

Notice that Lemma~\ref{genhelwrongedge} (i) fully answers the second question above, i.e.~``what is the error in $H^2$ if Algorithm~\ref{algo:genalgo}  makes a structural mistake?'' The answer is $22 \frac{\epsilon^2}{n}$. It also {\em partially} answers the first question, i.e.~``when can Algorithm~\ref{algo:genalgo} possibly make a structural mistake?'' The answer is that, if $H^2(\mathrm{P}_{ijk}, \mathrm{P}_{i\wideparen{ \,~~ j-}k}) > 22 \frac{\epsilon^2}{n}$, then Algorithm~\ref{algo:genalgo} cannot possibly prioritize $(i,k)$ over $(i,j)$. We only say ``partially'' because we would like the answer to be stated in terms of conditions on {\em pairwise} distributions $\mathrm{P}_{ij}$ and $\mathrm{P}_{jk}$ rather than three-wise distributions, as ultimately we want to use those conditions as guidance to classify the individual edges of $\widehat{T}$ into layers. (Indeed, recall that this is what we did in the symmetric case, by utilizing Lemma~\ref{wrongedgecondition}, which provided a necessary condition for a structural mistake involving only $\alpha_{ij}$ and $\alpha_{jk}$.) Thus, we want to find conditions involving only $\mathrm{P}_{ij}$ and $\mathrm{P}_{jk}$ that would guarantee $H^2(\mathrm{P}_{ijk}, \mathrm{P}_{i\wideparen{ \,~~ j-}k}) > 22 \frac{\epsilon^2}{n}$. To state these conditions, we define a few useful measures on (binary) pairwise distributions. It is easier to motivate the definitions of our measures in parallel with looking for reasonable ways to define the layers via a classification of the edges of $\widehat{T}$. We proceed to do this next.

\paragraph{The problem of using (estimated) mutual information to classify the edges of $\widehat{T}$.} Recall that, in the symmetric case, the edges of $\widehat{T}$ are classified into layers based on their $\abs{\widehat{\alpha}}$-estimates. Indeed, the $\abs{\widehat{\alpha}}$-estimates serve both as the edge-weights for computing the maximum weight spanning tree $\widehat{T}$ in Algorithm~\ref{algo:symalgo}, and as the measure according to which we classify the edges of $\widehat{T}$. The situation is different in the general case: while the estimated mutual informations are used by Algorithm~\ref{algo:genalgo} (the Chow-Liu algorithm) as edge-weights in computing the maximum weight spanning tree $\widehat{T}$, they cannot be used for classifying the edges of~$\widehat{T}$.

To see why, note that, to mimic what we did in the symmetric case, in the top layer of our layering we want to group edges of $\widehat{T}$ whose two end-nodes are almost always equal or almost always unequal (under $\mathrm{P}$), as doing so will allow us to prove that within each group of the top layer it doesn't matter much in terms of $H^2$ whether the nodes are held together by edges of $T$ or by edges of $\widehat{T}$ (so we can swap between the two configurations without incurring too much error). More precisely, we want to identify edges whose two end-nodes are equal with probability at least $1 - \Theta{\paren{\frac{\epsilon^2}{n}}}$, or unequal with probability at least $1 - \Theta{\paren{\frac{\epsilon^2}{n}}}$. However, unlike the symmetric case in which a pair of nodes are almost always equal or almost always unequal if and only if the $\abs{\alpha}$-value (and therefore the $\abs{\widehat{\alpha}}$-estimate) for the pair is close to one, for the general case a pair of nodes that are almost always equal or almost always unequal could have (estimated) mutual information that lies anywhere between $0$ and $\ln{2}$, depending on the degree of bias of the individual node marginals. Indeed, if $X_i$ and $X_j$ are always equal, and each assumes $1$ with probability $p$, then as $p$ goes from $0$ to $\frac{1}{2}$ the mutual information between $X_i$ and $X_j$ goes from $0$ to $\ln{2}$. Thus, if we set the estimated mutual information threshold for the top layer of edges close to the higher end (say at some $\ln{2} - \Theta{\paren{\frac{\epsilon^2}{n}}}$ threshold), then we would miss edges whose end-nodes are slightly biased but are nevertheless almost always equal or almost always unequal to each other. On the other hand, if we set the estimated mutual information threshold for the top layer at any value bounded away from the higher end, then we may include, for example, edges whose two end-nodes are (close to) uniform but are nevertheless neither almost always equal nor almost always unequal, in the sense we described above. 

The above discussion serves to motivate the use of different measures than mutual information for classifying the edges of $\widehat{T}$. Those new measures --- $\pmb{\mathrm{minmrg}}$, $\pmb{\mathrm{mindiag}}$, $\pmb{\mathrm{mindisc}}$, and $\pmb{I_{H^2}}$ --- are defined in Section~\ref{sec:genonly} (see Definition~\ref{minmrg}, Definition~\ref{mindiag}, Definition~\ref{mindisc}, and Definition~\ref{ih}).

\paragraph{The top layer.}  For the top layer we use a combination of $\mathrm{minmrg}$ and $\mathrm{mindiag}$, and classify an edge $(i,j)$ of $\widehat{T}$ as a $\pmb{\widehat{T}}$\textbf{-avenue} (we will explain later why we don't call it a ``road'' reflecting our terminology for the symmetric case) if $\minmrg{\widehat{\mathrm{P}}_{ij}} \geq 10^6 \frac{\epsilon^2}{n}$ and $\mindiag{\widehat{\mathrm{P}}_{ij}} \leq 10^5 \frac{\epsilon^2}{n}$, \footnote{Again, we could have let $T$ dictate the layering, or based our classification on $\mathrm{P}$ instead of $\widehat{\mathrm{P}}$, or both. See footnote~\ref{footnote:classification choice}  in Section~\ref{sec:outline of symmetric case} for a discussion of our choice.} where
$$\minmrg{\widehat{\mathrm{P}}_{ij}} = \min \brac{\widehat{\mathrm{P}}_i(1), \widehat{\mathrm{P}}_i(-1), \widehat{\mathrm{P}}_j(1), \widehat{\mathrm{P}}_j(-1)},$$
and
$$\mindiag{\widehat{\mathrm{P}}_{ij}} = \min \brac{\widehat{\mathrm{P}}_{ij} (X_i = X_j), \widehat{\mathrm{P}}_{ij} (X_i = -X_j)}.$$
The $\widehat{T}$-avenues cluster the $n$ nodes into connected components, each of which is called a \textbf{broken-city} (we will explain the terminology later). 

\paragraph{Striving to make the broken-cities $T$-connected.} The condition $\mindiag{\widehat{\mathrm{P}}_{ij}} \leq 10^5 \frac{\epsilon^2}{n}$ obviously says that $X_i$ and $X_j$ are estimated to be almost always equal or almost always unequal. But why do we need the condition $\minmrg{\widehat{\mathrm{P}}_{ij}} \geq 10^6 \frac{\epsilon^2}{n}$? Recall that in the symmetric case each city is shown to be not only $\widehat{T}$-connected, but also $T$-connected, which allows us to apply Corollary~\ref{edgeswitch} with $A_1, \ldots, A_{\Lambda}$ being the cities in order to bound $H(\mathrm{P}^{(2)},\mathrm{P}^{(3)})$ (the pair of distributions in our final hybrid argument related by the replacement of $T$-highways by the $\widehat{T}$-highways; see Section~\ref{sec:symbounding}). We hope the same is true in the general case: if $i$ and $j$ belong to the same broken-city, we would hope that all the nodes on the path in $T$ between $i$ and $j$ also belong to that broken-city. If we didn't impose the condition $\minmrg{\widehat{\mathrm{P}}_{ij}} \geq 10^6 \frac{\epsilon^2}{n}$ in our definition of $\widehat T$-avenues, then we can't exclude, for example, that $X_i$ and $X_j$ are both always equal to $1$. In this extreme case, the distributions for the intermediate nodes and edges on the path in $T$ between $i$ and $j$ can be arbitrary, and it would be impossible to prove that each intermediate node also belongs to the broken-city containing $i$ and $j$. The main issue here is that even though $X_i$ and $X_j$ are always equal, they are also independent of each other! This phenomenon can only occur when the individual marginals are allowed to be very biased, which can happen in the general case, but not in the symmetric case. In general, to avoid this kind of situations, we need to prevent including in the top layer edges of $\widehat{T}$ whose end-nodes are almost always equal or almost always unequal but are also almost independent. To do that, we have to exclude edges of $\widehat{T}$ whose end-nodes are ``too biased''. The condition $\minmrg{\widehat{\mathrm{P}}_{ij}} \geq 10^6 \frac{\epsilon^2}{n}$ does exactly that.

With a lower bound on $\mathrm{minmrg}$ included in the requirement for a $\widehat{T}$-avenue, the following lemma becomes applicable, as we will see shortly.

\newtheorem*{lem:mindiagmarkov}{Lemma \ref{mindiagmarkov}}
\begin{lem:mindiagmarkov}
If $i,j,k$ lie on a path in $T$, and $\minmrg{\mathrm{P}_{ik}} > 8 \cdot \mindiag{\mathrm{P}_{ik}}$, then
\begin{itemize}
\item[(i)] $\minmrg{\mathrm{P}_j} \geq \minmrg{\mathrm{P}_{ik}} - \mindiag{\mathrm{P}_{ik}}$;
\item[(ii)] $\max{\brac{\mindiag{\mathrm{P}_{ij}}, \mindiag{\mathrm{P}_{jk}}}} \leq \mindiag{\mathrm{P}_{ik}}$;
\item[(iii)] $\mindiag{\mathrm{P}_{ij}} + \mindiag{\mathrm{P}_{jk}} \leq \frac{7}{6} \mindiag{\mathrm{P}_{ik}}$.
\end{itemize}
\end{lem:mindiagmarkov}

Lemma~\ref{mindiagmarkov} is in some sense an analog of the multiplicativity of $\alpha$-values along a path in $T$ (Lemma~\ref{multiplicativity}), for the measure $\mathrm{mindiag}$ (recall that the $\alpha$-values provide a measure on pairwise distributions for the symmetric case). Roughly speaking, Lemma~\ref{mindiagmarkov} (ii) says that closer pairs along a path in $T$ are stronger as measured by $\mathrm{mindiag}$ (for $\alpha$-values for the symmetric case the analogous relationship is $\min \brac{ \abs{\alpha_{ij}}, \abs{\alpha_{jk}}} \geq \abs{\alpha_{ik}}$), and (iii) controls the manner in which the strengths of the pairs (as measured by $\mathrm{mindiag}$) concatenate/decompose along a path in $T$.

Note that Lemma~\ref{mindiagmarkov} (iii) implies that if $(i,k)$ is a $\widehat{T}$-avenue and $j$ is on the path in $T$ between $i$ and $k$ then at least one of $\mindiag{\mathrm{P}_{ij}}$ and $\mindiag{\mathrm{P}_{jk}}$ is \textit{significantly} less than $\mindiag{\mathrm{P}_{ik}}$ (it is important that we have ``significantly less than'' instead of just ``less than'' here because we need to allow space for estimation error). Say it is $\mindiag{\mathrm{P}_{ij}}$. Then following some quite involved argument (the main complication being that $(i,j)$ might not be an edge of $\widehat{T}$ so we can't simply deduce that $i$ and $j$ belong to the same broken-city even if $\widehat{\mathrm{P}}_{ij}$ satisfies the criteria for a $\widehat{T}$-avenue; instead, we have to reason about the edges of $\widehat{T}$ on the path in $\widehat{T}$ between $i$ and $j$; see Lemma~\ref{samebrokencity}, Definition~\ref{biasednode}, and Lemma~\ref{onlybiased}), we can arrive at the desired conclusion that $j$ belongs to the same broken-city as $i$ does, as long as $\minmrg{\mathrm{P}_j} \geq 10^7 \frac{\epsilon^2}{n}$. We call a node $j$ \textbf{biased} if $\minmrg{\mathrm{P}_j} < 10^7 \frac{\epsilon^2}{n}$. Consequently, every unbiased node on the path in $T$ between the two end-nodes of some $\widehat{T}$-avenue must belong to the same broken-city as those two end-nodes. In reality, the constant $10^7$ in the threshold for biased nodes can be replaced by any number strictly greater than $10^6$ (recall $10^6$ is the constant in the $\mathrm{minmrg}$ threshold for $\widehat{T}$-avenues) as long as we make $B$ (in the sample complexity of the algorithm) big enough.

\paragraph{Failure to make the broken-cities $T$-connected.} Unfortunately, the above argumentation still doesn't yield that every broken-city is $T$-connected due to potential ``holes'' created by biased nodes (a ``hole'' $j$ for a broken-city is a node which doesn't belong to the broken-city, but which lies on the path in $T$ between some $i$ and $k$ belonging to the broken-city). Indeed, our analysis is ultimately not going to guarantee that every broken-city is necessarily $T$-connected. Whatever we set the $\mathrm{minmrg}$ threshold for $\widehat{T}$-avenues to be, the $\mathrm{minmrg}$ threshold for a node to be considered unbiased has to be strictly greater than that in order for us to be able to prove that unbiased nodes won't create any ``holes.'' On the other hand, if $i,j,k$ lie on a path in $T$, we can't eliminate the possibility that $\minmrg{\mathrm{P}_j}$ is strictly less than both $\minmrg{\mathrm{P}_i}$ and $\minmrg{\mathrm{P}_k}$ (or their $\widehat{\mathrm{P}}$ counterparts) (see Lemma~\ref{mindiagmarkov} (i) for the best lower bound we can get). Consequently, it could potentially happen that $(i,k)$ is classified as a $\widehat{T}$-avenue, but $j$ is biased. When that happens, we can't prove that $j$ belongs to the same broken-city as $i$ and $k$ do. So a broken-city might not be $T$-connected (which is why we don't just call it a ``city''). We have to work around that.

\paragraph{The remedy.} The key to the remedy is, as we have already noted, that every ``hole'' must be biased. That is, every ``hole'' is either almost always equal to $1$, or almost always equal to $-1$, under $\mathrm{P}$. Roughly speaking, the ``holes'' can be neglected, because every edge incidental to a ``hole'' must be close to being independent in $H^2$ under $\mathrm{P}$ due to the ``hole'' having a small $\mathrm{minmrg}$ under $\mathrm{P}$ (see Lemma~\ref{ihminmrg}). We define the convex hull (in $T$) of a broken-city by filling in all the ``holes'' (see Definition~\ref{conv}). While any two different broken-cities must be disjoint, their convex hulls might intersect. Whenever the convex hulls of two broken-cities intersect, we have to merge those broken-cities together. The resulting clusters after we merge every pair of broken-cities whose convex hulls intersect are called the \textbf{cities} (see Definition~\ref{truehierarchy}). It can be shown that every city is $T$-connected (see Lemma~\ref{gencityconnected}). It is worth noting though that a city might not be $\widehat{T}$-connected, but we will be able to work around that.\footnote{It might appear from the preceding discussion that it would have been easier if we chose to let $T$ dictate the layering, so that the groups formed would be automatically $T$-connected at every layer. Unfortunately, in this case the fact that the groups may not be $\widehat{T}$ connected would present a bigger problem, which ultimately doesn't lead to a simpler proof.}

To complete our definitions at the top layer, we classify every edge of $T$ whose two end-nodes belong to the same city as a $\pmb{T}$\textbf{-road}.  Recall that in the symmetric case, one of the steps in the process of going from $T$ to $\widehat{T}$ is to replace the set of all $T$-roads by the set of all $\widehat{T}$-roads, and that each city is spanned by the $T$-roads in it, as well as by the $\widehat{T}$-roads in it. We would like something similar for the general case to make our argument down the road simpler and clearer. Because each city $C$ is $T$-connected, it is spanned by those $T$-roads whose two end-nodes belong to $C$. On the other hand, the $\widehat{T}$-avenues whose two end-nodes belong to $C$ are not necessarily enough to span $C$ (which is why we don't call them ``$\widehat{T}$-roads''); they only span the broken-cities contained in $C$ individually. To augment them into a spanning set for $C$ we have to throw in some edges that connect together the broken-cities contained in $C$. In fact, those additional edges can be selected from among the set of $T$-roads with end-nodes in $C$. See Definition~\ref{ttrail} and Remark~\ref{trailselection}. We call each of those additional edges a $\pmb{T}$\textbf{-trail}. Thus, each city $C$ is spanned by the union of the $\widehat{T}$-avenues with end-nodes in $C$ and the $T$-trails with end-nodes in $C$. One of the steps in the process of going from $T$ to $\widehat{T}$ is to replace the set of all $T$-roads by the union of the set of all $\widehat{T}$-avenues and the set of all $T$-trails. Notice that the $T$-trails are selected from among the $T$-roads, and so this round of edge replacement is in fact leaving those edges unchanged. Again, it is for the clarity of our argument that we insist on their involvement so that within each city we are replacing one spanning tree by another spanning tree (see the bounding argument for $H(\mathrm{P}^{(3)}, \mathrm{P}^{(4)})$ in Section~\ref{sec:genbounding}). Of course, the $T$-trails are not necessarily edges of $\widehat{T}$ so we have to get rid of them eventually. It can be shown that every $T$-trail must be incidental to at least one biased node, and so it is close to being independent in $H^2$ under $\mathrm{P}$ (see Lemma~\ref{trailih}). Consequently, they can be cut away without incurring too much error in $H^2$ (see the bounding argument for $H(\mathrm{P}^{(4)}, \mathrm{P}^{(5)})$ in Section~\ref{sec:genbounding}).

For a diagrammatic illustration of the top layer see Figure~\ref{brokencitiesfigure}. The solid connections represent edges of $T$, and the dashed connections represent edges of $\widehat{T}$. The small shaded circles represent the broken-cities, the shaded regions with irregular curved boundaries represent the convex hulls in $T$ of the broken-cities, and the big hollow circle represents the city $C$. The operations of taking convex hulls in $T$ fill in the ``holes.'' For example, $16 \in \conv{T}{\widetilde{C}_5}$ because the node $16$ lies on the path in $T$ between nodes $12$ and $13$, both of which belong to $\widetilde{C}_5$. Recall that all ``holes'' must be biased (that is, having $\mathrm{minmrg}$ less than $10^7 \frac{\epsilon^2}{n}$). In our case, the ``holes'' are marked with cross marks right next to them. Also, the broken-cities whose convex hulls in $T$ intersect are merged into the same city. In our case, the broken-cities $\widetilde{C}_1, \ldots, \widetilde{C}_7$ are merged into the city $C$, while $\widetilde{C}_8$ and $\widetilde{C}_9$ belong to some other city (not shown) because their convex hulls in $T$ don't intersect with the convex hull in $T$ of any of $\widetilde{C}_1, \ldots, \widetilde{C}_7$. Lastly, we use tiny hollow circles placed at the middle of the edges to mark the six $T$-roads that, along with the $\widehat{T}$-avenues in $C$, form a spanning tree of $C$. We can select those $T$-roads to be the $T$-trails in $C$ (the selection is not unique: we could have selected $(13,17)$ instead of $(13,16)$, for example).

\begin{figure}[h!]
\centering
\includegraphics[width=16cm]{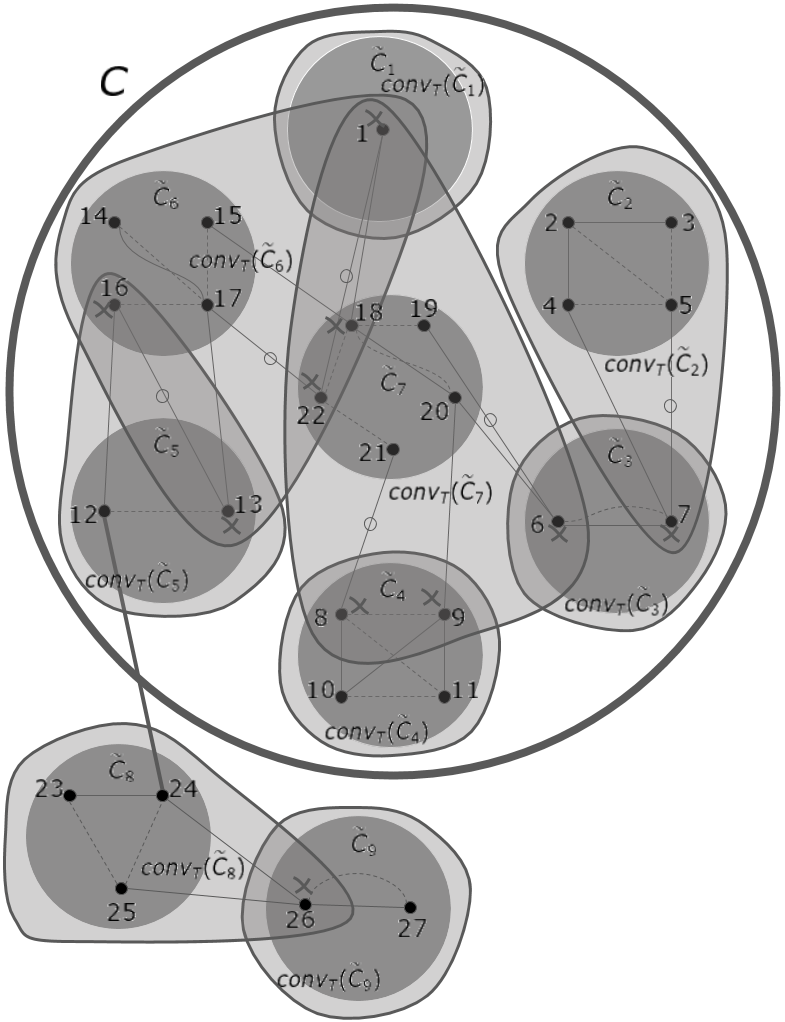}
\caption{A diagrammatic illustration of our hierarchy. For a description of the figure see the last paragraph of ``The remedy" in Section~\ref{sec:outline of general case}.}
\label{brokencitiesfigure}
\end{figure}

\paragraph{The second layer.} For the next layer, we classify an edge $(i,j)$ of $\widehat{T}$ with $\minmrg{\widehat{\mathrm{P}}_{ij}} \geq 10^8 \frac{\epsilon^2}{n}$ and $\mindisc{\widehat{\mathrm{P}}_{ij}} \geq \frac{1}{2}$ that is not a $\widehat{T}$-avenue as a $\pmb{\widehat{T}}$\textbf{-highway}, where
$$\mindisc{\widehat{\mathrm{P}}_{ij}} = \min \brac{\abs{\widehat{\mathrm{P}}_{i|j} (1|1) - \widehat{\mathrm{P}}_{i|j}(1|-1)}, \abs{\widehat{\mathrm{P}}_{j|i} (1|1) - \widehat{\mathrm{P}}_{j|i}(1|-1)}}.$$
The $\widehat{T}$-highways link the broken-cities into \textbf{broken-countries}, and link the cities into \textbf{countries} (from which it follows immediately that each country is a union of broken-countries). An edge $(i,j)$ of $T$ with $i,j$ belonging to different cities in the same country is classified as a $\pmb{T}$\textbf{-highway}. We can prove the following structural results (note that (iii) is an analog of Lemma~\ref{highwayparallel} for the symmetric case):

\newtheorem*{lem:genhighwayparallel}{Lemma \ref{genhighwayparallel}}
\begin{lem:genhighwayparallel}
\begin{itemize}
\item[(i)] A $\widehat{T}$-highway goes between different cities.
\item[(ii)] There can be at most one $\widehat{T}$-highway and at most one $T$-highway between any pair of cities.
\item[(iii)] There is a $\widehat{T}$-highway between a pair of cities if and only if there is a $T$-highway between the same pair of cities, and if so, those two highways in fact go between the same pair of broken-cities.
\end{itemize}
\end{lem:genhighwayparallel}
\begin{proof}
(a rough and slightly imprecise sketch of the ``only if'' direction of (iii)).
Suppose for the sake of contradiction that $(i,k)$ is a $\widehat{T}$-highway between broken-cities $\widetilde{C}$ and $\widetilde{D}$ that are contained respectively in cities $C$ and $D$ (note that (i) implies that $C$ and $D$ are different cities), and that there is no edge of $T$ between $\widetilde{C}$ and $\widetilde{D}$ (any such edge would in fact be classified as a $T$-highway because $C$ and $D$ are clearly different cities in the same country). Then, the path in $T$ between $i$ and $k$ must contain some $j$ that belongs to neither $\widetilde{C}$ nor $\widetilde{D}$. Note that $(i,k)$, being a $\widehat{T}$-highway, has a significant $\mathrm{mindisc}$ by definition. It follows that the pair $i,j$ must also have a significant $\mathrm{mindisc}$ because they are ``closer'' on the path (see Lemma~\ref{mindiscmarkov}). Furthermore, it can be shown that the pair $j,k$ must have a significant $\mathrm{mindiag}$ because they lie in different broken-cities (see Lemma~\ref{samebrokencity}). Lastly, it can be shown that $j$ must have a significant $\mathrm{minmrg}$ (see Lemma~\ref{minmrgmindisc3}). Thus, Lemma~\ref{helwrongedgecity} (stated below) implies that $H^2(\mathrm{P}_{ijk}, \mathrm{P}_{i\wideparen{ \,~~ j-}k})$ must be significant (that is, greater than $22 \frac{\epsilon^2}{n}$). It then follows from (the contrapositive of) Lemma~\ref{genhelwrongedge} (i) that $\mathrm{I}^{\widehat{\mathrm{P}}}(X_i;X_j) > \mathrm{I}^{\widehat{\mathrm{P}}}(X_i;X_k)$. Symmetrically we also have $\mathrm{I}^{\widehat{\mathrm{P}}}(X_j;X_k) > \mathrm{I}^{\widehat{\mathrm{P}}}(X_i;X_k)$. Together they imply that $(i,k)$ could not possibly be picked to be an edge of $\widehat{T}$ by the maximum spanning tree algorithm.
\end{proof}

\newtheorem*{lem:helwrongedgecity}{Lemma \ref{helwrongedgecity}}
\begin{lem:helwrongedgecity}
If $i,j,k$ lie on a path in $T$, then $H^2(\mathrm{P}_{ijk}, \mathrm{P}_{i\wideparen{ \,~~ j-}k}) \geq \frac{1}{100} {\mindisc{\mathrm{P}_{ij}}}^2 \cdot \min{\brac{\minmrg{\mathrm{P}_{j}}, \mindiag{\mathrm{P}_{jk}}}}.$
\end{lem:helwrongedgecity}

As in the symmetric case, we will replace the set of all $T$-highways with the set of all $\widehat{T}$-highways in a single round of edge replacements in the process of going from $T$ to $\widehat{T}$ (see the bounding argument for $H(\mathrm{P}^{(2)}, \mathrm{P}^{(3)})$ in Section~\ref{sec:genbounding}). Also note that Lemma~\ref{genhighwayparallel} (iii) and the fact that every city is $T$-connected imply that every country is $T$-connected.

\paragraph{The third layer.} For the next layer, we classify an edge $(i,j)$ of $\widehat{T}$ with $I_{H^2}(\widehat{\mathrm{P}}_{ij}) \geq 10^{10} \frac{\epsilon^2}{n}$ that is not a $\widehat{T}$-avenue or a $\widehat{T}$-highway as a $\pmb{\widehat{T}}$\textbf{-railway}, where
$$I_{H^2} (\widehat{\mathrm{P}}_{ij}) = H^2 (\widehat{\mathrm{P}}_{ij}, \widehat{\mathrm{P}}_{ij}^{\rm{(ind)}}),$$
and $\widehat{\mathrm{P}}_{ij}^{\rm{(ind)}}$ is the independent pairwise distribution that has the same individual node marginals at $i$ and $j$, respectively, as $\widehat{\mathrm{P}}_{ij}$ does. The $\widehat{T}$-railways link the broken-countries into \textbf{broken-continents}, and link the countries into \textbf{continents} (from which it follows immediately that each continent is a union of broken-continents). An edge $(i,j)$ of $T$ with $i,j$ belonging to different countries in the same continent is classified as a $\pmb{T}$\textbf{-railway}. We can prove the following structural results (note that (iii) is an analog of Lemma~\ref{railwayparallel} for the symmetric case):

\newtheorem*{lem:genrailwayparallel}{Lemma \ref{genrailwayparallel}}
\begin{lem:genrailwayparallel}
\begin{itemize}
\item[(i)] A $\widehat{T}$-railway goes between different countries.
\item[(ii)] There can be at most one $\widehat{T}$-railway and at most one $T$-railway between any pair of countries.
\item[(iii)] There is a $\widehat{T}$-railway between a pair of countries if and only if there is a $T$-railway between the same pair of countries, and if so, those two railways in fact go between the same pair of broken-countries.
\end{itemize}
\end{lem:genrailwayparallel}
\begin{proof}
(a rough and slightly imprecise sketch of the ``only if'' direction of (iii)).
Suppose for the sake of contradiction that $(i,k)$ is a $\widehat{T}$-railway between broken-countries $\widetilde{\mathcal{C}}$ and $\widetilde{\mathcal{D}}$ that are contained respectively in countries $\mathcal{C}$ and $\mathcal{D}$ (note that (i) implies that $\mathcal{C}$ and $\mathcal{D}$ are different countries), and that there is no edge of $T$ between $\widetilde{\mathcal{C}}$ and $\widetilde{\mathcal{D}}$ (any such edge would in fact be classified as a $T$-railway because $\mathcal{C}$ and $\mathcal{D}$ are clearly different countries in the same continent). Then, the path in $T$ between $i$ and $k$ must contain some $j$ that belongs to neither $\widetilde{\mathcal{C}}$ nor $\widetilde{\mathcal{D}}$. Note that $(i,k)$, being a $\widehat{T}$-railway, has a significant $I_{H^2}$ by definition. It follows that the pair $i,j$ must also have a significant $I_{H^2}$ because they are ``closer'' on the path (see Lemma~\ref{ihmarkov}). Furthermore, it can be shown that the pair $j,k$ must have a $\mathrm{mindisc}$ significantly far from $1$ because they lie in different broken-countries (see Lemma~\ref{samebrokencountry}). Thus, Lemma~\ref{helwrongedgecountry} (stated below) implies that $H^2(\mathrm{P}_{ijk}, \mathrm{P}_{i\wideparen{ \,~~ j-}k})$ must be significant (that is, greater than $22 \frac{\epsilon^2}{n}$). It then follows from (the contrapositive of) Lemma~\ref{genhelwrongedge} (i) that $\mathrm{I}^{\widehat{\mathrm{P}}}(X_i;X_j) > \mathrm{I}^{\widehat{\mathrm{P}}}(X_i;X_k)$. Symmetrically we also have $\mathrm{I}^{\widehat{\mathrm{P}}}(X_j;X_k) > \mathrm{I}^{\widehat{\mathrm{P}}}(X_i;X_k)$. Together they imply that $(i,k)$ could not possibly be picked to be an edge of $\widehat{T}$ by the maximum spanning tree algorithm.
\end{proof}

\newtheorem*{lem:helwrongedgecountry}{Lemma \ref{helwrongedgecountry}}
\begin{lem:helwrongedgecountry}
If $i,j,k$ lie on a path in $T$, then $H^2(\mathrm{P}_{ijk}, \mathrm{P}_{i\wideparen{ \,~~ j-}k}) \geq \frac{1}{100} I_{H^2}(\mathrm{P}_{ij}) \cdot \paren{1 - \mindisc{\mathrm{P}_{jk}}}^2.$
\end{lem:helwrongedgecountry}

As in the symmetric case, we will replace the set of all $T$-railways with the set of all $\widehat{T}$-railways in a single round of edge replacements in the process of going from $T$ to $\widehat{T}$ (see the bounding argument for $H(\mathrm{P}^{(1)}, \mathrm{P}^{(2)})$ in Section~\ref{sec:genbounding}). Also note that Lemma~\ref{genrailwayparallel} (iii) and the fact that every country is $T$-connected imply that every continent is $T$-connected.

\paragraph{The bottom layer.} Finally, an edge $(i,j)$ of $\widehat{T}$ that is yet to be classfied is classified as a $\pmb{\widehat{T}}$\textbf{-tunnel}, and an edge $(i,j)$ of $T$ that is yet to be classified is classified as a $\pmb{T}$\textbf{-airway}. While a $T$-airway must go between two different continents, a $\widehat{T}$-tunnel might not. In fact, the two end-nodes of a $\widehat{T}$-tunnel might belong to the same city, to different cities in the same country, to different countries in the same continent, or to different continents (that's why we call it a $\widehat{T}$-tunnel, instead of a ``$\widehat{T}$-airway''). Nevertheless, it can be shown that an edge that is a $\widehat{T}$-tunnel or $T$-airway must be close to being independent in $H^2$ under $\mathrm{P}$ (see Lemma~\ref{samebrokencontinent}), so that they can be cut away without incurring too much error in $H^2$ (see the bounding argument for $H(\mathrm{P}, \mathrm{P}^{(1)})$ and $H(\mathrm{P}^{(6)}, \mathrm{P}^{(7)})$ in Section~\ref{sec:genbounding}).

\paragraph{Summary of layering.} To summarize, the edges of $\widehat{T}$ are classified into $\widehat{T}$-avenues, $\widehat{T}$-highways, $\widehat{T}$-railways, and $\widehat{T}$-tunnels based on a combination of the measures $\mathrm{minmrg}$, $\mathrm{mindiag}$, $\mathrm{mindisc}$, and $I_{H^2}$. The $n$ nodes are clustered into broken-cities by the $\widehat{T}$-avenues, which are further clustered into broken-countries by the $\widehat{T}$-highways, which are further clustered into broken-continents by the $\widehat{T}$-railways, which are finally linked into one component by the $\widehat{T}$-tunnels.

Note that each broken-city, broken-country, or broken-continent is $\widehat{T}$-connected by definition. Unlike what happened in the symmetric case, it is however not necessarily $T$-connected. To remedy that, we define the convex hull (in $T$) of a broken-city by filling in all the ``holes'' (a ``hole'' for a broken-city is a node which doesn't belong to the broken-city, but which lies on the path in $T$ between some two nodes that belong to the broken-city; we can prove that every ``hole'' is biased, in that it is almost always $1$ or almost always $-1$). Whenever the convex hulls of two broken-cities intersect, we merge those broken-cities together. The resulting clusters after we merge every such pair of broken-cities are called the cities. The $\widehat{T}$-highways (each of which can be shown to go between different cities) link the cities into countries, and the $\widehat{T}$-railways (each of which can be shown to go between different countries) link the countries into continents. This hierarchical classification of nodes into groups then induces a classification of the edges of $T$: an edge of $T$ is classified as a $T$-road if its two end-nodes belong to the same city, a $T$-highway if its two end-nodes belong to different cities in the same country, a $T$-railway if its two end-nodes belong to different countries in the same continent, or a $T$-airway if its two end-nodes belong to different continents. For edges of $\widehat{T}$, we similarly have that the two end-nodes of each $\widehat{T}$-avenue belong to the same (broken-)city, those of each $\widehat{T}$-highway belong to different (broken-)cities in the same (broken-)country, and those of each $\widehat{T}$-railway belong to different (broken-)countries in the same (broken-)continent. However, while the two end-nodes of each $\widehat{T}$-tunnel surely belong to different broken-continents, they might in fact belong to the same city, to different cities in the same country, to different countries in the same continent, or to different continents.

We can prove that every city, country, or continent is $T$-connected (they are not necessarily $\widehat{T}$-connected, though). Furthermore, we can prove that the $T$-highways and $\widehat{T}$-highways can be matched into parallel pairs, that is, there is a $\widehat{T}$-highway between two broken-cities if and only if there is a $T$-highway between them (and there is at most one such pair between every pair of cities). The same goes for railways: there is a $\widehat{T}$-railway between two broken-countries if and only if there is a $T$-railway between them (and there is at most one such pair between every pair of countries). One consequence of everything we stated so far is that the exact same cities, countries, and continents would have been achieved if we were to cluster the $n$ nodes first by $T$-roads, then $T$-highways, then $T$-railways, and finally $T$-airways.

Lastly, for every city $C$ we select among the set of $T$-roads with end-nodes in $C$ a subset that, together with the $\widehat{T}$-avenues with end-nodes in $C$, form a spanning tree of $C$. We call those selected edges $T$-trails. We can prove that every $T$-trail must be incidental to a biased node and is therefore close to being independent in $H^2$.

For a diagrammatic illustration of our structural results see Figure~\ref{genhierarchyfigure}. The solid connections represent edges of $T$, and the dashed connections represent edges of $\widehat{T}$. The thickness of the line increases as we go down in the hierarchy, with the $T$-roads and $\widehat{T}$-avenues being the thinnest and the $T$-airways and $\widehat{T}$-tunnels being the thickest. The shaded circles with dashed boundaries represent the broken-cities, the shaded rectangular boxes with dashed boundaries represent the broken-countries, and the shaded regions with dashed saw-tooth boundaries represent the broken-continents. The hollow circular or oval regions with solid boundaries represent the cities, the hollow regions with solid axis-aligned boundaries represent the countries, and the hollow regions with solid saw-tooth boundaries represent the continents. To avoid cluttering the figure, we don't show the convex hulls (in $T$) of the broken-cities. Note that the $T$-highways and $\widehat{T}$-highways can be matched into parallel pairs, each going between the same pair of broken-cities (in different cities in the same country). An example is $(1,2)$ and $(3,4)$. Similarly, the $T$-railways and $\widehat{T}$-railways can be matched into parallel pairs, each going between the same pair of broken-countries (in different countries in the same continent). An example is $(5,6)$ and $(7,8)$. This parallelism needs not hold other types of edges. For example, there is no edge of $\widehat{T}$ that goes between the same pair of continents as the $T$-airway $(9,10)$ does, and there is no $T$-airway that goes between the same pair of continents as the $\widehat{T}$-tunnel $(11,12)$ does. The $\widehat{T}$-tunnels can behave wildly, in the sense that the two end-nodes of a $\widehat{T}$-tunnel could belong to different continents (e.g. $(13,14)$), to different countries in the same continent (e.g. $(15,16)$), to different cities in the same country (e.g. $(17,18)$), or even to the same city (e.g. $(19,20)$). Lastly, we use tiny hollow circles placed at the middle of the edges to mark one possible selection for the $T$-trails. (There are some structural properties illustrated in the diagram that were not discussed in the outline above. For example, the path in $T$ between the end-nodes of a $\widehat{T}$-highway or a $\widehat{T}$-railway must not contain any biased node. Consequently, it must not contain any ``hole'' for any of the broken-cities. Those properties will turn up in the proof details.)

\begin{figure}[h!]
\centering
\includegraphics[width=16.2cm]{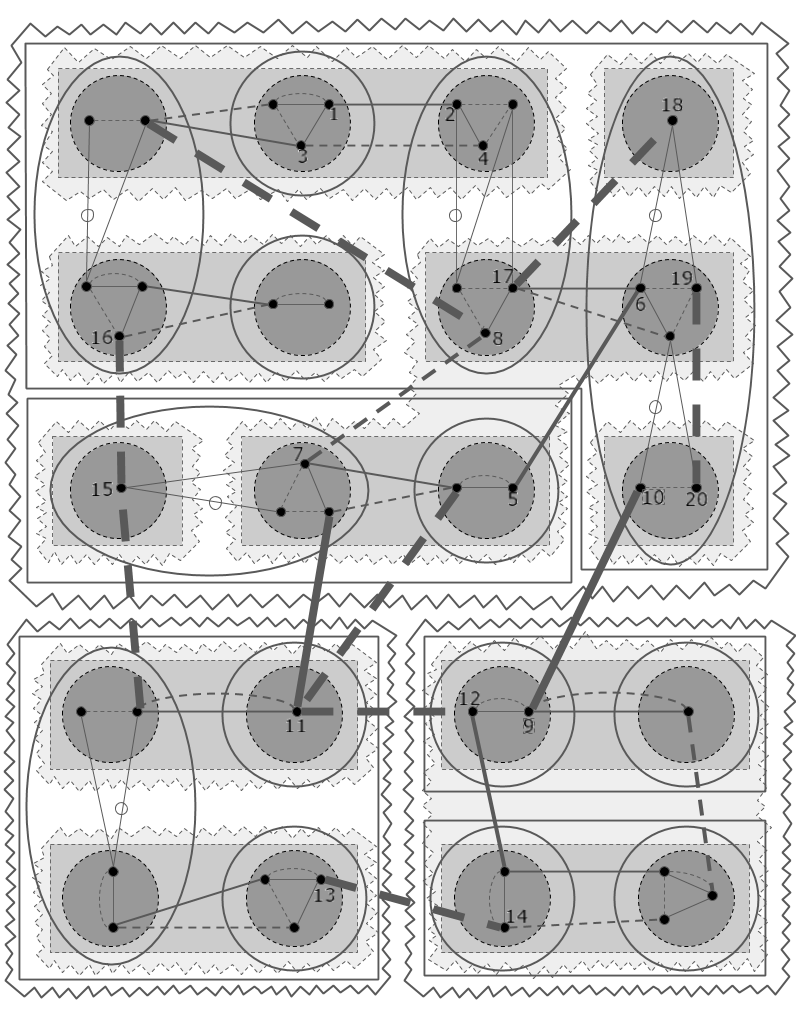}
\caption{A diagrammatic illustration of our hierarchy. For a description of the figure see the last paragraph of ``Summary of layering" in Section~\ref{sec:outline of general case}.}
\label{genhierarchyfigure}
\end{figure}

\paragraph{Going from $T$ to $\widehat{T}$ for the general case.} To go from $T$ to $\widehat{T}$, we first cut away all the $T$-airways, then replace the set of $T$-railways by the set of $\widehat{T}$-railways, then replace the set of $T$-highways by the set of $\widehat{T}$-highways, then replace the set of $T$-roads by the union of the set of $\widehat{T}$-avenues and the set of $T$-trails, then cut away all the $T$-trails, and finally install all the $\widehat{T}$-tunnels. See Section~\ref{sec:genbounding} for a rigorous treatment of the bounding of $\dtv{\mathrm{P}}{\mathrm{Q}}$.

\subsection{The Detailed Proof of the General Case}
\label{sec:genproof}

We present the detailed proof of Theorem~\ref{gensufficiency}, following the steps sketched in Section~\ref{sec:outline of general case}. Throughout the section, we assume that $\widehat{\mathrm{P}}$ satisfies strong 4-consistency, and our goal is to demonstrate that under this condition, Algorithm~\ref{algo:genalgo} will compute a distribution $\mathrm{Q}$ satisfying $\dtv{\mathrm{P}}{\mathrm{Q}} \leq O \paren{\epsilon}$. Section~\ref{sec:genlemmas} contains a collection of lemmas each involving a small number (1, 2, 3, or 4) of variables. Section~\ref{sec:genstructure} contains structural results on the relation between $T$ and $\widehat{T}$, stated in the language of various hierarchies to be defined at the beginning of that section. The layers defined here dictate how we switch from $T$ to $\widehat{T}$ in the hybrid argument later. Section~\ref{sec:genbounding} proves the desired bound in total variation distance between $\mathrm{P}$ and $\mathrm{Q}$ with the aid of a few intermediate hybrid distributions. Throughout the section we assume $\epsilon \leq \frac{1}{10^{100}}$ (it's easy to see that it suffices to deal only with sufficiently small $\epsilon$).

\subsubsection{Lemmas on 1, 2, 3, or 4 Variables}
\label{sec:genlemmas}

This subsection includes a collection of lemmas each involving a small number (1, 2, 3, or 4) of variables that form the basis of our analysis later of more complex structures (e.g. those of $T$ and $\widehat{T}$) and distributions.

Our first lemma provides bounds on the difference in $H^2$ for certain pairs of distributions involving a small number of variables, as a consequence of assuming 4-consistency. The distributions in (i), (ii), (iii), and (iv) are simply marginals of $\mathrm{P}$ or $\widehat{\mathrm{P}}$, while (most of) the distributions in (v), (vi), (vii), and (viii) are derived from $\mathrm{P}$ or $\widehat{\mathrm{P}}$ by means of Definition~\ref{arcnotation} or Definition~\ref{genarcnotation}.

\begin{lemma}
\label{genprecision}
\begin{itemize}
\item[(i)] For any $i \in [n]$, $H^2 (\mathrm{P}_i, \widehat{\mathrm{P}}_i) \leq \frac{1}{10^{20}} \frac{\epsilon^2}{n};$
\item[(ii)] For any $i,j \in [n]$, $H^2 (\mathrm{P}_{ij}, \widehat{\mathrm{P}}_{ij}) \leq 2 \cdot \frac{1}{10^{20}} \frac{\epsilon^2}{n};$
\item[(iii)] For any $i,j,k \in [n]$, $H^2 (\mathrm{P}_{ijk}, \widehat{\mathrm{P}}_{ijk}) \leq 4 \cdot \frac{1}{10^{20}} \frac{\epsilon^2}{n};$
\item[(iv)] For any $h,i,j,k \in [n]$, $H^2 (\mathrm{P}_{hijk}, \widehat{\mathrm{P}}_{hijk}) \leq 8 \cdot \frac{1}{10^{20}} \frac{\epsilon^2}{n};$
\item[(v)] For any $i,j,k \in [n]$, $H^2 (\mathrm{P}_{i-j-k}, \widehat{\mathrm{P}}_{i-j-k}) \leq 4 \cdot \frac{1}{10^{20}} \frac{\epsilon^2}{n};$
\item[(vi)] For any $h,i,j,k \in [n]$, $H^2 (\mathrm{P}_{(hi)-j-k}, \widehat{\mathrm{P}}_{(hi)-j-k}) \leq 6 \cdot \frac{1}{10^{20}} \frac{\epsilon^2}{n};$
\item[(vii)] If $i,j,k$ lie on a path in $T$, then $H^2 (\widehat{\mathrm{P}}_{i-j-k}, \widehat{\mathrm{P}}_{ijk}) \leq 16 \cdot \frac{1}{10^{20}} \frac{\epsilon^2}{n};$
\item[(viii)] If $h,i,j,k$ lie on a path in $T$, then $H^2 (\widehat{\mathrm{P}}_{(hi)-j-k}, \widehat{\mathrm{P}}_{hijk}) \leq 28 \cdot \frac{1}{10^{20}} \frac{\epsilon^2}{n}.$
\end{itemize}
\end{lemma}
\begin{proof}
First we show (iv). Note that for any $x_h,x_i,x_j,x_k \in \{1,-1\}$, 4-consistency implies that
$$\abs{\widehat{\mathrm{P}}_{hijk}(x_h,x_i,x_j,x_k) - \mathrm{P}_{hijk}(x_h,x_i,x_j,x_k)} \leq \frac{1}{10^{20}} \cdot \max \brac{ \sqrt{\mathrm{P}_{hijk}(x_h,x_i,x_j,x_k) \cdot \frac{\epsilon^2}{n}}, \frac{\epsilon^2}{n} }.$$
Lemma~\ref{bernoullihellinger} then implies $\paren{\sqrt{\widehat{\mathrm{P}}_{hijk}(x_h,x_i,x_j,x_k)} - \sqrt{\mathrm{P}_{hijk}(x_h,x_i,x_j,x_k)}}^2 \leq \frac{1}{10^{20}} \frac{\epsilon^2}{n}.$
We thus have
\begin{align*}
H^2 (\mathrm{P}_{hijk}, \widehat{\mathrm{P}}_{hijk}) &= \frac{1}{2} \sum_{x_h,x_i,x_j,x_k = \pm 1} \paren{\sqrt{\widehat{\mathrm{P}}_{hijk}(x_h,x_i,x_j,x_k)} - \sqrt{\mathrm{P}_{hijk}(x_h,x_i,x_j,x_k)}}^2 \\
		&\leq 8 \cdot \frac{1}{10^{20}} \frac{\epsilon^2}{n}.
\end{align*}
The proofs for (i), (ii), and (iii) are similar to that for (iv). For (i) we only need to bound for, and sum over, the two combinations corresponding to $x_i \in \{1,-1\}$, for (ii) we have four combinations corresponding to $x_i,x_j \in \{1,-1\}$, and for (iii) we have eight combinations corresponding to $x_i,x_j,x_k \in \{1,-1\}$.

To prove (v), note that $\mathrm{P}_{i-j-k}$ and $\widehat{\mathrm{P}}_{i-j-k}$ have the common factorization structure
\begin{align*}
\mathrm{P}_{i-j-k}(x_i,x_j,x_k) &= \mathrm{P}_{ij}(x_i,x_j) \mathrm{P}_{k|j}(x_k|x_j) \\
\widehat{\mathrm{P}}_{i-j-k}(x_i,x_j,x_k) &= \widehat{\mathrm{P}}_{ij}(x_i,x_j) \widehat{\mathrm{P}}_{k|j}(x_k|x_j) 
\end{align*}
By Lemma~\ref{subadditivity} and (ii) above, we have
$$H^2 (\mathrm{P}_{i-j-k}, \widehat{\mathrm{P}}_{i-j-k}) \leq H^2(\mathrm{P}_{ij},\widehat{\mathrm{P}}_{ij}) + H^2(\mathrm{P}_{jk},\widehat{\mathrm{P}}_{jk}) \leq 4 \cdot \frac{1}{10^{20}} \frac{\epsilon^2}{n}.$$

To prove (vi), note that $\mathrm{P}_{(hi)-j-k}$ and $\widehat{\mathrm{P}}_{(hi)-j-k}$ have the common factorization structure
\begin{align*}
\mathrm{P}_{(hi)-j-k}(x_h,x_i,x_j,x_k) &= \mathrm{P}_{hij}(x_h,x_i,x_j) \mathrm{P}_{k|j}(x_k|x_j) \\
\widehat{\mathrm{P}}_{(hi)-j-k}(x_h,x_i,x_j,x_k) &= \widehat{\mathrm{P}}_{hij}(x_h,x_i,x_j) \widehat{\mathrm{P}}_{k|j}(x_k|x_j) 
\end{align*}
By Lemma~\ref{subadditivity} and (ii), (iii) above, we have
$$H^2 (\mathrm{P}_{(hi)-j-k}, \widehat{\mathrm{P}}_{(hi)-j-k}) \leq H^2(\mathrm{P}_{hij},\widehat{\mathrm{P}}_{hij}) + H^2(\mathrm{P}_{jk},\widehat{\mathrm{P}}_{jk}) \leq 6 \cdot \frac{1}{10^{20}} \frac{\epsilon^2}{n}.$$

To prove (vii), note that the fact $i,j,k$ lie on a path in $T$ means $\mathrm{P}_{i-j-k} = \mathrm{P}_{ijk}$. Applying triangle inequality for Hellinger distance, and (iii), (v) above, we get
\begin{align*}
H^2 (\widehat{\mathrm{P}}_{i-j-k}, \widehat{\mathrm{P}}_{ijk}) &\leq \paren{H(\widehat{\mathrm{P}}_{i-j-k}, \mathrm{P}_{i-j-k}) + H(\mathrm{P}_{ijk}, \widehat{\mathrm{P}}_{ijk})}^2 \\
		&\leq \paren{\sqrt{4 \cdot \frac{1}{10^{20}} \frac{\epsilon^2}{n}} + \sqrt{4 \cdot \frac{1}{10^{20}} \frac{\epsilon^2}{n}}}^2 \\
		&= 16 \cdot \frac{1}{10^{20}} \frac{\epsilon^2}{n}.
\end{align*}

To prove (viii), note that the fact $h,i,j,k$ lie on a path in $T$ means $\mathrm{P}_{(hi)-j-k} = \mathrm{P}_{hijk}$. Applying triangle inequality for Hellinger distance, and (iv), (vi) above, we get
\begin{align*}
H^2 (\widehat{\mathrm{P}}_{(hi)-j-k}, \widehat{\mathrm{P}}_{hijk}) &\leq \paren{H(\widehat{\mathrm{P}}_{(hi)-j-k}, \mathrm{P}_{(hi)-j-k}) + H(\mathrm{P}_{hijk}, \widehat{\mathrm{P}}_{hijk})}^2 \\
		&\leq \paren{\sqrt{6 \cdot \frac{1}{10^{20}} \frac{\epsilon^2}{n}} + \sqrt{8 \cdot \frac{1}{10^{20}} \frac{\epsilon^2}{n}}}^2 \\
		&< 28 \cdot \frac{1}{10^{20}} \frac{\epsilon^2}{n}.
\end{align*}
\end{proof}

The next lemma indicates the precision to which $\widehat{\mathrm{P}}$ can be guaranteed to approximate $\mathrm{P}$ on events involving (up to) 4 variables, as a consequence of assuming 4-consistency.

\begin{lemma}
\label{genprobprecision}
Let $S \subseteq [n]$ be of cardinality $4$, and $W \subseteq {\{1, -1\}}^4$, then
\begin{itemize}
\item[(i)] If $\mathrm{P}(X_S \in W) \leq 2 \cdot 10^{20} \frac{\epsilon^2}{n}$, then $\abs{\widehat{\mathrm{P}}(X_S \in W) - \mathrm{P}(X_S \in W)} \leq \frac{\epsilon^2}{n}$;
\item[(ii)] If $\widehat{\mathrm{P}}(X_S \in W) \leq 10^{20} \frac{\epsilon^2}{n}$, then $\abs{\widehat{\mathrm{P}}(X_S \in W) - \mathrm{P}(X_S \in W)} \leq \frac{\epsilon^2}{n}$;
\item[(iii)] If $\mathrm{P}(X_S \in W) \geq \frac{\epsilon^2}{n}$, then $\abs{\frac{\widehat{\mathrm{P}}(X_S \in W)}{\mathrm{P}(X_S \in W)} - 1} \leq \frac{1}{10^{20}}$.
\end{itemize}
\end{lemma}
\begin{proof}
Let's prove (iii) first. If $\mathrm{P}(X_S \in W) \geq \frac{\epsilon^2}{n}$, 4-consistency implies
$$\abs{\widehat{\mathrm{P}}(X_S \in W) - \mathrm{P}(X_S \in W)} \leq \frac{1}{10^{20}} \cdot \max \brac{ \sqrt{\mathrm{P}(X_S \in W) \cdot \frac{\epsilon^2}{n}}, \frac{\epsilon^2}{n} } = \frac{1}{10^{20}} \cdot \sqrt{\mathrm{P}(X_S \in W) \cdot \frac{\epsilon^2}{n}},$$
and so
$$ \abs{\frac{\widehat{\mathrm{P}}(X_S \in W)}{\mathrm{P}(X_S \in W)} - 1} \leq \frac{1}{10^{20}} \cdot \sqrt{\frac{\frac{\epsilon^2}{n}}{\mathrm{P}(X_S \in W)}} \leq \frac{1}{10^{20}}.$$
To prove (i), note that if $\mathrm{P}(X_S \in W) \leq 2 \cdot 10^{20} \frac{\epsilon^2}{n}$, then 4-consistency implies
\begin{align*}
\abs{\widehat{\mathrm{P}}(X_S \in W) - \mathrm{P}(X_S \in W)} & \leq \frac{1}{10^{20}} \cdot \max \brac{ \sqrt{\mathrm{P}(X_S \in W) \cdot \frac{\epsilon^2}{n}}, \frac{\epsilon^2}{n} } \\
	& \leq \frac{1}{10^{20}} \cdot \max \brac{ \sqrt{2 \cdot 10^{20} \frac{\epsilon^2}{n} \cdot \frac{\epsilon^2}{n}}, \frac{\epsilon^2}{n} } \\
	& < \frac{\epsilon^2}{n}.
\end{align*}

To prove (ii), assume $\widehat{\mathrm{P}}(X_S \in W) \leq 10^{20} \frac{\epsilon^2}{n}$. If we were to have $\mathrm{P}(X_S \in W) \geq 2 \cdot 10^{20} \frac{\epsilon^2}{n}$, then we could apply (iii) to get $\widehat{\mathrm{P}}(X_S \in W) \geq (1 - \frac{1}{10^{20}}) \mathrm{P}(X_S \in W) > 10^{20} \frac{\epsilon^2}{n}$, a contradiction.

Thus, we must have $\mathrm{P}(X_S \in W) \leq 2 \cdot 10^{20} \frac{\epsilon^2}{n}$, and we can apply (i).
\end{proof}

The next four lemmas provide bounds on the error due to estimation of, respectively, the measures $\mathrm{minmrg}$, $\mathrm{mindiag}$, $\mathrm{mindisc}$, and $I_{H^2}$, as a consequence of assuming 4-consistency.

\begin{lemma}
\label{genminmrgprecision}
For any subset $U \subseteq [n]$,
\begin{itemize}
\item[(i)] If $\minmrg{\mathrm{P}_U} \leq 10^{20} \frac{\epsilon^2}{n}$, then $\abs{\minmrg{\mathrm{P}_U} - \minmrg{\widehat{\mathrm{P}}_U}} \leq \frac{\epsilon^2}{n}$;
\item[(ii)] If $\minmrg{\widehat{\mathrm{P}}_U} \leq 10^{20} \frac{\epsilon^2}{n}$, then $\abs{\minmrg{\mathrm{P}_U} - \minmrg{\widehat{\mathrm{P}}_U}} \leq \frac{\epsilon^2}{n}$.
\end{itemize}
\end{lemma}
\begin{proof}
Recall that 
\begin{align*}
\minmrg{\mathrm{P}_U} &= \min \brac{\mathrm{P}_i (x_i) \mid i \in U, x_i = \pm 1} \\
\minmrg{\widehat{\mathrm{P}}_U} &= \min \brac{\widehat{\mathrm{P}}_i (x_i) \mid i \in U, x_i = \pm 1}
\end{align*}
To prove (i), suppose that $\minmrg{\mathrm{P}_U} \leq 10^{20} \frac{\epsilon^2}{n}$. Suppose $\minmrg{\mathrm{P}_U} = \mathrm{P}_j (x_j)$ for some $j \in U$ and $x_j \in \{ 1, -1\}$. This implies $\mathrm{P}_j (x_j) \leq 10^{20} \frac{\epsilon^2}{n}$, and by Lemma~\ref{genprobprecision} (i), we have $\widehat{\mathrm{P}}_j (x_j) \leq \mathrm{P}_j (x_j) + \frac{\epsilon^2}{n} = \minmrg{\mathrm{P}_U} + \frac{\epsilon^2}{n}$. So $\minmrg{\widehat{\mathrm{P}}_U} \leq \widehat{\mathrm{P}}_j (x_j) \leq \minmrg{\mathrm{P}_U} + \frac{\epsilon^2}{n}$.

On the other hand, suppose that $\widehat{\mathrm{P}}_i (x_i) < \minmrg{\mathrm{P}_U} - \frac{\epsilon^2}{n}$ for some $i \in U$ and $x_i \in \{ 1, -1\}$. This implies $\widehat{\mathrm{P}}_i (x_i) < 10^{20} \frac{\epsilon^2}{n}$, and by Lemma~\ref{genprobprecision} (ii) we have $\mathrm{P}_i (x_i) \leq \widehat{\mathrm{P}}_i (x_i) + \frac{\epsilon^2}{n} < \minmrg{\mathrm{P}_U}$, a contradiction. Thus, $\widehat{\mathrm{P}}_i (x_i) \geq \minmrg{\mathrm{P}_U} - \frac{\epsilon^2}{n}$ for all $i \in U$ and $x_i \in \{ 1, -1\}$, which means that $\minmrg{\widehat{\mathrm{P}}_U} \geq \minmrg{\mathrm{P}_U} - \frac{\epsilon^2}{n}$. This concludes the proof of (i).

To prove (ii), suppose that $\minmrg{\widehat{\mathrm{P}}_U} \leq 10^{20} \frac{\epsilon^2}{n}$. Suppose $\minmrg{\widehat{\mathrm{P}}_U} = \widehat{\mathrm{P}}_j (x_j)$ for some $j \in U$ and $x_j \in \{ 1, -1\}$. This implies $\widehat{\mathrm{P}}_j (x_j) \leq 10^{20} \frac{\epsilon^2}{n}$, and by Lemma~\ref{genprobprecision} (ii), we have $\mathrm{P}_j (x_j) \leq \widehat{\mathrm{P}}_j (x_j) + \frac{\epsilon^2}{n} = \minmrg{\widehat{\mathrm{P}}_U} + \frac{\epsilon^2}{n}$. So $\minmrg{\mathrm{P}_U} \leq \mathrm{P}_j (x_j) \leq \minmrg{\widehat{\mathrm{P}}_U} + \frac{\epsilon^2}{n}$.

On the other hand, suppose that $\mathrm{P}_i (x_i) < \minmrg{\widehat{\mathrm{P}}_U} - \frac{\epsilon^2}{n}$ for some $i \in U$ and $x_i \in \{ 1, -1\}$. This implies $\mathrm{P}_i (x_i) < 10^{20} \frac{\epsilon^2}{n}$, and by Lemma~\ref{genprobprecision} (i) we have $\widehat{\mathrm{P}}_i (x_i) \leq\mathrm{P}_i (x_i) + \frac{\epsilon^2}{n} < \minmrg{\widehat{\mathrm{P}}_U}$, a contradiction. Thus, $\mathrm{P}_i (x_i) \geq \minmrg{\widehat{\mathrm{P}}_U} - \frac{\epsilon^2}{n}$ for all $i \in U$ and $x_i \in \{ 1, -1\}$, which means that $\minmrg{\mathrm{P}_U} \geq \minmrg{\widehat{\mathrm{P}}_U} - \frac{\epsilon^2}{n}$. This concludes the proof of (ii).
\end{proof}

\begin{lemma}
\label{genmindiagprecision}
For any $i,j \in [n]$,
\begin{itemize}
\item[(i)] If $\mindiag{\mathrm{P}_{ij}} \leq 2 \cdot 10^{20} \frac{\epsilon^2}{n}$, then $\abs{\mindiag{\mathrm{P}_{ij}} - \mindiag{\widehat{\mathrm{P}}_{ij}}} \leq \frac{\epsilon^2}{n}$;
\item[(ii)] If $\mindiag{\widehat{\mathrm{P}}_{ij}} \leq 10^{20} \frac{\epsilon^2}{n}$, then $\abs{\mindiag{\mathrm{P}_{ij}} - \mindiag{\widehat{\mathrm{P}}_{ij}}} \leq \frac{\epsilon^2}{n}$.
\end{itemize}
\end{lemma}
\begin{proof}
Recall that 
\begin{align*}
\mindiag{\mathrm{P}_{ij}} &= \min \brac{\mathrm{P} (X_i = X_j),  \mathrm{P} (X_i = -X_j)} \\
\mindiag{\widehat{\mathrm{P}}_{ij}} &= \min \brac{\widehat{\mathrm{P}} (X_i = X_j),  \widehat{\mathrm{P}} (X_i = -X_j)}
\end{align*}
To prove (i), suppose that $\mindiag{\mathrm{P}_{ij}} \leq 2 \cdot 10^{20} \frac{\epsilon^2}{n}$. Suppose without loss of generality that $\mindiag{\mathrm{P}_{ij}} = \mathrm{P} (X_i = -X_j)$ (the case $\mindiag{\mathrm{P}_{ij}} = \mathrm{P} (X_i = X_j)$ is symmetric). This implies $\mathrm{P} (X_i = -X_j) \leq 2 \cdot 10^{20} \frac{\epsilon^2}{n}$, and by Lemma~\ref{genprobprecision} (i), we have $\abs{\widehat{\mathrm{P}} (X_i = -X_j) - \mathrm{P} (X_i = -X_j)} \leq \frac{\epsilon^2}{n}$. Since $\mathrm{P} (X_i = X_j) = 1 - \mathrm{P} (X_i = -X_j)$ and $\widehat{\mathrm{P}} (X_i = X_j) = 1 - \widehat{\mathrm{P}} (X_i = -X_j)$, we also have $\abs{\widehat{\mathrm{P}} (X_i = X_j) - \mathrm{P} (X_i = X_j)} \leq \frac{\epsilon^2}{n}$. The conclusion in (i) then follows because perturbing each argument of the $\min$ function by at most $\frac{\epsilon^2}{n}$ changes the overall $\min$ value by at most $\frac{\epsilon^2}{n}$.

To prove (ii), suppose that $\mindiag{\widehat{\mathrm{P}}_{ij}} \leq 10^{20} \frac{\epsilon^2}{n}$. Suppose without loss of generality that $\mindiag{\widehat{\mathrm{P}}_{ij}} = \widehat{\mathrm{P}} (X_i = -X_j)$. This implies $\widehat{\mathrm{P}} (X_i = -X_j) \leq 10^{20} \frac{\epsilon^2}{n}$, and by Lemma~\ref{genprobprecision} (ii), we have $\abs{\mathrm{P} (X_i = -X_j) - \widehat{\mathrm{P}} (X_i = -X_j)} \leq \frac{\epsilon^2}{n}$. Since $\mathrm{P} (X_i = X_j) = 1 - \mathrm{P} (X_i = -X_j)$ and $\widehat{\mathrm{P}} (X_i = X_j) = 1 - \widehat{\mathrm{P}} (X_i = -X_j)$, we also have $\abs{\mathrm{P} (X_i = X_j) - \widehat{\mathrm{P}} (X_i = X_j)} \leq \frac{\epsilon^2}{n}$. From here on (ii) follows as in the proof of (i).
\end{proof}

\begin{lemma}
\label{genmindiscprecision}
For $i,j \in [n]$, if $\minmrg{\mathrm{P}_{ij}} \geq 100 \frac{\epsilon^2}{n}$, then $\abs{\mindisc{\mathrm{P}_{ij}} - \mindisc{\widehat{\mathrm{P}}_{ij}}} \leq \frac{1}{10^{20}}$.
\end{lemma}
\begin{proof}
Recall that 
\begin{align*}
\mindisc{\mathrm{P}_{ij}} = \min \brac{\abs{\mathrm{P}_{i|j} (1|1) - \mathrm{P}_{i|j}(1|-1)}, \abs{\mathrm{P}_{j|i} (1|1) - \mathrm{P}_{j|i}(1|-1)}} \\
\mindisc{\widehat{\mathrm{P}}_{ij}} = \min \brac{\abs{\widehat{\mathrm{P}}_{i|j} (1|1) - \widehat{\mathrm{P}}_{i|j}(1|-1)}, \abs{\widehat{\mathrm{P}}_{j|i} (1|1) - \widehat{\mathrm{P}}_{j|i}(1|-1)}}
\end{align*}
We have
\begingroup
\allowdisplaybreaks
\begin{align*}
\abs{\mathrm{P}_{i|j} (1|1) - \widehat{\mathrm{P}}_{i|j}(1|1)} &= \abs{\frac{\mathrm{P}_{ij}(1,1)}{\mathrm{P}_j(1)} - \frac{\widehat{\mathrm{P}}_{ij}(1,1)}{\widehat{\mathrm{P}}_j(1)}} \\
		&= \abs{\frac{\mathrm{P}_{ij}(1,1) \widehat{\mathrm{P}}_j(1) - \widehat{\mathrm{P}}_{ij}(1,1) \mathrm{P}_j(1)}{\mathrm{P}_j(1) \widehat{\mathrm{P}}_j(1)}} \\
		&\leq \abs{\frac{\paren{\mathrm{P}_{ij}(1,1) - \widehat{\mathrm{P}}_{ij}(1,1)} \cdot \widehat{\mathrm{P}}_j(1)}{\mathrm{P}_j(1) \widehat{\mathrm{P}}_j(1)}} + \abs{\frac{\widehat{\mathrm{P}}_{ij}(1,1) \cdot \paren{\widehat{\mathrm{P}}_j(1) - \mathrm{P}_j(1)}}{\mathrm{P}_j(1) \widehat{\mathrm{P}}_j(1)}} \\
		&\leq \frac{\abs{\mathrm{P}_{ij}(1,1) - \widehat{\mathrm{P}}_{ij}(1,1)}}{\mathrm{P}_j(1)} + \frac{\abs{\widehat{\mathrm{P}}_j(1) - \mathrm{P}_j(1)}}{\mathrm{P}_j(1)} \\
		&\leq \frac{\frac{1}{10^{20}} \cdot \max \brac{ \sqrt{\mathrm{P}_{ij}(1,1) \cdot \frac{\epsilon^2}{n}}, \frac{\epsilon^2}{n} }}{\mathrm{P}_j(1)} + \frac{\frac{1}{10^{20}} \cdot \max \brac{ \sqrt{\mathrm{P}_j(1) \cdot \frac{\epsilon^2}{n}}, \frac{\epsilon^2}{n} }}{\mathrm{P}_j(1)} \\
		&\leq 2 \cdot \frac{1}{10^{20}} 	\cdot \frac{\max \brac{ \sqrt{\mathrm{P}_j(1) \cdot \frac{\epsilon^2}{n}}, \frac{\epsilon^2}{n} }}{\mathrm{P}_j(1)} \\
		&= 2 \cdot \frac{1}{10^{20}} \cdot \frac{\sqrt{\mathrm{P}_j(1) \cdot \frac{\epsilon^2}{n}}}{\mathrm{P}_j(1)} \\
		&= 2 \cdot \frac{1}{10^{20}} \cdot \sqrt{\frac{\frac{\epsilon^2}{n}}{\mathrm{P}_j(1)}} \\
		&\leq 2 \cdot \frac{1}{10^{20}} \cdot \frac{1}{10},
\end{align*}
\endgroup
where we used the fact that $\minmrg{\mathrm{P}_{ij}} \geq 100 \frac{\epsilon^2}{n}$ implies $\mathrm{P}_j(1)\geq 100 \frac{\epsilon^2}{n}$. Similarly, we have
$$ \abs{\mathrm{P}_{i|j} (1|-1) - \widehat{\mathrm{P}}_{i|j}(1|-1)} \leq 2 \cdot \frac{1}{10^{20}} \cdot \frac{1}{10}.$$
Thus, we have
\begin{align*}
\biggl| \abs{\mathrm{P}_{i|j} (1|1) - \mathrm{P}_{i|j}(1|-1)} &- \abs{\widehat{\mathrm{P}}_{i|j} (1|1) - \widehat{\mathrm{P}}_{i|j}(1|-1)} \biggr| \\
		&\leq \abs{\paren{\mathrm{P}_{i|j} (1|1) - \mathrm{P}_{i|j}(1|-1)} - \paren{\widehat{\mathrm{P}}_{i|j} (1|1) - \widehat{\mathrm{P}}_{i|j}(1|-1)}} \\
		&\leq \abs{\paren{\mathrm{P}_{i|j} (1|1) - \widehat{\mathrm{P}}_{i|j} (1|1)} - \paren{\mathrm{P}_{i|j}(1|-1) - \widehat{\mathrm{P}}_{i|j}(1|-1)}} \\
		&\leq \abs{\mathrm{P}_{i|j} (1|1) - \widehat{\mathrm{P}}_{i|j} (1|1)} + \abs{\mathrm{P}_{i|j}(1|-1) - \widehat{\mathrm{P}}_{i|j}(1|-1)} \\
		&< \frac{1}{10^{20}}.
\end{align*}
Similarly, we have
$$\biggl| \abs{\mathrm{P}_{j|i} (1|1) - \mathrm{P}_{j|i}(1|-1)} - \abs{\widehat{\mathrm{P}}_{j|i} (1|1) - \widehat{\mathrm{P}}_{j|i}(1|-1)} \biggr| \leq \frac{1}{10^{20}}.$$
The desired conclusion follows because perturbing each argument of the $\min$ function by at most $\frac{1}{10^{20}}$ changes the overall $\min$ value by at most $\frac{1}{10^{20}}$.
\end{proof}

\begin{lemma}
\label{genihprecision}
For any $i,j \in [n]$,
\begin{itemize}
\item[(i)] If $I_{H^2}(\mathrm{P}_{ij}) \geq \frac{\epsilon^2}{n}$, then $\frac{1}{4} \leq \frac{I_{H^2}(\widehat{\mathrm{P}}_{ij})}{I_{H^2}(\mathrm{P}_{ij})} \leq 4$;
\item[(ii)] If $I_{H^2}(\widehat{\mathrm{P}}_{ij}) \geq \frac{\epsilon^2}{n}$, then $\frac{1}{4} \leq \frac{I_{H^2}(\widehat{\mathrm{P}}_{ij})}{I_{H^2}(\mathrm{P}_{ij})} \leq 4$.
\end{itemize}
\end{lemma}
\begin{proof}
By Lemma~\ref{genprecision} (ii), we have $H^2 (\mathrm{P}_{ij}, \widehat{\mathrm{P}}_{ij}) < \frac{1}{100} \frac{\epsilon^2}{n}$. So $H (\mathrm{P}_{ij}, \widehat{\mathrm{P}}_{ij}) < \frac{1}{10} \frac{\epsilon}{\sqrt{n}}.$

By Lemma~\ref{subadditivity} and Lemma~\ref{genprecision} (i), we have
$$H^2 (\mathrm{P}_{ij}^{\rm{(ind)}}, \widehat{\mathrm{P}}_{ij}^{\rm{(ind)}}) \leq H^2 (\mathrm{P}_i,\widehat{\mathrm{P}}_i) + H^2 (\mathrm{P}_j,\widehat{\mathrm{P}}_j) \leq \frac{1}{10^{20}} \frac{\epsilon^2}{n} + \frac{1}{10^{20}} \frac{\epsilon^2}{n} < \frac{1}{100} \frac{\epsilon^2}{n}.$$
So $H (\mathrm{P}_{ij}^{\rm{(ind)}}, \widehat{\mathrm{P}}_{ij}^{\rm{(ind)}}) < \frac{1}{10} \frac{\epsilon}{\sqrt{n}}.$

By triangle inequality for the Hellinger distance, we have
\begin{align}
\begin{split}
\label{ihdiff}
\abs{I_H(\mathrm{P}_{ij}) - I_H(\widehat{\mathrm{P}}_{ij})} &= \abs{H(\mathrm{P}_{ij}, \mathrm{P}_{ij}^{\rm{(ind)}}) - H(\widehat{\mathrm{P}}_{ij}, \widehat{\mathrm{P}}_{ij}^{\rm{(ind)}})} \\
		&\leq H (\mathrm{P}_{ij}, \widehat{\mathrm{P}}_{ij}) + H (\mathrm{P}_{ij}^{\rm{(ind)}}, \widehat{\mathrm{P}}_{ij}^{\rm{(ind)}}) \\
		&< \frac{1}{5} \frac{\epsilon}{\sqrt{n}}.
\end{split}
\end{align}

To prove (i), suppose $I_{H^2}(\mathrm{P}_{ij}) \geq \frac{\epsilon^2}{n}$. Then $I_H(\mathrm{P}_{ij}) \geq \frac{\epsilon}{\sqrt{n}}$. We have by (\ref{ihdiff})
$$I_H(\widehat{\mathrm{P}}_{ij}) \geq  I_H(\mathrm{P}_{ij}) - \abs{I_H(\mathrm{P}_{ij}) - I_H(\widehat{\mathrm{P}}_{ij})} > \frac{1}{2} \cdot I_H(\mathrm{P}_{ij}),$$
and
$$I_H(\widehat{\mathrm{P}}_{ij}) \leq  I_H(\mathrm{P}_{ij}) + \abs{I_H(\mathrm{P}_{ij}) - I_H(\widehat{\mathrm{P}}_{ij})} < 2 \cdot I_H(\mathrm{P}_{ij}).$$
Thus, $\frac{1}{4} \leq \frac{I_{H^2}(\widehat{\mathrm{P}}_{ij})}{I_{H^2}(\mathrm{P}_{ij})} \leq 4$.

To prove (ii), suppose $I_{H^2}(\widehat{\mathrm{P}}_{ij}) \geq \frac{\epsilon^2}{n}$. Then $I_H(\widehat{\mathrm{P}}_{ij}) \geq \frac{\epsilon}{\sqrt{n}}$. We have by (\ref{ihdiff})
$$I_H(\mathrm{P}_{ij}) \geq  I_H(\widehat{\mathrm{P}}_{ij}) - \abs{I_H(\mathrm{P}_{ij}) - I_H(\widehat{\mathrm{P}}_{ij})} > \frac{1}{2} \cdot I_H(\widehat{\mathrm{P}}_{ij}),$$
and
$$I_H(\mathrm{P}_{ij}) \leq  I_H(\widehat{\mathrm{P}}_{ij}) + \abs{I_H(\mathrm{P}_{ij}) - I_H(\widehat{\mathrm{P}}_{ij})} < 2 \cdot I_H(\widehat{\mathrm{P}}_{ij}).$$
Thus, we still have $\frac{1}{4} \leq \frac{I_{H^2}(\widehat{\mathrm{P}}_{ij})}{I_{H^2}(\mathrm{P}_{ij})} \leq 4$.
\end{proof}

The next three lemmas prove various relationships among the measures $\mathrm{minmrg}$, $\mathrm{mindiag}$, $\mathrm{mindisc}$, and $I_{H^2}$, for a single pair of variables. Lemma~\ref{ihminmrg} implies that, for example, a pair of variables must be close to being independent in $H^2$ under $\mathrm{P}$ if at least one of them has a very biased binary distribution under $\mathrm{P}$.

\begin{lemma}
\label{ihminmrg}
For any $i,j \in [n]$, $I_{H^2}(\mathrm{P}_{ij}) \leq 2 \cdot \minmrg{\mathrm{P}_{ij}}$.
\end{lemma}
\begin{proof}
Without loss of generality, suppose $\minmrg{\mathrm{P}_{ij}} = \mathrm{P}_i(1)$. Notice that
\begin{align*}
\abs{\mathrm{P}_{ij}(1,1) - \mathrm{P}_{ij}^{\rm{(ind)}}(1,1)} &= \abs{\mathrm{P}_{ij}(1,1) - \mathrm{P}_i(1) \mathrm{P}_j(1)} \\
		&= \abs{\paren{\mathrm{P}_i(1) - \mathrm{P}_{ij}(1,-1)} - \paren{\mathrm{P}_i(1) - \mathrm{P}_i(1) \mathrm{P}_j(-1)}} \\
		&= \abs{\mathrm{P}_{ij}(1,-1) - \mathrm{P}_i(1) \mathrm{P}_j(-1)} \\
		&= \abs{\mathrm{P}_{ij}(1,-1) - \mathrm{P}_{ij}^{\rm{(ind)}}(1,-1)}.
\end{align*}
Continuing this way, we see that the value $\abs{\mathrm{P}_{ij}(x_i,x_j) - \mathrm{P}_{ij}^{\rm{(ind)}}(x_i,x_j)}$ is the same for all four combinations $x_i,x_j \in \{1,-1\}$.
Thus,
\begingroup
\allowdisplaybreaks
\begin{align*}
I_{H^2}(\mathrm{P}_{ij}) &= H^2(\mathrm{P}_{ij}, \mathrm{P}_{ij}^{\rm{(ind)}}) \\
		&\stackrel{(\ast)}{\leq} \dtv{\mathrm{P}_{ij}}{\mathrm{P}_{ij}^{\rm{(ind)}}} \\
		&= \frac{1}{2} \sum_{x_i,x_j = \pm 1} \abs{\mathrm{P}_{ij}(x_i,x_j) - \mathrm{P}_{ij}^{\rm{(ind)}}(x_i,x_j)} \\
		&= 2 \cdot \abs{\mathrm{P}_{ij}(1,1) - \mathrm{P}_{ij}^{\rm{(ind)}}(1,1)} \\
		&\leq 2 \cdot \mathrm{P}_i(1) \\
		&= 2 \cdot \minmrg{\mathrm{P}_{ij}},
\end{align*}
\endgroup
where $(\ast)$ follows from Lemma~\ref{tvvshel}.
\end{proof}

\begin{lemma}
\label{minmrgmindiag}
For any $i,j \in [n]$, $\abs{\minmrg{\mathrm{P}_i} - \minmrg{\mathrm{P}_j}} \leq \mindiag{\mathrm{P}_{ij}}$.
\end{lemma}
\begin{proof}
Without loss of generality, suppose that $\mindiag{\mathrm{P}_{ij}} = \mathrm{P}(X_i=-X_j)$. We have
$$\mathrm{P}_j(1) = \mathrm{P}_i(1) - \mathrm{P}_{ij}(1,-1) + \mathrm{P}_{ij}(-1,1),$$
which implies that
$$\abs{\mathrm{P}_j(1) - \mathrm{P}_i(1)} \leq \mathrm{P}(X_i=-X_j) = \mindiag{\mathrm{P}_{ij}}.$$
Similarly, $\abs{\mathrm{P}_j(-1) - \mathrm{P}_i(-1)} \leq \mindiag{\mathrm{P}_{ij}}$. Recall that $\minmrg{\mathrm{P}_i} = \min{\brac{\mathrm{P}_i(1), \mathrm{P}_i(-1)}}$ and $\minmrg{\mathrm{P}_j} = \min{\brac{\mathrm{P}_j(1), \mathrm{P}_j(-1)}}$. The desired conclusion follows because perturbing each argument of the $\min$ function by at most $\mindiag{\mathrm{P}_{ij}}$ changes the overall min value by at most $\mindiag{\mathrm{P}_{ij}}$.
\end{proof}

\begin{lemma}
\label{minmrgmindiagmindisc}
For any $i,j \in [n]$, $\mindisc{\mathrm{P}_{ij}} \geq 1 - 2 \cdot \frac{\mindiag{\mathrm{P}_{ij}}}{\minmrg{\mathrm{P}_{ij}}}$.
\end{lemma}
\begin{proof}
Without loss of generality, suppose that $\mindiag{\mathrm{P}_{ij}} = \mathrm{P}(X_i=-X_j)$. Then
$$\mathrm{P}_{i|j}(1|1) = \frac{\mathrm{P}_{ij}(1,1)}{\mathrm{P}_j(1)} = 1 - \frac{\mathrm{P}_{ij}(-1,1)}{\mathrm{P}_j(1)} \geq 1 - \frac{\mindisc{\mathrm{P}_{ij}}}{\minmrg{\mathrm{P}_{ij}}},$$
and
$$\mathrm{P}_{i|j}(1|-1) = \frac{\mathrm{P}_{ij}(1,-1)}{\mathrm{P}_j(-1)} \leq \frac{\mindisc{\mathrm{P}_{ij}}}{\minmrg{\mathrm{P}_{ij}}}.$$
So $\abs{\mathrm{P}_{i|j}(1|1) - \mathrm{P}_{i|j}(1|-1)} \geq 1 - 2 \cdot \frac{\mindiag{\mathrm{P}_{ij}}}{\minmrg{\mathrm{P}_{ij}}}$. By symmetry the same lower bound also holds for $\abs{\mathrm{P}_{j|i}(1|1) - \mathrm{P}_{j|i}(1|-1)}$. So we have
\begin{align*}
\mindisc{\mathrm{P}_{ij}} &= \min{\brac{\abs{\mathrm{P}_{i|j}(1|1) - \mathrm{P}_{i|j}(1|-1)}, \abs{\mathrm{P}_{j|i}(1|1) - \mathrm{P}_{j|i}(1|-1)}}} \\
		&\geq 1 - 2 \cdot \frac{\mindiag{\mathrm{P}_{ij}}}{\minmrg{\mathrm{P}_{ij}}}.
\end{align*}
\end{proof}

The next lemma is in some sense an analog of the multiplicativity of $\alpha$-values along a path in $T$ (Lemma~\ref{multiplicativity}), for the measure $\mathrm{mindiag}$. Roughly speaking, Lemma~\ref{mindiagmarkov} (ii) says that closer pairs along a path in $T$ are stronger as measured by $\mathrm{mindiag}$ (for $\alpha$-values for the symmetric case the analogous relationship is $\min \brac{ \abs{\alpha_{ij}}, \abs{\alpha_{jk}}} \geq \abs{\alpha_{ik}}$), and (iii) controls the manner in which the strengths of the pairs (as measured by $\mathrm{mindiag}$) concatenate/decompose along a path in $T$.

\begin{lemma}
\label{mindiagmarkov}
If $i,j,k$ lie on a path in $T$, and $\minmrg{\mathrm{P}_{ik}} > 8 \cdot \mindiag{\mathrm{P}_{ik}}$, then
\begin{itemize}
\item[(i)] $\minmrg{\mathrm{P}_j} \geq \minmrg{\mathrm{P}_{ik}} - \mindiag{\mathrm{P}_{ik}}$;
\item[(ii)] $\max{\brac{\mindiag{\mathrm{P}_{ij}}, \mindiag{\mathrm{P}_{jk}}}} \leq \mindiag{\mathrm{P}_{ik}}$;
\item[(iii)] $\mindiag{\mathrm{P}_{ij}} + \mindiag{\mathrm{P}_{jk}} \leq \frac{7}{6} \mindiag{\mathrm{P}_{ik}}$.
\end{itemize}
\end{lemma}
\begin{proof}
Note that the condition implies $\mindiag{\mathrm{P}_{ik}} < \frac{1}{8}$.

Recall that $\mindiag{\mathrm{P}_{ik}} = \min \brac{\mathrm{P}(X_i=X_k), \mathrm{P}(X_i=-X_k)}$. Without loss of generality, suppose that $\mindiag{\mathrm{P}_{ik}} = \mathrm{P}(X_i=-X_k)$. The other case is symmetric, and can be reduced to the current case by switching the labels $1$ and $-1$ for $X_k$. Note that the relabeling does not change any of the $\minmrg{\cdot}$'s and $\mindiag{\cdot}$'s.

We can also suppose without loss of generality that $\mathrm{P}_{ij}(1,1) \geq \mathrm{P}_{ij}(-1,1)$. The other case is symmetric, and can be reduced to the current case by switching the labels $1$ and $-1$ for both $X_i$ and $X_k$ (so that the assumption $\mindiag{\mathrm{P}_{ik}} = \mathrm{P}(X_i=-X_k)$ in the previous paragraph is unaffected). Note that the relabeling does not change any of the $\minmrg{\cdot}$'s and $\mindiag{\cdot}$'s.

We first argue that we must have $\mathrm{P}_{ij}(1,-1) \leq \mathrm{P}_{ij}(-1,-1)$. Suppose for the sake of contradiction that $\mathrm{P}_{ij}(1,-1) > \mathrm{P}_{ij}(-1,-1)$. Then we have
\begingroup
\allowdisplaybreaks
\begin{align*}
\mindiag{\mathrm{P}_{ik}} &= \mathrm{P}(X_i = -X_k) \\
		&= \mathrm{P}_{ik}(1,-1) + \mathrm{P}_{ik}(-1,1) \\
		&= \mathrm{P}_{ijk}(1,1,-1) + \mathrm{P}_{ijk}(1,-1,-1) + \mathrm{P}_{ijk}(-1,1,1) + \mathrm{P}_{ijk}(-1,-1,1) \\
		&\begin{multlined}
		= \mathrm{P}_{ij}(1,1)\mathrm{P}_{k|j}(-1|1) + \mathrm{P}_{ij}(1,-1)\mathrm{P}_{k|j}(-1|-1) \\
		\qquad \qquad \qquad + \mathrm{P}_{ij}(-1,1)\mathrm{P}_{k|j}(1|1) + \mathrm{P}_{ij}(-1,-1)\mathrm{P}_{k|j}(1|-1)
		\end{multlined}\\
		&\begin{multlined}
		\geq \mathrm{P}_{ij}(-1,1)\mathrm{P}_{k|j}(-1|1) + \mathrm{P}_{ij}(-1,-1)\mathrm{P}_{k|j}(-1|-1) \\
		\qquad \qquad \qquad + \mathrm{P}_{ij}(-1,1)\mathrm{P}_{k|j}(1|1) + \mathrm{P}_{ij}(-1,-1)\mathrm{P}_{k|j}(1|-1)
		\end{multlined}\\
		&= \mathrm{P}_{ij}(-1,1) + \mathrm{P}_{ij}(-1,-1) \\
		&= \mathrm{P}_i(-1) \\
		&\geq \minmrg{\mathrm{P}_{ik}},
\end{align*}
\endgroup
contradicting the assumption that $\minmrg{\mathrm{P}_{ik}} > 8 \cdot \mindiag{\mathrm{P}_{ik}}$. So $\mathrm{P}_{ij}(1,-1) \leq \mathrm{P}_{ij}(-1,-1)$.

Consequently, we have
\begingroup
\allowdisplaybreaks
\begin{align*}
\mindiag{\mathrm{P}_{ik}} &= \mathrm{P}(X_i = -X_k) \\
		&= \mathrm{P}_{ik}(1,-1) + \mathrm{P}_{ik}(-1,1) \\
		&= \mathrm{P}_{ijk}(1,1,-1) + \mathrm{P}_{ijk}(1,-1,-1) + \mathrm{P}_{ijk}(-1,1,1) + \mathrm{P}_{ijk}(-1,-1,1) \\
		&\begin{multlined}
		= \mathrm{P}_{ij}(1,1)\mathrm{P}_{k|j}(-1|1) + \mathrm{P}_{ij}(1,-1)\mathrm{P}_{k|j}(-1|-1) \\
		\qquad \qquad \qquad + \mathrm{P}_{ij}(-1,1)\mathrm{P}_{k|j}(1|1) + \mathrm{P}_{ij}(-1,-1)\mathrm{P}_{k|j}(1|-1)
		\end{multlined}\\
		&\begin{multlined}
		\geq \mathrm{P}_{ij}(-1,1)\mathrm{P}_{k|j}(-1|1) + \mathrm{P}_{ij}(1,-1)\mathrm{P}_{k|j}(-1|-1) \\
		\qquad \qquad \qquad + \mathrm{P}_{ij}(-1,1)\mathrm{P}_{k|j}(1|1) + \mathrm{P}_{ij}(1,-1)\mathrm{P}_{k|j}(1|-1)
		\end{multlined}\\
		&= \mathrm{P}_{ij}(-1,1) + \mathrm{P}_{ij}(1,-1) \\
		&= \mathrm{P}(X_i = -X_j),
\end{align*}
\endgroup
which implies (recall $\mindiag{\mathrm{P}_{ik}} < \frac{1}{8}$)
\begin{equation}
\label{mindiagijleqmindiagik}
\mindiag{\mathrm{P}_{ij}} = \mathrm{P}(X_i = -X_j) \leq \mindiag{\mathrm{P}_{ik}}.
\end{equation}
Lemma~\ref{minmrgmindiag} and (\ref{mindiagijleqmindiagik}) then imply that
$$\minmrg{\mathrm{P}_j} \geq \minmrg{\mathrm{P}_i} - \mindiag{\mathrm{P}_{ij}} \geq \minmrg{\mathrm{P}_{ik}} - \mindiag{\mathrm{P}_{ik}},$$
proving (i).

We can also use (\ref{mindiagijleqmindiagik}) to lower bound $\mathrm{P}_{ij}(1,1)$:
\begin{align}
\label{jmrg}
\begin{split}
\mathrm{P}_{ij}(1,1) &= \mathrm{P}_i(1) - \mathrm{P}_{ij}(1,-1) \\
		&\geq \minmrg{\mathrm{P}_{ik}} - \mindiag{\mathrm{P}_{ij}} \\
		&\geq \minmrg{\mathrm{P}_{ik}} - \mindiag{\mathrm{P}_{ik}} \\
		&> 7 \cdot \mindiag{\mathrm{P}_{ik}}.
\end{split}
\end{align}

At this point, we can repeat the same argument above for $\mathrm{P}_{jk}$ instead of $\mathrm{P}_{ij}$. The only caveat is that we can not assume without loss of generality for free that $\mathrm{P}_{jk}(1,1) \geq \mathrm{P}_{jk}(1,-1)$ - our earlier ``without loss of generality assumptions" $\mindiag{\mathrm{P}_{ij}} = \mathrm{P}(X_i=-X_j)$ and $\mathrm{P}_{ij}(1,1) \geq \mathrm{P}_{ij}(-1,1)$ took that freedom away from us. We have to prove $\mathrm{P}_{jk}(1,1) \geq \mathrm{P}_{jk}(1,-1)$.

To that end, suppose for the sake of contradiction that $\mathrm{P}_{jk}(1,1) < \mathrm{P}_{jk}(1,-1)$. Earlier we proved (\ref{mindiagijleqmindiagik}) starting from $\mindiag{\mathrm{P}_{ik}} = \mathrm{P}(X_i=-X_k)$ and $\mathrm{P}_{ij}(1,1) \geq \mathrm{P}_{ij}(-1,1)$. It is easy to see that a similar argument (up to symmetry and relabeling) starting from $\mindiag{\mathrm{P}_{ik}} = \mathrm{P}(X_i=-X_k)$ and $\mathrm{P}_{jk}(1,1) < \mathrm{P}_{jk}(1,-1)$ would lead to
\begin{align}
\label{mindiagjkleqmindiagikwrong}
\mindiag{\mathrm{P}_{jk}} = \mathrm{P}(X_j = X_k) \leq \mindiag{\mathrm{P}_{ik}},
\end{align}
and so
\begingroup
\allowdisplaybreaks
\begin{align*}
\mindiag{\mathrm{P}_{ik}} &= \mathrm{P}(X_i = -X_k) \\
		&\geq \mathrm{P}_{ijk}(1,1,-1) \\
		&= \mathrm{P}_{ij}(1,1) \mathrm{P}_{k|j}(-1|1) \\
		& = \mathrm{P}_{ij}(1,1) \paren{1 - \frac{\mathrm{P}_{jk}(1,1)}{\mathrm{P}_j(1)}} \\
		&\stackrel{(\ast)}{\geq} 7 \cdot \mindiag{\mathrm{P}_{ik}} \cdot \paren{1 - \frac{\mindiag{\mathrm{P}_{ik}}}{7 \cdot \mindiag{\mathrm{P}_{ik}}}} \\
		&= 6 \cdot \mindiag{\mathrm{P}_{ik}},
\end{align*}
\endgroup
where ($\ast$) follows from (\ref{jmrg}), (\ref{mindiagjkleqmindiagikwrong}). This is a contradiction.

Thus, we must in fact have $\mathrm{P}_{jk}(1,1) \geq \mathrm{P}_{jk}(1,-1)$, and a similar argument as the one leading to (\ref{mindiagijleqmindiagik}) gives
\begin{align}
\label{mindiagjkleqmindiagik}
\mindiag{\mathrm{P}_{jk}} = \mathrm{P}(X_j = -X_k) \leq \mindiag{\mathrm{P}_{ik}}.
\end{align}
Note that (\ref{mindiagijleqmindiagik}) and (\ref{mindiagjkleqmindiagik}) prove (ii).

Lastly, we have
\begingroup
\allowdisplaybreaks
\begin{align*}
\mindiag{\mathrm{P}_{ik}} &= \mathrm{P}(X_i = -X_k) \\
		&= \mathrm{P}_{ik}(1,-1) + \mathrm{P}_{ik}(-1,1) \\
		&= \mathrm{P}_{ijk}(1,1,-1) + \mathrm{P}_{ijk}(1,-1,-1) + \mathrm{P}_{ijk}(-1,1,1) + \mathrm{P}_{ijk}(-1,-1,1) \\
		&\begin{multlined}
		= \mathrm{P}_{jk}(1,-1)\mathrm{P}_{i|j}(1|1) + \mathrm{P}_{ij}(1,-1)\mathrm{P}_{k|j}(-1|-1) \\
		\qquad \qquad \qquad + \mathrm{P}_{ij}(-1,1)\mathrm{P}_{k|j}(1|1) + \mathrm{P}_{jk}(-1,1)\mathrm{P}_{i|j}(-1|-1)
		\end{multlined}\\
		&\begin{multlined}
		= \mathrm{P}_{jk}(1,-1)\frac{\mathrm{P}_{ij}(1,1)}{\mathrm{P}_j(1)} + \mathrm{P}_{ij}(1,-1)\frac{\mathrm{P}_{jk}(-1,-1)}{\mathrm{P}_j(-1)} \\
		\qquad \qquad \qquad + \mathrm{P}_{ij}(-1,1)\frac{\mathrm{P}_{jk}(1,1)}{\mathrm{P}_j(1)} + \mathrm{P}_{jk}(-1,1)\frac{\mathrm{P}_{ij}(-1,-1)}{\mathrm{P}_j(-1)}
		\end{multlined}\\
		&\begin{multlined}
		= \mathrm{P}_{jk}(1,-1)\paren{1 - \frac{\mathrm{P}_{ij}(-1,1)}{\mathrm{P}_j(1)}} + \mathrm{P}_{ij}(1,-1)\paren{1 - \frac{\mathrm{P}_{jk}(-1,1)}{\mathrm{P}_j(-1)}} \\
		\qquad \qquad \quad + \mathrm{P}_{ij}(-1,1)\paren{1 - \frac{\mathrm{P}_{jk}(1,-1)}{\mathrm{P}_j(1)}} + \mathrm{P}_{jk}(-1,1)\paren{1 - \frac{\mathrm{P}_{ij}(1,-1)}{\mathrm{P}_j(-1)}}
		\end{multlined}\\
		&\begin{multlined}
		\stackrel{(\ast)}{\geq} \mathrm{P}_{jk}(1,-1)\paren{1 - \frac{\mindiag{\mathrm{P}_{ik}}}{7 \cdot \mindiag{\mathrm{P}_{ik}}}} + \mathrm{P}_{ij}(1,-1)\paren{1 - \frac{\mindiag{\mathrm{P}_{ik}}}{7 \cdot \mindiag{\mathrm{P}_{ik}}}} \\
		+ \mathrm{P}_{ij}(-1,1)\paren{1 -\frac{\mindiag{\mathrm{P}_{ik}}}{7 \cdot \mindiag{\mathrm{P}_{ik}}}} + \mathrm{P}_{jk}(-1,1)\paren{1 - \frac{\mindiag{\mathrm{P}_{ik}}}{7 \cdot \mindiag{\mathrm{P}_{ik}}}}
		\end{multlined}\\
		&= \frac{6}{7} \paren{\mathrm{P}_{ij}(1,-1) + \mathrm{P}_{ij}(-1,1) + \mathrm{P}_{jk}(1,-1) + \mathrm{P}_{jk}(-1,1)} \\
		&= \frac{6}{7} \paren{\mindiag{\mathrm{P}_{ij}} + \mindiag{\mathrm{P}_{jk}}},
\end{align*}
\endgroup
where $(\ast)$ follows from (\ref{mindiagijleqmindiagik}), (\ref{mindiagjkleqmindiagik}), and (i) (which implies $\minmrg{\mathrm{P}_j} > 7 \cdot \mindiag{\mathrm{P}_{ik}}$). This proves (iii).
\end{proof}

The next lemma generalizes Lemma~\ref{mindiagmarkov} (ii) to four nodes that lie on a path in $T$.

\begin{lemma}
\label{mindiagmarkov4}
If $h,i,j,k$ lie on a path in $T$, and $\minmrg{\mathrm{P}_{hk}} > 9 \cdot \mindiag{\mathrm{P}_{hk}}$, then $\mindiag{\mathrm{P}_{ij}} \leq \mindiag{\mathrm{P}_{hk}}$.
\end{lemma}
\begin{proof}
We can apply Lemma~\ref{mindiagmarkov} (i), (ii) to $h,i,k$ to get
$$\minmrg{\mathrm{P}_i} \geq \minmrg{\mathrm{P}_{hk}} - \mindiag{\mathrm{P}_{hk}} > 8 \cdot \mindiag{\mathrm{P}_{hk}}$$
and
$$\mindiag{\mathrm{P}_{ik}} \leq \mindiag{\mathrm{P}_{hk}}.$$
So $\minmrg{\mathrm{P}_i} > 8 \cdot \mindiag{\mathrm{P}_{ik}}$.

Also $\minmrg{\mathrm{P}_k} \geq \minmrg{\mathrm{P}_{hk}} > 9 \cdot \mindiag{\mathrm{P}_{hk}} \geq 9 \cdot \mindiag{\mathrm{P}_{ik}}$. So we have $\minmrg{\mathrm{P}_{ik}} > 8 \cdot \mindiag{\mathrm{P}_{ik}}$.

Therefore, we can apply Lemma~\ref{mindiagmarkov} (ii) to $i,j,k$ to get
$$\mindiag{\mathrm{P}_{ij}} \leq \mindiag{\mathrm{P}_{ik}} \leq \mindiag{\mathrm{P}_{hk}}.$$
\end{proof}

The next lemma is an analog of Lemma~\ref{mindiagmarkov} (ii), for the measure $\mathrm{mindisc}$. Roughly speaking, it says that closer pairs along a path in $T$ are stronger as measured by $\mathrm{mindisc}$.

\begin{lemma}
\label{mindiscmarkov}
If $i,j,k$ lie on a path in $T$, then we have $\mindisc{\mathrm{P}_{ij}} \geq \mindisc{\mathrm{P}_{ik}}$ and $\mindisc{\mathrm{P}_{jk}} \geq \mindisc{\mathrm{P}_{ik}}$.
\end{lemma}
\begin{proof}
Recall that $\mindisc{\mathrm{P}_{ik}} = \min \brac{\abs{\mathrm{P}_{i|k} (1|1) - \mathrm{P}_{i|k}(1|-1)}, \abs{\mathrm{P}_{k|i} (1|1) - \mathrm{P}_{k|i}(1|-1)}}.$

We have
\begingroup
\allowdisplaybreaks
\begin{align}
\begin{split}
\label{igivenk}
\mindisc{\mathrm{P}_{ik}} &\leq \abs{\mathrm{P}_{i|k} (1|1) - \mathrm{P}_{i|k}(1|-1)} \\
		&= \abs{\mathrm{P}_{ij|k} (1,1|1) + \mathrm{P}_{ij|k} (1,-1|1) - \mathrm{P}_{ij|k}(1,1|-1) - \mathrm{P}_{ij|k}(1,-1|-1)} \\
		&\begin{multlined}
		= \bigl| \mathrm{P}_{i|j} (1|1) \mathrm{P}_{j|k} (1|1) + \mathrm{P}_{i|j} (1|-1) \mathrm{P}_{j|k} (-1|1) \\
		\qquad \qquad \qquad - \mathrm{P}_{i|j} (1|1) \mathrm{P}_{j|k} (1|-1) - \mathrm{P}_{i|j} (1|-1) \mathrm{P}_{j|k} (-1|-1) \bigr|
		\end{multlined}\\
		&\begin{multlined}
		= \bigl| \mathrm{P}_{i|j} (1|1) \mathrm{P}_{j|k} (1|1) + \mathrm{P}_{i|j} (1|-1) \cdot \paren{1 - \mathrm{P}_{j|k} (1|1)} \\
		\qquad \qquad - \mathrm{P}_{i|j} (1|1) \mathrm{P}_{j|k} (1|-1) - \mathrm{P}_{i|j} (1|-1) \cdot \paren{1 - \mathrm{P}_{j|k} (1|-1)} \bigr|
		\end{multlined}\\
		&\begin{multlined}
		= \bigl| \mathrm{P}_{i|j} (1|1) \mathrm{P}_{j|k} (1|1) - \mathrm{P}_{i|j} (1|-1) \mathrm{P}_{j|k} (1|1) \\
		\qquad \qquad \qquad - \mathrm{P}_{i|j} (1|1) \mathrm{P}_{j|k} (1|-1) + \mathrm{P}_{i|j} (1|-1) \mathrm{P}_{j|k} (1|-1) \bigr|
		\end{multlined}\\
		&= \abs{\paren{\mathrm{P}_{i|j} (1|1) - \mathrm{P}_{i|j} (1|-1)} \cdot \paren{\mathrm{P}_{j|k} (1|1) - \mathrm{P}_{j|k} (1|-1)}}.
\end{split}
\end{align}
\endgroup
Symmetrically, we can also get
\begingroup
\allowdisplaybreaks
\begin{align}
\begin{split}
\label{kgiveni}
\mindisc{\mathrm{P}_{ik}} &\leq \abs{\mathrm{P}_{k|i} (1|1) - \mathrm{P}_{k|i}(1|-1)} \\
		&= \abs{\paren{\mathrm{P}_{k|j} (1|1) - \mathrm{P}_{k|j} (1|-1)} \cdot \paren{\mathrm{P}_{j|i} (1|1) - \mathrm{P}_{j|i} (1|-1)}}.
\end{split}
\end{align}
\endgroup
Note that (\ref{igivenk}) yields $\mindisc{\mathrm{P}_{ik}} \leq \abs{\mathrm{P}_{j|k} (1|1) - \mathrm{P}_{j|k} (1|-1)}$ and (\ref{kgiveni}) yields $\mindisc{\mathrm{P}_{ik}} \leq \abs{\mathrm{P}_{k|j} (1|1) - \mathrm{P}_{k|j} (1|-1)}$. Combining the two we have
\begin{align*}
\mindisc{\mathrm{P}_{jk}} &= \min \brac{\abs{\mathrm{P}_{j|k} (1|1) - \mathrm{P}_{j|k}(1|-1)}, \abs{\mathrm{P}_{k|j} (1|1) - \mathrm{P}_{k|j}(1|-1)}} \\
		&\geq \mindisc{\mathrm{P}_{ik}}.
\end{align*}
The proof for $\mindisc{\mathrm{P}_{ij}} \geq \mindisc{\mathrm{P}_{ik}}$ is symmetric.
\end{proof}

The next lemma is an analog of Lemma~\ref{mindiagmarkov} (ii) (or Lemma~\ref{mindiscmarkov}), for the measure $I_{H^2}$. Roughly speaking, it says that closer pairs along a path in $T$ are stronger as measured by $I_{H^2}$.

\begin{lemma}
\label{ihmarkov}
If $i,j,k$ lie on a path in $T$, then we have both $I_{H^2}(\mathrm{P}_{ij}) \geq I_{H^2}(\mathrm{P}_{ik})$ and $I_{H^2}(\mathrm{P}_{jk}) \geq I_{H^2}(\mathrm{P}_{ik})$.
\end{lemma}
\begin{proof}
Consider the convex function $f(t) = 1 - \sqrt{t}$. We have
\begingroup
\allowdisplaybreaks
\begin{align*}
I_{H^2}(\mathrm{P}_{ij}) &= H^2(\mathrm{P}_{ij}, \mathrm{P}_{ij}^{\rm{(ind)}}) \\
		&= 1 - \sum_{x_i,x_j = \pm 1} \sqrt{\mathrm{P}_i(x_i) \mathrm{P}_j(x_j) \mathrm{P}_{ij}(x_i,x_j)} \\
		&= \sum_{x_i,x_j = \pm 1} \paren{\mathrm{P}_i(x_i) \mathrm{P}_j(x_j) - \sqrt{\mathrm{P}_i(x_i) \mathrm{P}_j(x_j) \mathrm{P}_{ij}(x_i,x_j)}} \\
		&= \sum_{x_i,x_j = \pm 1} \mathrm{P}_i(x_i) \mathrm{P}_j(x_j) \cdot f\paren{\frac{\mathrm{P}_{ij}(x_i,x_j)}{\mathrm{P}_i(x_i) \mathrm{P}_j(x_j)}} \\
		&= \sum_{x_i,x_j,x_k = \pm 1} \mathrm{P}_i(x_i) \mathrm{P}_{jk}(x_j,x_k) \cdot f\paren{\frac{\mathrm{P}_{ij}(x_i,x_j)}{\mathrm{P}_i(x_i) \mathrm{P}_j(x_j)}} \\
		&= \sum_{x_i,x_k = \pm 1} \mathrm{P}_i(x_i) \mathrm{P}_k(x_k) \sum_{x_j = \pm 1} \mathrm{P}_{j|k}(x_j|x_k) \cdot f\paren{\frac{\mathrm{P}_{ij}(x_i,x_j)}{\mathrm{P}_i(x_i) \mathrm{P}_j(x_j)}} \\
		&\stackrel{(\ast)}{\geq} \sum_{x_i,x_k = \pm 1} \mathrm{P}_i(x_i) \mathrm{P}_k(x_k) \cdot f\paren{\sum_{x_j = \pm 1} \frac{\mathrm{P}_{j|k}(x_j|x_k) \mathrm{P}_{ij}(x_i,x_j)}{\mathrm{P}_i(x_i) \mathrm{P}_j(x_j)}} \\
		&= \sum_{x_i,x_k = \pm 1} \mathrm{P}_i(x_i) \mathrm{P}_k(x_k) \cdot f\paren{\sum_{x_j = \pm 1} \frac{\mathrm{P}_{jk}(x_j,x_k) \mathrm{P}_{ij}(x_i,x_j)}{\mathrm{P}_i(x_i) \mathrm{P}_j(x_j) \mathrm{P}_k(x_k)}} \\
		&= \sum_{x_i,x_k = \pm 1} \mathrm{P}_i(x_i) \mathrm{P}_k(x_k) \cdot f\paren{\sum_{x_j = \pm 1} \frac{\mathrm{P}_{k|j}(x_k|x_j) \mathrm{P}_{ij}(x_i,x_j)}{\mathrm{P}_i(x_i) \mathrm{P}_k(x_k)}} \\
		&= \sum_{x_i,x_k = \pm 1} \mathrm{P}_i(x_i) \mathrm{P}_k(x_k) \cdot f\paren{\sum_{x_j = \pm 1} \frac{\mathrm{P}_{ijk}(x_i,x_j,x_k)}{\mathrm{P}_i(x_i) \mathrm{P}_k(x_k)}} \\
		&= \sum_{x_i,x_k = \pm 1} \mathrm{P}_i(x_i) \mathrm{P}_k(x_k) \cdot f\paren{\frac{\mathrm{P}_{ik}(x_i,x_k)}{\mathrm{P}_i(x_i) \mathrm{P}_k(x_k)}} \\
		&= \sum_{x_i,x_k = \pm 1} \paren{\mathrm{P}_i(x_i) \mathrm{P}_k(x_k) - \sqrt{\mathrm{P}_i(x_i) \mathrm{P}_k(x_k) \mathrm{P}_{ik}(x_i,x_k)}} \\
		&= 1 - \sum_{x_i,x_k = \pm 1} \sqrt{\mathrm{P}_i(x_i) \mathrm{P}_k(x_k) \mathrm{P}_{ik}(x_i,x_k)} \\
		&= H^2(\mathrm{P}_{ik}, \mathrm{P}_{ik}^{\rm{(ind)}}) \\
		&= I_{H^2}(\mathrm{P}_{ik}),
\end{align*}
\endgroup
where ($\ast$) follows from Jensen's Inequality. The proof for $I_{H^2}(\mathrm{P}_{jk}) \geq I_{H^2}(\mathrm{P}_{ik})$ is similar.
\end{proof}

The next lemma is central in our analysis for the general case. Part (i), for example, provides a necessary condition ($H^2(\mathrm{P}_{ijk}, \mathrm{P}_{i\wideparen{ \,~~ j-}k}) \leq 22 \frac{\epsilon^2}{n}$) for Algorithm~\ref{algo:genalgo} to make a structural mistake, in the case of $i,j,k$ lying on a path in $T$, by ``reversing the order'' between $(i,j)$ and $(i,k)$. This not only bounds the error in $H^2$ when Algorithm~\ref{algo:genalgo} does make the structural mistake, but also provides a guideline on how to classify the edges and define the layers (which allows us to speak more precisely about the manners in which $\widehat{T}$ might differ from $T$) in Section~\ref{sec:genstructure}. The proof for this lemma is the only place in the entire argument where we need those implications of strong 4-consistency that are beyond those implies by just 3-consistency.

\begin{lemma}
\label{genhelwrongedge}
\begin{itemize}
\item[(i)] If $i,j,k$ lie on a path in $T$, and $\mathrm{I}^{\widehat{\mathrm{P}}}(X_i;X_j) \leq \mathrm{I}^{\widehat{\mathrm{P}}}(X_i;X_k)$, then $H^2(\mathrm{P}_{ijk}, \mathrm{P}_{i\wideparen{ \,~~ j-}k}) \leq 22 \frac{\epsilon^2}{n}$.
\item[(ii)] If $h,i,j,k$ lie on a path in $T$, and $\mathrm{I}^{\widehat{\mathrm{P}}}(X_i;X_j) \leq \mathrm{I}^{\widehat{\mathrm{P}}}(X_h;X_k)$, then $H^2(\mathrm{P}_{ijk}, \mathrm{P}_{i\wideparen{ \,~~ j-}k}) \leq 62 \frac{\epsilon^2}{n}$, $H^2(\mathrm{P}_{hik}, \mathrm{P}_{h\wideparen{-i \,~~ }k}) \leq 62 \frac{\epsilon^2}{n}$, $H^2(\mathrm{P}_{hijk}, \mathrm{P}_{h\wideparen{-i \,~~ j-}k}) \leq 248 \frac{\epsilon^2}{n}$ (symmetrically, $H^2(\mathrm{P}_{hij}, \mathrm{P}_{h\wideparen{-i \,~~ }j}) \leq 62 \frac{\epsilon^2}{n}$ and $H^2(\mathrm{P}_{hjk}, \mathrm{P}_{h\wideparen{ \,~~ j-}k}) \leq 62 \frac{\epsilon^2}{n}$).
\end{itemize}
\end{lemma}
\begin{proof}
We will prove (i) and (ii) in parallel in four steps. The first step provides a general inequality relating the mutual information and the squared Hellinger distance. The second step contains some numerical inequalities. The third step provides upper bounds on some estimated conditional mutual informations, using only the assumption that $i,j,k$ (or $h,i,j,k$) lie on a path in $T$. The fourth step finishes the proof by taking into account the assumed inequalities between estimated mutual informations.

\underline{\textit{Step 1}}\textit{:} In this step we show that for any $u,v,w \in [n]$ we have \footnote{It is easy to see that (\ref{ih23}) and (\ref{ih24}) are still true if $\widehat{\mathrm{P}}$ is replaced by any arbitrary joint distribution on $[n]$. We stated them only for $\widehat{\mathrm{P}}$ because that is the only distribution we will apply (\ref{ih23}) and (\ref{ih24}) to.}
\begin{equation}
\label{ih23}
\mathrm{I}^{\widehat{\mathrm{P}}}(X_u;X_v|X_w) = KL(\widehat{\mathrm{P}}_{uvw} \mid \mid \widehat{\mathrm{P}}_{u\wideparen{ \,~~ v-}w}) \geq 2H^2(\widehat{\mathrm{P}}_{uvw}, \widehat{\mathrm{P}}_{u\wideparen{ \,~~ v-}w}),
\end{equation}
and that for any $t,u,v,w \in [n]$ we have
\begin{equation}
\label{ih24}
\mathrm{I}^{\widehat{\mathrm{P}}}(X_t,X_u;X_v|X_w) = KL(\widehat{\mathrm{P}}_{tuvw} \mid \mid \widehat{\mathrm{P}}_{(tu)\wideparen{ \,~~ v-}w}) \geq 2H^2(\widehat{\mathrm{P}}_{tuvw}, \widehat{\mathrm{P}}_{(tu)\wideparen{ \,~~ v-}w}).
\end{equation}

To prove (\ref{ih23}), first note that
\begin{align*}
\mathrm{I}^{\widehat{\mathrm{P}}}(X_u;X_v|X_w) &= \sum_{x_w = \pm 1} \brac{\widehat{\mathrm{P}}_w(x_w) \sum_{x_u, x_v = \pm 1} \widehat{\mathrm{P}}_{uv|w}(x_u,x_v|x_w) \ln{\frac{\widehat{\mathrm{P}}_{uv|w}(x_u,x_v|x_w)}{\widehat{\mathrm{P}}_{u|w}(x_u|x_w) \widehat{\mathrm{P}}_{v|w}(x_v|x_w)}} } \\
		&= \sum_{x_u, x_v, x_w = \pm 1} \widehat{\mathrm{P}}_{uvw}(x_u,x_v,x_w) \ln{\frac{\widehat{\mathrm{P}}_{uvw}(x_u,x_v,x_w)}{\widehat{\mathrm{P}}_{u|w}(x_u|x_w) \widehat{\mathrm{P}}_{v|w}(x_v|x_w) \widehat{\mathrm{P}}_w(x_w)}} \\
		&= KL(\widehat{\mathrm{P}}_{uvw} \mid \mid \widehat{\mathrm{P}}_{u\wideparen{ \,~~ v-}w}).
\end{align*}
This proves the equality part of (\ref{ih23}). The inequality part follows from the general fact that whenever $\mathrm{p}' = (p'_1, \ldots, p'_L)$ and $\mathrm{p}'' = (p''_1, \ldots, p''_L)$ are discrete distributions over a domain of size $L$, we have
\begin{align*}
KL(\mathrm{p}' \mid \mid \mathrm{p}'') &= \sum_{l=1}^L p'_l \ln{\frac{p'_l}{p''_l}} = -\sum_{l=1}^L p'_l \ln{\frac{p''_l}{p'_l}} = -\sum_{l=1}^L p'_l \ln{\frac{p'_l p''_l}{{p'_l}^2}} = -2 \sum_{l=1}^L p'_l \ln{\frac{\sqrt{p'_l p''_l}}{p'_l}} \\
		&\stackrel{(\ast)}{\geq} -2 \ln{\paren{\sum_{l=1}^L \sqrt{p'_l p''_l}}} = -2 \ln{\paren{1 - H^2(\mathrm{p}',\mathrm{p}'')}} \stackrel{(\ast \ast)}{\geq} 2 H^2(\mathrm{p}',\mathrm{p}''),
\end{align*}
where $(\ast)$ follows from Jensen's Inequality and the concavity of $\ln(t)$, and $(\ast \ast)$ follows from the fact that $\ln{\paren{1-t}} \leq -t$ for $t \leq 1$.

The proof of (\ref{ih24}) is similar --- just treat the pair $(X_t,X_u)$ as one (super) random variable $X_{tu}$.

\underline{\textit{Step 2}}\textit{:} In this step we show that for $a,b \geq 0$ we have
\begin{align}
\label{numericalcases}
a \ln{\frac{a}{b}} \leq
\begin{cases}
(a-b) + a \ln{\frac{a}{b}} & \mbox{if } b < \frac{1}{2} a \\
(a-b) + 9 \paren{\sqrt{a} - \sqrt{b}}^2 & \mbox{if } \frac{1}{2} a \leq b \leq \frac{3}{2} a \\
(a-b) + 100 \paren{\sqrt{a} - \sqrt{b}}^2 & \mbox{if } b > \frac{3}{2} a
\end{cases}
\end{align}
The case $b < \frac{1}{2} a$ is trivial.

The case $b > \frac{3}{2} a$ follows from
$$a \ln{\frac{a}{b}} \leq 0 \leq (a-b) + b \leq (a-b) + 100 \paren{\sqrt{b} - \sqrt{\frac{2}{3} b}}^2 \leq (a-b) + 100 \paren{\sqrt{b} - \sqrt{a}}^2.$$

The case $\frac{1}{2} a \leq b \leq \frac{3}{2} a$ follows from
\begingroup
\allowdisplaybreaks
\begin{align*}
a \ln{\frac{a}{b}} &= -a \ln{\frac{b}{a}} \\
		&= -a \ln{\paren{1+\frac{b-a}{a}}} \\
		&\stackrel{(\ast)}{=} -a \paren{\frac{b-a}{a} - \frac{1}{2} \paren{\frac{b-a}{a}}^2 + \frac{1}{3} \paren{\frac{b-a}{a}}^3 - \frac{1}{4} \paren{\frac{b-a}{a}}^4 + \cdots} \\
		&= (a-b) + \frac{\paren{b-a}^2}{a} \paren{\frac{1}{2} - \frac{1}{3} \frac{b-a}{a} + \frac{1}{4} \paren{\frac{b-a}{a}}^2 - \cdots} \\
		&\stackrel{(\ast \ast)}{\leq} (a-b) + \frac{\paren{b-a}^2}{a} \\
		&= (a-b) + \frac{\paren{\sqrt{a}+\sqrt{b}}^2}{a} \paren{\sqrt{a} - \sqrt{b}}^2 \\
		&\leq (a-b) + \frac{\paren{\sqrt{a}+\sqrt{\frac{3}{2} a}}^2}{a} \paren{\sqrt{a} - \sqrt{b}}^2 \\
		&\leq (a-b) + 9 \paren{\sqrt{a} - \sqrt{b}}^2,
\end{align*}
\endgroup
where the Taylor expansion for $(\ast)$ is possible due to $\abs{\frac{b-a}{a}} \leq \frac{1}{2}$, and $(\ast \ast)$ follows from the fact that $\frac{1}{2} - \frac{1}{3} t + \frac{1}{4} t^2 - \cdots < 1$ for $\abs{t} \leq \frac{1}{2}$.

\underline{\textit{Step 3}}\textit{:} In this step we show that $i,j,k$ lie on a path in $T$ implies $\mathrm{I}^{\widehat{\mathrm{P}}}(X_i;X_k | X_j) \leq 41 \frac{\epsilon^2}{n}$, and that $h,i,j,k$ lie on a path in $T$ implies $\mathrm{I}^{\widehat{\mathrm{P}}}(X_h;X_j,X_k | X_i)  \leq 81 \frac{\epsilon^2}{n}$.

Note that both quantities above become $0$ if we replace $\widehat{\mathrm{P}}$ by $\mathrm{P}$ at the superscripts of $\mathrm{I}$ due to conditional independence.

Suppose that $i,j,k$ lie on a path in $T$. We have by the equality part of (\ref{ih23}) that
\begin{align}
\begin{split}
\label{iexpansion}
\mathrm{I}^{\widehat{\mathrm{P}}}(X_i;X_k | X_j) &= KL(\widehat{\mathrm{P}}_{ijk} \mid \mid \widehat{\mathrm{P}}_{i-j-k}) \\
		&= \sum_{x_i,x_j,x_k = \pm 1} \widehat{\mathrm{P}}_{ijk}(x_i,x_j,x_k) \ln{\frac{\widehat{\mathrm{P}}_{ijk}(x_i,x_j,x_k)}{\widehat{\mathrm{P}}_{i-j-k}(x_i,x_j,x_k)}}
\end{split}
\end{align}
Note that $i,j,k$ lie on a path in $T$ implies $\mathrm{P}_{ijk} = \mathrm{P}_{i-j-k}$, so
\begin{align}
\begin{split}
\label{ratioproduct}
\frac{\widehat{\mathrm{P}}_{ijk}(x_i,x_j,x_k)}{\widehat{\mathrm{P}}_{i-j-k}(x_i,x_j,x_k)} &= \frac{\frac{\widehat{\mathrm{P}}_{ijk}(x_i,x_j,x_k)}{\mathrm{P}_{ijk}(x_i,x_j,x_k)}}{\frac{\widehat{\mathrm{P}}_{i-j-k}(x_i,x_j,x_k)}{\mathrm{P}_{i-j-k}(x_i,x_j,x_k)}} \\
		&= \frac{\frac{\widehat{\mathrm{P}}_{ijk}(x_i,x_j,x_k)}{\mathrm{P}_{ijk}(x_i,x_j,x_k)}}{\frac{\widehat{\mathrm{P}}_j(x_j)}{\mathrm{P}_j(x_j)} \cdot \frac{\widehat{\mathrm{P}}_{i|j}(x_i|x_j)}{\mathrm{P}_{i|j}(x_i|x_j)} \cdot \frac{\widehat{\mathrm{P}}_{k|j}(x_k|x_j)}{\mathrm{P}_{k|j}(x_k|x_j)}} \\
		&= \frac{\frac{\widehat{\mathrm{P}}_{ijk}(x_i,x_j,x_k)}{\mathrm{P}_{ijk}(x_i,x_j,x_k)} \cdot \frac{\widehat{\mathrm{P}}_j(x_j)}{\mathrm{P}_j(x_j)}}{\frac{\widehat{\mathrm{P}}_{ij}(x_i,x_j)}{\mathrm{P}_{ij}(x_i,x_j)} \cdot \frac{\widehat{\mathrm{P}}_{jk}(x_j,x_k)}{\mathrm{P}_{jk}(x_j,x_k)}}.
\end{split}
\end{align}
To bound the summands in (\ref{iexpansion}) we divide into two cases $\mathrm{P}_{ijk}(x_i,x_j,x_k) \geq \frac{\epsilon^2}{n}$ and $\mathrm{P}_{ijk}(x_i,x_j,x_k) < \frac{\epsilon^2}{n}$.

\underline{\textit{Case 1}}\textit{:} $\mathrm{P}_{ijk}(x_i,x_j,x_k) \geq \frac{\epsilon^2}{n}$

Note that in this case we also have $\mathrm{P}_{ij}(x_i,x_j) \geq \frac{\epsilon^2}{n}$, $\mathrm{P}_{jk}(x_j,x_k) \geq \frac{\epsilon^2}{n}$, and $\mathrm{P}_{j}(x_j) \geq \frac{\epsilon^2}{n}$. Thus, Lemma~\ref{genprobprecision} (iii) implies that each of the quantities $\frac{\widehat{\mathrm{P}}_{ijk}(x_i,x_j,x_k)}{\mathrm{P}_{ijk}(x_i,x_j,x_k)}$, $\frac{\widehat{\mathrm{P}}_{ij}(x_i,x_j)}{\mathrm{P}_{ij}(x_i,x_j)}$, $\frac{\widehat{\mathrm{P}}_{jk}(x_j,x_k)}{\mathrm{P}_{jk}(x_j,x_k)}$, and $\frac{\widehat{\mathrm{P}}_{j}(x_j)}{\mathrm{P}_{j}(x_j)}$ lies within the interval $[1-\frac{1}{10^{20}},1+\frac{1}{10^{20}}]$. By (\ref{ratioproduct}) $\frac{\widehat{\mathrm{P}}_{ijk}(x_i,x_j,x_k)}{\widehat{\mathrm{P}}_{i-j-k}(x_i,x_j,x_k)}$ must lie between $\frac{1}{2}$ and $\frac{3}{2}$. Therefore, we can apply the middle case of (\ref{numericalcases}) in Step 2 to get
\begin{align}
\label{largecasebound}
\begin{multlined}
\widehat{\mathrm{P}}_{ijk}(x_i,x_j,x_k) \ln{\frac{\widehat{\mathrm{P}}_{ijk}(x_i,x_j,x_k)}{\widehat{\mathrm{P}}_{i-j-k}(x_i,x_j,x_k)}} \leq \paren{\widehat{\mathrm{P}}_{ijk}(x_i,x_j,x_k) - \widehat{\mathrm{P}}_{i-j-k}(x_i,x_j,x_k)} \\
\qquad \qquad \qquad \qquad \qquad \qquad \qquad \ \ \ + 9 \paren{\sqrt{\widehat{\mathrm{P}}_{ijk}(x_i,x_j,x_k)} - \sqrt{\widehat{\mathrm{P}}_{i-j-k}(x_i,x_j,x_k)}}^2
\end{multlined}
\end{align}

\underline{\textit{Case 2}}\textit{:} $\mathrm{P}_{ijk}(x_i,x_j,x_k) < \frac{\epsilon^2}{n}$

If $\widehat{\mathrm{P}}_{i-j-k}(x_i,x_j,x_k) \geq \frac{1}{2} \widehat{\mathrm{P}}_{ijk}(x_i,x_j,x_k)$, then by the last two cases of (\ref{numericalcases}) in Step 2 we have
\begin{align}
\label{smallcasebound1}
\begin{multlined}
\widehat{\mathrm{P}}_{ijk}(x_i,x_j,x_k) \ln{\frac{\widehat{\mathrm{P}}_{ijk}(x_i,x_j,x_k)}{\widehat{\mathrm{P}}_{i-j-k}(x_i,x_j,x_k)}} \leq \paren{\widehat{\mathrm{P}}_{ijk}(x_i,x_j,x_k) - \widehat{\mathrm{P}}_{i-j-k}(x_i,x_j,x_k)} \\
\qquad \qquad \qquad \qquad \qquad \qquad + 100 \paren{\sqrt{\widehat{\mathrm{P}}_{ijk}(x_i,x_j,x_k)} - \sqrt{\widehat{\mathrm{P}}_{i-j-k}(x_i,x_j,x_k)}}^2
\end{multlined}
\end{align}

Now assume $\widehat{\mathrm{P}}_{i-j-k}(x_i,x_j,x_k) < \frac{1}{2} \widehat{\mathrm{P}}_{ijk}(x_i,x_j,x_k)$. We first bound the four fractions in the last line of (\ref{ratioproduct}). Note that 4-consistency implies $\widehat{\mathrm{P}}_{ijk}(x_i,x_j,x_k) < (1+\frac{1}{10^{20}}) \frac{\epsilon^2}{n}$, and so
\begin{equation}
\label{firstratio}
\frac{\widehat{\mathrm{P}}_{ijk}(x_i,x_j,x_k)}{\mathrm{P}_{ijk}(x_i,x_j,x_k)} < (1 + \frac{1}{10^{20}}) \cdot \frac{\frac{\epsilon^2}{n}}{\mathrm{P}_{ijk}(x_i,x_j,x_k)}.
\end{equation}
If $\mathrm{P}_j(x_j) \geq \frac{\epsilon^2}{n}$, then by Lemma~\ref{genprobprecision} (iii) we have $\frac{\widehat{\mathrm{P}}_j(x_j)}{\mathrm{P}_j(x_j)} \leq 1+\frac{1}{10^{20}}$. If $\mathrm{P}_j(x_j) < \frac{\epsilon^2}{n}$, then 4-consistency implies $\widehat{\mathrm{P}}_j(x_j) < (1+\frac{1}{10^{20}}) \frac{\epsilon^2}{n}$, and so $\frac{\widehat{\mathrm{P}}_j(x_j)}{\mathrm{P}_j(x_j)} \leq (1+\frac{1}{10^{20}}) \frac{\frac{\epsilon^2}{n}}{\mathrm{P}_j(x_j)}$. Either way
\begin{equation}
\label{secondratio}
\frac{\widehat{\mathrm{P}}_j(x_j)}{\mathrm{P}_j(x_j)} < (1 + \frac{1}{10^{20}}) \cdot \frac{\frac{\epsilon^2}{n}}{\mathrm{P}_{ijk}(x_i,x_j,x_k)}.
\end{equation}
If $\mathrm{P}_{ij}(x_i,x_j) \geq \frac{\epsilon^2}{n}$, then by Lemma~\ref{genprobprecision} (iii) we have $\frac{\mathrm{P}_{ij}(x_i,x_j)}{\widehat{\mathrm{P}}_{ij}(x_i,x_j)} \leq \frac{1}{1-\frac{1}{10^{20}}}$. If $\mathrm{P}_{ij}(x_i,x_j) < \frac{\epsilon^2}{n}$,we have $\frac{\mathrm{P}_{ij}(x_i,x_j)}{\widehat{\mathrm{P}}_{ij}(x_i,x_j)} < \frac{\frac{\epsilon^2}{n}}{\widehat{\mathrm{P}}_{ij}(x_i,x_j)} \leq \frac{\frac{\epsilon^2}{n}}{\widehat{\mathrm{P}}_{ijk}(x_i,x_j,x_k)}$. Either way (using that $\widehat{\mathrm{P}}_{ijk}(x_i,x_j,x_k) < (1+\frac{1}{10^{20}}) \frac{\epsilon^2}{n}$)
\begin{equation}
\label{thirdratio}
\frac{\mathrm{P}_{ij}(x_i,x_j)}{\widehat{\mathrm{P}}_{ij}(x_i,x_j)} < \frac{1 + \frac{1}{10^{20}}}{1 - \frac{1}{10^{20}}} \cdot \frac{\frac{\epsilon^2}{n}}{\widehat{\mathrm{P}}_{ijk}(x_i,x_j,x_k)}.
\end{equation}
Similarly we also have
\begin{equation}
\label{fourthratio}
\frac{\mathrm{P}_{jk}(x_j,x_k)}{\widehat{\mathrm{P}}_{jk}(x_j,x_k)} < \frac{1 + \frac{1}{10^{20}}}{1 - \frac{1}{10^{20}}} \cdot \frac{\frac{\epsilon^2}{n}}{\widehat{\mathrm{P}}_{ijk}(x_i,x_j,x_k)}.
\end{equation}
Combining (\ref{ratioproduct}), (\ref{firstratio}), (\ref{secondratio}), (\ref{thirdratio}), (\ref{fourthratio}), and the first case of (\ref{numericalcases}) in Step 2, we get
\begingroup
\allowdisplaybreaks
\begin{align}
\label{smallcasebound2}
\widehat{\mathrm{P}}_{ijk}(x_i,x_j,x_k) & \ln{\frac{\widehat{\mathrm{P}}_{ijk}(x_i,x_j,x_k)}{\widehat{\mathrm{P}}_{i-j-k}(x_i,x_j,x_k)}} \nonumber \\
		&\begin{multlined}
		\hspace{-2em} \leq \paren{\widehat{\mathrm{P}}_{ijk}(x_i,x_j,x_k) - \widehat{\mathrm{P}}_{i-j-k}(x_i,x_j,x_k)} \\
		\qquad \qquad \qquad \qquad + \widehat{\mathrm{P}}_{ijk}(x_i,x_j,x_k) \ln{\frac{\widehat{\mathrm{P}}_{ijk}(x_i,x_j,x_k)}{\widehat{\mathrm{P}}_{i-j-k}(x_i,x_j,x_k)}}
		\end{multlined} \nonumber \\
		&\begin{multlined}
		\hspace{-2em} = \paren{\widehat{\mathrm{P}}_{ijk}(x_i,x_j,x_k) - \widehat{\mathrm{P}}_{i-j-k}(x_i,x_j,x_k)} \\
		\hspace{-1em} + \widehat{\mathrm{P}}_{ijk}(x_i,x_j,x_k) \ln{\brac{\frac{\widehat{\mathrm{P}}_{ijk}(x_i,x_j,x_k)}{\mathrm{P}_{ijk}(x_i,x_j,x_k)} \frac{\widehat{\mathrm{P}}_j(x_j)}{\mathrm{P}_j(x_j)}\frac{\mathrm{P}_{ij}(x_i,x_j)}{\widehat{\mathrm{P}}_{ij}(x_i,x_j)} \frac{\mathrm{P}_{jk}(x_j,x_k)}{\widehat{\mathrm{P}}_{jk}(x_j,x_k)}}}
		\end{multlined} \nonumber \\
		&\begin{multlined}
		\hspace{-2em} \leq \paren{\widehat{\mathrm{P}}_{ijk}(x_i,x_j,x_k) - \widehat{\mathrm{P}}_{i-j-k}(x_i,x_j,x_k)} \\
		\hspace{-1em} + \widehat{\mathrm{P}}_{ijk}(x_i,x_j,x_k) \ln{\brac{2 \cdot \paren{\frac{\frac{\epsilon^2}{n}}{\mathrm{P}_{ijk}(x_i,x_j,x_k)}}^2 \paren{\frac{\frac{\epsilon^2}{n}}{\widehat{\mathrm{P}}_{ijk}(x_i,x_j,x_k)}}^2}}
		\end{multlined} \nonumber \\
		&\begin{multlined}
		\hspace{-2em} = \paren{\widehat{\mathrm{P}}_{ijk}(x_i,x_j,x_k) - \widehat{\mathrm{P}}_{i-j-k}(x_i,x_j,x_k)} \\
		\hspace{-3em} + \widehat{\mathrm{P}}_{ijk}(x_i,x_j,x_k) \ln{2} + 2 \widehat{\mathrm{P}}_{ijk}(x_i,x_j,x_k) \ln{\frac{\frac{\epsilon^2}{n}}{\mathrm{P}_{ijk}(x_i,x_j,x_k)}} + 2 \frac{\epsilon^2}{n} \cdot \frac{\ln{\frac{\frac{\epsilon^2}{n}}{\widehat{\mathrm{P}}_{ijk}(x_i,x_j,x_k)}}}{\frac{\frac{\epsilon^2}{n}}{\widehat{\mathrm{P}}_{ijk}(x_i,x_j,x_k)}}
		\end{multlined} \nonumber \\
		&\hspace{-2em}\leq  \paren{\widehat{\mathrm{P}}_{ijk}(x_i,x_j,x_k)  - \widehat{\mathrm{P}}_{i-j-k}(x_i,x_j,x_k)} + 5\frac{\epsilon^2}{n},
\end{align}
\endgroup
where the last inequality follows from the bound $\widehat{\mathrm{P}}_{ijk}(x_i,x_j,x_k) < (1+\frac{1}{10^{20}}) \frac{\epsilon^2}{n}$, strong 4-consistency (this is the only place in the entire proof that the ``strong part" of strong 4-consistency is used), and the fact that $\frac{\ln{t}}{t} \leq 1$, for $t > 0$.

Combining (\ref{largecasebound}), (\ref{smallcasebound1}), and (\ref{smallcasebound2}), we have
\begin{align*}
\begin{multlined}
\widehat{\mathrm{P}}_{ijk}(x_i,x_j,x_k) \ln{\frac{\widehat{\mathrm{P}}_{ijk}(x_i,x_j,x_k)}{\widehat{\mathrm{P}}_{i-j-k}(x_i,x_j,x_k)}} \leq \paren{\widehat{\mathrm{P}}_{ijk}(x_i,x_j,x_k) - \widehat{\mathrm{P}}_{i-j-k}(x_i,x_j,x_k)} \\
\qquad \qquad \qquad \qquad \qquad \qquad + 100 \paren{\sqrt{\widehat{\mathrm{P}}_{ijk}(x_i,x_j,x_k)} - \sqrt{\widehat{\mathrm{P}}_{i-j-k}(x_i,x_j,x_k)}}^2 + 5 \frac{\epsilon^2}{n}
\end{multlined}
\end{align*}
for any of the eight combinations $x_i, x_j, x_k \in \{1,-1\}$. Plugging them into (\ref{iexpansion}), noting that the eight $\paren{\widehat{\mathrm{P}}_{ijk}(x_i,x_j,x_k) - \widehat{\mathrm{P}}_{i-j-k}(x_i,x_j,x_k)}$ terms sum to $0$, we get
\begin{align*}
\mathrm{I}^{\widehat{\mathrm{P}}}(X_i;X_k | X_j) &\leq 800 \paren{\sqrt{\widehat{\mathrm{P}}_{ijk}(x_i,x_j,x_k)} - \sqrt{\widehat{\mathrm{P}}_{i-j-k}(x_i,x_j,x_k)}}^2 + 40 \frac{\epsilon^2}{n} \\
		&= 1600 H^2(\widehat{\mathrm{P}}_{ijk}, \widehat{\mathrm{P}}_{i-j-k}) + 40 \frac{\epsilon^2}{n} \\
		&\stackrel{(\ast)}{<} 41 \frac{\epsilon^2}{n},
\end{align*}
where $(\ast)$ follows from Lemma~\ref{genprecision} (vii). This proves the first claim stated at the beginning of Step 3.

Now suppose that $h,i,j,k$ lie on a path in $T$. The argument for bounding $\mathrm{I}^{\widehat{\mathrm{P}}}(X_h;X_j,X_k | X_i)$ is completely parallel to the three-node case above if we view $(X_j,X_k)$ as a single random variable. For example we can apply the equality part of (\ref{ih24}) to get
\begin{align}
\label{iexpansion4}
\begin{split}
\mathrm{I}^{\widehat{\mathrm{P}}}(X_h;X_j,X_k | X_i) &= KL(\widehat{\mathrm{P}}_{hijk} \mid \mid \widehat{\mathrm{P}}_{h-i-(jk)}) \\
		&= \sum_{x_h,x_i,x_j,x_k = \pm 1} \widehat{\mathrm{P}}_{hijk}(x_h,x_i,x_j,x_k) \ln{\frac{\widehat{\mathrm{P}}_{hijk}(x_h,x_i,x_j,x_k)}{\widehat{\mathrm{P}}_{h-i-(jk)}(x_h,x_i,x_j,x_k)}}
\end{split}
\end{align}
Similarly, note that $h,i,j,k$ lie on a path in $T$ implies $\mathrm{P}_{hijk} = \mathrm{P}_{h-i-(jk)}$, so
\begin{align}
\label{ratioproduct4}
\begin{split}
\frac{\widehat{\mathrm{P}}_{hijk}(x_h,x_i,x_j,x_k)}{\widehat{\mathrm{P}}_{h-i-(jk)}(x_h,x_i,x_j,x_k)} &= \frac{\frac{\widehat{\mathrm{P}}_{hijk}(x_h,x_i,x_j,x_k)}{\mathrm{P}_{hijk}(x_h,x_i,x_j,x_k)}}{\frac{\widehat{\mathrm{P}}_{h-i-(jk)}(x_h,x_i,x_j,x_k)}{\mathrm{P}_{h-i-(jk)}(x_h,x_i,x_j,x_k)}} \\
		&= \frac{\frac{\widehat{\mathrm{P}}_{hijk}(x_h,x_i,x_j,x_k)}{\mathrm{P}_{hijk}(x_h,x_i,x_j,x_k)}}{\frac{\widehat{\mathrm{P}}_i(x_i)}{\mathrm{P}_i(x_i)} \cdot \frac{\widehat{\mathrm{P}}_{h|i}(x_h|x_i)}{\mathrm{P}_{h|i}(x_h|x_i)} \cdot \frac{\widehat{\mathrm{P}}_{jk|i}(x_j,x_k|x_i)}{\mathrm{P}_{jk|i}(x_j,x_k|x_i)}} \\
		&= \frac{\frac{\widehat{\mathrm{P}}_{hijk}(x_h,x_i,x_j,x_k)}{\mathrm{P}_{hijk}(x_h,x_i,x_j,x_k)} \cdot \frac{\widehat{\mathrm{P}}_i(x_i)}{\mathrm{P}_i(x_i)}}{\frac{\widehat{\mathrm{P}}_{hi}(x_h,x_i)}{\mathrm{P}_{hi}(x_h,x_i)} \cdot \frac{\widehat{\mathrm{P}}_{ijk}(x_i,x_j,x_k)}{\mathrm{P}_{ijk}(x_i,x_j,x_k)}}.
\end{split}
\end{align}
To bound the summands in (\ref{iexpansion4}) we similarly divide into two cases $\mathrm{P}_{hijk}(x_h,x_i,x_j,x_k) \geq \frac{\epsilon^2}{n}$ and $\mathrm{P}_{hijk}(x_h,x_i,x_j,x_k) < \frac{\epsilon^2}{n}$ and bound the four fractions at the end of (\ref{ratioproduct4}) separately. We omit the detail, noting only that at the end we will have
\begin{align*}
\mathrm{I}^{\widehat{\mathrm{P}}}(X_h;X_j,X_k | X_i) &\leq 1600 \paren{\sqrt{\widehat{\mathrm{P}}_{hijk}(x_h,x_i,x_j,x_k)} - \sqrt{\widehat{\mathrm{P}}_{h-i-(jk)}(x_h,x_i,x_j,x_k)}}^2 + 80 \frac{\epsilon^2}{n} \\
		&= 3200 H^2(\widehat{\mathrm{P}}_{hijk}, \widehat{\mathrm{P}}_{h-i-(jk)}) + 80 \frac{\epsilon^2}{n} \\
		&\stackrel{(\ast)}{<} 81 \frac{\epsilon^2}{n},
\end{align*}
where $(\ast)$ follows from Lemma~\ref{genprecision} (viii) (note that the coefficients are different from those in the three-node case because we have 16 terms instead of 8). This finishes Step 3.

\underline{\textit{Step 4}}\textit{:} We finish the proof in this step.

To prove (i), suppose $i,j,k$ lie on a path in $T$, and $\mathrm{I}^{\widehat{\mathrm{P}}}(X_i;X_j) \leq \mathrm{I}^{\widehat{\mathrm{P}}}(X_i;X_k)$. We will proceed in a way similar to the standard proof of the data-processing inequality for Markov chains. Let's expand $\mathrm{I}^{\widehat{\mathrm{P}}}(X_i;X_j,X_k)$ using the chain rule for mutual information in two different ways:
$$\mathrm{I}^{\widehat{\mathrm{P}}}(X_i;X_j,X_k) = \mathrm{I}^{\widehat{\mathrm{P}}}(X_i;X_j) + \mathrm{I}^{\widehat{\mathrm{P}}}(X_i;X_k|X_j),$$
and
$$\mathrm{I}^{\widehat{\mathrm{P}}}(X_i;X_j,X_k) = \mathrm{I}^{\widehat{\mathrm{P}}}(X_i;X_k) + \mathrm{I}^{\widehat{\mathrm{P}}}(X_i;X_j|X_k).$$
Compare the two we see that $\mathrm{I}^{\widehat{\mathrm{P}}}(X_i;X_j) \leq \mathrm{I}^{\widehat{\mathrm{P}}}(X_i;X_k)$ implies $\mathrm{I}^{\widehat{\mathrm{P}}}(X_i;X_k|X_j) \geq \mathrm{I}^{\widehat{\mathrm{P}}}(X_i;X_j|X_k)$. By Step 3, $\mathrm{I}^{\widehat{\mathrm{P}}}(X_i;X_k|X_j) \leq 41 \frac{\epsilon^2}{n}$. Thus we have $\mathrm{I}^{\widehat{\mathrm{P}}}(X_i;X_j|X_k) \leq 41 \frac{\epsilon^2}{n}$.

By (\ref{ih23}), we have $H^2(\widehat{\mathrm{P}}_{ijk}, \widehat{\mathrm{P}}_{i\wideparen{ \,~~ j-}k}) \leq \frac{1}{2} \mathrm{I}^{\widehat{\mathrm{P}}}(X_i;X_j|X_k) < 21 \frac{\epsilon^2}{n}$. By triangle inequality for Hellinger distance, and Lemma~\ref{genprecision} (iii), (v), we have
\begin{align*}
H^2(\mathrm{P}_{ijk}, \mathrm{P}_{i\wideparen{ \,~~ j-}k}) &\leq \paren{H(\mathrm{P}_{ijk}, \widehat{\mathrm{P}}_{ijk}) + H(\widehat{\mathrm{P}}_{ijk}, \widehat{\mathrm{P}}_{i\wideparen{ \,~~ j-}k}) + H(\widehat{\mathrm{P}}_{i\wideparen{ \,~~ j-}k}, \mathrm{P}_{i\wideparen{ \,~~ j-}k})}^2 \\
		&\leq \paren{\sqrt{4 \cdot \frac{1}{10^{20}} \frac{\epsilon^2}{n}} + \sqrt{21 \frac{\epsilon^2}{n}} + \sqrt{4 \cdot \frac{1}{10^{20}} \frac{\epsilon^2}{n}}}^2 \\
		&\leq 22 \frac{\epsilon^2}{n}.
\end{align*}

To prove (ii), suppose $h,i,j,k$ lie on a path in $T$, and $\mathrm{I}^{\widehat{\mathrm{P}}}(X_i;X_j) \leq \mathrm{I}^{\widehat{\mathrm{P}}}(X_h;X_k)$. Let's expand $\mathrm{I}^{\widehat{\mathrm{P}}}(X_h,X_i;X_j,X_k)$ using the chain rule for (conditional) mutual information in two different ways:
\begin{align*}
\mathrm{I}^{\widehat{\mathrm{P}}}(X_h,X_i;X_j,X_k) &= \mathrm{I}^{\widehat{\mathrm{P}}}(X_i;X_j,X_k) + \mathrm{I}^{\widehat{\mathrm{P}}}(X_h;X_j,X_k|X_i) \\
		&= \mathrm{I}^{\widehat{\mathrm{P}}}(X_i;X_j) + \mathrm{I}^{\widehat{\mathrm{P}}}(X_i;X_k|X_j) + \mathrm{I}^{\widehat{\mathrm{P}}}(X_h;X_j,X_k|X_i),
\end{align*}
and
\begin{align*}
\mathrm{I}^{\widehat{\mathrm{P}}}(X_h,X_i;X_j,X_k) &= \mathrm{I}^{\widehat{\mathrm{P}}}(X_h,X_i;X_k) + \mathrm{I}^{\widehat{\mathrm{P}}}(X_h,X_i;X_j|X_k) \\
		&= \mathrm{I}^{\widehat{\mathrm{P}}}(X_h;X_k) + \mathrm{I}^{\widehat{\mathrm{P}}}(X_i;X_k|X_h) + \mathrm{I}^{\widehat{\mathrm{P}}}(X_h,X_i;X_j|X_k),
\end{align*}
Compare the two we see that $\mathrm{I}^{\widehat{\mathrm{P}}}(X_i;X_j) \leq \mathrm{I}^{\widehat{\mathrm{P}}}(X_h;X_k)$ implies
$$\mathrm{I}^{\widehat{\mathrm{P}}}(X_i;X_k|X_j) + \mathrm{I}^{\widehat{\mathrm{P}}}(X_h;X_j,X_k|X_i) \geq \mathrm{I}^{\widehat{\mathrm{P}}}(X_i;X_k|X_h) + \mathrm{I}^{\widehat{\mathrm{P}}}(X_h,X_i;X_j|X_k).$$
By Step 3, $\mathrm{I}^{\widehat{\mathrm{P}}}(X_i;X_k|X_j) \leq 41 \frac{\epsilon^2}{n}$ and $\mathrm{I}^{\widehat{\mathrm{P}}}(X_h;X_j,X_k|X_i) \leq 81 \frac{\epsilon^2}{n}$. Thus we have $\mathrm{I}^{\widehat{\mathrm{P}}}(X_i;X_k|X_h) \leq 122 \frac{\epsilon^2}{n}$ and $\mathrm{I}^{\widehat{\mathrm{P}}}(X_h,X_i;X_j|X_k) \leq 122 \frac{\epsilon^2}{n}$.

By (\ref{ih23}), we have $H^2(\widehat{\mathrm{P}}_{hik}, \widehat{\mathrm{P}}_{h\wideparen{-i \,~~ }k}) \leq \frac{1}{2} \mathrm{I}^{\widehat{\mathrm{P}}}(X_i;X_k|X_h) \leq 61 \frac{\epsilon^2}{n}$. By triangle inequality for Hellinger distance, and Lemma~\ref{genprecision} (iii), (v), we have
\begin{align}
\label{heltriangle}
\begin{split}
H^2(\mathrm{P}_{hik}, \mathrm{P}_{h\wideparen{-i \,~~ }k}) &\leq \paren{H(\mathrm{P}_{hik}, \widehat{\mathrm{P}}_{hik}) + H(\widehat{\mathrm{P}}_{hik}, \widehat{\mathrm{P}}_{h\wideparen{-i \,~~ }k}) + H(\widehat{\mathrm{P}}_{h\wideparen{-i \,~~ }k}, \mathrm{P}_{h\wideparen{-i \,~~ }k})}^2 \\
		&\leq \paren{\sqrt{4 \cdot \frac{1}{10^{20}} \frac{\epsilon^2}{n}} + \sqrt{61 \frac{\epsilon^2}{n}} + \sqrt{4 \cdot \frac{1}{10^{20}} \frac{\epsilon^2}{n}}}^2 \\
		&\leq 62 \frac{\epsilon^2}{n}.
\end{split}
\end{align}

By (\ref{ih24}) with the pair $(X_h,X_i)$ viewed as a single random variable we have $H^2(\widehat{\mathrm{P}}_{hijk}, \widehat{\mathrm{P}}_{(hi)\wideparen{ \,~~ j-}k}) \leq \frac{1}{2} \mathrm{I}^{\widehat{\mathrm{P}}}(X_h,X_i;X_j|X_k) \leq 61 \frac{\epsilon^2}{n}$. We want to get rid of the index $h$ in $H^2(\widehat{\mathrm{P}}_{hijk}, \widehat{\mathrm{P}}_{(hi)\wideparen{ \,~~ j-}k})$. To that end, we have
\begin{align*}
H^2(\widehat{\mathrm{P}}_{hijk}, \widehat{\mathrm{P}}_{(hi)\wideparen{ \,~~ j-}k}) &= 1 - \sum_{x_h,x_i,x_j,x_k = \pm 1} \sqrt{\widehat{\mathrm{P}}_{hijk}(x_h,x_i,x_j,x_k) \cdot \widehat{\mathrm{P}}_{(hi)\wideparen{ \,~~ j-}k} (x_h,x_i,x_j,x_k)} \\
		&\hspace{-3em} = 1 - \sum_{x_i,x_j,x_k = \pm 1} \brac{\sum_{x_h = \pm 1} \sqrt{\widehat{\mathrm{P}}_{hijk}(x_h,x_i,x_j,x_k) \cdot \widehat{\mathrm{P}}_{(hi)\wideparen{ \,~~ j-}k} (x_h,x_i,x_j,x_k)}} \\
		&\hspace{-5em} \stackrel{(\ast)}{\geq} 1 - \sum_{x_i,x_j,x_k = \pm 1} \brac{\sqrt{\sum_{x_h = \pm 1} \widehat{\mathrm{P}}_{hijk}(x_h,x_i,x_j,x_k)} \sqrt{\sum_{x_h = \pm 1} \widehat{\mathrm{P}}_{(hi)\wideparen{ \,~~ j-}k} (x_h,x_i,x_j,x_k)}} \\
		&\hspace{-5em} \stackrel{(\ast \ast)}{=} 1 - \sum_{x_i,x_j,x_k = \pm 1} \sqrt{\widehat{\mathrm{P}}_{ijk}(x_i,x_j,x_k) \widehat{\mathrm{P}}_{i\wideparen{ \,~~ j-}k} (x_i,x_j,x_k)} \\
		&\hspace{-5em} = H^2(\widehat{\mathrm{P}}_{ijk}, \widehat{\mathrm{P}}_{i\wideparen{ \,~~ j-}k}),
\end{align*}
where $(\ast)$ follows from Cauchy's Inequality, and $(\ast \ast)$ follows from the facts that $\widehat{\mathrm{P}}_{hijk}$'s marginal distribution for $X_i,X_j,X_k$ is $\widehat{\mathrm{P}}_{ijk}$ and $\widehat{\mathrm{P}}_{(hi)\wideparen{ \,~~ j-}k}$'s marginal distribution for $X_i,X_j,X_k$ is $\widehat{\mathrm{P}}_{i\wideparen{ \,~~ j-}k}$. (In fact, it is generally true that restricting to a subset of variables will not increase the Hellinger distance, by an argument like the one we had above).

Thus, we in fact have $H^2(\widehat{\mathrm{P}}_{ijk}, \widehat{\mathrm{P}}_{i\wideparen{ \,~~ j-}k}) \leq 61 \frac{\epsilon^2}{n}$. We can then obtain  as in (\ref{heltriangle}) that $H^2(\mathrm{P}_{ijk}, \mathrm{P}_{i\wideparen{ \,~~ j-}k}) \leq 62 \frac{\epsilon^2}{n}$. Finally, by Lemma~\ref{hel4nodes} we have
$$H^2(\mathrm{P}_{hijk}, \mathrm{P}_{h\wideparen{-i \,~~ j-}k}) \leq \paren{H(\mathrm{P}_{ijk}, \mathrm{P}_{i\wideparen{ \,~~ j-}k}) + H(\mathrm{P}_{hik}, \mathrm{P}_{h\wideparen{-i \,~~ }k})}^2 \leq 248 \frac{\epsilon^2}{n}$$ This completes the entire proof.
\end{proof}

Lemma~\ref{genhelwrongedge} implies that if $i,j,k$ lie on a path in $T$ and $H^2(\mathrm{P}_{ijk}, \mathrm{P}_{i\wideparen{ \,~~ j-}k}) > 22 \frac{\epsilon^2}{n}$, then Algorithm~\ref{algo:genalgo} can not possibly make a structural mistake by ``reversing the order'' between $(i,j)$ and $(i,k)$. The condition $H^2(\mathrm{P}_{ijk}, \mathrm{P}_{i\wideparen{ \,~~ j-}k}) > 22 \frac{\epsilon^2}{n}$ involves $\mathrm{P}_{ijk}$. We would also want other forms of this condition that are stated in terms only of $\mathrm{P}_{ij}$ and $\mathrm{P}_{jk}$, as ultimately we are going to classify the individual edges (of $\widehat{T}$) independently in the definition of layers in Section~\ref{sec:genstructure}. We want to find conditions involving only $\mathrm{P}_{ij}$ and $\mathrm{P}_{jk}$ that would guarantee $H^2(\mathrm{P}_{ijk}, \mathrm{P}_{i\wideparen{ \,~~ j-}k}) > 22 \frac{\epsilon^2}{n}$. The following two lemmas enable us to do that.

\begin{lemma}
\label{helwrongedgecity}
If $i,j,k$ lie on a path in $T$, then $H^2(\mathrm{P}_{ijk}, \mathrm{P}_{i\wideparen{ \,~~ j-}k}) \geq \frac{1}{100} {\mindisc{\mathrm{P}_{ij}}}^2 \cdot \min{\brac{\minmrg{\mathrm{P}_{j}}, \mindiag{\mathrm{P}_{jk}}}}.$
\end{lemma}
\begin{proof}
First we show there are $x_j, x_k \in \{1,-1\}$ so that $\abs{\mathrm{P}_{i|j}(1|x_j) - \mathrm{P}_{i|k}(1|x_k)} \geq \frac{1}{2} \mindisc{\mathrm{P}_{ij}}$ and $\mathrm{P}_{jk}(x_j,x_k) \geq \frac{1}{2} \min{\brac{\minmrg{\mathrm{P}_{j}}, \mindiag{\mathrm{P}_{jk}}}}$.
 
We have
\begin{align*}
\mindisc{\mathrm{P}_{ij}} &\leq \abs{\mathrm{P}_{i|j}(1|1) - \mathrm{P}_{i|j}(1|-1)} \\
		&\leq \abs{\mathrm{P}_{i|j}(1|1) - \mathrm{P}_{i|k}(1|1)} + \abs{\mathrm{P}_{i|k}(1|1) - \mathrm{P}_{i|j}(1|-1)},
\end{align*}
so 
\begin{equation}
\label{j1k1j-1k1}
\max{\brac{\abs{\mathrm{P}_{i|j}(1|1) - \mathrm{P}_{i|k}(1|1)}, \abs{\mathrm{P}_{i|j}(1|-1) - \mathrm{P}_{i|k}(1|1)}}} \geq \frac{1}{2} \mindisc{\mathrm{P}_{ij}}.
\end{equation}

Similarly, we have
\begin{align*}
\mindisc{\mathrm{P}_{ij}} &\leq \abs{\mathrm{P}_{i|j}(1|1) - \mathrm{P}_{i|j}(1|-1)} \\
		&\leq \abs{\mathrm{P}_{i|j}(1|1) - \mathrm{P}_{i|k}(1|-1)} + \abs{\mathrm{P}_{i|k}(1|-1) - \mathrm{P}_{i|j}(1|-1)},
\end{align*}
so
\begin{equation}
\label{j1k-1j-1k-1}
\max{\brac{\abs{\mathrm{P}_{i|j}(1|1) - \mathrm{P}_{i|k}(1|-1)}, \abs{\mathrm{P}_{i|j}(1|-1) - \mathrm{P}_{i|k}(1|-1)}}} \geq \frac{1}{2} \mindisc{\mathrm{P}_{ij}}.
\end{equation}

It's easy to see that we also have the following:
\begingroup
\allowdisplaybreaks
\begin{align}
\label{fourxjxkcombos}
\begin{split}
\max{\brac{\mathrm{P}_{jk}(1,1), \mathrm{P}_{jk}(1,-1)}} \geq \frac{1}{2} \mathrm{P}_j(1) &\geq \frac{1}{2} \minmrg{\mathrm{P}_j} \\
\max{\brac{\mathrm{P}_{jk}(-1,1), \mathrm{P}_{jk}(-1,-1)}} \geq \frac{1}{2} \mathrm{P}_j(-1) &\geq \frac{1}{2} \minmrg{\mathrm{P}_j} \\
\max{\brac{\mathrm{P}_{jk}(1,1), \mathrm{P}_{jk}(-1,-1)}} \geq \frac{1}{2} \mathrm{P}(X_j = X_k) &\geq \frac{1}{2} \mindiag{\mathrm{P}_{jk}} \\
\max{\brac{\mathrm{P}_{jk}(1,-1), \mathrm{P}_{jk}(-1,1)}} \geq \frac{1}{2} \mathrm{P}(X_j = -X_k) &\geq \frac{1}{2} \mindiag{\mathrm{P}_{jk}}
\end{split}
\end{align}
\endgroup
The claim at the beginning of the lemma follows from (\ref{j1k1j-1k1}), (\ref{j1k-1j-1k-1}), (\ref{fourxjxkcombos}), and a bit of case work. Without loss of generality we can assume that $\abs{\mathrm{P}_{i|j}(1|1) - \mathrm{P}_{i|k}(1|1)} \geq \frac{1}{2} \mindisc{\mathrm{P}_{ij}}$ and $\mathrm{P}_{jk}(1,1) \geq \frac{1}{2} \min{\brac{\minmrg{\mathrm{P}_{j}}, \mindiag{\mathrm{P}_{jk}}}}$.

We then have
\begingroup
\allowdisplaybreaks
\begin{align*}
H^2(\mathrm{P}_{i-j-k}, \mathrm{P}_{i\wideparen{ \,~~ j-}k}) &\geq \frac{1}{2} \paren{\sqrt{\mathrm{P}_{i-j-k}(1,1,1)} - \sqrt{\mathrm{P}_{i\wideparen{ \,~~ j-}k}(1,1,1)}}^2 \\
		&= \frac{1}{2} \paren{\sqrt{\mathrm{P}_{jk}(1,1) \mathrm{P}_{i|j}(1|1)} - \sqrt{\mathrm{P}_{jk}(1,1) \mathrm{P}_{i|k}(1|1)}}^2 \\
		&= \frac{1}{2} \paren{\sqrt{\mathrm{P}_{i|j}(1|1)} - \sqrt{\mathrm{P}_{i|k}(1|1)}}^2 \cdot \mathrm{P}_{jk}(1,1)\\
		&\stackrel{(\ast)}{\geq} \frac{1}{2} \paren{\sqrt{1} - \sqrt{1 - \abs{\mathrm{P}_{i|j}(1|1) - \mathrm{P}_{i|k}(1|1)}}}^2 \cdot \mathrm{P}_{jk}(1,1)\\
		&\geq \frac{1}{2} \paren{1 - \sqrt{1 - \frac{1}{2} \mindisc{\mathrm{P}_{ij}}}}^2 \cdot \mathrm{P}_{jk}(1,1)\\
		&= \frac{1}{2} \paren{\frac{\frac{1}{2} \mindisc{\mathrm{P}_{ij}}}{1 + \sqrt{1 - \frac{1}{2} \mindisc{\mathrm{P}_{ij}}}}}^2 \cdot \mathrm{P}_{jk}(1,1) \\
		&\geq \frac{1}{100} {\mindisc{\mathrm{P}_{ij}}}^2 \cdot \min{\brac{\minmrg{\mathrm{P}_{j}}, \mindiag{\mathrm{P}_{jk}}}},
\end{align*}
\endgroup
where $(\ast)$ follows from the fact that $\sqrt{a} - \sqrt{b} \geq \sqrt{a+c} - \sqrt{b+c}$, $\forall a \geq b \geq 0, c \geq 0$.
\end{proof}

\begin{lemma}
\label{helwrongedgecountry}
If $i,j,k$ lie on a path in $T$, then $H^2(\mathrm{P}_{ijk}, \mathrm{P}_{i\wideparen{ \,~~ j-}k}) \geq \frac{1}{100} I_{H^2}(\mathrm{P}_{ij}) \cdot \paren{1 - \mindisc{\mathrm{P}_{jk}}}^2.$
\end{lemma}
\begin{proof}
Since $I_{H^2}(\mathrm{P}_{ij}) = \frac{1}{2} \sum_{x_i,x_j = \pm 1} \paren{\sqrt{\mathrm{P}_{ij}(x_i,x_j)} - \sqrt{\mathrm{P}_i(x_i) \mathrm{P}_j(x_j)}}^2$, we can suppose without loss of generality that
\begin{equation}
\label{significanthel}
\paren{\sqrt{\mathrm{P}_{ij}(1,1)} - \sqrt{\mathrm{P}_i(1) \mathrm{P}_j(1)}}^2 \geq \frac{1}{2} I_{H^2}(\mathrm{P}_{ij}).
\end{equation}

It is easy to see that $\mathrm{P}_{ij}(1,1) - \mathrm{P}_i(1) \mathrm{P}_j(1) = -\paren{\mathrm{P}_{ij}(1,-1) - \mathrm{P}_i(1) \mathrm{P}_j(-1)}$. To shorten the length of subsequent equation-displays, We use the shorthand
$$V = \mathrm{P}_{ij}(1,1) - \mathrm{P}_i(1) \mathrm{P}_j(1),$$
and so $-V = \mathrm{P}_{ij}(1,-1) - \mathrm{P}_i(1) \mathrm{P}_j(-1)$.

We have
\begin{align*}
\mathrm{P}_{ijk} (1,1,1) = \mathrm{P}_{ij}(1,1) \mathrm{P}_{k|j}(1|1) &= \mathrm{P}_i(1) \mathrm{P}_j(1) \mathrm{P}_{k|j}(1|1) + V \cdot \mathrm{P}_{k|j}(1|1) \\
		&= \mathrm{P}_i(1) \mathrm{P}_{jk}(1,1) + V \cdot \mathrm{P}_{k|j}(1|1)
\end{align*}
and
\begin{align*}
\mathrm{P}_{ijk} (1,1,-1) = \mathrm{P}_{ij}(1,1) \mathrm{P}_{k|j}(-1|1) &= \mathrm{P}_i(1) \mathrm{P}_j(1) \mathrm{P}_{k|j}(-1|1) + V \cdot \mathrm{P}_{k|j}(-1|1) \\
		&= \mathrm{P}_i(1) \mathrm{P}_{jk}(1,-1) + V \cdot \mathrm{P}_{k|j}(-1|1)
\end{align*}

Combining the two we have
\begin{equation}
\label{AB}
A + B = V,
\end{equation}
where
\begin{align*}
A &= \mathrm{P}_{ijk} (1,1,1) - \mathrm{P}_i(1) \mathrm{P}_{jk}(1,1) \\
B &= \mathrm{P}_{ijk} (1,1,-1) - \mathrm{P}_i(1) \mathrm{P}_{jk}(1,-1)
\end{align*}

We also have
\begingroup
\allowdisplaybreaks
\begin{align*}
\mathrm{P}_{i\wideparen{ \,~~ j-}k}(1,1,1) &= \mathrm{P}_{jk}(1,1) \mathrm{P}_{i|k}(1|1) \\
		&= \mathrm{P}_i(1) \mathrm{P}_{jk}(1,1) + \mathrm{P}_{jk}(1,1) \cdot \paren{\mathrm{P}_{i|k}(1|1) - \mathrm{P}_i(1)} \\
		&= \mathrm{P}_i(1) \mathrm{P}_{jk}(1,1) + \mathrm{P}_{j|k}(1|1) \mathrm{P}_k(1) \cdot \paren{\mathrm{P}_{i|k}(1|1) - \mathrm{P}_i(1)} \\
		&= \mathrm{P}_i(1) \mathrm{P}_{jk}(1,1) + \mathrm{P}_{j|k}(1|1) \cdot \paren{\mathrm{P}_{ik}(1,1) - \mathrm{P}_i(1) \mathrm{P}_k(1)} \\
		&\begin{multlined}
		= \mathrm{P}_i(1) \mathrm{P}_{jk}(1,1) + \mathrm{P}_{j|k}(1|1) \cdot \bigl[\paren{\mathrm{P}_{ijk}(1,1,1) - \mathrm{P}_i(1) \mathrm{P}_{jk}(1,1)} \\
		\qquad \qquad \qquad \qquad \qquad + \paren{\mathrm{P}_{ijk}(1,-1,1) - \mathrm{P}_i(1) \mathrm{P}_{jk}(-1,1)}\bigr]
		\end{multlined}\\
		&\begin{multlined}
		= \mathrm{P}_i(1) \mathrm{P}_{jk}(1,1) + \mathrm{P}_{j|k}(1|1) \cdot \bigl[\paren{\mathrm{P}_{ij}(1,1)\mathrm{P}_{k|j}(1|1) - \mathrm{P}_i(1) \mathrm{P}_j(1)\mathrm{P}_{k|j}(1|1)} \\
		\qquad \qquad \qquad \qquad + \paren{\mathrm{P}_{ij}(1,-1)\mathrm{P}_{k|j}(1|-1) - \mathrm{P}_i(1) \mathrm{P}_j(-1)\mathrm{P}_{k|j}(1|-1)}\bigr]
		\end{multlined}\\
		&= \mathrm{P}_i(1) \mathrm{P}_{jk}(1,1) + \mathrm{P}_{j|k}(1|1) \cdot \sqbrac{\mathrm{P}_{k|j}(1|1) \cdot V + \mathrm{P}_{k|j}(1|-1) \cdot (-V)} \\
		&= \mathrm{P}_i(1) \mathrm{P}_{jk}(1,1) + \mathrm{P}_{j|k}(1|1) \cdot \paren{\mathrm{P}_{k|j}(1|1) - \mathrm{P}_{k|j}(1|-1)} \cdot V,
\end{align*}
\endgroup
and similarly (for obtaining the first equality below)
\begin{align*}
\mathrm{P}_{i\wideparen{ \,~~ j-}k}(1,1,-1) &= \mathrm{P}_i(1) \mathrm{P}_{jk}(1,-1) + \mathrm{P}_{j|k}(1|-1) \cdot \paren{\mathrm{P}_{k|j}(-1|1) - \mathrm{P}_{k|j}(-1|-1)} \cdot V \\
		&= \mathrm{P}_i(1) \mathrm{P}_{jk}(1,-1) + \mathrm{P}_{j|k}(1|-1) \cdot \sqbrac{\paren{1 - \mathrm{P}_{k|j}(1|1)} - \paren{1 - \mathrm{P}_{k|j}(1|-1)}} \cdot V \\
		&= \mathrm{P}_i(1) \mathrm{P}_{jk}(1,-1) - \mathrm{P}_{j|k}(1|-1) \cdot \paren{\mathrm{P}_{k|j}(1|1) - \mathrm{P}_{k|j}(1|-1)} \cdot V.
\end{align*}

Combining the two we have
\begin{equation}
\label{CD}
C + D = \paren{\mathrm{P}_{j|k}(1|1) - \mathrm{P}_{j|k}(1|-1)} \cdot \paren{\mathrm{P}_{k|j}(1|1) - \mathrm{P}_{k|j}(1|-1)} \cdot V,
\end{equation}
where
\begin{align*}
C &= \mathrm{P}_{i\wideparen{ \,~~ j-}k}(1,1,1) - \mathrm{P}_i(1) \mathrm{P}_{jk}(1,1) \\
D &= \mathrm{P}_{i\wideparen{ \,~~ j-}k}(1,1,-1) - \mathrm{P}_i(1) \mathrm{P}_{jk}(1,-1)
\end{align*}

Combining (\ref{AB}) and (\ref{CD}), recall the definition
$$\mindisc{\mathrm{P}_{jk}} = \min \brac{\abs{\mathrm{P}_{j|k}(1|1) - \mathrm{P}_{j|k}(1|-1)},\abs{\mathrm{P}_{k|j}(1|1) - \mathrm{P}_{k|j}(1|-1)}},$$
so that one of $\abs{\mathrm{P}_{j|k}(1|1) - \mathrm{P}_{j|k}(1|-1)}$ and $\abs{\mathrm{P}_{k|j}(1|1) - \mathrm{P}_{k|j}(1|-1)}$ is equal to $\mindisc{\mathrm{P}_{jk}}$, and the other is at most $1$, we have
\begin{align*}
\abs{A-C} + \abs{B-D} &\geq \abs{A+B-\paren{C+D}} \\
		&= \sqbrac{1 - \paren{\mathrm{P}_{j|k}(1|1) - \mathrm{P}_{j|k}(1|-1)} \cdot \paren{\mathrm{P}_{k|j}(1|1) - \mathrm{P}_{k|j}(1|-1)}} \cdot \abs{V} \\
		&\geq \paren{1 - \mindisc{\mathrm{P}_{jk}}} \cdot \abs{V}.
\end{align*}
This means that at least one of $\abs{A-C}$ and $\abs{B-D}$ is at least $\frac{1}{2} \paren{1 - \mindisc{\mathrm{P}_{jk}}} \cdot \abs{V}$. Suppose without loss of generality that it is $\abs{A-C}$. Plugging in the values we represent using $A$ and $C$, we have
\begin{equation}
\label{fracVdiff}
\abs{\mathrm{P}_{ijk}(1,1,1) - \mathrm{P}_{i\wideparen{ \,~~ j-}k}(1,1,1)} \geq \frac{1}{2} \paren{1 - \mindisc{\mathrm{P}_{jk}}} \cdot \abs{V}.
\end{equation}

Inequalities (\ref{significanthel}) and (\ref{fracVdiff}) are all we need to establish our desired bound. We have
\begingroup
\allowdisplaybreaks
\begin{align}
\begin{split}
\label{diffsqrt}
& \abs{\sqrt{\mathrm{P}_{ijk}(1,1,1)} - \sqrt{\mathrm{P}_{i\wideparen{ \,~~ j-}k}(1,1,1)}} \\
& \qquad \qquad \qquad \stackrel{(\ast)}{\geq} \sqrt{\mathrm{P}_{ijk}(1,1,1) + \abs{\mathrm{P}_{ijk}(1,1,1) - \mathrm{P}_{i\wideparen{ \,~~ j-}k}(1,1,1)}} - \sqrt{\mathrm{P}_{ijk}(1,1,1)} \\
& \qquad \qquad \qquad \geq \sqrt{\mathrm{P}_{ijk}(1,1,1) + \frac{1}{2} \paren{1 - \mindisc{\mathrm{P}_{jk}}} \cdot \abs{V}} - \sqrt{\mathrm{P}_{ijk}(1,1,1)} \\
& \qquad \qquad \qquad = \frac{\frac{1}{2} \paren{1 - \mindisc{\mathrm{P}_{jk}}} \cdot \abs{V}}{\sqrt{\mathrm{P}_{ijk}(1,1,1) + \frac{1}{2} \paren{1 - \mindisc{\mathrm{P}_{jk}}} \cdot \abs{V}} + \sqrt{\mathrm{P}_{ijk}(1,1,1)}} \\
& \qquad \qquad \qquad \geq \frac{1}{2} \paren{1 - \mindisc{\mathrm{P}_{jk}}} \cdot \frac{\abs{V}}{\sqrt{\mathrm{P}_{ijk}(1,1,1) + \abs{V}} + \sqrt{\mathrm{P}_{ijk}(1,1,1)}} \\
& \qquad \qquad \qquad = \frac{1}{2} \paren{1 - \mindisc{\mathrm{P}_{jk}}} \cdot \paren{\sqrt{\mathrm{P}_{ijk}(1,1,1) + \abs{V}} - \sqrt{\mathrm{P}_{ijk}(1,1,1)}} \\
& \qquad \qquad \qquad \stackrel{(\ast \ast)}{\geq} \frac{1}{2} \paren{1 - \mindisc{\mathrm{P}_{jk}}} \cdot \paren{\sqrt{\mathrm{P}_{ij}(1,1) + \abs{V}} - \sqrt{\mathrm{P}_{ij}(1,1)}}
\end{split}
\end{align}
\endgroup
where ($\ast$) is an equality in the case $\mathrm{P}_{ijk}(1,1,1) < \mathrm{P}_{i\wideparen{ \,~~ j-}k}(1,1,1)$, and follows from the fact that $\sqrt{a} - \sqrt{b} \geq \sqrt{a+c} - \sqrt{b+c}$, $\forall a \geq b \geq 0, c \geq 0$ in the case $\mathrm{P}_{ijk}(1,1,1) \geq \mathrm{P}_{i\wideparen{ \,~~ j-}k}(1,1,1)$; $(\ast \ast)$ follows from the same fact.

If $V < 0$, then $\sqrt{\mathrm{P}_{ij}(1,1) + \abs{V}} - \sqrt{\mathrm{P}_{ij}(1,1)} = \sqrt{\mathrm{P}_i(1) \mathrm{P}_j(1)} - \sqrt{\mathrm{P}_{ij}(1,1)}$.

If $V \geq 0$, then the definition $V = \mathrm{P}_{ij}(1,1) - \mathrm{P}_i(1) \mathrm{P}_j(1)$ implies $V \leq \mathrm{P}_{ij}(1,1)$, and so
\begingroup
\allowdisplaybreaks
\begin{align*}
\sqrt{\mathrm{P}_{ij}(1,1) + \abs{V}} - \sqrt{\mathrm{P}_{ij}(1,1)} &= \frac{V}{\sqrt{\mathrm{P}_{ij}(1,1) + V} + \sqrt{\mathrm{P}_{ij}(1,1)}} \\
		&\geq \frac{1}{\sqrt{2}+1} \cdot \frac{V}{\sqrt{\mathrm{P}_{ij}(1,1)} + \sqrt{\mathrm{P}_{ij}(1,1) - V}} \\
		&= \frac{1}{\sqrt{2} + 1} \cdot \paren{\sqrt{\mathrm{P}_{ij}(1,1)} - \sqrt{\mathrm{P}_{ij}(1,1) - V}} \\
		&= \frac{1}{\sqrt{2} + 1} \cdot \paren{\sqrt{\mathrm{P}_{ij}(1,1)} - \sqrt{\mathrm{P}_i(1) \mathrm{P}_j(1)}}.
\end{align*}
\endgroup

Therefore, no matter $V < 0$ or $V \geq 0$, we have 
\begin{equation}
\label{Vdiff}
\sqrt{\mathrm{P}_{ij}(1,1) + \abs{V}} - \sqrt{\mathrm{P}_{ij}(1,1)} \geq \frac{1}{\sqrt{2} + 1} \cdot \abs{\sqrt{\mathrm{P}_{ij}(1,1)} - \sqrt{\mathrm{P}_i(1) \mathrm{P}_j(1)}}.
\end{equation}
Combining (\ref{significanthel}, (\ref{diffsqrt}), and (\ref{Vdiff}) we have
\begin{align*}
H^2(\mathrm{P}_{ijk}, \mathrm{P}_{i\wideparen{ \,~~ j-}k}) &\geq \frac{1}{2} \paren{\sqrt{\mathrm{P}_{ijk}(1,1,1)} - \sqrt{\mathrm{P}_{i\wideparen{ \,~~ j-}k}(1,1,1)}}^2 \\
		&\geq \frac{1}{2} \paren{\frac{1}{2} \paren{1 - \mindisc{\mathrm{P}_{jk}}} \cdot \frac{1}{\sqrt{2} + 1} \cdot \abs{\sqrt{\mathrm{P}_{ij}(1,1)} - \sqrt{\mathrm{P}_i(1) \mathrm{P}_j(1)}}}^2 \\
		&\geq \frac{1}{100} \paren{1 - \mindisc{\mathrm{P}_{jk}}}^2 \cdot I_{H^2}(\mathrm{P}_{ij}).
\end{align*}
\end{proof}

The last five lemmas in this subsection prove various relationships among the measures $\mathrm{minmrg}$, $\mathrm{mindiag}$, $\mathrm{mindisc}$, and $I_{H^2}$, for $i,j,k$ that lie on a path in $T$. For example, Lemma~\ref{mindiscih3} says that the pair $(i,k)$ must be strong as measured by $I_{H^2}$ if the pairs $(i,j)$ and $(j,k)$ are strong as measured by $I_{H^2}$ and $\mathrm{mindisc}$, respectively. Lemma~\ref{minmrgih3} generalizes Lemma~\ref{ihminmrg}. It implies that, for example, a pair of variables must be close to being independent in $H^2$ under $\mathrm{P}$ if there is a node on the path in $T$ between them that has a very biased distribution under $\mathrm{P}$.

\begin{lemma}
\label{minmrgmindiagmindisc3}
If $i,j,k$ lie on a path in $T$, then $\abs{\mindisc{\mathrm{P}_{ij}} - \mindisc{\mathrm{P}_{ik}}} \leq 4 \cdot \frac{\mindiag{\mathrm{P}_{jk}}}{\minmrg{\mathrm{P}_{ij}}}$.
\end{lemma}
\begin{proof}
Without loss of generality, suppose that $\mindiag{\mathrm{P}_{jk}} = \mathrm{P}(X_j = -X_k)$. We have
$$\mathrm{P}_k(1) = \mathrm{P}_j(1) - \mathrm{P}_{jk}(1,-1) + \mathrm{P}_{jk}(-1,1),$$
which implies that
$$\abs{\mathrm{P}_k(1) - \mathrm{P}_j(1)} \leq \mathrm{P}(X_j=-X_k) = \mindiag{\mathrm{P}_{jk}}.$$
Similarly, $\abs{\mathrm{P}_k(-1) - \mathrm{P}_j(-1)} \leq \mindiag{\mathrm{P}_{jk}}$.

We also have
$$\mathrm{P}_{ik}(1,1) = \mathrm{P}_{ij}(1,1) - \mathrm{P}_{ijk}(1,1,-1) + \mathrm{P}_{ijk}(1,-1,1),$$
which implies that
$$\abs{\mathrm{P}_{ik}(1,1) - \mathrm{P}_{ij}(1,1)} \leq \mathrm{P}(X_j=-X_k) = \mindiag{\mathrm{P}_{jk}}.$$
Similarly, $\abs{\mathrm{P}_{ik}(1,-1) - \mathrm{P}_{ij}(1,-1)}$, $\abs{\mathrm{P}_{ik}(-1,1) - \mathrm{P}_{ij}(-1,1)}$, $\abs{\mathrm{P}_{ik}(-1,-1) - \mathrm{P}_{ij}(-1,-1)}$ are all at most $\mindiag{\mathrm{P}_{jk}}$.

So we have
$$\abs{\mathrm{P}_{k|i}(1|1) - \mathrm{P}_{j|i}(1|1)} = \frac{\abs{\mathrm{P}_{ik}(1,1) - \mathrm{P}_{ij}(1,1)}}{\mathrm{P}_i(1)} \leq \frac{\mindiag{\mathrm{P}_{jk}}}{\minmrg{\mathrm{P}_{ij}}},$$
and similarly $\abs{\mathrm{P}_{k|i}(1|-1) - \mathrm{P}_{j|i}(1|-1)} \leq \frac{\mindiag{\mathrm{P}_{jk}}}{\minmrg{\mathrm{P}_{ij}}}.$ Thus,
\begin{align}
\label{mindisckiji}
\Bigl| \abs{\mathrm{P}_{k|i}(1|1) - \mathrm{P}_{k|i}(1|-1)} - \abs{\mathrm{P}_{j|i}(1|1) - \mathrm{P}_{j|i}(1|-1)} \Bigr| \leq 2 \cdot \frac{\mindiag{\mathrm{P}_{jk}}}{\minmrg{\mathrm{P}_{ij}}}.
\end{align}

We also have
\begin{align*}
\abs{\mathrm{P}_{i|k}(1|1) - \mathrm{P}_{i|j}(1|1)} &= \abs{\frac{\mathrm{P}_{ik}(1,1)}{\mathrm{P}_k(1)} - \frac{\mathrm{P}_{ij}(1,1)}{\mathrm{P}_j(1)}} \\
		&\leq \abs{\frac{\mathrm{P}_{ik}(1,1)}{\mathrm{P}_k(1)} - \frac{\mathrm{P}_{ik}(1,1)}{\mathrm{P}_j(1)}} + \abs{\frac{\mathrm{P}_{ik}(1,1)}{\mathrm{P}_j(1)} - \frac{\mathrm{P}_{ij}(1,1)}{\mathrm{P}_j(1)}} \\
		&= \frac{\mathrm{P}_{ik}(1,1) \cdot \abs{\mathrm{P}_j(1) - \mathrm{P}_k(1)}}{\mathrm{P}_k(1) \mathrm{P}_j(1)} + \frac{\abs{\mathrm{P}_{ik}(1,1) - \mathrm{P}_{ij}(1,1)}}{\mathrm{P}_j(1)} \\
		&\leq 2 \cdot \frac{\mindiag{\mathrm{P}_{jk}}}{\minmrg{\mathrm{P}_{ij}}},
\end{align*}
and similarly $\abs{\mathrm{P}_{i|k}(1|-1) - \mathrm{P}_{i|j}(1|-1)} \leq 2 \cdot \frac{\mindiag{\mathrm{P}_{jk}}}{\minmrg{\mathrm{P}_{ij}}}$. Thus,
\begin{align}
\label{mindiscikij}
\Bigl| \abs{\mathrm{P}_{i|k}(1|1) - \mathrm{P}_{i|k}(1|-1)} - \abs{\mathrm{P}_{i|j}(1|1) - \mathrm{P}_{i|j}(1|-1)} \Bigr| \leq 4 \cdot \frac{\mindiag{\mathrm{P}_{jk}}}{\minmrg{\mathrm{P}_{ij}}}.
\end{align}

Since
\begin{align*}
\mindisc{\mathrm{P}_{ij}} = \min \brac{\abs{\mathrm{P}_{i|j} (1|1) - \mathrm{P}_{i|j}(1|-1)}, \abs{\mathrm{P}_{j|i} (1|1) - \mathrm{P}_{j|i}(1|-1)}} \\
\mindisc{\mathrm{P}_{ik}} = \min \brac{\abs{\mathrm{P}_{i|k} (1|1) - \mathrm{P}_{i|k}(1|-1)}, \abs{\mathrm{P}_{k|i} (1|1) - \mathrm{P}_{k|i}(1|-1)}}
\end{align*}
the desired conclusion follows from (\ref{mindisckiji}) and (\ref{mindiscikij}) because perturbing each argument of the $\min$ function by at most $4 \cdot \frac{\mindiag{\mathrm{P}_{jk}}}{\minmrg{\mathrm{P}_{ij}}}$ changes the overall min value by at most $4 \cdot \frac{\mindiag{\mathrm{P}_{jk}}}{\minmrg{\mathrm{P}_{ij}}}$.
\end{proof}

\begin{lemma}
\label{mindiscih3}
If $i,j,k$ lie on a path in $T$, then $I_{H^2}(\mathrm{P}_{ik}) \geq \frac{1}{100} {\mindisc{\mathrm{P}_{jk}}}^3 \cdot I_{H^2}(\mathrm{P}_{ij})$.
\end{lemma}
\begin{proof}
Since $I_{H^2}(\mathrm{P}_{ij}) = \frac{1}{2} \sum_{x_i,x_j = \pm 1} \paren{\sqrt{\mathrm{P}_{ij}(x_i,x_j)} - \sqrt{\mathrm{P}_i(x_i) \mathrm{P}_j(x_j)}}^2$, we can suppose without loss of generality that
\begin{equation}
\label{besthelterm}
\paren{\sqrt{\mathrm{P}_{ij}(1,1)} - \sqrt{\mathrm{P}_i(1) \mathrm{P}_j(1)}}^2 \geq \frac{1}{2} I_{H^2}(\mathrm{P}_{ij}).
\end{equation}

It is easy to see that $\mathrm{P}_{ij}(1,1) - \mathrm{P}_i(1) \mathrm{P}_j(1) = -\paren{\mathrm{P}_{ij}(1,-1) - \mathrm{P}_i(1) \mathrm{P}_j(-1)}$. To shorten the length of subsequent equation-displays, We use the shorthand
$$V = \mathrm{P}_{ij}(1,1) - \mathrm{P}_i(1) \mathrm{P}_j(1),$$
and so $-V = \mathrm{P}_{ij}(1,-1) - \mathrm{P}_i(1) \mathrm{P}_j(-1)$.

Since $\mindisc{\mathrm{P}_{jk}} \leq \abs{\mathrm{P}_{j|k}(1|1) - \mathrm{P}_{j|k}(1|-1)}$, at least one of $\mathrm{P}_{j|k}(1|1)$ and $\mathrm{P}_{j|k}(1|-1)$ is at least $\mindisc{\mathrm{P}_{jk}}$. Suppose without loss of generality $\mathrm{P}_{j|k}(1|1) = \frac{\mathrm{P}_{jk}(1,1)}{\mathrm{P}_k(1)} \geq \mindisc{\mathrm{P}_{jk}}$. We thus have
\begin{equation}
\label{kmrgbound}
\mathrm{P}_k(1) \leq \frac{\mathrm{P}_{jk}(1,1)}{\mindisc{\mathrm{P}_{jk}}} \leq \frac{\mathrm{P}_j(1)}{\mindisc{\mathrm{P}_{jk}}}.
\end{equation}

We also have the following:
\begingroup
\allowdisplaybreaks
\begin{align*}
\mathrm{P}_{ik}(1,1) - \mathrm{P}_i(1) \mathrm{P}_k(1) &= \sqbrac{\mathrm{P}_{ijk}(1,1,1) + \mathrm{P}_{ijk}(1,-1,1)} - \sqbrac{\mathrm{P}_i(1) \mathrm{P}_{jk}(1,1) + \mathrm{P}_i(1) \mathrm{P}_{jk}(-1,1)} \\
		&\begin{multlined}= \sqbrac{\mathrm{P}_{ij}(1,1) \mathrm{P}_{k|j}(1|1) + \mathrm{P}_{ij}(1,-1) \mathrm{P}_{k|j}(1|-1)} \\
		\qquad \qquad - \sqbrac{\mathrm{P}_i(1) \mathrm{P}_j(1) \mathrm{P}_{k|j}(1|1) + \mathrm{P}_i(1) \mathrm{P}_j(-1) \mathrm{P}_{k|j}(1|-1)}
		\end{multlined}\\
		&\begin{multlined}= \sqbrac{\mathrm{P}_{ij}(1,1) \mathrm{P}_{k|j}(1|1) - \mathrm{P}_i(1) \mathrm{P}_j(1) \mathrm{P}_{k|j}(1|1)} \\
		\qquad \qquad + \sqbrac{\mathrm{P}_{ij}(1,-1) \mathrm{P}_{k|j}(1|-1) - \mathrm{P}_i(1) \mathrm{P}_j(-1) \mathrm{P}_{k|j}(1|-1)}
		\end{multlined}\\
		&= V \cdot \mathrm{P}_{k|j}(1|1) - V \cdot \mathrm{P}_{k|j}(1|-1),
\end{align*}
\endgroup
which gives the bound
\begin{equation}
\label{probdiff}
\abs{\mathrm{P}_{ik}(1,1) - \mathrm{P}_i(1) \mathrm{P}_k(1)} = \abs{V} \cdot \abs{\mathrm{P}_{k|j}(1|1) - \mathrm{P}_{k|j}(1|-1)} \geq \abs{V} \cdot \mindisc{\mathrm{P}_{jk}}.
\end{equation}

The inequalities (\ref{besthelterm}), (\ref{kmrgbound}), and (\ref{probdiff}) are all we need. We have
\begingroup
\allowdisplaybreaks
\begin{align}
\label{diffsqrt2}
\abs{\sqrt{\mathrm{P}_{ik}(1,1)} - \sqrt{\mathrm{P}_i(1) \mathrm{P}_k(1)}} &\stackrel{(\ast)}{\geq} \sqrt{\mathrm{P}_i(1) \mathrm{P}_k(1) + \abs{\mathrm{P}_{ik}(1,1) - \mathrm{P}_i(1) \mathrm{P}_k(1)}} - \sqrt{\mathrm{P}_i(1) \mathrm{P}_k(1)} \nonumber \\
		&\hspace{-5em}\geq \sqrt{\mathrm{P}_i(1) \mathrm{P}_k(1) + \abs{V} \cdot \mindisc{\mathrm{P}_{jk}}} - \sqrt{\mathrm{P}_i(1) \mathrm{P}_k(1)} \nonumber \\
		&\hspace{-5em}= \frac{\abs{V} \cdot \mindisc{\mathrm{P}_{jk}}}{\sqrt{\mathrm{P}_i(1) \mathrm{P}_k(1) + \abs{V} \cdot \mindisc{\mathrm{P}_{jk}}} + \sqrt{\mathrm{P}_i(1) \mathrm{P}_k(1)}} \nonumber \\
		&\hspace{-5em}\geq {\mindisc{\mathrm{P}_{jk}}}^2 \cdot \frac{\frac{\abs{V}}{\mindisc{\mathrm{P}_{jk}}}}{\sqrt{\mathrm{P}_i(1) \mathrm{P}_k(1) + \frac{\abs{V}}{\mindisc{\mathrm{P}_{jk}}}} + \sqrt{\mathrm{P}_i(1) \mathrm{P}_k(1)}} \\
		&\hspace{-5em}= {\mindisc{\mathrm{P}_{jk}}}^2 \cdot \paren{\sqrt{\mathrm{P}_i(1) \mathrm{P}_k(1) + \frac{\abs{V}}{\mindisc{\mathrm{P}_{jk}}}} - \sqrt{\mathrm{P}_i(1) \mathrm{P}_k(1)}} \nonumber \\
		&\hspace{-11em}\stackrel{(\ast \ast)}{\geq} {\mindisc{\mathrm{P}_{jk}}}^2 \cdot \paren{\sqrt{\mathrm{P}_i(1) \cdot \frac{\mathrm{P}_j(1)}{\mindisc{\mathrm{P}_{jk}}} + \frac{\abs{V}}{\mindisc{\mathrm{P}_{jk}}}} - \sqrt{\mathrm{P}_i(1) \cdot \frac{\mathrm{P}_j(1)}{\mindisc{\mathrm{P}_{jk}}}}} \nonumber \\
		&\hspace{-5em}= {\mindisc{\mathrm{P}_{jk}}}^\frac{3}{2} \cdot \paren{\sqrt{\mathrm{P}_i(1) \mathrm{P}_j(1) + \abs{V}} - \sqrt{\mathrm{P}_i(1) \mathrm{P}_j(1)}} \nonumber
\end{align}
\endgroup
where ($\ast$) is an equality in the case $\mathrm{P}_{ik}(1,1) \geq \mathrm{P}_i(1) \mathrm{P}_k(1)$, and follows from the fact that $\sqrt{a} - \sqrt{b} \geq \sqrt{a+c} - \sqrt{b+c}$, $\forall a \geq b \geq 0, c \geq 0$ in the case $\mathrm{P}_{ik}(1,1) < \mathrm{P}_i(1) \mathrm{P}_k(1)$; $(\ast \ast)$ follows from the same fact.

If $V \geq 0$, then $\sqrt{\mathrm{P}_i(1) \mathrm{P}_j(1) + \abs{V}} - \sqrt{\mathrm{P}_i(1) \mathrm{P}_j(1)} = \sqrt{\mathrm{P}_{ij}(1,1)} - \sqrt{\mathrm{P}_i(1) \mathrm{P}_j(1)}$.

If $V < 0$, then the definition $V = \mathrm{P}_{ij}(1,1) - \mathrm{P}_i(1) \mathrm{P}_j(1)$ implies $\abs{V} < \mathrm{P}_i(1) \mathrm{P}_j(1)$, and so
\begingroup
\allowdisplaybreaks
\begin{align*}
\sqrt{\mathrm{P}_i(1) \mathrm{P}_j(1) + \abs{V}} - \sqrt{\mathrm{P}_i(1) \mathrm{P}_j(1)} &= \frac{\abs{V}}{\sqrt{\mathrm{P}_i(1) \mathrm{P}_j(1) + \abs{V}} + \sqrt{\mathrm{P}_i(1) \mathrm{P}_j(1)}} \\
		&\geq \frac{1}{\sqrt{2}+1} \cdot \frac{V}{\sqrt{\mathrm{P}_i(1) \mathrm{P}_j(1)} + \sqrt{\mathrm{P}_i(1) \mathrm{P}_j(1) - \abs{V}}} \\
		&= \frac{1}{\sqrt{2} + 1} \cdot \paren{\sqrt{\mathrm{P}_i(1) \mathrm{P}_j(1)} - \sqrt{\mathrm{P}_i(1) \mathrm{P}_j(1) - \abs{V}}} \\
		&= \frac{1}{\sqrt{2} + 1} \cdot \paren{\sqrt{\mathrm{P}_i(1) \mathrm{P}_j(1)} - \sqrt{\mathrm{P}_{ij}(1,1)}}.
\end{align*}
\endgroup

Therefore, no matter $V \geq 0$ or $V < 0$, we have 
\begin{equation}
\label{Vdiff2}
\sqrt{\mathrm{P}_i(1) \mathrm{P}_j(1) + \abs{V}} - \sqrt{\mathrm{P}_i(1) \mathrm{P}_j(1)} \geq \frac{1}{\sqrt{2} + 1} \cdot \abs{\sqrt{\mathrm{P}_i(1) \mathrm{P}_j(1)} - \sqrt{\mathrm{P}_{ij}(1,1)}}.
\end{equation}
Combining (\ref{besthelterm}, (\ref{diffsqrt2}), and (\ref{Vdiff2}) we have
\begin{align*}
I_{H^2}(\mathrm{P}_{ik}) &\geq \frac{1}{2} \paren{\sqrt{\mathrm{P}_{ik}(1,1)} - \sqrt{\mathrm{P}_i(1) \mathrm{P}_k(1)}}^2 \\
		&\geq \frac{1}{2} \paren{{\mindisc{\mathrm{P}_{jk}}}^\frac{3}{2} \cdot \frac{1}{\sqrt{2} + 1} \cdot \abs{\sqrt{\mathrm{P}_i(1) \mathrm{P}_j(1)} - \sqrt{\mathrm{P}_{ij}(1,1)}}}^2 \\
		&\geq \frac{1}{100} {\mindisc{\mathrm{P}_{jk}}}^3 \cdot I_{H^2}(\mathrm{P}_{ij}).
\end{align*}
\end{proof}

\begin{lemma}
\label{mindiagconcat}
If $i,j,k$ lie on a path in $T$, then $\mindiag{\mathrm{P}_{ik}} \leq \mindiag{\mathrm{P}_{ij}} + \mindiag{\mathrm{P}_{jk}}$.
\end{lemma}
\begin{proof}
Without loss of generality (by possibly switching the labels $1$ and $-1$ for some nodes) we can assume that $\mindiag{\mathrm{P}_{ij}} = \mathrm{P}(X_i=-X_j)$ and $\mindiag{\mathrm{P}_{jk}} = \mathrm{P}(X_j=-X_k)$. Since $X_i=-X_k$ implies either $X_i=-X_j$ or $X_j=-X_k$ we have
\begin{align*}
\mindiag{\mathrm{P}_{ik}} &= \min \brac{\mathrm{P}(X_i=X_k), \mathrm{P}(X_i=-X_k)} \\
		&\leq \mathrm{P}(X_i=-X_k) \\
		&\leq \mathrm{P}(X_i=-X_j) + \mathrm{P}(X_j=-X_k) \\
		&= \mindiag{\mathrm{P}_{ij}} + \mindiag{\mathrm{P}_{jk}}.
\end{align*}
\end{proof}

\begin{lemma}
\label{minmrgmindisc3}
If $i,j,k$ lie on a path in $T$, then $\minmrg{\mathrm{P}_j} \geq \minmrg{\mathrm{P}_{ik}} \mindisc{\mathrm{P}_{ik}}$. 
\end{lemma}
\begin{proof}
We have $\mathrm{P}_{j|i}(1|1) = \frac{\mathrm{P}_{ij}(1,1)}{\mathrm{P}_i(1)} \leq \frac{\mathrm{P}_j(1)}{\minmrg{\mathrm{P}_{ik}}}$, and similarly $\mathrm{P}_{j|i}(1|-1) \leq \frac{\mathrm{P}_j(1)}{\minmrg{\mathrm{P}_{ik}}}$. So
$$\mindisc{\mathrm{P}_{ij}} \leq \abs{\mathrm{P}_{j|i}(1|1) - \mathrm{P}_{j|i}(1|-1)} \leq \frac{\mathrm{P}_j(1)}{\minmrg{\mathrm{P}_{ik}}}.$$
Thus, $\mathrm{P}_j(1) \geq \minmrg{\mathrm{P}_{ik}} \mindisc{\mathrm{P}_{ij}} \geq \minmrg{\mathrm{P}_{ik}} \mindisc{\mathrm{P}_{ik}}$, where we used Lemma~\ref{mindiscmarkov}.

Similary, $\mathrm{P}_j(-1) \geq \minmrg{\mathrm{P}_{ik}} \mindisc{\mathrm{P}_{ik}}$, and the conclusion follows.
\end{proof}

\begin{lemma}
\label{minmrgih3}
If $i,j,k$ lie on a path in $T$, then $\minmrg{\mathrm{P}_j} \geq \frac{1}{2} I_{H^2}(\mathrm{P}_{ik})$. 
\end{lemma}
\begin{proof}
By Lemma~\ref{ihminmrg} and Lemma~\ref{ihmarkov}, we have
$$\minmrg{\mathrm{P}_j} \geq \minmrg{\mathrm{P}_{ij}} \geq \frac{1}{2} \cdot \mathrm{I}_{H^2}(\mathrm{P}_{ij}) \geq \frac{1}{2} \cdot I_{H^2}(\mathrm{P}_{ik}).$$
\end{proof}

\subsubsection{Hierarchies and Structural Results for $T$ vs. $\widehat{T}$}
\label{sec:genstructure}

This subsection contains structural results on the relation between $T$ and $\widehat{T}$, stated in the language of various hierarchies to be defined first. The layers we define here dictate how we switch from $T$ to $\widehat{T}$ in the hybrid argument for bounding $\dtv{\mathrm{P}}{\mathrm{Q}}$ in the next subsection. The major complication is that, unlike for the symmetric case, the groups at the top layer as directly induced by the classification of edges of $\widehat{T}$ are not necessarily $T$-connected (in fact this issue propagates to the groups at lower layers too). This would prevent us from describing our structural results and from applying Corollary~\ref{edgeswitch} in the way we wanted. Additional steps have to be taken to work around this issue.

We first classify all the edges of $\widehat{T}$ based on a combination of $\mathrm{minmrg}$, $\mathrm{mindiag}$, $\mathrm{mindisc}$, and $I_{H^2}$. This induces a preliminary hierarchical classification of nodes into groups.

\begin{definition}
\label{preliminaryhierarchy}
We establish the following hierarchies, starting with the empty graph $G$ on $[n]$ and adding edges incrementally:
\begin{itemize}
\item An edge $(i,j)$ of $\widehat{T}$ with $\minmrg{\widehat{\mathrm{P}}_{ij}} \geq 10^6 \frac{\epsilon^2}{n}$ and $\mindiag{\widehat{\mathrm{P}}_{ij}} \leq 10^5 \frac{\epsilon^2}{n}$ is classified as a $\pmb{\widehat{T}}$\textbf{-avenue}. Adding the $\widehat{T}$-avenues to $G$ links the $n$ nodes into connected components, each of which is called a \textbf{broken-city}.
\item An edge $(i,j)$ of $\widehat{T}$ with $\minmrg{\widehat{\mathrm{P}}_{ij}} \geq 10^8 \frac{\epsilon^2}{n}$ and $\mindisc{\widehat{\mathrm{P}}_{ij}} \geq \frac{1}{2}$ that is not a $\widehat{T}$-avenue is classified as a $\pmb{\widehat{T}}$\textbf{-highway}. Adding the $\widehat{T}$-highways to $G$ further links the broken-cities into bigger connected components, each of which is called a \textbf{broken-country}.
\item An edge $(i,j)$ of $\widehat{T}$ with $I_{H^2}(\widehat{\mathrm{P}}_{ij}) \geq 10^{10} \frac{\epsilon^2}{n}$ that is not a $\widehat{T}$-avenue or a $\widehat{T}$-highway is classified as a $\pmb{\widehat{T}}$\textbf{-railway}. Adding the $\widehat{T}$-railways to $G$ further links the broken-countries into even bigger connected components, each of which is called a \textbf{broken-continent}.
\item Each remaining edge of $\widehat{T}$ is classified as a $\pmb{\widehat{T}}$\textbf{-tunnel}. Adding the $\widehat{T}$-tunnels to $G$ further links the broken-continents into one single connected component of all $n$ nodes.
\end{itemize}
\end{definition}

As we mentioned, the broken-cities might not be $T$-connected. The following definition allows us to obtain sets from them that are $T$-connected by filling in the ``holes'' (a ``hole'' for a broken-city is a node which doesn't belong to the broken-city, but which lies on the path in $T$ between some pair of nodes that belong to the broken-city).

\begin{definition}[$\pmb{\conv{T}{\cdot}}$]
\label{conv}
For a subset $S$ of $[n]$, let $\pmb{\conv{T}{S}}$ (called the \textbf{convex hull} of $S$ in $T$) denote the set of nodes consisting of nodes in $S$, as well as every node that lies on the path in $T$ between some $i,j \in S$. In other words, $\conv{T}{S}$ is the minimal $T$-connected subset of $[n]$ that contains $S$.
\end{definition}

The convex hull in $T$ of each broken-city is obviously $T$-connected. However, the convex hulls of different broken-cities might intersect, and whenever that happens, we need to merge the two broken-cities. Doing that leads to a partition of the $n$ nodes into $T$-connected subsets, which forms the top layer of our hierarchy. The definitions for the lower layers and the (induced) classification of edges of $T$ are also contained in the following:

\begin{definition}
\label{truehierarchy}
We establich the following hierarchy:
\begin{itemize}
\item The set of nodes $[n]$ is partitioned into \textbf{cities} such that each city is a union of broken-cities, and that two broken-cities $\widetilde{C}$ and $\widetilde{D}$ are contained in the same city if and only if there exists a sequence of broken-cities $\widetilde{C} = \widetilde{C}_0, \widetilde{C}_1, \ldots, \widetilde{C}_r = \widetilde{D}$ such that $\conv{T}{\widetilde{C}_s} \cap \conv{T}{\widetilde{C}_{s+1}} \neq \emptyset$ for each $s = 0, \ldots, r-1$. An edge $(i,j)$ of $T$ with $i,j$ belonging to the same city is classified as a $\pmb{T}$\textbf{-road}.
\item The set of nodes $[n]$ is partitioned into \textbf{countries} such that each country is a union of cities, and that two cities $C$ and $D$ are contained in the same country if and only if there exists a sequence of cities $C = C_0, C_1, \ldots, C_r = D$ such that there exists a $\widehat{T}$-highway between $C_s$ and $C_{s+1}$ for each $s = 0, \ldots, r-1$. An edge $(i,j)$ of $T$ with $i,j$ belonging to the different cities in the same country is classified as a $\pmb{T}$\textbf{-highway}.
\item The set of nodes $[n]$ is partitioned into \textbf{continents} such that each continent is a union of countries, and that two countries $\mathcal{C}$ and $\mathcal{D}$ are contained in the same continent if and only if there exists a sequence of countries $\mathcal{C} = \mathcal{C}_0, \mathcal{C}_1, \ldots, \mathcal{C}_r = \mathcal{D}$ such that there exists a $\widehat{T}$-railway between $\mathcal{C}_s$ and $\mathcal{C}_{s+1}$ for each $s = 0, \ldots, r-1$. An edge $(i,j)$ of $T$ with $i,j$ belonging to the different countries in the same continent is classified as a $\pmb{T}$\textbf{-railway}.
\item An edge $(i,j)$ of $T$ with $i,j$ belonging to different continents is classified as a $\pmb{T}$\textbf{-airway}.
\end{itemize}
\end{definition}

Our first lemma compiles a list of containment relationships among the (preliminary) groups we defined.

\begin{lemma}
\label{containment}
\begin{itemize}
\item[(i)] A broken-country is a union of broken-cities.
\item[(ii)] A broken-continent is a union of broken-countries.
\item[(iii)] A country is a union of cities.
\item[(iv)] A continent is a union of countries.
\item[(v)] A city is a union of broken-cities.
\item[(vi)] A city is the union of the convex hulls in $T$ of those broken-cities contained in the city.
\item[(vii)] A country is a union of broken-countries.
\item[(viii)] A continent is a union of broken-continents.
\end{itemize}
\end{lemma}
\begin{proof}
Note that (i), (ii), (iii), (iv), (v) follow directly from Definition~\ref{preliminaryhierarchy} and Definition~\ref{truehierarchy}.

To prove (vi) it suffices to show that if $\widetilde{C}$ is a broken-city contained in a city $C$, then $\conv{T}{\widetilde{C}}$ is also contained in $C$. To see that, let $i$ be any node in $\conv{T}{\widetilde{C}}$. If $i \in \widetilde{C}$ then obviously $i \in C$. Otherwise, let $\widetilde{D}$ be the broken-city containing $i$. Since $i \in \conv{T}{\widetilde{C}} \cap \conv{T}{\widetilde{D}} \neq \emptyset$, by Definition~\ref{truehierarchy} $\widetilde{C}$ and $\widetilde{D}$ are contained in the same city. So $i \in C$ and (vi) is proved.

By Definition~\ref{preliminaryhierarchy} a broken-country consists of a collection of broken-cities linked together by $\widehat{T}$-highways. By Definition~\ref{truehierarchy} the two end-nodes of each $\widehat{T}$-highway belong to the same country. Since each broken-city is contained entirely in one city (by (v)), and therefore entirely in one country, we see that each broken-country is contained entirely in one country, and (vii) follows. 

By Definition~\ref{preliminaryhierarchy} a broken-continent consists of a collection of broken-countries linked together by $\widehat{T}$-railways. By Definition~\ref{truehierarchy} the two end-nodes of each $\widehat{T}$-railway belong to the same continent. Since each broken-country is contained entirely in one country (by (vii)), and therefore entirely in one continent, we see that each broken-continent is contained entirely in one continent, and (viii) follows. 
\end{proof}

The formal proof that every city is $T$-connected is as follows:

\begin{lemma}
\label{gencityconnected}
Every city is $T$-connected.
\end{lemma}
\begin{proof}
Let $C$ be any city, and $i,j$ be any pair of nodes in $C$. It suffices to show that we can go from $i$ to $j$ using only edges of $T$ without leaving $C$.

Let $\widetilde{C}$ and $\widetilde{D}$ be broken-cities such that $i \in \widetilde{C}$ and $j \in \widetilde{D}$. By Definition~\ref{truehierarchy} there exists a sequence of broken-cities $\widetilde{C} = \widetilde{C}_0, \widetilde{C}_1, \ldots, \widetilde{C}_r = \widetilde{D}$ such that $\conv{T}{\widetilde{C}_s} \cap \conv{T}{\widetilde{C}_{s+1}} \neq \emptyset$ for each $s = 0,\ldots, r-1$. It is easy to see that Definition~\ref{truehierarchy} implies $\widetilde{C}_s \subset C$ for each $s = 0, \ldots, r$, and so $\conv{T}{\widetilde{C}_s} \subset C$ by Lemma~\ref{containment} (vi). For each $s = 0, \ldots, r-1$, let $k_s$ be any node in $\conv{T}{\widetilde{C}_s} \cap \conv{T}{\widetilde{C}_{s+1}}$.

Since $i,k_0 \in \conv{T}{\widetilde{C}_0}$, and $\conv{T}{\widetilde{C}_0}$ is $T$-connected, the path in $T$ between $i$ and $k_0$ lies entirely in $\conv{T}{\widetilde{C}_0} \subset C$. Similarly, for each $s = 0, \ldots, r-1$, we have $k_s,k_{s+1} \in \conv{T}{\widetilde{C}_{s+1}}$, and so the path in $T$ between them lies entirely in $\conv{T}{\widetilde{C}_{s+1}} \subset C$. Lastly, $k_{r-1},j \in \conv{T}{\widetilde{C}_r}$, so the path in $T$ between them lies entirely in $\conv{T}{\widetilde{C}_r} \subset C$. Tracing all those paths in order allows us to go from $i$ to $j$ using only edges of $T$ without leaving $C$.
\end{proof}

As discussed in the part titled ``The remedy'' in Section~\ref{sec:outline of general case}, the fact that a city is not necessarily spanned by the $\widehat{T}$-avenues whose two end-nodes belong to the city prompted us to find edges that augment them into a spanning set for the city (we will use this spanning set to replace the spanning set formed by the $T$-roads with both end-nodes in the city in the process of going from $T$ to $\widehat{T}$). Specifically, those edges can be selected from among the $T$-roads with both end-nodes in the city. We now formalize the definition and the selection procedure.

\begin{definition}[$\pmb{T}$\textbf{-trails}]
\label{ttrail}
For every city $C$, we select a subset of edges among the $T$-roads with both end-nodes in $C$ such that the selected edges, together with the set of all $\widehat{T}$-avenues with both end-nodes in $C$, form a spanning tree of $C$. We do not select any $\widehat{T}$-avenue. See Remark~\ref{trailselection} below for one way to do it. Each selected edge is called a $\pmb{T}$\textbf{-trail}.
\end{definition}
\begin{remark}
\label{trailselection}
One way to make the selection in Definition~\ref{ttrail} is as follows: go through all $T$-roads with both end-nodes in $C$, in an arbitrary order. For a $T$-road $(i,j)$, we include it in our selection if and only if adding it to the union of all $\widehat{T}$-avenues with both end-nodes in $C$ and the edges we already selected does not create any cycle (and that $(i,j)$ is not itself a $\widehat{T}$-avenue).

To see that the set of selected edges by the end of the procedure above indeed, together with all $\widehat{T}$-avenues with both end-nodes in $C$, form a spanning tree of $C$, first note that obviously there can be no cycle. On the other hand, suppose $C$ can be partitioned into two non-empty subsets with no selected edge or $\widehat{T}$-avenue going in between. Since $C$ is $T$-connected by Lemma~\ref{gencityconnected}, there must exist a $T$-road between those two subsets, and this $T$-road should have been selected when we examined it in the selection procedure above. This is a contradiction.
\end{remark}

We show that a pair of nodes that is either a $T$-road or a $\widehat{T}$-avenue must be almost always equal or almost always unequal under $\mathrm{P}$ (in other words, the pair is strong as measured by $\mathrm{mindiag}$). Note that the case of a $T$-trail is also included because a $T$-trail is a $T$-road by Definition~\ref{ttrail}. This will enable us to prove that within a city it doesn't matter much in terms of $H^2$ whether the nodes are held together by the $T$-roads in the city or by the $\widehat{T}$-avenues and $T$-trails in the city.

\begin{lemma}
\label{roadavenuemindiag}
If $(i,j)$ is a $T$-road or a $\widehat{T}$-avenue, then $\mindiag{\mathrm{P}_{ij}} \leq (10^5 + 1) \frac{\epsilon^2}{n}$.
\end{lemma}
\begin{proof}
If $(i,j)$ is a $\widehat{T}$-avenue, then by definition $\mindiag{\widehat{\mathrm{P}}_{ij}} \leq 10^5 \frac{\epsilon^2}{n}$. By Lemma~\ref{genmindiagprecision} (ii), we have $\mindiag{\mathrm{P}_{ij}} \leq (10^5 + 1) \frac{\epsilon^2}{n}$, as desired.

Suppose from now on that $(i,j)$ is a $T$-road. First we show that $(i,j)$ is on the path in $T$ between the two end-nodes of some $\widehat{T}$-avenue. Suppose not for the sake of contradiction. Imagine removing $(i,j)$ from $T$. This partitions $[n]$ into two $T$-connected subsets $U$ and $V$. Then for each $\widehat{T}$-avenue, its two end-nodes must both belong to $U$ or both belong to $V$. Since a broken-city is spanned by $\widehat{T}$-avenues, we see that each broken-city must be contained either entirely in $U$ or entirely in $V$. Since $U$ and $V$ are $T$-connected, the convex hull in $T$ of each broken-city must be contained either entirely in $U$ or entirely in $V$. It is then easy to see that $i$ and $j$ can't belong to the same city because the sequence of broken-cities bridging between $i$'s and $j$'s broken-cities as required according to Definition~\ref{truehierarchy} can't possibly exist. This contradicts the assumption that $(i,j)$ is a $T$-road.

So we can indeed find $h,k$ such that $(h,k)$ is a $\widehat{T}$-avenue and that $(i,j)$ is on the path in $T$ between $h$ and $k$. By definition $\minmrg{\widehat{\mathrm{P}}_{hk}} \geq 10^6 \frac{\epsilon^2}{n}$ and $\mindiag{\widehat{\mathrm{P}}_{hk}} \leq 10^5 \frac{\epsilon^2}{n}$. Note that $\minmrg{\widehat{\mathrm{P}}_{hk}} \geq 10^6 \frac{\epsilon^2}{n}$ implies $\minmrg{\mathrm{P}_{hk}} \geq (10^6 - 1) \frac{\epsilon^2}{n}$ (the contrapositive of this follows from Lemma~\ref{genminmrgprecision} (i)). Also $\mindiag{\mathrm{P}_{hk}} \leq (10^5 + 1) \frac{\epsilon^2}{n}$ by Lemma~\ref{genmindiagprecision} (ii). Thus, we can apply Lemma~\ref{mindiagmarkov} (or Lemma~\ref{mindiagmarkov4}) with $h$ and $k$ at the two ends and any node(s) on the path in $T$ between $h$ and $k$ in the middle. It is easy to argue that in any case we have $\mindiag{\mathrm{P}_{ij}} \leq \mindiag{\mathrm{P}_{hk}} \leq (10^5 + 1) \frac{\epsilon^2}{n}$, as desired.
\end{proof}

The next lemma provides sufficient conditions for a pair $i,j$ to belong to the same broken-city. The proof of this lemma is relatively involved. The main complication is that, even though the stated conditions appear stronger than those for $\widehat{T}$-avenues (Definition~\ref{preliminaryhierarchy}), we can not directly deduce that $i,j$ belong to the same broken-city because the pair might not be an edge of $\widehat{T}$ and so might not even be subject to the classification as described in Definition~\ref{preliminaryhierarchy} at all. Instead, we have to reason about the edges of $\widehat{T}$ on the path in $\widehat{T}$ between $i$ and $j$.

\begin{lemma}
\label{samebrokencity}
If $\minmrg{\mathrm{P}_{ij}} \geq (5 \cdot 10^6 + 1) \frac{\epsilon^2}{n}$ and $\mindiag{\mathrm{P}_{ij}} \leq (6 \cdot 10^4 - 1) \frac{\epsilon^2}{n}$, then $i,j$ belong to the same broken-city.
\end{lemma}
\begin{proof}
Suppose the path in $\widehat{T}$ between $i$ and $j$ consists of nodes $i = i'_0, \ldots, i'_r = j$, in that order. We first show that $\minmrg{\mathrm{P}_{i'_s}} \geq 3 \cdot 10^6 \frac{\epsilon^2}{n}$, for all $s = 0, \ldots, r$.

By the cycle property of maximum weight spanning tree $\mathrm{I}^{\widehat{\mathrm{P}}}(X_i;X_j) \leq \mathrm{I}^{\widehat{\mathrm{P}}}(X_{i'_s};X_{i'_{s+1}})$ for all $s = 0, \ldots, r-1$. Note that $\minmrg{\mathrm{P}_{ij}} \geq (5 \cdot 10^6 + 1) \frac{\epsilon^2}{n}$ implies $\minmrg{\widehat{\mathrm{P}}_{ij}} \geq 5 \cdot 10^6 \frac{\epsilon^2}{n}$ (the contrapositive of this follows from Lemma~\ref{genminmrgprecision} (ii)). Also $\mindiag{\widehat{\mathrm{P}}_{ij}} \leq 6 \cdot 10^4 \frac{\epsilon^2}{n}$ by Lemma~\ref{genmindiagprecision} (i). So we have
\begingroup
\allowdisplaybreaks
\begin{align*}
\mathrm{I}^{\widehat{\mathrm{P}}}(X_i;X_j) &= \sum_{x_i,x_j = \pm 1} \widehat{\mathrm{P}}_{ij}(x_i,x_j) \ln{\frac{\widehat{\mathrm{P}}_{ij}(x_i,x_j)}{\widehat{\mathrm{P}}_i(x_i) \widehat{\mathrm{P}}_j(x_j)}} \\
		&= \sum_{x_i = \pm 1} \widehat{\mathrm{P}}_i(x_i) \ln{\frac{1}{\widehat{\mathrm{P}}_i(x_i)}} - \sum_{x_j = \pm 1} \brac{\widehat{\mathrm{P}}_j(x_j) \sqbrac{\sum_{x_i = \pm 1} \frac{\widehat{\mathrm{P}}_{ij}(x_i,x_j)}{\widehat{\mathrm{P}}_j(x_j)} \ln{\frac{1}{\frac{\widehat{\mathrm{P}}_{ij}(x_i,x_j)}{\widehat{\mathrm{P}}_j(x_j)}}}}} \\
		&\begin{multlined}
		\stackrel{(\ast)}{\geq} 5 \cdot 10^6 \frac{\epsilon^2}{n} \cdot \ln{\frac{1}{5 \cdot 10^6 \frac{\epsilon^2}{n}}} \\
		- \sum_{x_j = \pm 1} \brac{\widehat{\mathrm{P}}_j(x_j) \sqbrac{\frac{6 \cdot 10^4 \frac{\epsilon^2}{n}}{\widehat{\mathrm{P}}_j(x_j)} \ln{\frac{1}{\frac{6 \cdot 10^4 \frac{\epsilon^2}{n}}{\widehat{\mathrm{P}}_j(x_j)}}} + \paren{1 - \frac{6 \cdot 10^4 \frac{\epsilon^2}{n}}{\widehat{\mathrm{P}}_j(x_j)}} \ln{\frac{1}{\paren{1 - \frac{6 \cdot 10^4 \frac{\epsilon^2}{n}}{\widehat{\mathrm{P}}_j(x_j)}}}}}}
		\end{multlined}\\
		&\begin{multlined}
		= 5 \cdot 10^6 \frac{\epsilon^2}{n} \cdot \ln{\frac{1}{5 \cdot 10^6 \frac{\epsilon^2}{n}}} \\
		- \sum_{x_j = \pm 1} \brac{6 \cdot 10^4 \frac{\epsilon^2}{n} \cdot \ln{\frac{\widehat{\mathrm{P}}_j(x_j)}{6 \cdot 10^4 \frac{\epsilon^2}{n}}} + \widehat{\mathrm{P}}_j(x_j) \paren{1 - \frac{6 \cdot 10^4 \frac{\epsilon^2}{n}}{\widehat{\mathrm{P}}_j(x_j)}} \ln{\frac{1}{\paren{1 - \frac{6 \cdot 10^4 \frac{\epsilon^2}{n}}{\widehat{\mathrm{P}}_j(x_j)}}}}}
		\end{multlined}\\
		&\stackrel{(\ast \ast)}{\geq} 5 \cdot 10^6 \frac{\epsilon^2}{n} \cdot \ln{\frac{1}{5 \cdot 10^6 \frac{\epsilon^2}{n}}} - \sum_{x_j = \pm 1} \brac{6 \cdot 10^4 \frac{\epsilon^2}{n} \cdot \ln{\frac{1}{6 \cdot 10^4 \frac{\epsilon^2}{n}}} + \widehat{\mathrm{P}}_j(x_j) \cdot 2 \cdot \frac{6 \cdot 10^4 \frac{\epsilon^2}{n}}{\widehat{\mathrm{P}}_j(x_j)}} \\
		&\begin{multlined}
		= 5 \cdot 10^6 \frac{\epsilon^2}{n} \cdot \ln{\frac{1}{\frac{\epsilon^2}{n}}} - 5 \cdot 10^6 \frac{\epsilon^2}{n} \cdot \ln{\brac{5 \cdot 10^6}} \\
		\qquad \qquad \qquad \qquad - 12 \cdot 10^4 \frac{\epsilon^2}{n} \ln{\frac{1}{\frac{\epsilon^2}{n}}} + 12 \cdot 10^4 \frac{\epsilon^2}{n} \ln{\brac{6 \cdot 10^4}} - 24 \cdot 10^4 \frac{\epsilon^2}{n}
		\end{multlined} \\
		&\stackrel{(\ast \ast \ast)}{\geq} 4 \cdot 10^6 \frac{\epsilon^2}{n} \cdot \ln{\frac{1}{\frac{\epsilon^2}{n}}}.
\end{align*}
\endgroup
Here, $(\ast)$ follows from the fact that the binary entropy function $t \ln{\frac{1}{t}} + (1-t) \ln{\frac{1}{1-t}}$ is increasing for $t \in [0,\frac{1}{2}]$ and decreasing for $t \in [\frac{1}{2},1]$ (symmetrically with respect to $t = \frac{1}{2}$), that $\min{\brac{\widehat{\mathrm{P}}_i(1), \widehat{\mathrm{P}}_i(-1)}} \geq 5 \cdot 10^6 \frac{\epsilon^2}{n}$, and that $\min{\brac{\frac{\widehat{\mathrm{P}}_{ij}(1,x_j)}{\widehat{\mathrm{P}}_j(x_j)}, \frac{\widehat{\mathrm{P}}_{ij}(-1,x_j)}{\widehat{\mathrm{P}}_j(x_j)}}} \leq \frac{6 \cdot 10^4 \frac{\epsilon^2}{n}}{\widehat{\mathrm{P}}_j(x_j)} < \frac{1}{2}$ for both $x_j \in \{1,-1\}$ (due to $\minmrg{\widehat{\mathrm{P}}_{ij}} \geq 5 \cdot 10^6 \frac{\epsilon^2}{n}$ and $\mindiag{\widehat{\mathrm{P}}_{ij}} \leq 6 \cdot 10^4 \frac{\epsilon^2}{n}$). Inequality $(\ast \ast)$ follows from the fact that $\ln{\frac{1}{1-t}} \leq 2t$ for $t \in [0,\frac{1}{2}]$. Inequality $(\ast \ast \ast)$ holds because we assumed $\epsilon \leq \frac{1}{10^{100}}$.

Suppose for the sake of contradiction that $\minmrg{\mathrm{P}_{i'_s}} < 3 \cdot 10^6 \frac{\epsilon^2}{n}$ for some $s = 0, \ldots, r-1$. By Lemma~\ref{genminmrgprecision} (i) we have $\minmrg{\widehat{\mathrm{P}}_{i'_s}} < (3 \cdot 10^6 + 1) \frac{\epsilon^2}{n}$. So
\begingroup
\allowdisplaybreaks
\begin{align*}
\mathrm{I}^{\widehat{\mathrm{P}}}(X_{i'_s};X_{i'_{s+1}}) &\leq \sum_{x_{i'_s} = \pm 1} \widehat{\mathrm{P}}_{i'_s}(x_{i'_s}) \ln{\frac{1}{\widehat{\mathrm{P}}_{i'_s}(x_{i'_s})}} \\
		&\begin{multlined}
		< \paren{3 \cdot 10^6 + 1} \frac{\epsilon^2}{n} \cdot \ln{\frac{1}{\paren{3 \cdot 10^6 + 1} \frac{\epsilon^2}{n}}} \\
		\qquad \qquad \qquad \qquad + \paren{1 - \paren{3 \cdot 10^6 + 1} \frac{\epsilon^2}{n}}  \cdot \ln{\frac{1}{1 - \paren{3 \cdot 10^6 + 1} \frac{\epsilon^2}{n}}}
		\end{multlined}\\
		&< \paren{3 \cdot 10^6 + 1} \frac{\epsilon^2}{n} \cdot \ln{\frac{1}{\frac{\epsilon^2}{n}}} + 2 \cdot \paren{3 \cdot 10^6 + 1} \frac{\epsilon^2}{n} \\
		&< 4 \cdot 10^6 \frac{\epsilon^2}{n} \cdot \ln{\frac{1}{\frac{\epsilon^2}{n}}}.
\end{align*}
\endgroup
Here, the first inequality follows from the fact that the mutual information between a pair of variables is at most the entropy of either of them. The justifications for the rest of the inequalities are similar to the justifications for those in the bounding of $\mathrm{I}^{\widehat{\mathrm{P}}}(X_i;X_j)$ above.

Note that our assumption leads to $\mathrm{I}^{\widehat{\mathrm{P}}}(X_i;X_j)  > \mathrm{I}^{\widehat{\mathrm{P}}}(X_{i'_s};X_{i'_{s+1}})$, a contradiction. Consequently, we must have $\minmrg{\mathrm{P}_{i'_s}} \geq 3 \cdot 10^6 \frac{\epsilon^2}{n}$, for all $s = 0, \ldots, r$.

Suppose the path in $T$ between $i$ and $j$ consists of nodes $i = i_0, \ldots, i_q = j$, in that order. Since $\minmrg{\mathrm{P}_{ij}} \geq (5 \cdot 10^6 + 1) \frac{\epsilon^2}{n}$ and $\mindiag{\mathrm{P}_{ij}} \leq (6 \cdot 10^4 - 1) \frac{\epsilon^2}{n}$, applications of Lemma~\ref{mindiagmarkov} (i), (ii), and Lemma~\ref{mindiagmarkov4} yield, say, $\minmrg{\mathrm{P}_{i_u}} \geq 2 \cdot 10^6 \frac{\epsilon^2}{n}$ for all $u = 0, \ldots, q$ and $\mindiag{\mathrm{P}_{i_u i_{u+1}}} \leq 6 \cdot 10^4 \frac{\epsilon^2}{n}$ for all $u = 0, \ldots, q-1$. Recall that the path in $\widehat{T}$ between $i$ and $j$ is $i = i'_0, \ldots, i'_r = j$, and that we already proved $\minmrg{\mathrm{P}_{i'_s}} \geq 3 \cdot 10^6 \frac{\epsilon^2}{n}$, for all $s = 0, \ldots, r$. Our goal is to prove that $i$ and $j$ belong to the same broken-city. It suffices to prove the following claim, which we deal with exclusively for the rest of the proof of this lemma. For convenience we will keep using symbols like $i,j,q,r,s,u$ in the statement and proof of the claim below, where they might refer to different things from what they have been referring to so far in the proof. For example, we are not assuming $\minmrg{\mathrm{P}_{ij}} \geq (5 \cdot 10^6 + 1) \frac{\epsilon^2}{n}$ any more. Basically think of what follows as the statement and proof of a fresh lemma.

\underline{\textit{Claim}}\textit{:} For any $i,j \in [n]$, let the path in $T$ between them be $i = i_0, \ldots, i_q = j$, in that order, and the path in $\widehat{T}$ between them be $i = i'_0, \ldots, i'_r = j$, in that order (the two paths might share any number of intermediate nodes). If
\begin{itemize}
\item $\minmrg{\mathrm{P}_{i_u}} \geq 2 \cdot 10^6 \frac{\epsilon^2}{n}$ for all $u = 0, \ldots, q$
\item $\mindiag{\mathrm{P}_{i_u i_{u+1}}} \leq 6 \cdot 10^4 \frac{\epsilon^2}{n}$ for all $u = 0, \ldots, q-1$
\item $\minmrg{\mathrm{P}_{i'_s}} \geq 3 \cdot 10^6 \frac{\epsilon^2}{n}$, for all $s = 0, \ldots, r$
\end{itemize}
then $i$ and $j$ belong to the same broken-city.

For an illustration of the setting described in our claim, which also contains the notation introduced for its proof below, see Figure~\ref{samebrokencityfigure}.

\begin{figure}[h]
\centering
\includegraphics[width=16.5cm]{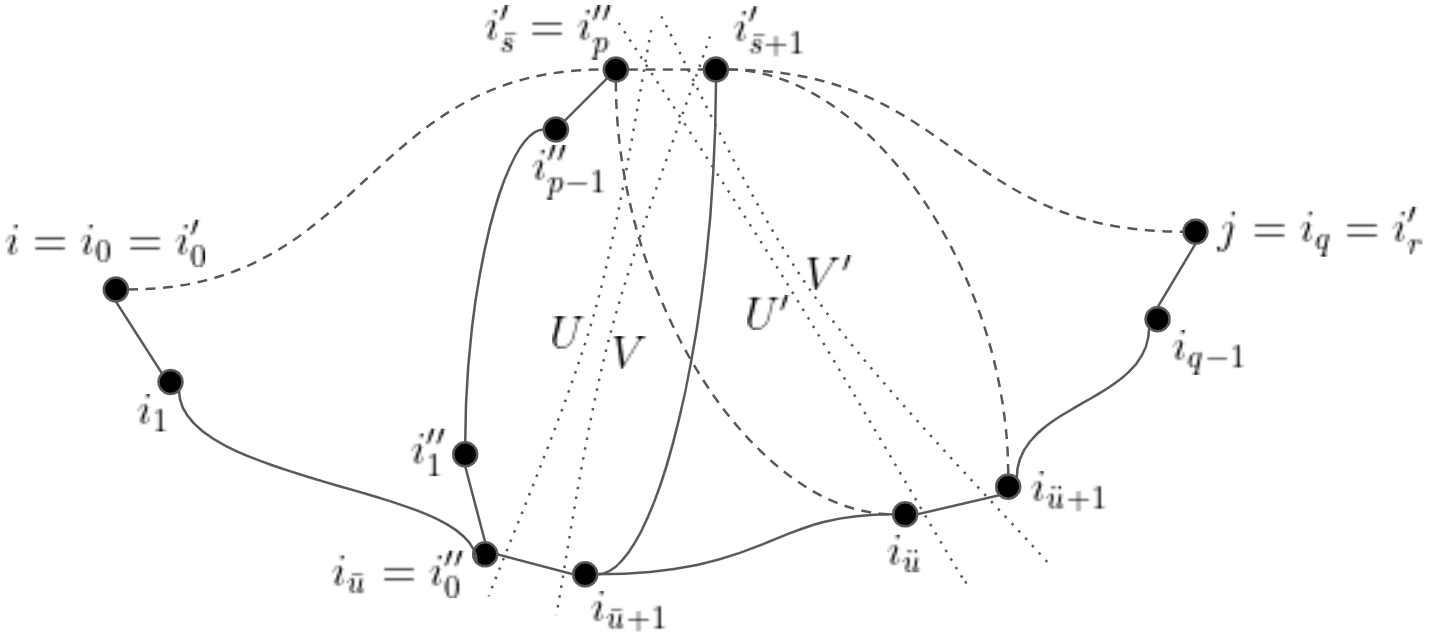}
\caption{An illustration of the setting described in our claim, which also contains the notation introduced for its proof: The solid connections represent edges of $T$, while the broken connections represent edges of $\widehat{T}$. The solid curves represent paths in $T$, while the broken curves represent paths in $\widehat{T}$. We also show the $T$-connected subsets $U$ and $V$ obtained from removing the edge $(i_{\bar{u}}, i_{\bar{u}+1})$ of $T$, and the $\widehat{T}$-connected subsets $U'$ and $V'$ obtained from removing the edge $(i_{\bar{s}}, i_{\bar{s}+1})$ of $\widehat{T}$. This illustration is intended only for conveying the basic paradigm. In reality, some of the nodes shown as distinct in the figure might coincide. For example, the path in $T$ between $i$ and $j$ (the upper arch) and the path in $\widehat{T}$ between $i$ and $j$ (the lower arch) might share some nodes. Also, while the path in $T$ between $i_{\bar{u}}$ and $i_{\bar{s}}'$ must be contained $U$, nodes such as $i$ or $i_1$ might not. Similar cautions apply to $V$, $U'$ and $V'$.}
\label{samebrokencityfigure}
\end{figure}

We prove the claim by induction on $r$. We will jump into the induction step right away; the base case $r=1$ will be dealt with naturally in the midst of that.

Let $\bar{u}$ be so that $\mathrm{I}^{\widehat{\mathrm{P}}}(X_{i_{\bar{u}}};X_{i_{\bar{u}+1}})$ is the smallest among $\mathrm{I}^{\widehat{\mathrm{P}}}(X_{i_u};X_{i_{u+1}})$ for $u = 0, \ldots, q-1$. Imagine removing $(i_{\bar{u}},i_{\bar{u}+1})$ from $T$. This partitions $[n]$ into two $T$-connected subsets $U$ and $V$, with $i_{\bar{u}} \in U$ and $i_{\bar{u}+1} \in V$. Clearly we have $i=i_0 \in U$ and $j=j_0 \in V$, so at least one edge on the path in $\widehat{T}$ between $i=i'_0$ and $j=j'_0$ goes between $U$ and $V$. Say $(i'_{\bar{s}}, i'_{\bar{s}+1})$ does. Suppose that $i'_{\bar{s}} \in U$ and $i'_{\bar{s}+1} \in V$ (the argument for $i'_{\bar{s}} \in V$ and $i'_{\bar{s}+1} \in U$ is similar). Then $i'_{\bar{s}}, i_{\bar{u}}, i_{\bar{u}+1}, i'_{\bar{s}+1}$ lie on a path in $T$.

Now imagine removing $(i'_{\bar{s}}, i'_{\bar{s}+1})$ from $\widehat{T}$. This partitions $[n]$ into two $\widehat{T}$-connected subsets $U'$ and $V'$, with $i_{\bar{s}} \in U'$ and $i_{\bar{s}+1} \in V'$. Clearly we have $i=i'_0 \in U'$ and $j=j'_0 \in V'$, so at least one edge on the path in $T$ between $i=i_0$ and $j=j_0$ goes between $U'$ and $V'$. Say $(i_{\ddot{u}}, i_{\ddot{u}+1})$ does. Thus, the path in $\widehat{T}$ between $i_{\ddot{u}}$ and $i_{\ddot{u}+1}$ must include $(i'_{\bar{s}}, i'_{\bar{s}+1})$. By the cycle property of maximum weight spanning tree we have $\mathrm{I}^{\widehat{\mathrm{P}}}(X_{i_{\ddot{u}}}, X_{i_{\ddot{u}+1}}) \leq \mathrm{I}^{\widehat{\mathrm{P}}}(X_{i'_{\bar{s}}}, X_{i'_{\bar{s}+1}})$.

Note that the way we initially picked $\bar{u}$ implies $\mathrm{I}^{\widehat{\mathrm{P}}}(X_{i_{\bar{u}}};X_{i_{\bar{u}+1}}) \leq \mathrm{I}^{\widehat{\mathrm{P}}}(X_{i_{\ddot{u}}}, X_{i_{\ddot{u}+1}})$, so we have $\mathrm{I}^{\widehat{\mathrm{P}}}(X_{i_{\bar{u}}};X_{i_{\bar{u}+1}}) \leq \mathrm{I}^{\widehat{\mathrm{P}}}(X_{i'_{\bar{s}}}, X_{i'_{\bar{s}+1}})$. This, combined with the fact that $i'_{\bar{s}}, i_{\bar{u}}, i_{\bar{u}+1}, i'_{\bar{s}+1}$ lie on a path in $T$, allows us to apply Lemma~\ref{genhelwrongedge} (ii) to deduce 
$H^2(\mathrm{P}_{i'_{\bar{s}} i_{\bar{u}} i_{\bar{u}+1}}, \mathrm{P}_{{i'_{\bar{s}}}\wideparen{ \, -i_{{\text{$\overline{u}$}}}~~~}i_{\bar{u}+1}}) \leq 62 \frac{\epsilon^2}{n}$ and 
$H^2(\mathrm{P}_{i_{\bar{u}} i_{\bar{u}+1} i'_{\bar{s}+1}}, \mathrm{P}_{i_{\bar{u}}\wideparen{ \,~~ i_{\text{$\overline{u}$}+1} -}i'_{\bar{s}+1}}) \leq 62 \frac{\epsilon^2}{n}.$

Combining with Lemma~\ref{helwrongedgecity}, we have
\begin{align}
\label{claimineq}
&\begin{multlined}
\frac{1}{100} {\mindisc{\mathrm{P}_{i_{\bar{u}} i_{\bar{u}+1}}}}^2 \cdot \min{\brac{\minmrg{\mathrm{P}_{i_{\bar{u}+1}}}, \mindiag{\mathrm{P}_{i_{\bar{u}+1} i'_{\bar{s}+1}}}}} \\
\qquad \qquad \qquad \qquad \qquad \qquad \qquad \leq H^2(\mathrm{P}_{i_{\bar{u}} i_{\bar{u}+1} i'_{\bar{s}+1}}, \mathrm{P}_{i_{\bar{u}}\wideparen{ \,~~ {{i_{\text{$\overline{u}$+1}}}} -}i'_{\bar{s}+1}}) \leq 62 \frac{\epsilon^2}{n}.
\end{multlined}
\end{align}
Since $\minmrg{\mathrm{P}_{i_{\bar{u}} i_{\bar{u}+1}}} \geq 2 \cdot 10^6 \frac{\epsilon^2}{n}$ and $\mindiag{\mathrm{P}_{i_{\bar{u}} i_{\bar{u}+1}}} \leq 6 \cdot 10^4 \frac{\epsilon^2}{n}$, Lemma~\ref{minmrgmindiagmindisc} implies that $\mindisc{\mathrm{P}_{i_{\bar{u}} i_{\bar{u}+1}}} \geq 0.94$. We also have $\minmrg{\mathrm{P}_{i_{\bar{u}+1}}} \geq 2 \cdot 10^6 \frac{\epsilon^2}{n}$. Thus, (\ref{claimineq}) implies $\mindiag{\mathrm{P}_{i_{\bar{u}+1} i'_{\bar{s}+1}}} \leq 10^4 \frac{\epsilon^2}{n}$. Similarly, $\mindiag{\mathrm{P}_{i'_{\bar{s}} i_{\bar{u}}}} \leq 10^4 \frac{\epsilon^2}{n}$. We can then apply Lemma~\ref{mindiagconcat} (twice) to get
$$\mindiag{\mathrm{P}_{i'_{\bar{s}} i'_{\bar{s}+1}}} \leq \mindiag{\mathrm{P}_{i'_{\bar{s}} i_{\bar{u}}}} + \mindiag{\mathrm{P}_{i_{\bar{u}} i_{\bar{u}+1}}} + \mindiag{\mathrm{P}_{i_{\bar{u}+1} i'_{\bar{s}+1}}} \leq 8 \cdot 10^4 \frac{\epsilon^2}{n}.$$
By Lemma~\ref{genmindiagprecision} (i) we certainly have $\mindiag{\widehat{\mathrm{P}}_{i'_{\bar{s}} i'_{\bar{s}+1}}} \leq 10^5 \frac{\epsilon^2}{n}$. Note $\minmrg{\mathrm{P}_{i'_{\bar{s}} i'_{\bar{s}+1}}} \geq 3 \cdot 10^6 \frac{\epsilon^2}{n}$ certainly implies $\minmrg{\widehat{\mathrm{P}}_{i'_{\bar{s}} i'_{\bar{s}+1}}} \geq 10^6 \frac{\epsilon^2}{n}$ (the contrapositive of this follows from Lemma~\ref{genminmrgprecision} (ii)). We see that $(i'_{\bar{s}}, i'_{\bar{s}+1})$ satisfies the criterion for a $\widehat{T}$-avenue and thus must be classified as one.

If $r=1$ then $(i'_{\bar{s}}, i'_{\bar{s}+1})$ must simply be $(i,j)$. So $i$ and $j$ belong to the same broken-city and we are done (this is the base case).

If $r>1$ then at least one of $\bar{s}$ and $\bar{s}+1$ belongs to $\{1, \ldots, r-1\}$ (so $i'_{\bar{s}}$ is a strict intermediate node on the path in $\widehat{T}$ between $i$ and $j$). Say $\bar{s}$ does (the other case is argued in the same way).

Let the path in $T$ between $i_{\bar{u}}$ and $i'_{\bar{s}}$ be $i_{\bar{u}} = i''_0, \ldots, i''_p = i'_{\bar{s}}$. Since $\minmrg{\mathrm{P}_{i'_{\bar{s}}}} \geq 3 \cdot 10^6 \frac{\epsilon^2}{n}$, and we already proved $\mindiag{\mathrm{P}_{i_{\bar{u}} i'_{\bar{s}}}} \leq 10^4 \frac{\epsilon^2}{n}$, Lemma~\ref{minmrgmindiag} implies $\minmrg{\mathrm{P}_{i_{\bar{u}}}} \geq (3 \cdot 10^6 - 10^4) \frac{\epsilon^2}{n}$, and so $\minmrg{\mathrm{P}_{i_{\bar{u}} i'_{\bar{s}}}} \geq (3 \cdot 10^6 - 10^4) \frac{\epsilon^2}{n}$. It is easy to see that applications of Lemma~\ref{mindiagmarkov} (i), (ii), and Lemma~\ref{mindiagmarkov4} yield $\minmrg{\mathrm{P}_{i''_v}} \geq (3 \cdot 10^6 - 2 \cdot 10^4) \frac{\epsilon^2}{n}$ for all $v = 0, \ldots, p$, and $\mindiag{\mathrm{P}_{i''_v i''_{v+1}}} \leq 10^4 \frac{\epsilon^2}{n}$ for all $v = 0, \ldots, p-1$.

Therefore, all the nodes and edges on the path in $T$ between $i_{\bar{u}}$ and $i'_{\bar{s}}$ satisfy the first two bullet points of the claim we are proving. Note that all the nodes and edges on the path in $T$ between $i$ and $i_{\bar{u}}$ also satisfy those two bullet points by assumption. Since the path in $T$ between $i$ and $i'_{\bar{s}}$ is contained in the union of the path in $T$ between $i$ and $i_{\bar{u}}$ and the path in $T$ between $i_{\bar{u}}$ and $i'_{\bar{s}}$, and the path in $\widehat{T}$ between $i$ and $i'_{\bar{s}}$ is part of the path in $\widehat{T}$ between $i$ and $j$ (so all its nodes satisfy the third bullet point of the claim), we see that all the conditions listed in our claim are satisfied with $i$ and $i'_{\bar{s}}$ in place of $i$ and $j$. Furthermore, the length $\bar{s}$ of the path in $\widehat{T}$ between $i$ and $i'_{\bar{s}}$ is strictly smaller than the length $r$ of the path in $\widehat{T}$ between $i$ and $j$. The induction setup (assuming the claim is true for all lengths less than $r$) thus allows us to conclude that $i$ and $i'_{\bar{s}}$ belong to the same broken-city

Similarly, all the conditions listed in our claim are satisfied with $i'_{\bar{s}}$ and $j$ in place of $i$ and $j$, and the length $r-\bar{s}$ of the path in $\widehat{T}$ between $i'_{\bar{s}}$ and $j$ is strictly smaller than the length $r$ of the path in $\widehat{T}$ between $i$ and $j$. The induction setup thus allows us to conclude that $i'_{\bar{s}}$ and $j$ belong to the same broken-city.

Combining the two yields that $i$ and $j$ belong to the same broken-city. The induction step is established, and we are done.
\end{proof}

Recall that a broken-city might not be $T$-connected due to the presence of ``holes''. We now show that every such ``hole'' must have a small $\mathrm{minmrg}$. That is, it is either almost always equal to $1$, or almost always equal to $-1$.

\begin{definition}
\label{biasednode}
A node $i \in [n]$ with $\minmrg{\mathrm{P}_i} < 10^7 \frac{\epsilon^2}{n}$ is called \textbf{biased}.
\end{definition}

\begin{lemma}
\label{onlybiased}
For a broken-city $\widetilde{C}$, every node in $\conv{T}{\widetilde{C}} \setminus \widetilde{C}$ is biased.
\end{lemma}
\begin{proof}
Let $j$ be any node in $\conv{T}{\widetilde{C}} \setminus \widetilde{C}$, we want to show that $j$ is biased.

First we show that there exist $i,k \in \widetilde{C}$ such that $(i,k)$ is a $\widehat{T}$-avenue and that $j$ lies on the path in $T$ between $i$ and $k$. Suppose not for the sake of contradiction. Imagine removing $j$ and its incidental edges from $T$. This separates the rest of the nodes into a number (equal to the degree of $j$ in $T$) of $T$-connected subsets. The two end-nodes of every $\widehat{T}$-avenue in $\widetilde{C}$ must belong to the same such $T$-connected subset (call it $U$), else the path in $T$ between those two end-nodes would pass through $j$. Since $\widetilde{C}$ is spanned by the $\widehat{T}$-avenues in $\widetilde{C}$, the entire $\widetilde{C}$ must in fact be contained in $U$, and so $\conv{T}{\widetilde{C}}$ must, too. This contradicts the fact that $j \in \conv{T}{\widetilde{C}}$.

So we can indeed find $i,k \in \widetilde{C}$ such that $(i,k)$ is a $\widehat{T}$-avenue and that $i,j,k$ lie on a path in $T$. By definition $\minmrg{\widehat{\mathrm{P}}_{ik}} \geq 10^6 \frac{\epsilon^2}{n}$ and $\mindiag{\widehat{\mathrm{P}}_{ik}} \leq 10^5 \frac{\epsilon^2}{n}$. Note that $\minmrg{\widehat{\mathrm{P}}_{ik}} \geq 10^6 \frac{\epsilon^2}{n}$ implies $\minmrg{\mathrm{P}_{ik}} \geq (10^6 - 1) \frac{\epsilon^2}{n}$ (the contrapositive of this follows from Lemma~\ref{genminmrgprecision} (i)). Also $\mindiag{\mathrm{P}_{ik}} \leq (10^5 + 1) \frac{\epsilon^2}{n}$ by Lemma~\ref{genmindiagprecision} (ii). Thus, we can apply Lemma~\ref{mindiagmarkov} (iii) to get
$$\mindiag{\mathrm{P}_{ij}} + \mindiag{\mathrm{P}_{jk}} \leq \frac{7}{6} \mindiag{\mathrm{P}_{ik}}.$$
So at least one of $\mindiag{\mathrm{P}_{ij}}$ and $\mindiag{\mathrm{P}_{jk}}$ is at most $(6 \cdot 10^4 - 1) \frac{\epsilon^2}{n}$. Without loss of generality suppose $\mindiag{\mathrm{P}_{ij}} \leq (6 \cdot 10^4 - 1) \frac{\epsilon^2}{n}$.

Suppose for the sake of contradiction that $j$ is not biased, so $\minmrg{\mathrm{P}_j} \geq 10^7 \frac{\epsilon^2}{n}$. Lemma~\ref{minmrgmindiag} thus implies $\minmrg{\mathrm{P}_i} \geq (10^7 - 6 \cdot 10^4 + 1) \frac{\epsilon^2}{n}$, and so $\minmrg{\mathrm{P}_{ij}} \geq (10^7 - 6 \cdot 10^4 + 1) \frac{\epsilon^2}{n}$. We can therefore apply Lemma~\ref{samebrokencity} to conclude that $i$ and $j$ belong to the same broken-city. So $j \in \widetilde{C}$, a contradiction. Consequently, $j$ must be biased.
\end{proof}

Intuitively, the following lemma says that the nodes responsible for ``gluing together'' different broken-cities in the same city must be biased in the sense of Definition~\ref{biasednode}.

\begin{lemma}
\label{differentbrokencitysamecity}
The path in $T$ between two nodes from different broken-cities in the same city must contain at least one biased node (possibly at one of its two ends).
\end{lemma}
\begin{proof}
Suppose for the sake of contradiction that there is a path in $T$ containing no biased node between some two nodes from different broken-cities in the same city. Since a city is $T$-connected by Lemma~\ref{gencityconnected}, the entire path must be contained in that same city. Since the two end-nodes of that path belong to different broken-cities, there must exist an edge $(i,j)$ on the path such that $i$ and $j$ belong to different broken-cities (in the same city).

Imagine removing $(i,j)$ from $T$. This partitions $[n]$ into two $T$-connected subsets $U$ and $V$. Note that any broken-city must be contained either entirely in $U$ or entirely in $V$. Otherwise, its convex hull in $T$ would contain both $i$ and $j$, and so both $i$ and $j$ must in fact be contained in that broken-city by Lemma~\ref{onlybiased} because none of them is biased, a contradiction. Consequently, since $U$ and $V$ are $T$-connected, the convex hull in $T$ of each broken-city must be contained either entirely in $U$ or entirely in $V$. It is then easy to see that $i$ and $j$ can't belong to the same city because the sequence of broken-cities bridging between $i$'s and $j$'s broken-cities as required according to Definition~\ref{truehierarchy} can't possibly exist. This is a contradiction.
\end{proof}

The following lemma shows that each $T$-trail is incidental to at least one biased node and so is close to being independent in $H^2$. As a result, they can be cut away without incurring too much error in $H^2$.

\begin{lemma}
\label{trailih}
If $(i,j)$ is a $T$-trail, then at least one of $i,j$ is biased, and $I_{H^2}(\mathrm{P}_{ij}) < 2 \cdot 10^7 \frac{\epsilon^2}{n}$.
\end{lemma}
\begin{proof}
Suppose for the sake of contradiction that none of $i,j$ is biased. The fact that $(i,j)$ is a $T$-trail means that $i,j$ belong to the same city. Since $T$-trails are selected among $T$-roads (see Definition~\ref{ttrail}), we have that $(i,j)$ is a $T$-road and so Lemma~\ref{differentbrokencitysamecity} implies that $i,j$ must belong to the same broken-city. Since $(i,j)$ is not a $\widehat{T}$-avenue (see Definition~\ref{ttrail}), it creates a cycle with the $\widehat{T}$-avenues spanning that broken-city. This contradicts the description in Definition~\ref{ttrail}.

So at least one of $i,j$ is biased, which implies $\minmrg{\mathrm{P}_{ij}} < 10^7 \frac{\epsilon^2}{n}$. Thus, $I_{H^2}(\mathrm{P}_{ij}) < 2 \cdot 10^7 \frac{\epsilon^2}{n}$ by Lemma~\ref{ihminmrg}.
\end{proof}

The next lemma contains important structural results regarding the highways. Definition~\ref{preliminaryhierarchy} obviously implies that a $\widehat{T}$-highway goes between different broken-cities, and (i) below shows that it in fact must go between different cities. Also, (ii) and (iii) show that the $T$-highways and the $\widehat{T}$-highways can be matched into parallel pairs, where each parallel pair consists of a $T$-highway and a $\widehat{T}$-highway going between the same pair of broken-cities (residing in different cities in the same country). This is an analog of Lemma~\ref{highwayparallel} for the symmetric case.

\begin{lemma}
\label{genhighwayparallel}
\begin{itemize}
\item[(i)] A $\widehat{T}$-highway goes between different cities.
\item[(ii)] There can be at most one $\widehat{T}$-highway and at most one $T$-highway between any pair of cities.
\item[(iii)] There is a $\widehat{T}$-highway between a pair of cities if and only if there is a $T$-highway between the same pair of cities, and if so, those two highways in fact go between the same pair of broken-cities.
\end{itemize}
\end{lemma}
\begin{proof}
Suppose $(i,k)$ is a $\widehat{T}$-highway. Then $\minmrg{\widehat{\mathrm{P}}_{ik}} \geq 10^8 \frac{\epsilon^2}{n}$ and $\mindisc{\widehat{\mathrm{P}}_{ik}} \geq \frac{1}{2}$. Note that $\minmrg{\widehat{\mathrm{P}}_{ik}} \geq 10^8 \frac{\epsilon^2}{n}$ implies $\minmrg{\mathrm{P}_{ik}} \geq (10^8 - 1) \frac{\epsilon^2}{n}$ (the contrapositive of this follows from Lemma~\ref{genminmrgprecision} (i)). By Lemma~\ref{genmindiscprecision} we also have $\mindisc{\mathrm{P}_{ik}} \geq \frac{1}{2} - \frac{1}{10^{20}}$. By Lemma~\ref{minmrgmindisc3}, any $j$ on the path in $T$ between $i$ and $k$ satisfies
$$\minmrg{\mathrm{P}_j} \geq \minmrg{\mathrm{P}_{ik}}\mindisc{\mathrm{P}_{ik}} > 10^7 \frac{\epsilon^2}{n}.$$
Therefore, none of the nodes on the path in $T$ between $i$ and $k$ is biased.

Note that $(i,k)$, being a $\widehat{T}$-highway, obviously goes between different broken-cities. It must in fact go between different cities, as otherwise the path in $T$ between $i$ and $k$ would have to contain at least one biased node by Lemma~\ref{differentbrokencitysamecity}, a contradiction. This proves (i).

We move on to prove the only if direction of (iii). Keep the notation so far. Let $\widetilde{C},\widetilde{D}$ be broken-cities containing $i,k$, respectively. 

We want to show that there exists a $T$-highway between $\widetilde{C}$ and $\widetilde{D}$. Suppose for the sake of contradiction that there is no $T$-highway between $\widetilde{C}$ and $\widetilde{D}$. Note that by (i) $\widetilde{C}$ and $\widetilde{D}$ are contained in different cities. They are in the same country because they are linked together directly by the $\widehat{T}$-highway $(i,k)$. Consequently, there must exist no edge of $T$ at all between $\widetilde{C}$ and $\widetilde{D}$, because any such edge would have been classified as a $T$-highway according to Definition~\ref{truehierarchy}.

As a result, there must exist some $j$ on the path in $T$ between $i$ and $k$ such that $j$ belongs to neither $\widetilde{C}$ nor $\widetilde{D}$. Since $(i,k)$ is an edge of $\widehat{T}$, either $k$ sits on the path in $\widehat{T}$ between $i$ and $j$, or $i$ sits on the path in $\widehat{T}$ between $k$ and $j$. Without loss of generality we assume it is the former case. Then, by the cycle property of maximum weight spanning trees, we have $\mathrm{I}^{\widehat{\mathrm{P}}}(X_i;X_j) \leq \mathrm{I}^{\widehat{\mathrm{P}}}(X_i;X_k)$. Since $i,j,k$ lie on a path in $T$, Lemma~\ref{genhelwrongedge} (i) implies that $H^2(\mathrm{P}_{ijk}, \mathrm{P}_{i\wideparen{ \,~~ j-}k}) \leq 22 \frac{\epsilon^2}{n}$.

On the other hand, we know that none of $i,j,k$ is biased, so $\minmrg{\mathrm{P}_{jk}} \geq 10^7 \frac{\epsilon^2}{n}$. Since $j,k$ belong to different broken-cities, Lemma~\ref{samebrokencity} implies that $\mindiag{\mathrm{P}_{jk}} > (6 \cdot 10^4 - 1) \frac{\epsilon^2}{n}$. By Lemma~\ref{mindiscmarkov} we have $\mindisc{\mathrm{P}_{ij}} \geq \mindisc{\mathrm{P}_{ik}} \geq \frac{1}{2} - \frac{1}{10^{20}}$. Applying Lemma~\ref{helwrongedgecity}, we have
\begin{align*}
H^2(\mathrm{P}_{ijk}, \mathrm{P}_{i\wideparen{ \,~~ j-}k}) &\geq \frac{1}{100} {\mindisc{\mathrm{P}_{ij}}}^2 \cdot \min{\brac{\minmrg{\mathrm{P}_{j}}, \mindiag{\mathrm{P}_{jk}}}} \\
		&\geq \frac{1}{100} \paren{\frac{1}{2} - \frac{1}{10^{20}}}^2 \cdot \min{\brac{10^7 \frac{\epsilon^2}{n}, (6 \cdot 10^4 - 1) \frac{\epsilon^2}{n}}} \\
		&> 22 \frac{\epsilon^2}{n},
\end{align*}
a contradiction.

We just proved that the existence of a $\widehat{T}$-highway between a pair of broken-cities (which necessarily lie in different cities by (i)) implies the existence of a $T$-highway between the same pair of broken-cities. This is the only if direction of (iii). Before tackling the if direction, let's prove (ii) first.

That there can be at most one $T$-highway between any pair of cities is clear because every city is $T$-connected by Lemma~\ref{gencityconnected}.

Suppose there are two $\widehat{T}$-highways between some pair of cities. Because every broken-city is obviously $\widehat{T}$-connected according to Definition~\ref{preliminaryhierarchy}, there can't be more than one edge of $\widehat{T}$ between the same pair of broken-cities. Therefore, those two $\widehat{T}$-highways must go between two different pairs of broken-cities. The only if direction of (iii) we already proved implies that for each of those two pairs of broken-cities there must be a $T$-highway going between them. But then there are two $T$-highways between the same pair of cities, a contradiction. This proves (ii).

We now prove the if direction of (iii). We reset the meaning of all symbols defined so far in this proof (e.g. $i,j,k$). Suppose $(i,j)$ is a $T$-highway between broken-cities $\widetilde{C}$ and $\widetilde{D}$, which are contained in cities $C$ and $D$, respectively.

We first show that there must exist a $\widehat{T}$-highway between $C$ and $D$. Suppose not. According to Definition~\ref{truehierarchy} the fact that $(i,j)$ is a $T$-highway means $C$ and $D$ are different cities in the same country. Since there is no $\widehat{T}$-highway between $C$ and $D$, according to Definition~\ref{truehierarchy} there must exist a sequence of distinct cities $C = C_0, C_1, \ldots, C_r = D$, $r \geq 2$, such that there is a $\widehat{T}$-highway between $C_s$ and $C_{s+1}$ for each $s$. By the only if direction of (iii) we already proved, there is a $T$-highway between $C_s$ and $C_{s+1}$ for each $s$. Since each $C_s$ is $T$-connected by Lemma~\ref{gencityconnected}, it is easy to see that there is a path in $T$ between $i$ and $j$ that visits the cities $C = C_0, C_1, \ldots, C_r = D$, in that order. On the other hand, $(i,j)$ by itself constitutes another path in $T$ between $i$ and $j$. So we have more than one paths in $T$ between $i$ and $j$, contradicting the fact that $T$ is a tree.

So there exists a $\widehat{T}$-highway between $C$ and $D$. We show that it must in fact go between $\widetilde{C}$ and $\widetilde{D}$. Suppose for the sake of contradiction that it goes between a different pair of broken-cities. By the only if direction of (iii) we already proved, there is a $T$-highway between that other pair of broken-cities. This means that there is an edge of $T$ different from $(i,j)$ that also goes between $C$ and $D$, which can't happen because $C$ and $D$ are $T$-connected by Lemma~\ref{gencityconnected}. This contradiction shows that the existing $\widehat{T}$-highway between $C$ and $D$ must in fact go between $\widetilde{C}$ and $\widetilde{D}$. This finishes the proof of (iii).
\end{proof}

Now we can easily prove that every country is $T$-connected.

\begin{lemma}
\label{gencountryconnected}
Every country is $T$-connected.
\end{lemma}
\begin{proof}
Recall that a country consists of a collection of cities linked together by $\widehat{T}$-highways. Its $T$-connectedness then follows from Lemma~\ref{gencityconnected} and Lemma~\ref{genhighwayparallel} (iii).
\end{proof}

The following lemma is the country-level version of Lemma~\ref{differentbrokencitysamecity}.

\begin{lemma}
\label{differentbrokencountrysamecountry}
The path in $T$ between two nodes from different broken-countries in the same country must contain at least one biased node (possibly at one of its two ends).
\end{lemma}
\begin{proof}
Suppose for the sake of contradiction that there is a path in $T$ containing no biased node between some two nodes from different broken-countries in the same country. Since a country is $T$-connected by Lemma~\ref{gencountryconnected}, the entire path must be contained in that same country. Since the two end-nodes of that path belong to different broken-countries, there must exist an edge $(i,j)$ on that path such that $i$ and $j$ belong to different broken-countries (in the same country).

We divide the rest of argument into two cases, depending on whether $i$ and $j$ belong to the same city.

First consider the case $i$ and $j$ belong to the same city. Note that the fact that $i$ and $j$ belong to different broken-countries obviously implies that they belong to different broken-cities. As a result, at least one of $i$ and $j$ must be a biased node by Lemma~\ref{differentbrokencitysamecity} (remember that $(i,j)$ is an edge of $T$). This is a contradiction.

Next consider the case $i$ and $j$ belong to different cities. Since $i$ and $j$ belong to the same country, $(i,j)$ is a $T$-highway according to Definition~\ref{truehierarchy}. By Lemma~\ref{genhighwayparallel}, there exists a $\widehat{T}$-highway between the broken-city containing $i$ and the broken-city containing $j$. But then those two broken-cities (and therefore $i$ and $j$) must in fact be contained in the same broken-country according to Definition~\ref{preliminaryhierarchy}, contradicting our assumption.
\end{proof}

The following lemma is the country-level version of Lemma~\ref{samebrokencity}.

\begin{lemma}
\label{samebrokencountry}
If $\minmrg{\mathrm{P}_{ij}} \geq 10^9 \frac{\epsilon^2}{n}$ and $\mindisc{\mathrm{P}_{ij}} \geq \frac{2}{3}$, then $i,j$ belong to the same broken-country.
\end{lemma}
\begin{proof}
Suppose for the sake of contradiction that $i,j$ belong to different broken-countries. If $i,j$ were to belong to the same country, then by Lemma~\ref{differentbrokencountrysamecountry} the path in $T$ between them must contain at least one biased node. However, for any $i'$ on the path in $T$ between $i$ and $j$ we have by Lemma~\ref{minmrgmindisc3} that
$$\minmrg{\mathrm{P}_{i'}} \geq \minmrg{\mathrm{P}_{ij}}\mindisc{\mathrm{P}_{ij}} > 10^7 \frac{\epsilon^2}{n},$$
and so $i'$ is not biased. This is a contradiction.

Therefore, $i,j$ must belong to different countries. So there must exist an edge $(i',j')$ on the path in $T$ between $i$ and $j$ such that $i',j'$ belong to different countries. Suppose the path in $\widehat{T}$ between $i'$ and $j'$ consists of nodes $i' = i'_0, \ldots, i'_r = j'$, in that order. By the cycle property of maximum weight spanning tree, we have $\mathrm{I}^{\widehat{\mathrm{P}}}(X_{i'};X_{j'}) \leq \mathrm{I}^{\widehat{\mathrm{P}}}(X_{i'_s};X_{i'_{s+1}})$ for all $s = 0, \ldots, r-1$.

\begin{figure}[h]
\centering
\includegraphics[width=16.5cm]{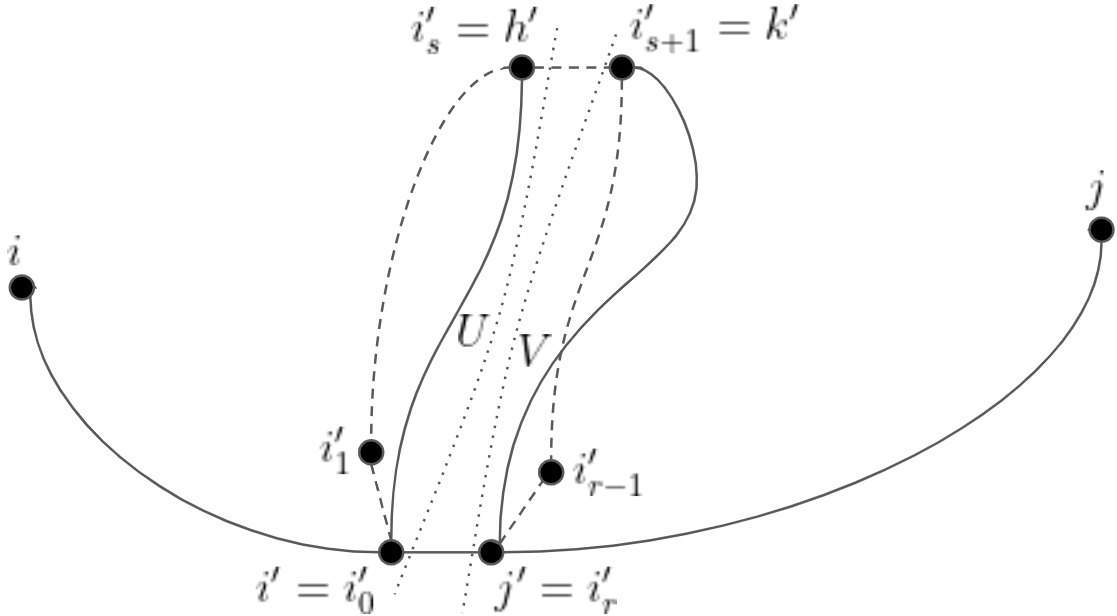}
\caption{An illustration of our setting: The solid connections represent edges of $T$, while the broken connections represent edges of $\widehat{T}$. The solid curves represent paths in $T$, while the broken curves represent paths in $\widehat{T}$. We also show the $T$-connected subsets $U$ and $V$ obtained from removing the edge $(i',j')$ of $T$. This illustration is intended only for conveying the basic paradigm. In reality, some of the nodes shown as distinct in the figure might coincide. For example, the path in $T$ between $i' $ and $j'$ and the path in $T$ between $i$ and $j$ might share some nodes. Also, while the path in $T$ between $i' = i_0'$ and $i_s' = h'$ must be contained $U$, nodes such as $i$ or $i_1'$ might not. A similar caution applies to $V$.}
\label{samebrokencountryfigure}
\end{figure}

Imagine removing $(i',j')$ from $T$. This partitions $[n]$ into two $T$-connected subsets $U$ and $V$, with $i' \in U$ and $j' \in V$. It is easy to see that there exists $s$ such that $i'_s \in U$ and $i'_{s+1} \in V$. Denote $h' = i'_s$ and $k' = i'_{s+1}$. See Figure~\ref{samebrokencountryfigure} for an illustration of our setting so far. Our goal is to show that $h',k'$ belong to the same country by showing that $(h',k')$ is a $\widehat{T}$-road or a $\widehat{T}$-highway. Note that $h', i', j', k'$ lie on a path in $T$, and $\mathrm{I}^{\widehat{\mathrm{P}}}(X_{i'};X_{j'}) \leq \mathrm{I}^{\widehat{\mathrm{P}}}(X_{h'};X_{k'})$. Lemma~\ref{genhelwrongedge} (ii) thus implies
\begin{align}
\label{helupperboundcity}
\begin{split}
H^2(\mathrm{P}_{h'i'j'}, \mathrm{P}_{h'\wideparen{-i' \,~~ }j'}) &\leq 62 \frac{\epsilon^2}{n} \\
H^2(\mathrm{P}_{i'j'k'}, \mathrm{P}_{i'\wideparen{ \,~~ j'-}k'}) &\leq 62 \frac{\epsilon^2}{n} 
\end{split}
\end{align}

Since $(i',j')$ is an edge on the path in $T$ between $i$ and $j$, by Lemma~\ref{minmrgmindisc3} we have
$$\minmrg{\mathrm{P}_{i'}} \geq \minmrg{\mathrm{P}_{ij}}\mindisc{\mathrm{P}_{ij}} > 5 \cdot 10^8 \frac{\epsilon^2}{n},$$
and similarly $\minmrg{\mathrm{P}_{j'}} > 5 \cdot 10^8 \frac{\epsilon^2}{n}.$ So $\minmrg{\mathrm{P}_{i'j'}} \geq 5 \cdot 10^8 \frac{\epsilon^2}{n}$.

Also, applying Lemma~\ref{mindiscmarkov} (possibly twice) we have $\mindisc{\mathrm{P}_{i'j'}} \geq \mindisc{\mathrm{P}_{ij}} \geq \frac{2}{3}$.

By Lemma~\ref{helwrongedgecity} we have
\begin{align}
\label{hellowerboundcity}
\begin{split}
H^2(\mathrm{P}_{h'i'j'}, \mathrm{P}_{h'\wideparen{-i' \,~~ }j'}) &\geq \frac{1}{100} {\mindisc{\mathrm{P}_{i'j'}}}^2 \cdot \min{\brac{\minmrg{\mathrm{P}_{i'}}, \mindiag{\mathrm{P}_{h'i'}}}} \\
H^2(\mathrm{P}_{i'j'k'}, \mathrm{P}_{i'\wideparen{ \,~~ j'-}k'}) &\geq \frac{1}{100} {\mindisc{\mathrm{P}_{i'j'}}}^2 \cdot \min{\brac{\minmrg{\mathrm{P}_{j'}}, \mindiag{\mathrm{P}_{j'k'}}}}
\end{split}
\end{align}
Combining (\ref{helupperboundcity}), (\ref{hellowerboundcity}), and the various bounds we established in between, we see that we have 
$\mindiag{\mathrm{P}_{h'i'}} < 10^5 \frac{\epsilon^2}{n}$ and $\mindiag{\mathrm{P}_{j'k'}} < 10^5 \frac{\epsilon^2}{n}$.

By Lemma~\ref{minmrgmindiag} we have $\minmrg{\mathrm{P}_{h'}} \geq \minmrg{\mathrm{P}_{i'}} - \mindiag{\mathrm{P}_{h'i'}} > 4 \cdot 10^8 \frac{\epsilon^2}{n}$, and $\minmrg{\mathrm{P}_{k'}} \geq \minmrg{\mathrm{P}_{j'}} - \mindiag{\mathrm{P}_{j'k'}} > 4 \cdot 10^8 \frac{\epsilon^2}{n}$. So $\minmrg{\mathrm{P}_{i'k'}} > 4 \cdot 10^8 \frac{\epsilon^2}{n}$ and $\minmrg{\mathrm{P}_{h'k'}} > 4 \cdot 10^8 \frac{\epsilon^2}{n}$.

By Lemma~\ref{minmrgmindiagmindisc3} we have $\mindisc{\mathrm{P}_{i'k'}} \geq \mindisc{\mathrm{P}_{i'j'}} - 4 \cdot \frac{\mindiag{\mathrm{P}_{j'k'}}}{\minmrg{\mathrm{P}_{i'j'}}} > \frac{2}{3} - \frac{1}{10^3}$.

By Lemma~\ref{minmrgmindiagmindisc3} again we have $\mindisc{\mathrm{P}_{h'k'}} \geq \mindisc{\mathrm{P}_{i'k'}} - 4 \cdot \frac{\mindiag{\mathrm{P}_{h'i'}}}{\minmrg{\mathrm{P}_{i'k'}}} > \frac{2}{3} - 2 \cdot \frac{1}{10^3}$.

Finally, $\minmrg{\mathrm{P}_{h'k'}} > 4 \cdot 10^8 \frac{\epsilon^2}{n}$ implies $\minmrg{\widehat{\mathrm{P}}_{h'k'}} > (4 \cdot 10^8 - 1) \frac{\epsilon^2}{n}$ (the contrapositive of this follows from Lemma~\ref{genminmrgprecision} (ii)), and $\mindisc{\widehat{\mathrm{P}}_{h'k'}} \geq \mindisc{\mathrm{P}_{h'k'}} - \frac{1}{10^{20}} > \frac{1}{2}$ by Lemma~\ref{genmindiscprecision}.

According to Definition~\ref{preliminaryhierarchy}, $(h',k')$ is classified as either a $\widehat{T}$-road or a $\widehat{T}$-highway. In either case, $h',k'$ belong to the same broken-country, and therefore to the same country. Since a country is $T$-connected by Lemma~\ref{gencountryconnected}, and $h',i',j',k'$ lie on a path in $T$, we see that $h',i',j',k'$ all belong to the same country. This contradicts the fact that $i',j'$ belong to different countries (see start of the second paragraph), and the proof is completed.
\end{proof}

The next four lemmas parallelize the four previous lemmas but are one layer lower in the hierarchy.

\begin{lemma}
\label{genrailwayparallel}
\begin{itemize}
\item[(i)] A $\widehat{T}$-railway goes between different countries.
\item[(ii)] There can be at most one $\widehat{T}$-railway and at most one $T$-railway between any pair of countries.
\item[(iii)] There is a $\widehat{T}$-railway between a pair of countries if and only if there is a $T$-railway between the same pair of countries, and if so, those two railways in fact go between the same pair of broken-countries.
\end{itemize}
\end{lemma}
\begin{proof}
Suppose $(i,k)$ is a $\widehat{T}$-railway. Then $I_{H^2}(\widehat{\mathrm{P}}_{ik}) \geq 10^{10} \frac{\epsilon^2}{n}$. By Lemma~\ref{genihprecision} we have $I_{H^2}(\mathrm{P}_{ik}) \geq \frac{1}{4} \cdot 10^{10} \frac{\epsilon^2}{n}$. By Lemma~\ref{minmrgih3}, any $j$ on the path in $T$ between $i$ and $k$ satisfies
$$\minmrg{\mathrm{P}_j} \geq \frac{1}{2} I_{H^2}(\mathrm{P}_{ik}) > 10^7 \frac{\epsilon^2}{n}.$$
Therefore, none of the nodes on the path in $T$ between $i$ and $k$ is biased.

Note that $(i,k)$, being a $\widehat{T}$-railway, obviously goes between different broken-countries. It must in fact go between different countries, as otherwise the path in $T$ between $i$ and $k$ would have to contain at least one biased node by Lemma~\ref{differentbrokencountrysamecountry}, a contradiction. This proves (i).

We move on to prove the only if direction of (iii). Keep the notation so far. Let $\widetilde{\mathcal{C}},\widetilde{\mathcal{D}}$ be broken-countries containing $i,k$, respectively. 

We want to show that there exists a $T$-railway between $\widetilde{\mathcal{C}}$ and $\widetilde{\mathcal{D}}$. Suppose for the sake of contradiction that there is no $T$-railway between $\widetilde{\mathcal{C}}$ and $\widetilde{\mathcal{D}}$. Note that by (i) $\widetilde{\mathcal{C}}$ and $\widetilde{\mathcal{D}}$ are contained in different countries. They are in the same continent because they are linked together directly by the $\widehat{T}$-railway $(i,k)$. Consequently, there must exist no edge of $T$ at all between $\widetilde{\mathcal{C}}$ and $\widetilde{\mathcal{D}}$, because any such edge would have been classified as a $T$-railway according to Definition~\ref{truehierarchy}.

As a result, there must exist some $j$ on the path in $T$ between $i$ and $k$ such that $j$ belongs to neither $\widetilde{\mathcal{C}}$ nor $\widetilde{\mathcal{D}}$. Since $(i,k)$ is an edge of $\widehat{T}$, either $k$ sits on the path in $\widehat{T}$ between $i$ and $j$, or $i$ sits on the path in $\widehat{T}$ between $k$ and $j$. Without loss of generality we assume it is the former case. Then, by the cycle property of maximum weight spanning trees, we have $\mathrm{I}^{\widehat{\mathrm{P}}}(X_i;X_j) \leq \mathrm{I}^{\widehat{\mathrm{P}}}(X_i;X_k)$. Since $i,j,k$ lie on a path in $T$, Lemma~\ref{genhelwrongedge} (i) implies that $H^2(\mathrm{P}_{ijk}, \mathrm{P}_{i\wideparen{ \,~~ j-}k}) \leq 22 \frac{\epsilon^2}{n}$.

On the other hand, by Lemma~\ref{ihminmrg} we have $\minmrg{\mathrm{P}_{ik}} \geq \frac{1}{2} I_{H^2}(\mathrm{P}_{ik}) > 10^9 \frac{\epsilon^2}{n}$, and by Lemma~\ref{minmrgih3} we have $\minmrg{\mathrm{P}_j} \geq \frac{1}{2} I_{H^2}(\mathrm{P}_{ik}) > 10^9 \frac{\epsilon^2}{n}.$ So $\minmrg{\mathrm{P}_{jk}} > 10^9 \frac{\epsilon^2}{n}$. Since $j,k$ belong to different broken-countries, Lemma~\ref{samebrokencountry} implies that $\mindisc{\mathrm{P}_{jk}} < \frac{2}{3}$. By Lemma~\ref{ihmarkov} we have $I_{H^2}(\mathrm{P}_{ij}) \geq I_{H^2}(\mathrm{P}_{ik}) \geq \frac{1}{4} \cdot 10^{10} \frac{\epsilon^2}{n}$. Applying Lemma~\ref{helwrongedgecountry}, we have
\begin{align*}
H^2(\mathrm{P}_{ijk}, \mathrm{P}_{i\wideparen{ \,~~ j-}k}) &\geq \frac{1}{100} I_{H^2}(\mathrm{P}_{ij}) \cdot \paren{1 - \mindisc{\mathrm{P}_{jk}}}^2 \\
		&\geq \frac{1}{100} \cdot \frac{1}{4} \cdot 10^{10} \frac{\epsilon^2}{n} \cdot \paren{1 - \frac{2}{3}}^2 \\
		&> 22 \frac{\epsilon^2}{n},
\end{align*}
a contradiction.

We just proved that the existence of a $\widehat{T}$-railway between a pair of broken-countries (which necessarily lie in different countries by (i)) implies the existence of a $T$-railway between the same pair of broken-countries. This is the only if direction of (iii). Before tackling the if direction, let's prove (ii) first.

That there can be at most one $T$-railway between any pair of countries is clear because every country is $T$-connected by Lemma~\ref{gencountryconnected}.

Suppose there are two $\widehat{T}$-railways between some pair of countries. Because every broken-country is obviously $\widehat{T}$-connected according to Definition~\ref{preliminaryhierarchy}, there can't be more than one edge of $\widehat{T}$ between the same pair of broken-countries. Therefore, those two $\widehat{T}$-railways must go between two different pairs of broken-countries. The only if direction of (iii) we already proved implies that for each of those two pairs of broken-countries there must be a $T$-railway going between them. But then there are two $T$-railways between the same pair of countries, a contradiction. This proves (ii).

We now prove the if direction of (iii). We reset the meaning of all symbols defined so far in this proof (e.g. $i,j,k$). Suppose $(i,j)$ is a $T$-railway between broken-countries $\widetilde{\mathcal{C}}$ and $\widetilde{\mathcal{D}}$, which are contained in countries $\mathcal{C}$ and $\mathcal{D}$, respectively.

We first show that there must exist a $\widehat{T}$-railway between $\mathcal{C}$ and $\mathcal{D}$. Suppose not. According to Definition~\ref{truehierarchy} the fact that $(i,j)$ is a $T$-railway means $\mathcal{C}$ and $\mathcal{D}$ are different countries in the same continent. Since there is no $\widehat{T}$-railway between $\mathcal{C}$ and $\mathcal{D}$, according to Definition~\ref{truehierarchy} there must exist a sequence of distinct countries $\mathcal{C} = \mathcal{C}_0, \mathcal{C}_1, \ldots, \mathcal{C}_r = \mathcal{D}$, $r \geq 2$, such that there is a $\widehat{T}$-railway between $\mathcal{C}_s$ and $\mathcal{C}_{s+1}$ for each $s$. By the only if direction of (iii) we already proved, there is a $T$-railway between $\mathcal{C}_s$ and $\mathcal{C}_{s+1}$ for each $s$. Since each $\mathcal{C}_s$ is $T$-connected by Lemma~\ref{gencountryconnected}, it is easy to see that there is a path in $T$ between $i$ and $j$ that visits the countries $\mathcal{C} = \mathcal{C}_0, \mathcal{C}_1, \ldots, \mathcal{C}_r = \mathcal{D}$, in that order. On the other hand, $(i,j)$ by itself constitutes another path in $T$ between $i$ and $j$. So we have more than one paths in $T$ between $i$ and $j$, contradicting the fact that $T$ is a tree.

So there exists a $\widehat{T}$-railway between $\mathcal{C}$ and $\mathcal{D}$. We show that it must in fact go between $\widetilde{\mathcal{C}}$ and $\widetilde{\mathcal{D}}$. Suppose for the sake of contradiction that it goes between a different pair of broken-countries. By the only if direction of (iii) we already proved, there is a $T$-railway between that other pair of broken-countries. This means that there is an edge of $T$ different from $(i,j)$ that also goes between $\mathcal{C}$ and $\mathcal{D}$, which can't happen because $\mathcal{C}$ and $\mathcal{D}$ are $T$-connected by Lemma~\ref{gencountryconnected}. This contradiction shows that the existing $\widehat{T}$-railway between $\mathcal{C}$ and $\mathcal{D}$ must in fact go between $\widetilde{\mathcal{C}}$ and $\widetilde{\mathcal{D}}$. This finishes the proof of (iii).
\end{proof}

\begin{lemma}
\label{gencontinentconnected}
Every continent is $T$-connected.
\end{lemma}
\begin{proof}
Recall that a continent consists of a collection of countries linked together by $\widehat{T}$-railways. Its $T$-connectedness then follows from Lemma~\ref{gencountryconnected} and Lemma~\ref{genrailwayparallel} (iii).
\end{proof}

\begin{lemma}
\label{differentbrokencontinentsamecontinent}
The path in $T$ between two nodes from different broken-continents in the same continent must contain at least one biased node (possibly at one of its two ends).
\end{lemma}
\begin{proof}
Suppose for the sake of contradiction that there is a path in $T$ containing no biased node between some two nodes from different broken-continents in the same continent. Since a continent is $T$-connected by Lemma~\ref{gencontinentconnected}, the entire path must be contained in that same continent. Since the two end-nodes of that path belong to different broken-continents, there must exist an edge $(i,j)$ on that path such that $i$ and $j$ belong to different broken-continents (in the same continent).

We divide the rest of argument into two cases, depending on whether $i$ and $j$ belong to the same country.

First consider the case $i$ and $j$ belong to the same country. Note that the fact that $i$ and $j$ belong to different broken-continents obviously implies that they belong to different broken-countries. As a result, at least one of $i$ and $j$ must be a biased node by Lemma~\ref{differentbrokencountrysamecountry} (remember that $(i,j)$ is an edge of $T$). This is a contradiction.

Next consider the case $i$ and $j$ belong to different countries. Since $i$ and $j$ belong to the same continent, $(i,j)$ is a $T$-railway according to Definition~\ref{truehierarchy}. By Lemma~\ref{genrailwayparallel}, there exists a $\widehat{T}$-railway between the broken-country containing $i$ and the broken-country containing $j$. But then those two broken-countries (and therefore $i$ and $j$) must in fact be contained in the same broken-continent according to Definition~\ref{preliminaryhierarchy}, contradicting our assumption.
\end{proof}

\begin{lemma}
\label{samebrokencontinent}
If $I_{H^2}(\mathrm{P}_{ij}) \geq 10^{17} \frac{\epsilon^2}{n}$, then $i,j$ belong to the same broken-continent.
\end{lemma}
\begin{proof}
Suppose for the sake of contradiction that $i,j$ belong to different broken-continents. If $i,j$ were to belong to the same continent, then by Lemma~\ref{differentbrokencontinentsamecontinent} the path in $T$ between them must contain at least one biased node. However, for any $i'$ on the path in $T$ between $i$ and $j$ we have by Lemma~\ref{minmrgih3} that
$$\minmrg{\mathrm{P}_{i'}} \geq \frac{1}{2} I_{H^2} (\mathrm{P}_{ij}) > 10^7 \frac{\epsilon^2}{n},$$
and so $i'$ is not biased. This is a contradiction.

Therefore, $i,j$ must belong to different continents. So there must exist an edge $(i',j')$ on the path in $T$ between $i$ and $j$ such that $i',j'$ belong to different continents. Suppose the path in $\widehat{T}$ between $i'$ and $j'$ consists of nodes $i' = i'_0, \ldots, i'_r = j'$, in that order. By the cycle property of maximum weight spanning tree, we have $\mathrm{I}^{\widehat{\mathrm{P}}}(X_{i'};X_{j'}) \leq \mathrm{I}^{\widehat{\mathrm{P}}}(X_{i'_s};X_{i'_{s+1}})$ for all $s = 0, \ldots, r-1$.

Imagine removing $(i',j')$ from $T$. This partitions $[n]$ into two $T$-connected subsets $U$ and $V$, with $i' \in U$ and $j' \in V$. It is easy to see that there exists $s$ such that $i'_s \in U$ and $i'_{s+1} \in V$. Denote $h' = i'_s$ and $k' = i'_{s+1}$. See Figure~\ref{samebrokencityfigure} in Lemma~\ref{samebrokencountry} for an illustration of our setting. Our goal is to show that $h',k'$ belong to the same continent by showing that $(h',k')$ is a $\widehat{T}$-road, a $\widehat{T}$-highway, or a $\widehat{T}$-railway. Note that $h', i', j', k'$ lie on a path in $T$, and $\mathrm{I}^{\widehat{\mathrm{P}}}(X_{i'};X_{j'}) \leq \mathrm{I}^{\widehat{\mathrm{P}}}(X_{h'};X_{k'})$. Lemma~\ref{genhelwrongedge} (ii) thus implies
\begin{align}
\label{helupperboundcountry}
\begin{split}
H^2(\mathrm{P}_{h'i'j'}, \mathrm{P}_{h'\wideparen{-i' \,~~ }j'}) &\leq 62 \frac{\epsilon^2}{n} \\
H^2(\mathrm{P}_{i'j'k'}, \mathrm{P}_{i'\wideparen{ \,~~ j'-}k'}) &\leq 62 \frac{\epsilon^2}{n} 
\end{split}
\end{align}

Applying Lemma~\ref{ihmarkov} (possibly twice) we have $I_{H^2}(\mathrm{P}_{i'j'}) \geq I_{H^2}(\mathrm{P}_{ij}) \geq 10^{17} \frac{\epsilon^2}{n}$.

By Lemma~\ref{helwrongedgecountry} we have
\begin{align}
\label{hellowerboundcountry}
\begin{split}
H^2(\mathrm{P}_{h'i'j'}, \mathrm{P}_{h'\wideparen{-i' \,~~ }j'}) &\geq \frac{1}{100} I_{H^2}(\mathrm{P}_{i'j'}) \cdot \paren{1 - \mindisc{\mathrm{P}_{j'k'}}}^2 \\
H^2(\mathrm{P}_{i'j'k'}, \mathrm{P}_{i'\wideparen{ \,~~ j'-}k'}) &\geq \frac{1}{100} I_{H^2}(\mathrm{P}_{i'j'}) \cdot \paren{1 - \mindisc{\mathrm{P}_{h'i'}}}^2
\end{split}
\end{align}
Combining (\ref{helupperboundcountry}), (\ref{hellowerboundcountry}), and the one bound we established in between, we see that we have 
$\mindisc{\mathrm{P}_{h'i'}} > \frac{9}{10}$ and $\mindisc{\mathrm{P}_{j'k'}} > \frac{9}{10}$.

By Lemma~\ref{mindiscih3} we have $I_{H^2}(\mathrm{P}_{i'k'}) \geq \frac{1}{100} {\mindisc{\mathrm{P}_{j'k'}}}^3 \cdot I_{H^2}(\mathrm{P}_{i'j'}) > 10^{14} \frac{\epsilon^2}{n}$.

By Lemma~\ref{mindiscih3} again we have $I_{H^2}(\mathrm{P}_{h'k'}) \geq \frac{1}{100} {\mindisc{\mathrm{P}_{h'i'}}}^3 \cdot I_{H^2}(\mathrm{P}_{i'k'}) > 10^{11} \frac{\epsilon^2}{n}$.

Finally, $I_{H^2}(\widehat{\mathrm{P}}_{h'k'}) \geq \frac{1}{4} I_{H^2}(\mathrm{P}_{h'k'}) > 10^{10} \frac{\epsilon^2}{n}$ by Lemma~\ref{genihprecision} (i).

According to Definition~\ref{preliminaryhierarchy}, $(h',k')$ is classified as either a $\widehat{T}$-road, a $\widehat{T}$-highway, or a $\widehat{T}$-railway. In any case, $h',k'$ belong to the same broken-continent, and therefore to the same continent. Since a continent is $T$-connected by Lemma~\ref{gencontinentconnected}, and $h',i',j',k'$ lie on a path in $T$, we see that $h',i',j',k'$ all belong to the same continent. This contradicts the fact that $i',j'$ belong to different continents (see start of the second paragraph), and the proof is completed.
\end{proof}

For a summary of the structural results proved in this section, as well as a diagrammatic illustration, see the part that begins with ``Summary of layering'' in Section~\ref{sec:outline of general case}.

\subsubsection{Bounding the Distance between $\mathrm{P}$ and $\mathrm{Q}$}
\label{sec:genbounding}

This subsection puts together everything we have so far and proves the desired bound in total variation distance between $\mathrm{P}$ and $\mathrm{Q}$.

As outlined in Section~\ref{sec:outline of general case}, we need a few hybrid distributions to serve as intermediate steps in bounding the total variation distance between $\mathrm{P}$ and $\mathrm{Q}$. We specify the underlying tree and the pairwise marginal distributions for the tree edges for all the distributions involved:
\begin{itemize}
\item $\mathrm{P}$: the underlying tree is $T$; the marginal for each edge $(i,j)$ of $T$ is $\mathrm{P}_{ij}$;
\item $\mathrm{P}^{(1)}$: the underlying tree is $T$; the marginal for each edge $(i,j)$ of $T$ is $\mathrm{P}_{ij}$, except if $(i,j)$ is a $T$-airway, in which case the marginal is $\mathrm{P}_{ij}^{\rm{(ind)}}$;
\item $\mathrm{P}^{(2)}$: the underlying tree $T^{(2)}$ consists of all the $T$-roads, $T$-highways, $\widehat{T}$-railways, and $T$-airways; the marginal for each edge $(i,j)$ of $T^{(2)}$ is $\mathrm{P}_{ij}$, except if $(i,j)$ is a $T$-airway, in which case the marginal is $\mathrm{P}_{ij}^{\rm{(ind)}}$;
\item $\mathrm{P}^{(3)}$: the underlying tree $T^{(3)}$ consists of all the $T$-roads, $\widehat{T}$-highways, $\widehat{T}$-railways, and $T$-airways; the marginal for each edge $(i,j)$ of $T^{(3)}$ is $\mathrm{P}_{ij}$, except if $(i,j)$ is a $T$-airway, in which case the marginal is $\mathrm{P}_{ij}^{\rm{(ind)}}$;
\item $\mathrm{P}^{(4)}$: the underlying tree $\ddot{T}$ consists of all the $\widehat{T}$-avenues, $T$-trails, $\widehat{T}$-highways, $\widehat{T}$-railways, and $T$-airways; the marginal for each edge $(i,j)$ of $\ddot{T}$ is $\mathrm{P}_{ij}$, except if $(i,j)$ is a $T$-airway, in which case the marginal is $\mathrm{P}_{ij}^{\rm{(ind)}}$;
\item $\mathrm{P}^{(5)}$: the underlying tree is $\ddot{T}$; the marginal for each edge $(i,j)$ of $\ddot{T}$ is $\mathrm{P}_{ij}$, except if $(i,j)$ is a $T$-trail or a $T$-airway, in which case the marginal is $\mathrm{P}_{ij}^{\rm{(ind)}}$;
\item $\mathrm{P}^{(6)}$: the underlying tree is $\widehat{T}$; the marginal for each edge $(i,j)$ of $\widehat{T}$ is $\mathrm{P}_{ij}$, except if $(i,j)$ is a $\widehat{T}$-tunnel, in which case the marginal is $\mathrm{P}_{ij}^{\rm{(ind)}}$;
\item $\mathrm{P}^{(7)}$: the underlying tree is $\widehat{T}$; the marginal for each edge $(i,j)$ of $\widehat{T}$ is $\mathrm{P}_{ij}$;
\item $\mathrm{Q}$: the underlying tree is $\widehat{T}$; the marginal for each edge $(i,j)$ of $\widehat{T}$ is $\widehat{\mathrm{P}}_{ij}$.
\end{itemize}
We bound the Hellinger distance separately between each adjacent pair in the list above. The desired total variation distance bound between $\mathrm{P}$ and $\mathrm{Q}$ as stated in Theorem~\ref{gensufficiency} then follows from the triangle inequality and the fact that $d_{\mathrm{TV}}$ is upper bounded by $\sqrt{2}$ times $H$ (Lemma~\ref{tvvshel}).

\subsubsection*{\underline{Bounding $H(\mathrm{P}, \mathrm{P}^{(1)})$}}

To bound the Hellinger distance between $\mathrm{P}$ and $\mathrm{P}^{(1)}$, we can apply Corollary~\ref{edgeswitch} with $\mathrm{P}' = \mathrm{P}$, $\mathrm{P}'' = \mathrm{P}^{(1)}$, $T' = T'' = T$, and the sets $A_{\lambda}$'s being the continents $\mathfrak{C}_1, \ldots, \mathfrak{C}_{\Lambda}$, to obtain
\begin{equation}
\label{gen01edgeswitch}
H^2(\mathrm{P},\mathrm{P}^{(1)}) \leq \sum_{\lambda=1}^{\Lambda} H^2(\mathrm{P}_{\mathfrak{C}_\lambda}, \mathrm{P}^{(1)}_{\mathfrak{C}_\lambda}) + \sum_{(\lambda, \mu) \in \mathcal{E}} H^2(\mathrm{P}_{W_{\lambda \mu}}, \mathrm{P}^{(1)}_{W_{\lambda \mu}}),
\end{equation}
where $\mathcal{E}$ contains all $(\lambda, \mu)$ for which $\mathfrak{C}_{\lambda}$ and $\mathfrak{C}_{\mu}$ are straddled by a $T$-airway, and $W_{\lambda \mu}$ consists of the two end-nodes of that straddling edge.

The first summation on the right hand side of (\ref{gen01edgeswitch}) is zero.

For a term in the second summation corresponding to $(\lambda,\mu) \in \mathcal{E}$, let the $T$-airway straddling $\mathfrak{C}_{\lambda}$ and $\mathfrak{C}_{\mu}$ be $(i,j)$. So $W_{\lambda \mu} = \{i,j\}$. Note that $\mathrm{P}^{(1)}_{ij}$ equals $\mathrm{P}_{ij}^{\rm{(ind)}}$. By Lemma~\ref{samebrokencontinent} we thus have 
$$H^2 (\mathrm{P}_{W_{\lambda \mu}}, \mathrm{P}^{(1)}_{W_{\lambda \mu}}) = H^2 (\mathrm{P}_{ij}, \mathrm{P}^{(1)}_{ij}) = H^2 (\mathrm{P}_{ij}, \mathrm{P}_{ij}^{\rm{(ind)}}) = I_{H^2}(\mathrm{P}_{ij}) < 10^{17} \frac{\epsilon^2}{n}.$$

Since $\abs{\mathcal{E}} < n$, (\ref{gen01edgeswitch}) implies $H^2(\mathrm{P}, \mathrm{P}^{(1)}) \leq n \cdot 10^{17} \frac{\epsilon^2}{n} = 10^{17} \epsilon^2$, and so $H(\mathrm{P}, \mathrm{P}^{(1)}) \leq  10^9 \epsilon.$

\subsubsection*{\underline{Bounding $H(\mathrm{P}^{(1)}, \mathrm{P}^{(2)})$}}

To bound the Hellinger distance between $\mathrm{P}^{(1)}$ and $\mathrm{P}^{(2)}$, we can apply Corollary~\ref{edgeswitch} with $\mathrm{P}' = \mathrm{P}^{(1)}$, $\mathrm{P}'' = \mathrm{P}^{(2)}$, $T' = T$, $T'' = T^{(2)}$, and the sets $A_{\lambda}$'s being the countries $\mathcal{C}_1, \ldots, \mathcal{C}_{\Lambda}$ to obtain
\begin{equation}
\label{gen12edgeswitch}
H^2(\mathrm{P}^{(1)},\mathrm{P}^{(2)}) \leq \sum_{\lambda=1}^{\Lambda} H^2(\mathrm{P}^{(1)}_{\mathcal{C}_\lambda}, \mathrm{P}^{(2)}_{\mathcal{C}_\lambda}) + \sum_{(\lambda, \mu) \in \mathcal{E}} H^2(\mathrm{P}^{(1)}_{W_{\lambda \mu}}, \mathrm{P}^{(2)}_{W_{\lambda \mu}}),
\end{equation}
where $\mathcal{E}$ contains all $(\lambda, \mu)$ for which $\mathcal{C}_{\lambda}$ and $\mathcal{C}_{\mu}$ are straddled by (1) a $T$-airway or (2) a $T$-railway and a $\widehat{T}$-railway, and $W_{\lambda \mu}$ consists of all end-nodes of those straddling edges.

The only possibly nonzero terms on the right hand side of (\ref{gen12edgeswitch}) are those in the second summation corresponding to $(\lambda,\mu)$ such that $\mathcal{C}_{\lambda}$ and $\mathcal{C}_{\mu}$ are straddled by a $T$-railway and a $\widehat{T}$-railway. For such a pair $(\lambda,\mu)$, let the $T$-railway straddling $\mathcal{C}_{\lambda}$ and $\mathcal{C}_{\mu}$ be $(i,j)$, with $i \in \mathcal{C}_{\lambda}$ and $j \in \mathcal{C}_{\mu}$, and the $\widehat{T}$-railway straddling $\mathcal{C}_{\lambda}$ and $\mathcal{C}_{\mu}$ be $(h,k)$, with $h \in \mathcal{C}_{\lambda}$ and $k \in \mathcal{C}_{\mu}$. By Lemma~\ref{genrailwayparallel} (iii), $h,i$ belong to the same broken-country, and $j,k$ belong to the same broken-country.

First suppose $h \neq i$ and $j \neq k$. In this case, $W_{\lambda \mu} = \{h,i,j,k\}$, and $h,i,j,k$ lie on a path in $T$ (because every country is $T$-connected). It is easy to see that $i,h,k,j$ lie on a path in $\widehat{T}$ (because every broken-country is $\widehat{T}$-connected). So the cycle property of maximum weight spanning trees implies that $\mathrm{I}^{\widehat{\mathrm{P}}}(X_i;X_j) \leq \mathrm{I}^{\widehat{\mathrm{P}}}(X_h;X_k)$. Note that $\mathrm{P}^{(1)}_{hijk}$ equals $\mathrm{P}_{hijk}$, and $\mathrm{P}^{(2)}_{hijk}$ equals $\mathrm{P}_{h\wideparen{-i \,~~ j-}k}$ (because the edges within each country haven't been replaced yet and are still edges of the true tree $T$). By Lemma~\ref{genhelwrongedge} (ii) we thus have
$$H^2(\mathrm{P}^{(1)}_{W_{\lambda \mu}}, \mathrm{P}^{(2)}_{W_{\lambda \mu}}) = H^2 (\mathrm{P}^{(1)}_{hijk}, \mathrm{P}^{(2)}_{hijk}) = H^2 (\mathrm{P}_{hijk}, \mathrm{P}_{h\wideparen{-i \,~~ j-}k}) \leq 248 \frac{\epsilon^2}{n}.$$

Next suppose $h = i$ and $j \neq k$ (the case $h \neq i$ and $j = k$ is symmetric). In this case, $W_{\lambda \mu} = \{i,j,k\}$, and $i,j,k$ lie on a path in $T$. It is easy to see that $i,k,j$ lie on a path in $\widehat{T}$. So the cycle property of maximum weight spanning trees implies that $\mathrm{I}^{\widehat{\mathrm{P}}}(X_i;X_j) \leq \mathrm{I}^{\widehat{\mathrm{P}}}(X_i;X_k)$. Note that $\mathrm{P}^{(1)}_{ijk}$ equals $\mathrm{P}_{ijk}$, and $\mathrm{P}^{(2)}_{ijk}$ equals $\mathrm{P}_{i\wideparen{ \,~~ j-}k}$. By Lemma~\ref{genhelwrongedge} (i) we thus have
$$H^2(\mathrm{P}^{(1)}_{W_{\lambda \mu}}, \mathrm{P}^{(2)}_{W_{\lambda \mu}}) = H^2 (\mathrm{P}^{(1)}_{ijk}, \mathrm{P}^{(2)}_{ijk}) = H^2(\mathrm{P}_{ijk}, \mathrm{P}_{i\wideparen{ \,~~ j-}k}) \leq 22 \frac{\epsilon^2}{n}.$$

Lastly, if $h = i$ and $j = k$, then $W_{\lambda \mu} = \{i,j\}$. It's easy to see that $H^2(\mathrm{P}^{(1)}_{W_{\lambda \mu}}, \mathrm{P}^{(2)}_{W_{\lambda \mu}}) = 0$.

Thus, in any case we have $H^2(\mathrm{P}^{(1)}_{W_{\lambda \mu}}, \mathrm{P}^{(2)}_{W_{\lambda \mu}}) \leq 248 \frac{\epsilon^2}{n}$.

Since $\abs{\mathcal{E}} < n$, (\ref{gen12edgeswitch}) implies $H^2(\mathrm{P}^{(1)}, \mathrm{P}^{(2)}) \leq n \cdot 248 \frac{\epsilon^2}{n} = 248 \epsilon^2$, and so $H(\mathrm{P}^{(1)}, \mathrm{P}^{(2)}) \leq  16 \epsilon.$

\subsubsection*{\underline{Bounding $H(\mathrm{P}^{(2)}, \mathrm{P}^{(3)})$}}

This is similar to the bounding between $\mathrm{P}^{(1)}$ and $\mathrm{P}^{(2)}$, but centered at the city/highway level of the hierarchy instead of the country/railway level. We apply Corollary~\ref{edgeswitch} with $\mathrm{P}' = \mathrm{P}^{(2)}$, $\mathrm{P}'' = \mathrm{P}^{(3)}$, $T' = T^{(2)}$, $T'' = T^{(3)}$, and the sets $A_{\lambda}$'s being the cities $C_1, \ldots, C_{\Lambda}$. The only possibly nonzero terms on the right hand side of the resulting inequality correspond to $(\lambda, \mu)$ for which $C_\lambda$ and $C_\mu$ are straddled by a $T$-highway and a $\widehat{T}$-highway. An analysis parallel to that for the pair $\mathrm{P}^{(1)}$ and $\mathrm{P}^{(2)}$ leads to $H(\mathrm{P}^{(2)}, \mathrm{P}^{(3)}) \leq  16 \epsilon.$

\subsubsection*{\underline{Bounding $H(\mathrm{P}^{(3)}, \mathrm{P}^{(4)})$}}

To bound the Hellinger distance between $\mathrm{P}^{(3)}$ and $\mathrm{P}^{(4)}$, we can apply Corollary~\ref{edgeswitch} with $\mathrm{P}' = \mathrm{P}^{(3)}$, $\mathrm{P}'' = \mathrm{P}^{(4)}$, $T' = T^{(3)}$, $T'' = \ddot{T}$, and the sets $A_{\lambda}$'s being the cities $C_1, \ldots, C_{\Lambda}$ to obtain
\begin{equation}
\label{gen34edgeswitch}
H^2(\mathrm{P}^{(3)},\mathrm{P}^{(4)}) \leq \sum_{\lambda=1}^{\Lambda} H^2(\mathrm{P}^{(3)}_{C_\lambda}, \mathrm{P}^{(4)}_{C_\lambda}) + \sum_{(\lambda, \mu) \in \mathcal{E}} H^2(\mathrm{P}^{(3)}_{W_{\lambda \mu}}, \mathrm{P}^{(4)}_{W_{\lambda \mu}}),
\end{equation}
where $\mathcal{E}$ contains all $(\lambda, \mu)$ for which $C_\lambda$ and $C_\mu$ are straddled by a $T$-airway, a $\widehat{T}$-railway, or a $\widehat{T}$-highway, and $W_{\lambda \mu}$ consists of both end-nodes of that straddling edge.

The second summation on the right hand side of (\ref{gen34edgeswitch}) is zero.

We now focus on a term in the first summation corresponding to $\lambda$. Let $\ddot{E}_{C_\lambda}$ be the set of all $\widehat{T}$-avenues and $T$-trails in $C_\lambda$.

Fix an arbitrary $u \in C_\lambda$. We define $\ddot{x}_{C_\lambda} = \paren{\ddot{x}_i}_{i \in C_\lambda} \in \brac{1,-1}^{C_\lambda}$ such that $\ddot{x}_u = 1$, and
\begin{align*}
\ddot{x}_i = 
\begin{cases}
\ddot{x}_j &\mbox{if } \mathrm{P}(X_i=X_j) \geq \frac{1}{2} \\
-\ddot{x}_j &\mbox{if } \mathrm{P}(X_i=X_j) < \frac{1}{2}
\end{cases}
\end{align*}
for each $(i,j) \in \ddot{E}_{C_\lambda}$. Since the edges in $\ddot{E}_{C_\lambda}$ form a spanning tree of $C_\lambda$, $\ddot{x}_{C_\lambda}$ is uniquely defined.

Note that the event ``$X_{C_\lambda} = \ddot{x}_{C_\lambda}$'' obviously implies ``$X_u = 1$'', and that ``$X_u = 1, X_{C_\lambda} \neq \ddot{x}_{C_\lambda}$'' implies ``$\frac{X_i}{X_j} \neq \frac{\ddot{x}_i}{\ddot{x}_j}$ for some $(i,j) \in \ddot{E}_{C_\lambda}$''. Taking probabilities according to $\mathrm{P}^{(3)}$ and $\mathrm{P}^{(4)}$, respectively, we have
\begin{align*}
\mathrm{P}^{(3)}_u(1) \geq \mathrm{P}^{(3)}(X_{C_\lambda} &= \ddot{x}_{C_\lambda}) \geq \mathrm{P}^{(3)}_u(1) - \sum_{(i,j) \in \ddot{E}_{C_\lambda}} \mathrm{P}^{(3)}\paren{\frac{X_i}{X_j} \neq \frac{\ddot{x}_i}{\ddot{x}_j}} \\
\mathrm{P}^{(4)}_u(1) \geq \mathrm{P}^{(4)}(X_{C_\lambda} &= \ddot{x}_{C_\lambda}) \geq \mathrm{P}^{(4)}_u(1) - \sum_{(i,j) \in \ddot{E}_{C_\lambda}} \mathrm{P}^{(4)}\paren{\frac{X_i}{X_j} \neq \frac{\ddot{x}_i}{\ddot{x}_j}}
\end{align*}
Note that $\mathrm{P}^{(3)}_i = \mathrm{P}^{(4)}_i = \mathrm{P}_i$ for any $i$, and that $\mathrm{P}^{(3)}_{ij} = \mathrm{P}^{(4)}_{ij} = \mathrm{P}_{ij}$ for any $(i,j)$ that is a $\widehat{T}$-avenue or a $T$-trail. Also, it is easy to see that $\mathrm{P}\paren{\frac{X_i}{X_j} \neq \frac{\ddot{x}_i}{\ddot{x}_j}} = \mindiag{\mathrm{P}_{ij}}$ by the way $\ddot{x}_{C_\lambda}$ has been defined. By Lemma~\ref{roadavenuemindiag} we thus have
\begin{align}
\label{xclambdabound}
\begin{split}
\mathrm{P}_u(1) &\geq \mathrm{P}^{(3)}(X_{C_\lambda} = \ddot{x}_{C_\lambda}) \geq \mathrm{P}_u(1) - \abs{C_\lambda} \cdot (10^5+1) \frac{\epsilon^2}{n} \\
\mathrm{P}_u(1) &\geq \mathrm{P}^{(4)}(X_{C_\lambda} = \ddot{x}_{C_\lambda}) \geq \mathrm{P}_u(1) - \abs{C_\lambda} \cdot (10^5+1) \frac{\epsilon^2}{n}
\end{split}
\end{align}
Similarly, we have
\begin{align}
\label{negativexclambdabound}
\begin{split}
\mathrm{P}_u(-1) &\geq \mathrm{P}^{(3)}(X_{C_\lambda} = -\ddot{x}_{C_\lambda}) \geq \mathrm{P}_u(-1) - \abs{C_\lambda} \cdot (10^5+1) \frac{\epsilon^2}{n} \\
\mathrm{P}_u(-1) &\geq \mathrm{P}^{(4)}(X_{C_\lambda} = -\ddot{x}_{C_\lambda}) \geq \mathrm{P}_u(-1) - \abs{C_\lambda} \cdot (10^5+1) \frac{\epsilon^2}{n}
\end{split}
\end{align}
where $-\ddot{x}_{C_\lambda}$ is the component-wise negation of $\ddot{x}_{C_\lambda}$.

From (\ref{xclambdabound}) and (\ref{negativexclambdabound}) it is easy to see that $\dtv{\mathrm{P}^{(3)}_{C_\lambda}}{\mathrm{P}^{(4)}_{C_\lambda}} \leq 3 \cdot \abs{C_\lambda} \cdot (10^5+1) \frac{\epsilon^2}{n}$ (for one thing note that the probability that $X_{C_\lambda}$ does not equal to either $\ddot{x}_{C_\lambda}$ or $-\ddot{x}_{C_\lambda}$ is at most $2 \cdot \abs{C_\lambda} \cdot (10^5+1) \frac{\epsilon^2}{n}$, under either $\mathrm{P}^{(3)}_{C_\lambda}$ or $\mathrm{P}^{(4)}_{C_\lambda}$). By Lemma~\ref{tvvshel} we thus have
$$H^2(\mathrm{P}^{(3)}_{C_\lambda}, \mathrm{P}^{(4)}_{C_\lambda}) \leq \dtv{\mathrm{P}^{(3)}_{C_\lambda}}{\mathrm{P}^{(4)}_{C_\lambda}} \leq 3 \cdot \abs{C_\lambda} \cdot (10^5+1) \frac{\epsilon^2}{n}.$$

By (\ref{gen34edgeswitch}) we thus have
$$H^2(\mathrm{P}^{(3)}, \mathrm{P}^{(4)}) \leq \sum_{\lambda = 1}^\Lambda H^2(\mathrm{P}^{(3)}_{C_\lambda}, \mathrm{P}^{(4)}_{C_\lambda}) \leq  \sum_{\lambda = 1}^\Lambda 3 \cdot \abs{C_\lambda} \cdot (10^5+1) \frac{\epsilon^2}{n} = 3 \cdot (10^5+1) \epsilon^2,$$
and so $H(\mathrm{P}^{(3)}, \mathrm{P}^{(4)}) \leq  10^3 \epsilon$.

\subsubsection*{\underline{Bounding $H(\mathrm{P}^{(4)}, \mathrm{P}^{(5)})$}}

To bound the Hellinger distance between $\mathrm{P}^{(4)}$ and $\mathrm{P}^{(5)}$, we can apply Corollary~\ref{edgeswitch} with $\mathrm{P}' = \mathrm{P}^{(4)}$, $\mathrm{P}'' = \mathrm{P}^{(5)}$, $T' = T'' = \ddot{T}$, and the sets $A_{\lambda}$'s being the single-node sets $\{1\}, \ldots, \{n\}$, to obtain
\begin{equation}
\label{gen45edgeswitch}
H^2(\mathrm{P}^{(4)},\mathrm{P}^{(5)}) \leq \sum_{i=1}^{n} H^2(\mathrm{P}^{(4)}_i, \mathrm{P}^{(5)}_i) + \sum_{(i,j) \in E} H^2(\mathrm{P}^{(4)}_{ij}, \mathrm{P}^{(5)}_{ij}),
\end{equation}
where $E$ is the set of edges of $\ddot{T}$.

The only possibly nonzero terms on the right hand side of (\ref{gen45edgeswitch}) are those in the second summation corresponding to $(i,j)$ being a $T$-trail. For such a term, note that $\mathrm{P}^{(4)}_{ij}$ equals $\mathrm{P}_{ij}$, and $\mathrm{P}^{(5)}_{ij}$ equals $\mathrm{P}^{\rm{(ind)}}_{ij}$. By Lemma~\ref{trailih} we thus have
$$H^2 (\mathrm{P}^{(4)}_{ij}, \mathrm{P}^{(5)}_{ij}) = H^2(\mathrm{P}_{ij}, \mathrm{P}^{\rm{(ind)}}_{ij}) = I_{H^2}(\mathrm{P}_{ij}) \leq 2 \cdot 10^7 \frac{\epsilon^2}{n}.$$

Since $\abs{E} < n$, (\ref{gen45edgeswitch}) implies $H^2(\mathrm{P}^{(4)}, \mathrm{P}^{(5)}) \leq n \cdot 2 \cdot 10^7 \frac{\epsilon^2}{n} = 2 \cdot 10^7 \epsilon^2$, and so $H(\mathrm{P}^{(4)}, \mathrm{P}^{(5)}) <  10^4 \epsilon.$

\subsubsection*{\underline{Bounding $H(\mathrm{P}^{(5)}, \mathrm{P}^{(6)})$}}

Note that $\mathrm{P}^{(5)}$ and $\mathrm{P}^{(6)}$ have the same marginal for each continent, and that different continents are mutually independent according to both. So $\mathrm{P}^{(5)}$ and $\mathrm{P}^{(6)}$ are in fact equal!

\subsubsection*{\underline{Bounding $H(\mathrm{P}^{(6)}, \mathrm{P}^{(7)})$}}

To bound the Hellinger distance between $\mathrm{P}^{(6)}$ and $\mathrm{P}^{(7)}$, we can apply Corollary~\ref{edgeswitch} with $\mathrm{P}' = \mathrm{P}^{(6)}$, $\mathrm{P}'' = \mathrm{P}^{(7)}$, $T' = T'' = \widehat{T}$, and the sets $A_{\lambda}$'s being the single-node sets $\{1\}, \ldots, \{n\}$, to obtain
\begin{equation}
\label{gen67edgeswitch}
H^2(\mathrm{P}^{(6)},\mathrm{P}^{(7)}) \leq \sum_{i=1}^{n} H^2(\mathrm{P}^{(6)}_i, \mathrm{P}^{(7)}_i) + \sum_{(i,j) \in E} H^2(\mathrm{P}^{(6)}_{ij}, \mathrm{P}^{(7)}_{ij}),
\end{equation}
where $E$ is the set of edges of $\widehat{T}$.

The only possibly nonzero terms on the right hand side of (\ref{gen67edgeswitch}) are those in the second summation corresponding to $(i,j)$ being a $\widehat{T}$-tunnel. For such a term, note that $\mathrm{P}^{(6)}_{ij}$ equals $\mathrm{P}^{\rm{(ind)}}_{ij}$, and $\mathrm{P}^{(7)}_{ij}$ equals $\mathrm{P}_{ij}$. According to Definition~\ref{preliminaryhierarchy} $I_{H^2}(\widehat{\mathrm{P}}_{ij}) < 10^{10} \frac{\epsilon^2}{n}$, which implies $I_{H^2}(\mathrm{P}_{ij}) < 4 \cdot 10^{10} \frac{\epsilon^2}{n}$ (the contrapositive of this follows from Lemma~\ref{genihprecision} (i)). We thus have
$$H^2 (\mathrm{P}^{(6)}_{ij}, \mathrm{P}^{(7)}_{ij}) = H^2(\mathrm{P}^{\rm{(ind)}}_{ij}, \mathrm{P}_{ij}) = I_{H^2}(\mathrm{P}_{ij}) \leq 4 \cdot 10^{10} \frac{\epsilon^2}{n}.$$

Since $\abs{E} < n$, (\ref{gen67edgeswitch}) implies $H^2(\mathrm{P}^{(6)}, \mathrm{P}^{(7)}) \leq n \cdot 4 \cdot 10^{10} \frac{\epsilon^2}{n} = 4 \cdot 10^{10} \epsilon^2$, and so $H(\mathrm{P}^{(6)}, \mathrm{P}^{(7)}) <  10^6 \epsilon.$

\subsubsection*{\underline{Bounding $H(\mathrm{P}^{(7)}, \mathrm{Q})$}}

To bound the Hellinger distance between $\mathrm{P}^{(7)}$ and $\mathrm{Q}$, we can apply Corollary~\ref{edgeswitch} with $\mathrm{P}' = \mathrm{P}^{(7)}$, $\mathrm{P}'' = \mathrm{Q}$, $T' = T'' = \widehat{T}$, and the sets $A_{\lambda}$'s being the single-node sets $\{1\}, \ldots, \{n\}$, to obtain
\begin{equation}
\label{gen7qedgeswitch}
H^2(\mathrm{P}^{(7)},\mathrm{Q}) \leq \sum_{i=1}^{n} H^2(\mathrm{P}^{(7)}_i, \mathrm{Q}_i) + \sum_{(i,j) \in E} H^2(\mathrm{P}^{(7)}_{ij}, \mathrm{Q}_{ij}),
\end{equation}
where $E$ is the set of edges of $\widehat{T}$.

For a term in the first summation corresponding to $i$, note that $\mathrm{P}^{(7)}_i$ equals $\mathrm{P}_i$, and $\mathrm{Q}_i$ equals $\widehat{\mathrm{P}}_i$. By Lemma~\ref{genprecision} (i) we thus have
$$H^2 (\mathrm{P}^{(7)}_i, \mathrm{Q}_i) = H^2 (\mathrm{P}_i, \widehat{\mathrm{P}}_i) \leq \frac{1}{10^{20}} \frac{\epsilon^2}{n}.$$

For a term in the second summation corresponding to $(i,j) \in E$, note that $\mathrm{P}^{(7)}_{ij}$ equals $\mathrm{P}_{ij}$, and $\mathrm{Q}_{ij}$ equals $\widehat{\mathrm{P}}_{ij}$. By Lemma~\ref{genprecision} (ii) we thus have
$$H^2 (\mathrm{P}^{(7)}_{ij}, \mathrm{Q}_{ij}) = H^2 (\mathrm{P}_{ij}, \widehat{\mathrm{P}}_{ij}) \leq 2 \cdot \frac{1}{10^{20}} \frac{\epsilon^2}{n}.$$

Since $\abs{E} < n$, (\ref{gen7qedgeswitch}) implies $H^2(\mathrm{P}^{(4)}, \mathrm{Q}) \leq n \cdot \frac{1}{10^{20}} \frac{\epsilon^2}{n} + n \cdot 2 \cdot \frac{1}{10^{20}} \frac{\epsilon^2}{n} < \epsilon^2$, and so $H(\mathrm{P}^{(7)}, \mathrm{Q}) <  \epsilon.$

\begin{proof}[\textbf{\underline{Proof of Theorem~\ref{gensufficiency}}}]

Combining the bounds between adjacent pairs from the list of distributions at the beginning of Section~\ref{sec:genbounding}, we have by the triangle inequality
\begin{align*}
H(\mathrm{P}, \mathrm{Q}) &\leq H(\mathrm{P}, \mathrm{P}^{(1)}) + H(\mathrm{P}^{(1)}, \mathrm{P}^{(2)}) + H(\mathrm{P}^{(2)}, \mathrm{P}^{(3)}) + H(\mathrm{P}^{(3)}, \mathrm{P}^{(4)}) \\
		& \qquad \qquad \qquad \qquad \qquad H(\mathrm{P}^{(4)}, \mathrm{P}^{(5)})+ H(\mathrm{P}^{(5)}, \mathrm{P}^{(6)})+ H(\mathrm{P}^{(6)}, \mathrm{P}^{(7)}) + H(\mathrm{P}^{(7)}, \mathrm{Q}) \\
		& \leq 10^9 \epsilon + 16 \epsilon + 16 \epsilon + 10^3 \epsilon + 10^4 \epsilon + 0 + 10^6 \epsilon + \epsilon \\
		& <  2 \cdot 10^9 \epsilon
\end{align*}
Finally, we have $\dtv{\mathrm{P}}{\mathrm{Q}} \leq \sqrt{2} H(\mathrm{P}, \mathrm{Q}) < 10^{10} \epsilon$ by Lemma~\ref{tvvshel}.
\end{proof}

\newpage

\section{Experimental Results}
\label{sec:experiments}

In this section, we describe an experiment designed to get insight into the absolute constant hidden in the $O\left({n (\ln n + \ln {1/\gamma}) \over \eps^2}\right)$ sample complexity of Theorem~\ref{thm:main}. For this purpose, for various sample sizes $m$, we run Algorithm~\ref{algo:genalgo} (namely, the Chow-Liu algorithm with the plug-in estimator for mutual information computation) on a corresponding collection of ``hard instances'' (randomly generated depending on $m$ in a way as suggested by our analysis) of binary tree-structured Bayesnets on $n = 100$ nodes and estimate the total variation reconstruction errors. Our experiment described below suggests that the hidden constant is much smaller than that implied by the crude bounds in our proof. 

\paragraph{Our experimental design.} We first describe our experiment. We fix the tree size at $n = 100$. We go through ten different sample sizes $m$ that form a geometric sequence starting with $m = 1000$ and ending with $m = 100000$ (specifically, we go through $m = \lfloor 1000 \cdot 100^{\frac{t}{9}} \rfloor$, for $t = 0, \ldots, 9$). For each $m$ we generate $25$ binary tree-structured Bayesnets on $n = 100$ nodes identically and independently, as follows: to generate each Bayesnet $\mathrm{P}$, we first generate its underlying tree $T$ by randomly assigning to each edge of the complete graph on $n = 100$ nodes a weight from the interval $(0,1)$ uniformly, and then taking the maximum weight spanning tree as $T$. We root $T$ at the node $0$ and orient all the edges away from the root. Note that to complete the specification of $\mathrm{P}$ it suffices to specify its marginal on the root node and its edge conditionals. That is, it suffices to specify $\mathrm{P}_0(1)$, as well as $\mathrm{P}_{j|i}(1|1)$ and $\mathrm{P}_{j|i}(1|-1)$ for each edge $(i,j)$ of $T$ (oriented away from the root).

We generate $\mathrm{P}_0(1)$ uniformly from the interval $(0,1)$. To generate the edge conditionals, we first generate $(p_1,p_2,p_3)$ uniformly from the two-dimensional probability simplex (that is, according to Dirichlet(1,1,1)). Then, for each edge $(i,j)$ of $T$ independently, we mark it as a ``normal'' edge with probability $p_1$, a ``strong'' edge with probability $p_2$, or a ``weak'' edge with probability $p_3$. For each edge $(i,j)$ of $T$ that is marked as a ``normal'' edge, we generate each of $\mathrm{P}_{j|i}(1|1)$ and $\mathrm{P}_{j|i}(1|-1)$ uniformly from the interval $(0,1)$, independently. For each edge $(i,j)$ of $T$ that is marked as a ``strong'' edge, we set $\mathrm{P}_{j|i}(1|1) = Z_1 \cdot \frac{\ln{n}}{m}$ and $\mathrm{P}_{j|i}(1|-1) = 1 - Z_{-1} \cdot \frac{\ln{n}}{m}$ with probability $\frac{1}{2}$, or $\mathrm{P}_{j|i}(1|1) = 1 - Z_1 \cdot \frac{\ln{n}}{m}$ and $\mathrm{P}_{j|i}(1|-1) = Z_{-1} \cdot \frac{\ln{n}}{m}$ with probability $\frac{1}{2}$, where $Z_1$ and $Z_{-1}$ are independent standard log-normal random variables (whose logarithms have standard normal distributions). For each edge $(i,j)$ of $T$ that is marked as a ``weak'' edge, we generate $\mathrm{P}_{j|i}(1|1)$ uniformly from the interval $(0,1)$ and then set $\mathrm{P}_{j|i}(1|-1)$ to equal $\mathrm{P}_{j|i}(1|1) + Z \cdot \sqrt{\frac{\ln{n}}{m}}$ or $\mathrm{P}_{j|i}(1|1) - Z \cdot \sqrt{\frac{\ln{n}}{m}}$ with probability $\frac{1}{2}$ each, where $Z$ is a standard log-normal random variable. Remember that $n = 100$ throughout. Some of the expressions above may yield values less than $0$ or greater than $1$, in which cases we round them to $0$ and $1$, respectively. This completes the description of our procedure for generating $\mathrm{P}$.

On each generated $\mathrm{P}$, we run Algorithm~\ref{algo:genalgo} (the Chow-Liu algorithm) seven times, drawing $m$ new samples for each run. Let $\mathrm{Q}$ denote the output binary tree-structured Bayesnet for any given run. We estimate the error in total variation distance for that given run, $\dtv{\mathrm{P}}{\mathrm{Q}}$, based on the following identity:
$$\dtv{\mathrm{P}}{\mathrm{Q}} = \frac{1}{2} \sum_{x \in \brac{1,-1}^n} \abs{\mathrm{P}(x) - \mathrm{Q}(x)} = \frac{1}{2} \sum_{x \in \brac{1,-1}^n} \mathrm{P}(x) \cdot \abs{1 - \frac{\mathrm{Q}(x)}{\mathrm{P}(x)}} = \frac{1}{2} \mathbb{E}_{x \sim \mathrm{P}} \abs{1 - \frac{\mathrm{Q}(x)}{\mathrm{P}(x)}}.$$
More specifically, we draw $40000$ i.i.d samples $x^{(1)}, \ldots, x^{(40000)}$ from $\brac{1,-1}^n$ according to $\mathrm{P}$, and our estimate for $\dtv{\mathrm{P}}{\mathrm{Q}}$ is $\frac{1}{2} \cdot \frac{1}{40000} \sum_{t=1}^{40000} \abs{1 - \frac{\mathrm{Q}(x^{(t)})}{\mathrm{P}(x^{(t)})}}$.

Out of the seven estimates for the error in total variation distance obtained from the seven runs of Algorithm~\ref{algo:genalgo} on $\mathrm{P}$, we pick the second largest value and denote it as $\pmb{\hat{\epsilon}_{\mathrm{P},0.25}(m)}$ (for reasons that will become clear later). For each sample size $m$, the value of the largest $\hat{\epsilon}_{\mathrm{P},0.25}(m)$ among the $25$ instances generated for $m$ is denoted as $\pmb{\hat{\epsilon}_{100,0.25}(m)}$.

To summarize, we fix $n = 100$ and go through $m = \lfloor 1000 \cdot 100^{\frac{t}{9}} \rfloor$, for $t = 0, \ldots, 9$. For each $m$, we generate $25$ binary tree-structured Bayesnets on $n = 100$ nodes identically and independently (in a way that depends on $m$). On each generated Bayesnet $\mathrm{P}$, we run Algorithm~\ref{algo:genalgo} (the Chow-Liu algorithm) seven times, each on $m$ new samples. We estimate the error in total variation distance between the output and the true Bayesnet $\mathrm{P}$ for each of the seven runs, and define $\hat{\epsilon}_{\mathrm{P},0.25}(m)$ to be the second largest value among the seven estimates. We define $\hat{\epsilon}_{100,0.25}(m)$ to be the largest value of $\hat{\epsilon}_{\mathrm{P},0.25}(m)$ among the $25$ instances generated for $m$.

\paragraph{Remarks on our experimental design.} We generate the edge conditionals in the manner described above to ensure that our learning instances encompass a wide variety of mixing proportions of ``normal'' edges, ``strong'' edges, and ``weak'' edges  (for each instance the associated triplet $(p_1,p_2,p_3)$ specifies the expected proportion of each type of edges). Our analysis in earlier parts of the paper suggests that the hard instances for the learning problem likely contain combinations of those different types of edges, because the interactions among edges across those different regimes could potentially lead to many edges being learned incorrectly. Note that the definitions for ``strong'' edges and ``weak'' edges depend on $m$ (they are motivited by the classification of edges in our analysis; see Definition~\ref{hierarchy} and Definition~\ref{preliminaryhierarchy}). This is because, as suggested by our analysis, even for a fixed tree size $n$, the set of instances that are ``hard'' depends on the sample size $m$.

It is important to note that we are not merely experimenting with different sample sizes $m$ on some fixed binary tree-structured Bayesnet. In fact, for a fixed binary tree-structured Bayesnet, it is easy to see that as long as none of the tree edges is degenerate (a tree edge being degenerate means that the two end nodes are independent, always equal, or always unequal) then with at least $\Theta \paren{\ln \frac{1}{\gamma}}$ samples (the hidden constant may depend on the distribution itself) all edges will be simultaneously learned correctly with probability at least, say, $1-\frac{\gamma}{2}$ (see~\cite{TanATW11}). When that happens, the error in total variation distance becomes very easy to bound, because we don't need to worry about errors that are induced by structural mistakes (in fact one straightforward application of Corollary~\ref{edgeswitch}, the squared Hellinger subadditivity, would suffice). Thus, for each fixed distribution, it is relatively easy to establish that $O\left( \ln {1/\gamma} \over \eps^2 \right)$ samples suffice to learn the distribution up to $\epsilon$ in total variation distance with error probability at most $\gamma$, asymptotically. However, it is unclear whether the hidden constant can be chosen to be, for example, dependent only on $n$ and not on the distribution itself. 

Let $\pmb{\epsilon_{\mathrm{P},\gamma}(m)}$ denote the smallest value that can be picked so that running Algorithm~\ref{algo:genalgo} on $\mathrm{P}$ using $m$ samples returns $\mathrm{Q}$ such that $\dtv{\mathrm{P}}{\mathrm{Q}} \leq \epsilon_{\mathrm{P},\gamma}(m)$, with probability at least $1-\gamma$. Let $\pmb{\epsilon_{n,\gamma}(m)}$ denote the smallest value that can be picked so that, for \textit{any} binary tree-structured bayesnet $\mathrm{P}$ on $n$ nodes, running Algorithm~\ref{algo:genalgo} on $\mathrm{P}$ using $m$ samples returns $\mathrm{Q}$ such that $\dtv{\mathrm{P}}{\mathrm{Q}} \leq \epsilon_{n,\gamma}(m)$, with probability at least $1-\gamma$. The core of our experiment consists of estimating $\epsilon_{n,\gamma}(m)$, because it represents the level of precision that can be achieved even for the worst case distribution on $n$ nodes, with $m$ samples and error probability at most $\gamma$. Recall that the sample complexity $O\left({n (\ln n + \ln {1/\gamma}) \over \eps^2}\right)$ in Theorem~\ref{thm:main} hides an absolute constant. One of the implications is that for fixed $n$ and $\gamma$, $\epsilon_{n,\gamma}(m)$ should decrease at least at the rate of (some multiple of) $\frac{1}{\sqrt{m}}$ as $m$ increases. On the other hand, the fact that the sample complexity in Theorem~\ref{thm:main} is optimal up to a constant factor suggests that it should decrease at most at the rate of (some multiple of) $\frac{1}{\sqrt{m}}$ as $m$ increases (at least as far as $\limsup$ is concerned). Thus, barring the possibility of extremely pathological behavior, the rate of decrease for $\epsilon_{n,\gamma}(m)$ as $m$ increases should be asymptotically equal to some fixed multiple of $\frac{1}{\sqrt{m}}$. In the experiment we described, $\hat{\epsilon}_{\mathrm{P},0.25}(m)$ serves as an estimate for $\epsilon_{\mathrm{P},0.25}(m)$, and $\hat{\epsilon}_{100,0.25}(m)$ serves as an estimate for $\epsilon_{100,0.25}(m)$. Recall that $\hat{\epsilon}_{\mathrm{P},0.25}(m)$ is taken to be the second largest value among the estimated TV errors from seven independent runs of Algorithm~\ref{algo:genalgo} on $\mathrm{P}$; this choice is based on the fact that the second largest value among seven independent uniform samples from the interval $(0,1)$ has expected value $0.75$. Also recall that $\hat{\epsilon}_{100,0.25}(m)$ is taken to be the largest value of $\hat{\epsilon}_{\mathrm{P},0.25}(m)$ among the $25$ ``hard instances'' generated for $m$; this choice is based on the fact that $\epsilon_{100,0.25}(m)$ is the supremum of $\epsilon_{\mathrm{P},0.25}(m)$ as $\mathrm{P}$ ranges over all binary tree-structured bayesnets on $n=100$ nodes.

Our experimental design is relatively crude. For example, the way we generated our ``hard instances'' is based on the challenges encountered in our theoretical analysis, as well as the layering approach we adopted to resolve them. It is not clear whether that constitutes an effective way to reach the actual worst case (or near worst case) instances for learning using Algorithm~\ref{algo:genalgo}. A more comprehensive experiment would include more extensive search for those actual worst case instances, as well as investigate different values of $n$ and $\gamma$ (and also perform more runs of Algorithm~\ref{algo:genalgo} on each generated instance). We leave that as a potentially interesting step to undertake in the future, and content ourselves here with providing just an indication. 

\renewcommand{\arraystretch}{1.2}
\begin{table}[H]
\begin{center}
\begin{tabular}{ | m{5em} | m{8em} | } 
\hline
$m$ & $\hat{\epsilon}_{100,0.25}(m)$ \\
\hline
1000 & 0.328577064 \\ 
1668 & 0.229263919 \\ 
2782 & 0.192228084 \\
4641 & 0.152060819 \\
7742 & 0.109599911 \\
12915 & 0.091076733 \\
21544 & 0.066171029 \\
35938 & 0.050152842 \\
59948 & 0.044043771 \\
100000 & 0.029603917 \\ 
\hline
\end{tabular}
\caption{values of $\hat{\epsilon}_{100,0.25}(m)$ for the ten sample sizes $m$ under investigation}
\label{table}
\end{center}
\end{table}
\renewcommand{\arraystretch}{1}

\begin{figure}[H]
\centering
\includegraphics[width=17cm]{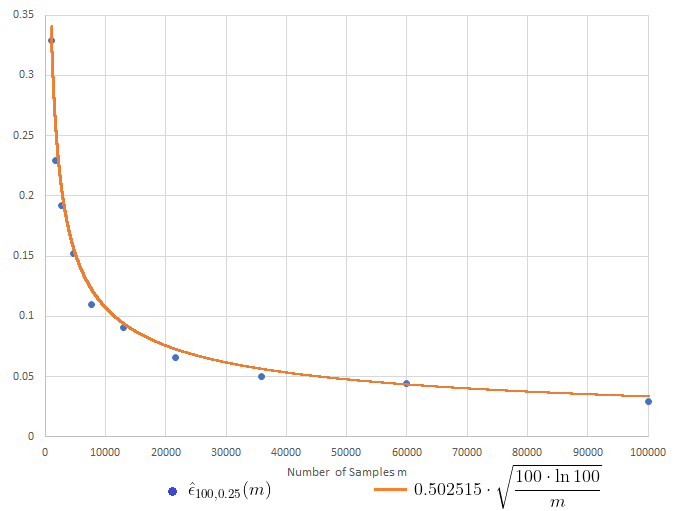}
\caption{Plots of the points $(m,\hat{\epsilon}_{100,0.25}(m))$ for the ten sample sizes $m$ under investigation, and the function $0.502515 \cdot \sqrt{100 \cdot \ln 100 \over m}$}
\label{plot}
\end{figure}

\paragraph{Results and discussion.} We'd like to get insight into the absolute constant hidden in the $O\left({n (\ln n + \ln {1/\gamma}) \over \eps^2}\right)$ sample complexity of Theorem~\ref{thm:main}. Instead of trying to relate the sample size $m$ (required to achieve a learning precision of $\epsilon$ in total variation distance with worst case error probability at most $\gamma$ on $n$ nodes) with the value of $n (\ln n + \ln {1/\gamma}) \over \eps^2$ directly, we investigate the relation between $\epsilon_{n,\gamma}(m)$ and $\sqrt{n (\ln n + \ln {1/\gamma}) \over m}$. For simplicity, we fix $n = 100$ and $\gamma = 0.25$, and vary $m$. Table~\ref{table} lists the $\hat{\epsilon}_{100,0.25}(m)$ values (estimates for $\epsilon_{100,0.25}(m)$) obtained for the ten sample sizes $m$ under investigation. The ten points $(m,\hat{\epsilon}_{100,0.25}(m))$ are plotted in Figure~\ref{plot}.

Ignoring the term $\ln {1/\gamma}$, we look for the smallest constant $c$ such that $\hat{\epsilon}_{100,0.25}(m)$ is dominated by $c \cdot \sqrt{100 \cdot \ln 100 \over m}$ for each of the ten sample sizes $m$ under investigation (an even smaller constant can be chosen if we take the term $\ln {1/\gamma}$ into consideration). That constant turns out to be $0.502515$. The function $0.502515 \cdot \sqrt{100 \cdot \ln 100 \over m}$ is plotted in Figure~\ref{plot}, along with the ten points $(m,\hat{\epsilon}_{100,0.25}(m))$. Note the nice fit between the two. This provides an indication that $\hat{\epsilon}_{100,0.25}(m)$ decreases at the rate of some multiple of $\frac{1}{\sqrt{m}}$ as $m$ increases, asymptotically, as predicted earlier. Finally, the constant $0.502515$ relating $\epsilon_{n,\gamma}(m)$ and $\sqrt{n (\ln n + \ln {1/\gamma}) \over m}$ points roughly to a constant of $0.502515^2$ hidden in the $O\left({n (\ln n + \ln {1/\gamma}) \over \eps^2}\right)$ sample complexity in Theorem~\ref{thm:main}. This provides an indication that the hidden constant is much smaller than that implied by the crude bounds in our theoretical analysis.

\newpage

\section{Conclusion and Future Work}
\label{sec:conclusion}

In this paper, we establish that there exist computationally efficient algorithms for properly learning tree-structured Ising models---which are an equivalent class to tree-structured binary-alphabet Bayesian networks and tree-structured binary-alphabet Markov Random Fields, requiring an asymptotically tight $\Theta(n \ln n/\epsilon^2)$ number of samples, where $n$ is the number of nodes and $\epsilon$ is the total variation distance between the true model and the learned model, which is constrained to also be a tree-structured Ising model. We note that the constants hidden inside the $\Theta(\cdot)$ notation are absolute constants that do not depend on the model being learned. Moreover, the algorithm attaining these guarantees is the celebrated Chow-Liu algorithm using the plug-in estimator to compute, using the available samples, its required pairwise mutual information quantities. Our results are the first to simultaneously attain computational efficiency and optimal sample complexity for this problem, matching the sample requirements of the  inefficient algorithm suggested by~\cite{devroye2019minimax}. On the other hand, we note that the vast recent literature on learning graphical models, including~\cite{NarasimhanB04,ravikumar2010high,TanATW11,jalali2011learning,SanthanamW12,Bresler15,VuffrayMLC16,KlivansM17,hamilton2017information,WuSD19,vuffray2019efficient}, provides sample inefficient methods or computationally inefficient methods, or obtains time and sample efficient algorithms by imposing assumptions on the model being learned, such as upper bounds on the strengths of the edges of the Ising model being learned and bounds on the degree of the underlying graph (or more generally bounds on the total strength of edges adjacent to each node). In contrast, we make no assumptions whatsoever about the model being learned. Finally, consistent with the notion of proper learning, we insist on the output of our algorithm being a tree-structured Ising model, which has downstream computational benefits for doing inference with the learned model. We note, however, that
it is not known that the problem becomes any easier if we only targeted learning some (potentially improper) representation of the density of the tree-structured model.

More broadly, our work is motivated by the central challenge of learning high-dimensional probability distributions. Without making any modeling assumptions about the distribution being learned, this task is of course meaningless, facing the curse of dimensionality in terms of both representing the distribution and in terms of the number of samples required to learn it, motivating the definition of various models of high-dimensional distributions with structure, such as Ising models, Bayesian networks, Markov Random Fields, Boltzmann Machines, etc. These models have been used extensively in applications, and recent work has shown that under appropriate  assumptions these distributions can be learned computationally and statistically efficiently with respect to the dimension of these distributions~\cite{Bresler15,VuffrayMLC16,KlivansM17,hamilton2017information,WuSD19,vuffray2019efficient,kelner2019learning,BreslerKM19}. The broader goal of our line of work is to provide computationally and statistically efficient algorithms for learning such distributions without making any additional assumptions about the distribution being learned, except that it belongs to one of these families, and we do obtain sample-optimal and computationally efficient algorithms for tree-structured Ising models. Going forward, our work motivates several open problems. The obvious one is (i) extending our work from computationally efficient and sample-optimal learning of tree-structured Ising models to computationally efficient and sample-optimal learning of tree-structured Markov Random Fields whose variables take values in a larger alphabet. As discussed in Section~\ref{sec:intro}, the independent work of Bhattacharyya et al.~\cite{bhattacharyya2020near} proves that Chow-Liu  solves this problem computationally and statistically efficiently for discrete alphabets, albeit its sample-optimality is not established by their work.
More ambitiously, our work motivates learning (ii) general-structured Ising models, (iii) general-structured Markov Random Fields, whose maximum clique size is bounded by an absolute constant, and (iv) general-structured Bayesian networks, whose maximum in-degree is bounded by an absolute constant. In all cases, we don't want to make any additional assumptions about the distributions targeted in (i), (ii), (iii), (iv). We believe that the techniques developed in this paper can readily serve as a starting point towards resolving (i), although additional techniques must be developed to handle the new intricacies imposed by more general alphabets. Resolving (ii), (iii) and (iv) will necessitate significantly new ideas. While it is unclear what these will be, it is important to point out that a polynomial number of samples suffices to solve these problems, as shown for (ii) and the Gaussian case of (iii) by~\cite{devroye2019minimax}, and as shown for (ii), (iii) and (iv) by~\cite{Brustle0D20}.

\section*{Acknowledgements} 
Supported by NSF Awards IIS-1741137, CCF-1617730 and CCF-1901292, by a Simons Investigator Award, and by the DOE PhILMs project (No. DE-AC05-76RL01830).

\newpage

\bibliographystyle{plain}
\bibliography{sigproc3}
\end{document}